\documentclass[a4paper]{article}
\usepackage[a4paper,hmargin=1.5cm]{geometry}

\usepackage{times}
\usepackage[T1]{fontenc}
\usepackage[utf8]{inputenc}
\usepackage{array}
\usepackage{url}
\usepackage{amsthm}
\usepackage{amsmath}
\usepackage{amssymb}
\usepackage{graphicx}
\usepackage{color}
\usepackage[english]{babel}
\usepackage{blindtext}

\usepackage[tableposition=top]{caption}
\usepackage[floatrowsep=quad]{floatrow}      % for table and figure side by side
\DeclareMathOperator*{\argmax}{arg\,max}
\DeclareMathOperator*{\argmin}{arg\,min}

\newcommand{\G}{\mathcal{G}} 
\newcommand{\E}{\mathcal{E}} 
\newcommand{\Dc}{\mathcal{D}} 
\renewcommand{\gg}{{\mathrm{g}}} 
\newcommand{\F}{\mathcal{F}} 
\renewcommand{\L}{\mathcal{L}} 
\newcommand{\V}{\mathcal{V}} 
\newcommand{\W}{\mathbf{W}} 
\newcommand{\D}{\mathbf{D}} 
\newcommand{\x}{u} 
\newcommand{\T}{T}
\newcommand{\TT}{{\mathcal T}}
\newcommand{\A}{\mathcal{A}}

\renewcommand{\u}{u}
\renewcommand{\l}{\ell}

\newcommand{\lmax}{\lambda_{\rm max}}
\newcommand{\hg}{\hat{g}}

\newcommand{\Rbb}{\mathbb{R}} \newcommand{\Cbb}{\mathbb{C}}

\newcommand{\scp}[2]{\langle #1, #2 \rangle}
\providecommand{\keywords}[1]{\textbf{\textit{Index terms---}} #1}

\newtheorem{theorem}{Theorem} \newtheorem{definition}{Definition}
 \newtheorem{lemma}{Lemma}
 \newtheorem{corollary}{Corollary}
\newtheorem{example}{Example}
\newtheorem*{mexample}{Motivating Example}

\begin{document}
\title{{Global and Local Uncertainty Principles for Signals on Graphs}}
%\author{\IEEEauthorblockN{Nathanael Perraudin, Benjamin Ricaud, David Shuman and Pierre Vandergheynst} \IEEEauthorblockA{Signal Processing Laboratory (LTS2) \\ Swiss Federal Institute of Technology (EPFL)\\ Station 11, CH-1015 Lausanne, Switzerland \\  Email: nathanael.perraudin@epfl.ch, benjamin.ricaud@epfl.ch, david.shuman@epfl.ch, pierre.vandergheynst@epfl.ch} }

\author{Nathanael Perraudin, Benjamin Ricaud, David I Shuman, and Pierre Vandergheynst}
%\address{$^a$Signal Processing Laboratory (LTS2), Swiss Federal Institute of Technology (EPFL), %Station 11, CH-1015 
%Lausanne, Switzerland \\ $^b$Department of Mathematics, Statistics, and Computer Science, Macalester College, Saint Paul, MN, USA}

\maketitle

\begin{abstract}
Uncertainty principles such as Heisenberg's provide limits on the time-frequency concentration of a signal, and constitute an important theoretical tool for designing and evaluating linear signal transforms. Generalizations of such principles to the graph setting can inform dictionary design for graph signals, lead to algorithms for reconstructing missing information from graph signals via sparse representations, and yield new graph analysis tools. While previous work has focused on generalizing notions of spreads of a graph signal in the vertex and graph spectral domains, our approach is to generalize the methods of Lieb in order to develop uncertainty principles that provide limits on the concentration of the analysis coefficients of any graph signal under a dictionary transform whose atoms are jointly localized in the vertex and graph spectral domains. One challenge we highlight is that due to the inhomogeneity of the underlying graph data domain, the local structure in a single small region of the graph can drastically affect the uncertainty bounds for signals concentrated in different regions of the graph, limiting the information provided by global uncertainty principles. Accordingly, we suggest a new way to incorporate a notion of locality, and develop local uncertainty principles that bound the concentration of the analysis coefficients of each atom of a localized graph spectral filter frame in terms of quantities that depend on the local structure of the graph around the center vertex of the given atom. Finally, we demonstrate how our proposed local uncertainty measures can improve the random sampling of graph signals.
\end{abstract}


%\newpage
%\tableofcontents
%\newpage

%%%%%%%%%%%%%%%%%%%%%%%%%%%%%%%%%%%%%%%%%%%
%%      Introduction
%%%%%%%%%%%%%%%%%%%%%%%%%%%%%%%%%%%%%%%%%%%

\section{Introduction} \label{Se:intro}

%\nati{Shall we cite Moura or not?}

% Intro
%{\color{blue}
%\begin{itemize}
%\item emerging field of signal processing on graphs , plethora of new dictionaries \cite{hammond2011wavelets,coifman2006diffusion,narang2012perfect,shuman2015vertex}.
%\item dictionary evaluation questions: sparse representations, structure
%%\item specifically, it is desirable that the dictionary atoms are jointly localized in time and frequency. uncertainty principles characterize this tradeoff
%%\item one main tool in the classical case: uncertainty principles \cite{donoho1989uncertainty,Donoho01uncertainty,Elad02Generalized,gribonval2003sparse,candes2006quantitative,rictorrefined}
%%\item Other constructive uses of uncertainty principles (and role in dictionary design) (Check language from Ambizione) grant
%%\item
%%\item maybe an example here
%%\item Intuition from classical signal processing: some carriers over, but some does not (shift-invariant notion of translation), and localization of eigenvectors
%%
%%\item Prime example: localized graph Laplacian eigenvectors. In classical case, the incoherence leads to uncertainty principles such as Heisenberg. In graph case, not true (show a few examples)
%%\item As shown in \cite{mcgraw, saito} and discussed in \cite[Section 3.2]{shuman2015vertex}, 
%\end{itemize}
%
%}

The major research thrust to date in the emerging area of signal processing on graphs \cite{shuman2013emerging,sandryhaila2014discrete}  has been to design multiscale wavelet and vertex-frequency transforms  \cite{Crovella2003}-\nocite{Maggioni_biorthogonal,szlam,coifman2006diffusion,bremer_packets,lafon_coarse,wang,narang_lifting_graphs,jansen,gavish,hammond2011wavelets,ram2011generalized,narang2012perfect,leonardi_multislice,ekambaram_globalsip,narang_bior_filters,liu_coarsening,sakiyama,nguyen,shuman2013spectrum,shuman2015vertex}\cite{shuman2016multiscale}. Objectives of these transforms are to sparsely represent different classes of  graph signals and/or efficiently reveal relevant structural properties of high-dimensional data on graphs.
%efficiently extract information from high-dimensional data on graphs by 
%Dictionaries used to transform signal and reveal structure or sparsely represent
As we move forward, it is important to both test these transforms on myriad applications, as well as to develop additional theory to help answer the question of which transforms are best suited to which types of data. 

%(1) applications (2) theory
%
%dictionary evaluation questions: sparse representations, structure

Uncertainty principles such as the ones presented in \cite{donoho1989uncertainty}-\nocite{Donoho01uncertainty,Elad02Generalized,gribonval2003sparse,candes2006quantitative}\cite{rictorrefined} are an important tool in designing and evaluating linear transforms for processing ``classical'' signals such as audio signals, time series, and images residing on Euclidean domains. It is desirable that the dictionary atoms are jointly localized in time and frequency, and uncertainty principles characterize the resolution tradeoff between these two domains. Moreover, while ``the uncertainty principle is [often] used to show that certain things are impossible,'' Donoho and Stark \cite{donoho1989uncertainty} present ``examples where the generalized uncertainty principle shows something unexpected is \emph{possible}; specifically, the recovery of a signal or image despite significant amounts of missing information.'' In particular, uncertainty principles can provide guarantees that if a signal has a sparse decomposition in a dictionary of incoherent atoms, this is indeed a unique representation that can be recovered via optimization \cite{Donoho01uncertainty,Elad02Generalized}. This idea underlies the recent wave of sparse signal processing techniques, with applications such as denoising, source separation, inpainting, and compressive sensing. While there is still limited theory showing that different mathematical classes of graph signals are sparsely represented by the recently proposed transforms (see \cite{ricaud_SPIE_2013} for one preliminary work along these lines), there is far more empirical work showing the potential of these transforms to sparsely represent graph signals in various applications.

Many of the multiscale transforms designed for graph signals attempt to leverage intuition from signal processing techniques designed for signals on Euclidean data domains by generalizing fundamental operators and transforms to the graph setting (e.g., by checking that they correspond on a ring graph). While some intuition, such as the notion of filtering with a Fourier basis of functions that oscillate at different rates (see, e.g., \cite{shuman2013emerging}) carries over to the graph setting, the irregular structure of the graph domain often restricts our ability to generalize ideas. One prime example is the lack of a shift-invariant notion of translation of a graph signal. As shown in \cite{mcgraw, saito} and discussed in \cite[Section 3.2]{shuman2015vertex}, the concentration of the Fourier basis functions is another example where the intuition does not carry over directly. Complex exponentials, the basis functions for the classical Fourier transform, have global support across the real line. On the other hand, the eigenvectors of the combinatorial or normalized graph Laplacians, which are most commonly used as the basis functions for a graph Fourier transform, are sometimes localized to small regions of the graph. Because the incoherence between the Fourier basis functions and the standard normal basis underlies many uncertainty principles, we demonstrate this issue with a short example.

\begin{mexample}[Part I: Laplacian eigenvector localization] \label{Ex:motivating}
Let us consider the two manifolds (surfaces) embedded in $\Rbb^3$ and shown in the first row of Figure~\ref{fig:intro manifold}. The first one is a flat square. The second is identical except for the center where it contains a spike. We sample both of these manifolds uniformly across the $x$-$y$ plane and create a graph by connecting the $8$ nearest neighbors with weights depending on the distance ($W_{ij}=e^{-d_{ij}/\sigma}$). The energy of each Laplacian eigenvector of the graph arising from the first manifold is not concentrated on any particular vertex;
 %For the first manifold, there are no localized eigenvectors, 
 i.e., $\max_{i,\ell}|u_\ell(i)| << 1$, where $u_\ell$ is the eigenvector associated with eigenvalue $\lambda_\ell$. %eigenvector. %This ensures a large bound for the uncertainty. Indeed it is known that $\frac{\| f \|_1 \|\hat{f}\|_1}{\|f\|_2^2} \geq \frac{1}{\max_{i,\ell}|u_\ell(i)|}$. 
However, the graph arising from the second manifold does have a few eigenvectors, such as eigenvector 3 shown in the middle row Figure \ref{fig:intro manifold}, whose energy is highly concentrated on the region of the spike; i.e: $\max_{i,\ell}|u_\ell(i)| \approx 1$. Yet, the Laplacian eigenvectors of this second graph whose energy resides primarily on the flatter regions of the manifold, such as eigenvector 17 shown in the bottom row of Figure \ref{fig:intro manifold}, are not too concentrated on any single vertex. Rather, they more closely resemble some of the Laplacian eigenvectors of the graph arising from the first manifold.
%this is no more the case for the second manifold which has a few very localized eigenvectors, %These decreases the uncertainty, i.e: they example of signal concentrated in both the vertex and the spectral domain. Nevertheless, the second manifold is composed of two different parts and the behavior far away from the spike should behave similarly as for the manifold 1. 
%In fact, there is no localized eigenvectors away from the spike (e.g ). 
%As a result, on manifold 2 the bound on vertex/frequency concentration should depend on the signal localization. If the energy of the signal is concentrated on the spike, then it can be very concentrated in both the vertex and the spectral domain. On the contrary, if its energy is located on the flat part of the manifold, the bound on the concentration is higher. This will lead us to a local uncertainty principle.

\begin{figure}[t!]
\centering
\begin{minipage}[b]{.5\linewidth}
\begin{minipage}[b]{.49\linewidth}
\centerline{\small{Manifold 1}}
\centerline{\includegraphics[width=\linewidth]{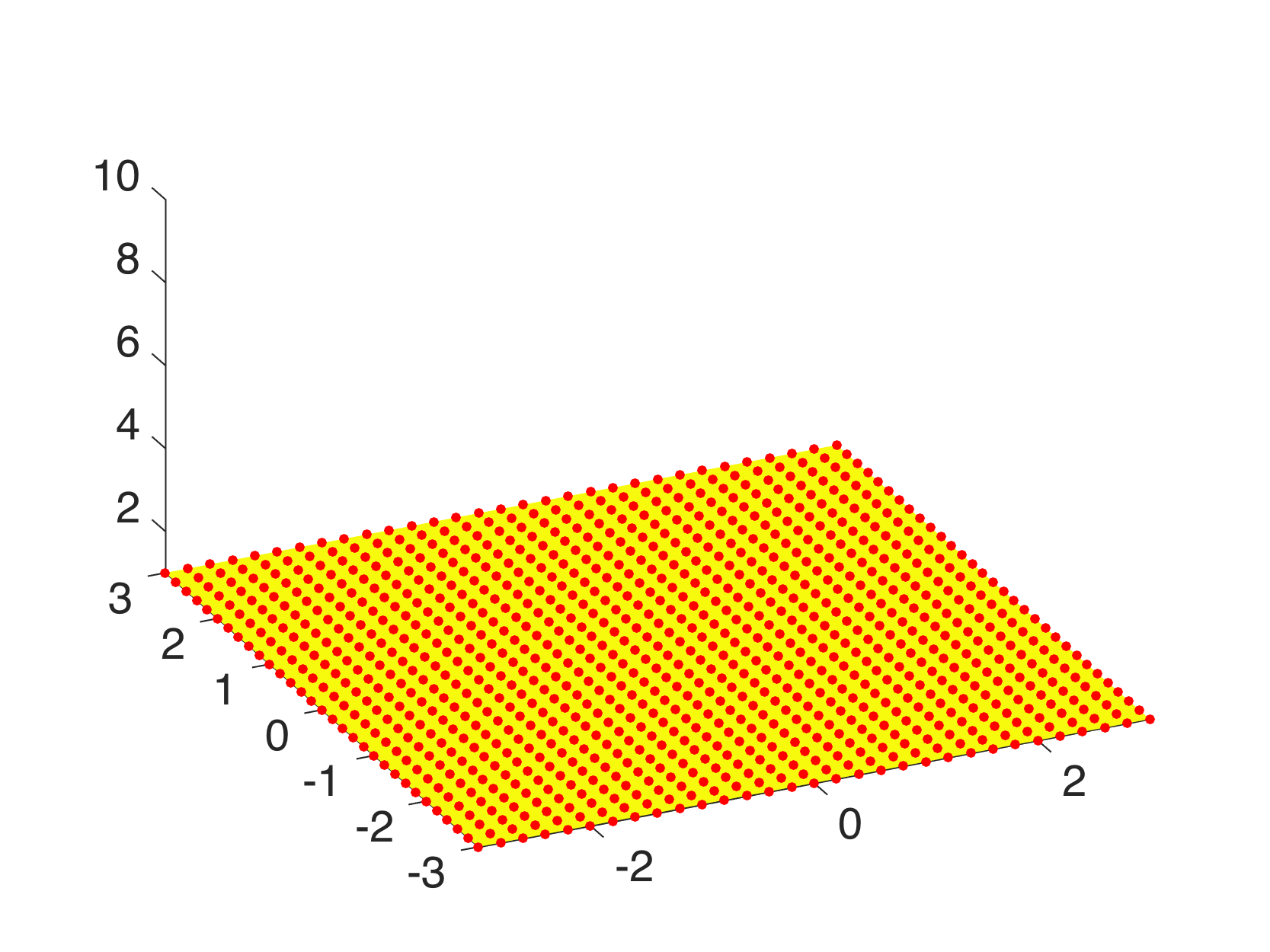}} 
%\centerline{\small{(a)}}
\end{minipage}
\hfill
\begin{minipage}[b]{.49\linewidth}
\centerline{\small{Manifold 2}}
\centerline{\includegraphics[width=\linewidth]{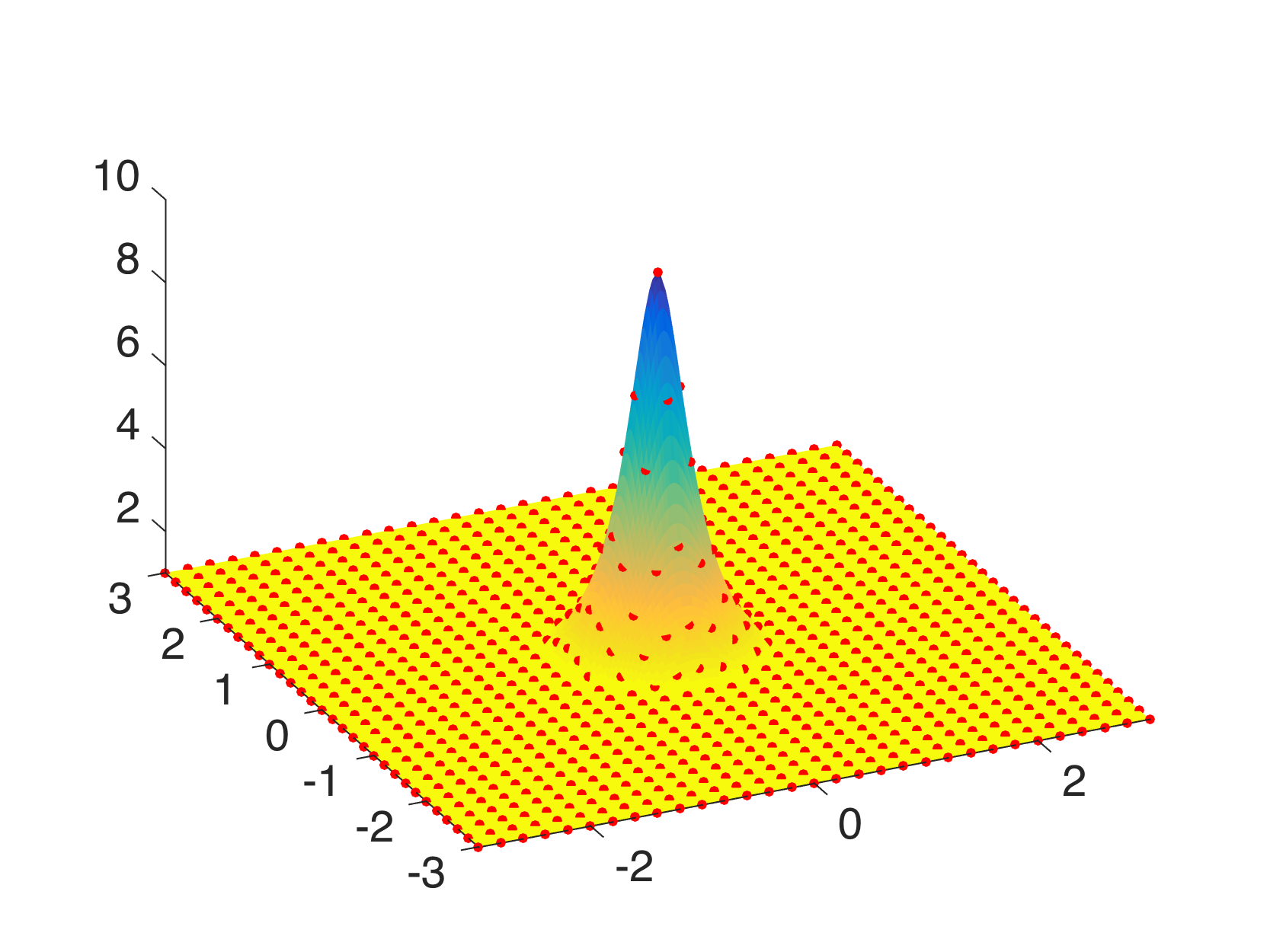}} 
%\centerline{\small{(b)}}
\end{minipage}
\\

\vspace{.1in}
\begin{minipage}[b]{.49\linewidth}
\centerline{\small{Eigenvector 5}}
\centerline{\small{(Manifold 1)}}
\centerline{\includegraphics[width=\linewidth]{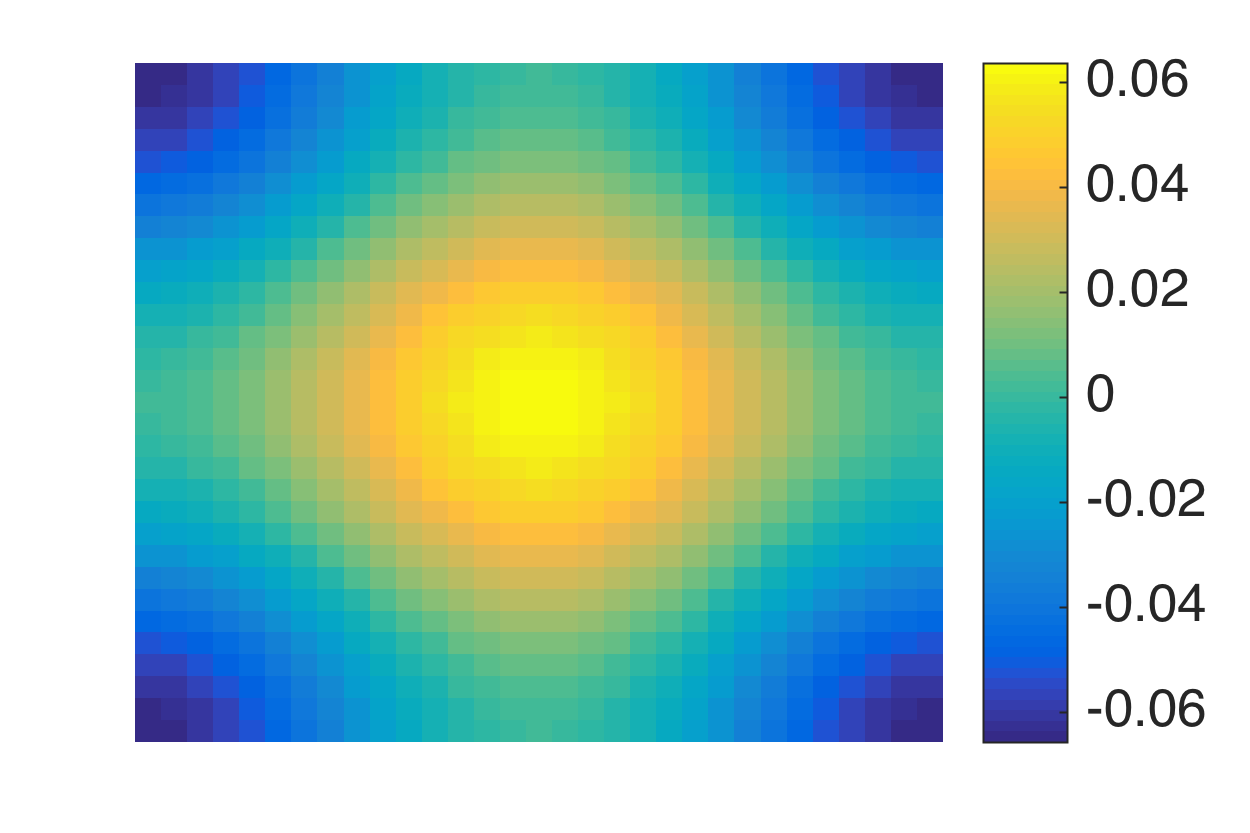}} 
%\centerline{\small{(c)}}
\end{minipage}
\hfill
\begin{minipage}[b]{.49\linewidth}
\centerline{\small{Eigenvector 3}}
\centerline{\small{(Manifold 2)}}
\centerline{\includegraphics[width=\linewidth]{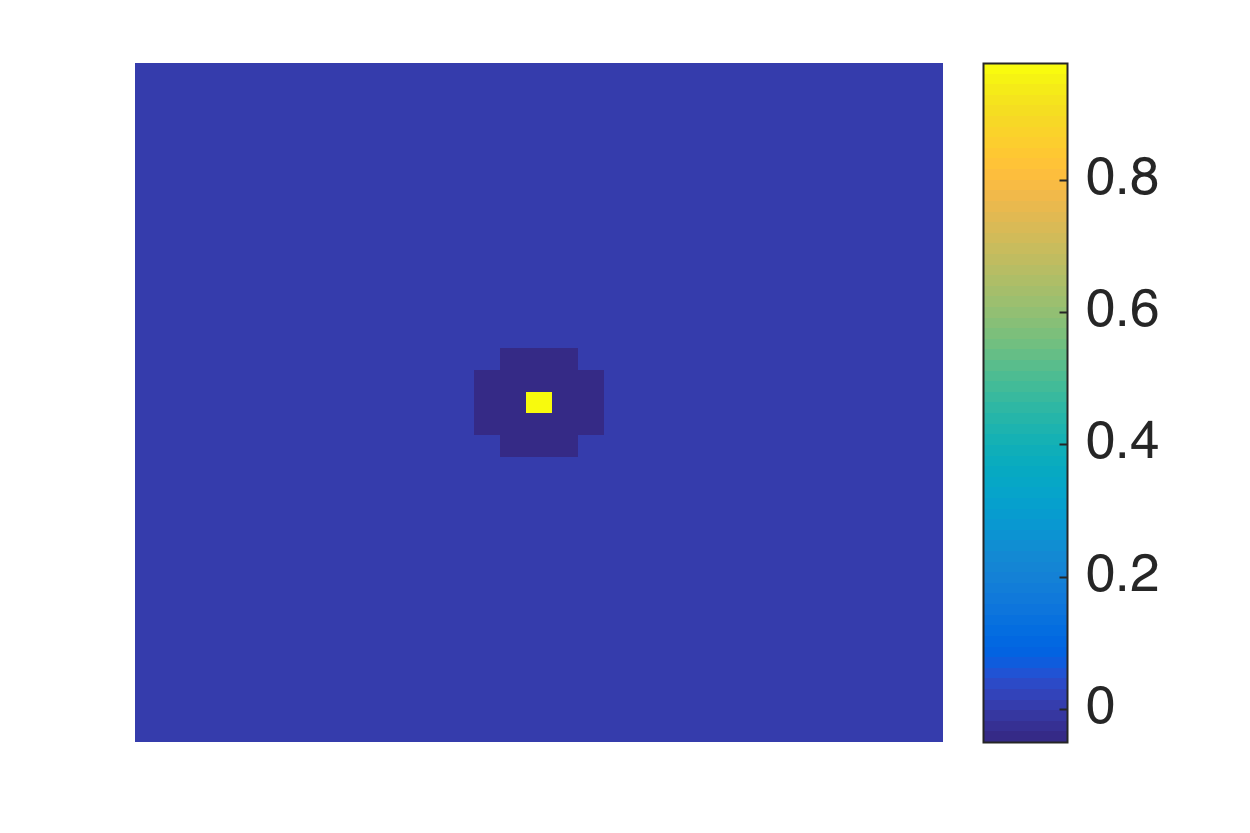}} 
%\centerline{\small{(d)}}
\end{minipage}
\\

\vspace{.1in}
\begin{minipage}[b]{.49\linewidth}
\centerline{\small{Eigenvector 13}}
\centerline{\small{(Manifold 1)}}
\centerline{\includegraphics[width=\linewidth]{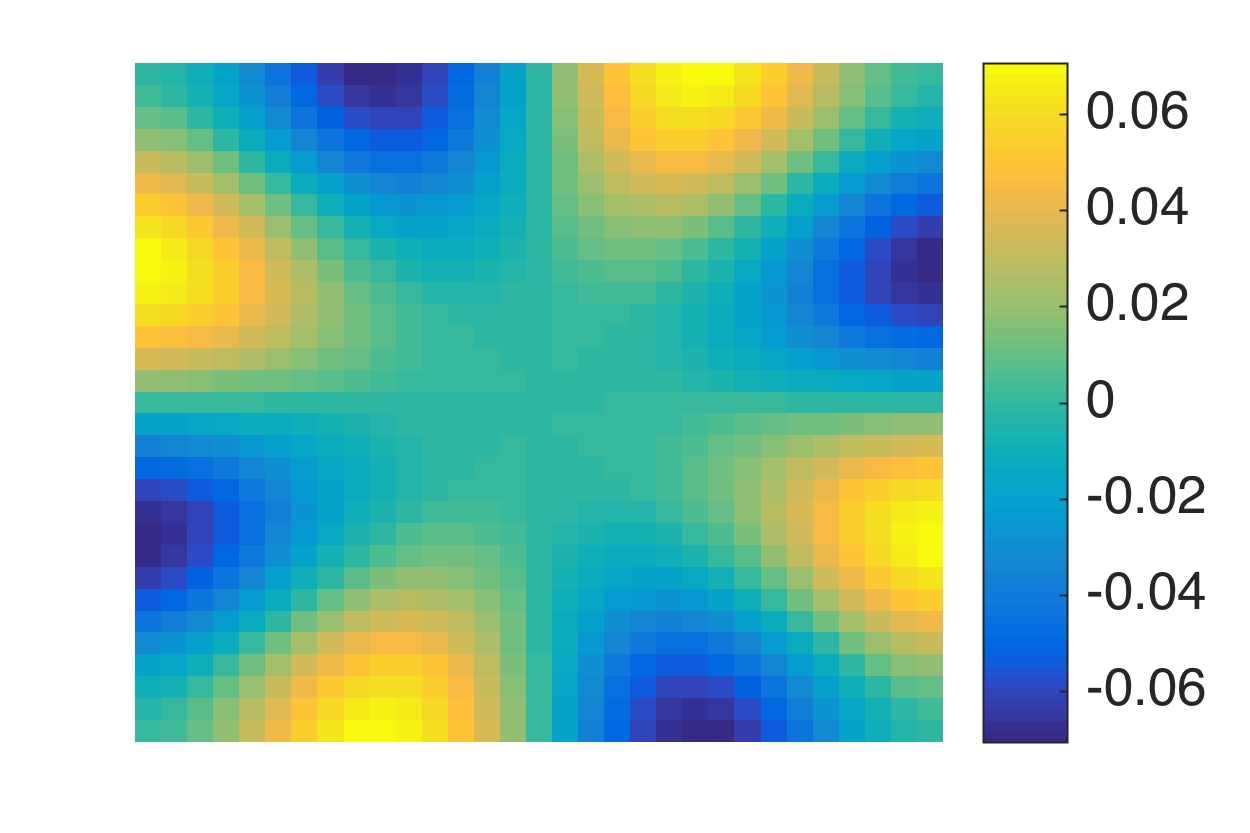}} 
%\centerline{\small{(e)}}
\end{minipage}
\hfill
\begin{minipage}[b]{.49\linewidth}
\centerline{\small{Eigenvector 17}}
\centerline{\small{(Manifold 2)}}
\centerline{\includegraphics[width=\linewidth]{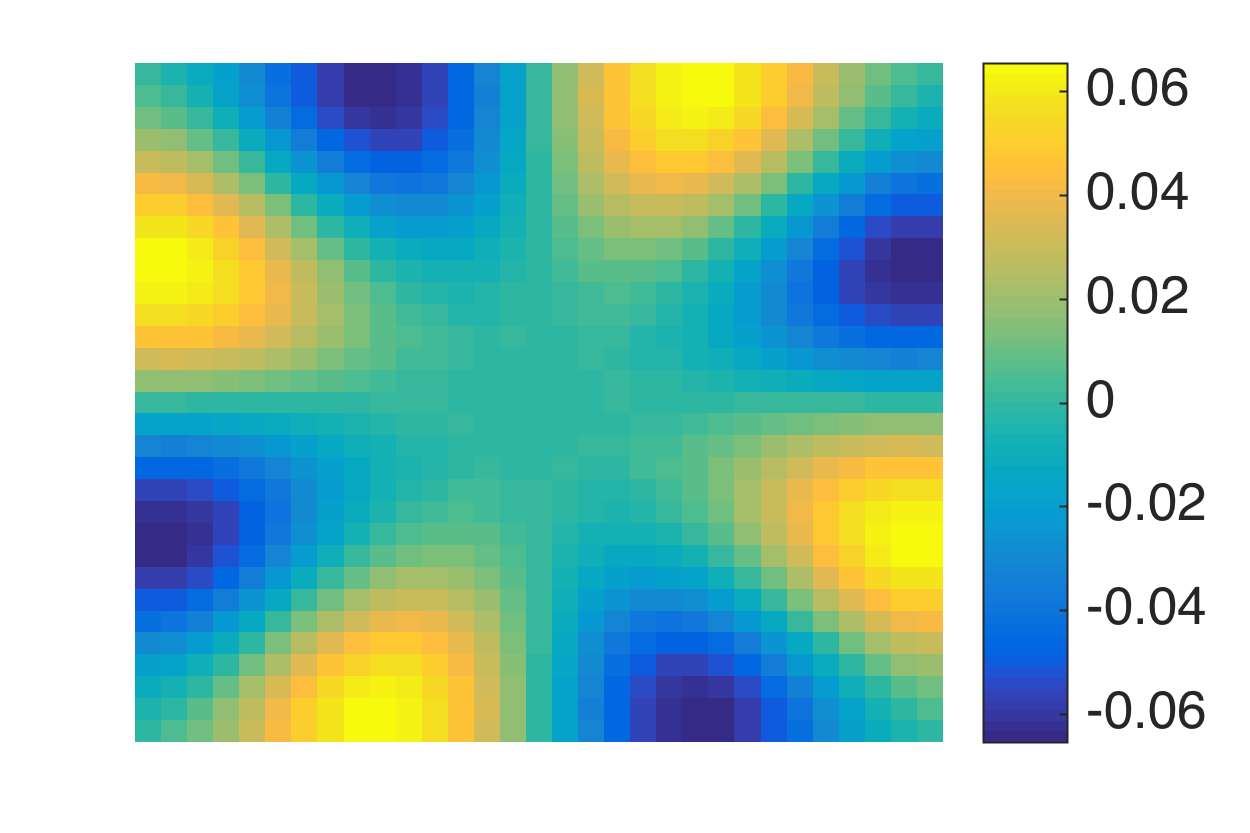}} 
%\centerline{\small{(f)}}
\end{minipage}
\end{minipage}
\caption{Concentration of graph Laplacian eigenvectors. We discretize two different manifolds by sampling uniformly across the $x$-$y$ plane. Due its bumpy central part, the graph arising from manifold 2 has a graph Laplacian eigenvector (shown in the middle row of the right column) that is highly concentrated in both the vertex and graph spectral domains. However, the eigenvectors of this graph whose energy primarily resides in the flatter parts of the manifold (such as the one shown in the bottom row of the right column) are less concentrated, and some closely resemble the Laplacian eigenvectors of the graph arising from the flat manifold 1 (such as the corresponding eigenvector shown in the bottom row of the left column.}
%Due to to the manifold topology, some eigenvectors might be concentrated. Column 1: manifold 1. Column 2: manifold 2. Row 2 and 3: top view on the graph eigenvectors. The sampling is uniform along the axis $x$ and $y$. In the flat part, the eigenvectors of Manifold 1 and 2 are similar. However, they differ widely in the bumpy central part of Manifold 2.}
\label{fig:intro manifold}
\end{figure}
\end{mexample}

Below we discuss three different families of uncertainty principles, and their extensions to the graph setting, both in prior work and in this contribution.
\begin{itemize}
\item The first family of uncertainty principles measure the spreading \emph{around some reference point}, usually the mean position of the energy contained in the signal. The well-known Heisenberg uncertainty principle \cite{folland1997, mallat2008wavelet} belongs to this family. It views the modulus square of the signal in both the time and Fourier domains as energy probability density functions, and takes the variance of those energy distributions as measures of the spreading in each domain. The uncertainty principle states that the product of variances in the time and in the Fourier domains cannot be arbitrarily small. The generalization of this uncertainty principle to the graph setting is complex since there does not exist a simple formula for the mean value or the variance of graph signals, in either the vertex or the graph spectral domains. 
For unweighted graphs, Agaskar and Lu \cite{agaskar_spie,agaskar_icassp,agaskarlu13} also view the square modulus of the signal in the vertex domain as an energy probability density function and use the geodesic graph distance (shortest number of hops) to define the spread of a graph signal around a given center vertex.
%In an effort to generalize the work of \cite{agaskarlu13} to weighted graphs, \cite{pasdeloup} introduce {\color{red}(spreads based on new distance measures).} 
For the spread of a signal $f$ in the graph spectral domain, Agaskar and Lu use the normalized variation $\frac{f^{\top}\L f}{||f||_2^2}$, which captures the smoothness of a signal. They then specify uncertainty curves %generated in \cite{agaskarlu13} 
that characterize the tradeoff between the smoothness of a graph signal and its localization in the vertex domain. This idea is  generalized to weighted graphs in \cite{pasdeloup}.
As pointed out in \cite{agaskarlu13}, the tradeoff between smoothness and localization in the vertex domain is intuitive as a signal that is smooth with respect to the graph topology cannot feature values that decay too quickly from the peak value. However, as shown in Figure \ref{fig:intro manifold} %Example \ref{Ex:motivating} 
(and subsequent examples in Table \ref{tab:uncertainty tight}), graph signals can indeed be simultaneously highly localized or concentrated in both the vertex domain and the graph spectral domain. 
This discrepancy is because the normalized variation used as the spectral spread in \cite{agaskarlu13}  is one method to measure the spread of the spectral representation around the eigenvalue 0, rather than around some mean of that signal in the graph spectral domain.
%why? because spread is around 0 instead of a mean in the spectral domain.
In fact, using the notion of spectral spread presented in \cite{agaskarlu13}, the graph signal with the highest spectral spread on a graph $\G$ is the graph Laplacian eigenvector associated with the highest eigenvalue. The graph spectral representation of that signal is a Kronecker delta whose energy is completely localized at a single eigenvalue. One might argue that its \emph{spread} should in fact be zero. So, in summary, while there does exist a tradeoff between the smoothness of a graph signal and its localization around any given center vertex in the vertex domain, the classical idea that a signal cannot be simultaneously localized in the time and frequency domains does not always carry over to the graph setting. % graph and spectral domains
%{\color{red} Expand this section. Add spread definitions. Add extra citations. Summarize main results. Point out drawbacks. One difficulty resides in the definition of a mean position, a second one is to define a concentration measure around a mean position.
%A first attempt has been proposed by~\cite{agaskarlu13}. The authors suggest an interesting measure of concentration around the mean value of the signal for the graph setting that is formulated as a minimization problem. However, precise mathematical bounds involving graph properties can not be computed analytically as the measures as well as the uncertainty bound are given by solving minimization problems (except maybe for very simple graphs). In addition, their measure of localization is only applied to the vertex domain and not to the (generalized) Fourier domain. Instead they only measure the spreading on the Fourier domain around coordinate zero, which limits the generality of their uncertainty principle.}
%
While certainly an interesting avenue for continued investigation, we do not discuss uncertainty principles based on spreads in the vertex and graph spectral domains any further in this paper.

% Despite this, a first attempt has been proposed by~\cite{agaskarlu13}. On the vertex domain, they suggest an interesting measure of concentration around a vertex that can be freely chosen. However, they only measures the spreading around coordinate zero in the spectral domain, which limits the generality of their principle. In addition, precise mathematical bounds involving graph properties can only be computed for a limited number of graphs (the bound is defined as a minimization problem). 
%%
\item The second family of uncertainty principles involve the absolute sparsity or concentration of a signal. The key quantities are typically either support measures counting the number of non-zero elements, or concentration measures, such as $\ell^p$-norms. 
An important distinction is that these sparsity and concentration measures are not localization measures.  
They can give the same values for different signals, independent of whether the predominant signal components are clustered in a small region of the vertex domain or spread across different regions of the graph. An example of a recent work from the graph signal processing literature that falls into this family is \cite{tsitsvero2015signals}, in which  Tsitsvero et al. propose an uncertainty principle that characterizes how jointly concentrated graph signals can be in the vertex and spectral domains.
Generalizing prolate spheroidal wave functions \cite{slepian1961prolate}, their notion of concentration is
based on the percentage of energy of a graph signal that is concentrated on a given set of vertices in the vertex domain and a given set of frequencies in the graph spectral domain.

% Work in this direction has already been done using a sampling point of view. Generalizing prolate spheroidal wave functions \cite{slepian1961prolate}, Tsitsvero et al. \cite{tsitsvero2015signals}  propose an uncertainty principle that characterizes the support of a signal both in the vertex and in the spectral domain. Based on this principle, they derive a sampling method for band limited signals.
% We follow a different approach i}

Since we can interpret signals defined on graphs as finite dimensional vectors with well-defined $\ell^p$-norms, we can also apply directly the results of existing uncertainty principles for finite dimensional signals. As one example, the Elad-Bruckstein uncertainty principle of \cite{Elad02Generalized} states that if $\alpha$ and $\beta$ are the coefficients of a vector $f \in \Rbb^N$ in two different orthonormal bases, then 
\begin{align}\label{Eq:elbr_gen}
\frac{||\alpha||_0 + ||\beta||_0}{2} \geq \sqrt{||\alpha||_0 \cdot ||\beta||_0} \geq \frac{1}{\mu},
\end{align}
where $\mu$ is the maximum magnitude of the inner product between any vector in the first basis with any vector in the second basis. In Section \ref{Se:direct_extensions}, we apply \eqref{Eq:elbr_gen} to graph signals by taking one basis to be the canonical basis of Kronecker delta functions in the graph vertex domain and the other to be a Fourier basis of graph Laplacian eigenvectors. We also apply
other such finite dimensional uncertainty principles from %{\color{red} (Add [7], [13] here?)} 
\cite{rictorrefined}, \cite{folland1997}, and \cite{Maassen88generalized} to the graph setting.
%We can apply this to graph setting 
%we apply such finite dimensional uncertainty principles to graph signals. 
In Section \ref{Se:HYI}, we adapt the Hausdorff-Young inequality \cite[Section IX.4]{reed1975methods}, a classical result for infinite dimensional signals, to the graph setting. 
%{\color{red} Do we also need to mention the Babenko–Beckner inequality for $\Rbb^N$ here? \nati{No, it is complicated to generalize and it will not bring anything to the paper.} } 
These results typically depend on the mutual coherence between the graph Laplacian eigenvectors and the canonical basis of deltas. For the special case of shift-invariant graphs with circulant graph Laplacians \cite[Section 5.1]{grady}, such as ring graphs, these bases are incoherent, and we can attain meaningful uncertainty bounds. However, for less homogeneous graphs (e.g., a graph with a vertex with a much higher or lower degree than other vertices), the two bases can be more coherent, leading to weaker bounds. Moreover, as we discuss in Section \ref{sec:towarduncertaintyprinciple}, the bounds are global bounds, so even if the majority of a graph is for example very homogenous, inhomogeneity in one small area can prevent the result from informing the behavior of graph signals across the rest of the graph.

\item The third family of uncertainty principles characterize a single joint representation of time and frequency. The short-time Fourier transform (STFT) is an example of a time-frequency representation that projects a function $f$ onto a set of translated and modulated copies of a function $g$. Usually, $g$ is a function localized in the time-frequency plane, for example a Gaussian, vanishing away from some known reference point in the joint time and frequency domain. Hence this transformation reveals \emph{local properties in time and frequency} of $f$ by separating the time-frequency domain into regions where the translated and modulated copies of $g$ are localized. This representation obeys an uncertainty principle:  the STFT coefficients cannot be arbitrarily concentrated. This can be shown by estimating the different $\l^p$-norms of this representation (note that the concentration measures of the second family of uncertainty principles are used). For example, Lieb \cite{lieb1990integral} proves a concentration bound on the ambiguity function (e.g., the STFT coefficients of the STFT atoms). Lieb's approach is more general than the Heisenberg uncertainty principle, because it handles the case where the signal is concentrated around multiple different points (see, e.g., the signal $f_3$ in Figure~\ref{fig:concentration_example}).

In Section \ref{Se:Lieb_global}, we generalize Lieb's uncertainty principle to the graph setting to provide upper bounds on the concentration of the transform coefficients of any graph signal under (i) any frame of dictionary atoms, and (ii) a  special class of dictionaries called \emph{localized spectral graph filter frames}, whose atoms are of the form $T_i g_k$, where $T_i$ is a localization operator that centers on vertex $i$ a pattern described in the graph spectral domain by the kernel $\widehat{g_k}$. 

%{\color{blue} 

%Continue discussion of ambiguity functions (STFT)and interpretation 
%
%In the graph setting we can define generalized forms of the translation $\T_i$ (called localization operator) and modulation $\M_k$ and obtain similar uncertainty bounds.
%We shall see however that for functions defined on graphs the bounds depend on the graph structure.
%A major discrepancy between classical and graph cases is that the shape of the atoms $g_{i,k}$ in the vertex-"frequency" plane vary from place to place following the graph structure.}
\end{itemize}

While the second family of uncertainty principles above yields \emph{global uncertainty principles}, we can generalize the third family to the graph setting in a way that yields \emph{local uncertainty principles}. 
In the classical Euclidean setting, the underlying domain is homogenous, and thus uncertainty principles apply to all signals equally, regardless of where on the real line they are concentrated. However, in the graph setting, the underlying domain is irregular, and a change in the graph structure in a single small region of the graph can drastically affect the uncertainty bounds. For instance, the second family of uncertainty principles  all depend on the coherence between the graph Laplacian eigenvectors and the standard normal basis of Kronecker deltas, which is a global quantity in the sense that it incorporates local behavior from all regions of the graph. To see how this can limit the usefulness of such global uncertainty principles, we return to the motivating example from above.
% to see how this can limit the usefulness of such \emph{global uncertainty principles}.

\begin{mexample}[Part II: Global versus local uncertainty principles]
In Section \ref{Se:direct_extensions}, we show that a direct application of a result from \cite{rictorrefined} to the graph setting yields the following uncertainty relationship, which falls into the second family described above, for any signal $f\in \Rbb^N$:
\begin{align} \label{Eq:mot_unc}
\left(\frac{\| f \|_2}{\| f \|_1}\right)\left(\frac{\| \hat{f} \|_2}{\| \hat{f} \|_1}\right) \leq {\max_{i,\ell}|u_\ell(i)|}.
\end{align}
Each fraction in the left-hand side of \eqref{Eq:mot_unc} is a measure of concentration that lies in the interval $[\frac{1}{\sqrt{N}},1]$ ($N$ is the number of vertices), and the coherence between the graph Laplacian eigenvectors and the Kronecker deltas on the right-hand side lies in the same interval. On the graph arising from manifold 1, the coherence is close to $\frac{1}{\sqrt{N}}$, and \eqref{Eq:mot_unc} yields a meaningful uncertainty principle. However, on the graph arising from manifold 2, the coherence is close to 1 due to the localized eigenvector 3 in Figure \ref{fig:intro manifold}, \eqref{Eq:mot_unc} is trivially true for any signal in $\Rbb^n$ from the properties of vector norms, and thus the uncertainty principle is not particularly useful. Nevertheless, far away from the spike, signals should behave similarly on manifold 2 to how they behave on manifold 1. Part of the issue here is that the uncertainty relationship holds for any graph signal $f$, even those concentrated on the spike, which we know can be jointly localized in both the vertex and graph spectral domains. An alternative approach is to develop a local uncertainty principle that characterizes the uncertainty in different regions of the graph on a separate basis. Then, if the energy of a given signal is concentrated on a more homogeneous part of the graph, the concentration bounds will be tighter.
\end{mexample}

%{
%\color{blue}
%In fact, there is no localized eigenvectors away from the spike (e.g eigenvector 17 of Figure \ref{fig:intro manifold}). As a result, on manifold 2 the bound on vertex/frequency concentration should depend on the signal localization. If the energy of the signal is concentrated on the spike, then it can be very concentrated in both the vertex and the spectral domain. On the contrary, if its energy is located on the flat part of the manifold, the bound on the concentration is higher. This will lead us to a local uncertainty principle.
%}

In Section \ref{subsec:local}, we generalize the approach of Lieb to build a local uncertainty principle that bounds the concentration of the analysis coefficients of each atom of a localized graph spectral filter frame in terms of quantities that depend on the local structure of the graph around the center vertex of the given atom. Thus, atoms localized to different regions of the graph feature different concentration bounds. Such local uncertainty principles also have constructive applications, and we conclude with an example of non-uniform sampling for graph inpainting, where the varying uncertainty levels across the graph suggest a strategy of sampling more densely in areas of higher uncertainty. For example, if we were to take $M$ measurements of a smooth signal on manifold 2 in Figure \ref{fig:intro manifold}, this method would lead to a higher probability of sampling signal values near the spike, and a lower probability of sampling signal values in the more homogenous flat parts of the manifold, where reconstruction of the missing signal values is inherently easier.

\section{Notation and graph signal concentration}\label{sec:towarduncertaintyprinciple}
%\section{Uncertainty principles based on concentration measures}
%\section{Influence of the graph on the uncertainty principle}

In this section, we introduce some notation and illustrate further how certain intuition from signal processing on Euclidean spaces does not carry over to the graph setting.

%The goal of this section is twofold. On the one hand, we investigate some of the state-of-the-art uncertainty principles, which may be applied to the graph setting in a straight-forward manner. It allows to compare our contribution to the literature. On the other hand we show through the use of these principles, the particularities of the graph setting and the importance of some key quantities. It brings a first insight on the graph uncertainty principle, its challenges and the limitations of the actual principles dedicated to standard spaces.
%As a byproduct of the analysis of the uncertainty principle, we also present a version of the Hausdorff-Young inequality in the graph setting. It provides additional information on the peculiarity of the graph setting.

%\subsection{Notation, graph Fourier coherence, and graph Laplacian eigenvector concentration}
\subsection{Notation}
%Let us start by establishing the graph framework.
Throughout the paper, we consider signals residing on an undirected, connected, and weighted graph $\G=\{ \V,\E,\W\}$, where $\V$ is a finite set of $N$ vertices ($|\V|=N$), $\E$ is a finite set of edges, and $\W$ is the weight or adjacency matrix. The entry $W_{ij}$ of $\W$ represents the weight of an edge connecting vertices $i$ and $j$. %We denote by $W_{ij}$ the entries of $\W$. 
A graph signal $f: \V \rightarrow  \mathbb{C}$ is a function assigning one value to each vertex. Such a signal $f$ can be written as a vector of size $N$ with the $n^{th}$ component representing the signal value at the $n^{th}$ vertex. The generalization of Fourier analysis to the graph setting requires a graph Fourier basis $\{u_{\l}\}_{\l \in \{0,1,\ldots,N-1\}}$.  The most commonly used graph Fourier bases are the eigenvectors of the combinatorial (or non-normalized) graph Laplacian, which is defined as $\L=\D-\W$, where $\D$ is the diagonal degree matrix with diagonal entries $\D_{ii}=\sum_{j=1}^N W_{ij}$, and $i\in \V$, or the eigenvectors of the normalized graph Laplacian $\tilde{\L}=\D^{-\frac{1}{2}} \L \D^{-\frac{1}{2}}$. However, the eigenbases (or Jordan eigenbases) of other matrices such as the adjacency matrix have also been used as graph Fourier bases \cite{sandryhaila_2013, sandryhaila2014discrete}. All of our results in this paper hold for any choice of the graph Fourier basis. For concreteness, we use the combinatorial Laplacian, which has a complete set of orthonormal eigenvectors $\{\u_{\l}\}_{\l \in \{0,1,\ldots,N-1\}}$ associated with the real eigenvalues $0=\lambda_0 < \lambda_1 \leq \lambda_2 \leq ... \leq \lambda_{N-1} = \lambda_{\rm max}$. 
The graph Fourier transform $\hat{f}\in\Cbb^N$ of a function $f\in\Cbb^N$ defined on a graph $\G$ is the projection of the signal onto the orthonormal graph Fourier basis $\{\u_{\l}\}_{\l=0,1,\ldots,N-1}$, which we take to be the eigenvectors of the graph Laplacian associated with $\G$:
\begin{equation} \label{def: graph Fourier transform}
\hat{f}(\lambda_\l)=\langle f,\u_{\l} \rangle= \sum_{n=1}^N f(n) \overline{\u_{\l}(n)},\quad \l \in \{0,1,\ldots,N-1\}.
\end{equation}
See for example~\cite{chung1997spectral} for more details on spectral graph theory, and \cite{shuman2013emerging} for more details on signal processing on graphs.

\subsection{Concentration measures}
In order to discuss uncertainty principles, we must first introduce some concentration/sparsity measures.
Throughout the paper, we use the terms \emph{sparsity} and \emph{concentration} somewhat interchangeably, but we reserve the term \emph{spread} to describe the spread of a function around some mean or center point, as discussed in the first family of uncertainty principles in Section \ref{Se:intro}. The first concentration measure is the support measure of $f$, denoted $\|f\|_0$, which counts the number of non-zero elements of $f$. The second concentration measure is the Shannon entropy, which is used often in information theory and physics:
$$H(f)=-\sum_{n} |f(n)|^2\ln|f(n)|^2,$$
where the variable $n$ has values in $\{1,2,\ldots,N\}$ for functions on graphs and in $\{0,1,\ldots,N-1\}$ in the graph Fourier representation. 
Another class of concentration measures is the $\ell^p$-norms, with $p\in[1,\infty]$. For $p\neq 2$, the sparsity of $f$ may be measured using the following quantity: 
$$s_p(f)=
\begin{cases}
\frac{\|f\|_2}{\|f\|_p},&\hbox{ if } 1 \leq p \leq 2 \\
\\
\frac{\|f\|_p}{\|f\|_2},&\hbox{ if } 2 < p \leq \infty 
\end{cases}.
$$ 
For any vector $f\in \Cbb^N$ and any $p \in [1,\infty]$, $s_p(f) \in \left[N^{-\left|\frac{1}{p}-\frac{1}{2}\right|},1\right]$. If $s_p(f)$ is high (close to 1), then $f$ is sparse, and if $s_p(f)$ is low, then $f$ is not concentrated.
%For $p<2$, a small value for $\|f\|_p/\|f\|_2$ implies that the signal $f$ is sparse, whereas for $p>2$, a sparse $f$ corresponds to a large value of $\|f\|_p/\|f\|_2$. For example, $s_1(\cdot)$ and $s_\infty(\cdot)$ are commonly used sparsity measures.
%For $p=1$ and $p=\infty$  the above expression is well known to be a sparsity measure. In fact $p$ is not restricted to these sole values for the measure of sparsity. 
Figure \ref{fig:concentration_example} uses some basic signals to illustrate this notion of concentration, for different values of $p$.
In addition to sparsity, one can also relate $\ell^p$-norms to the Shannon entropy via Renyi entropies (see, e.g., \cite{renyi1961measures,RicTorrACOM} for more details). %\end{remark}

\begin{figure}[ht!]
\begin{minipage}[h]{.45\linewidth}
\centering
\includegraphics[width=\textwidth]{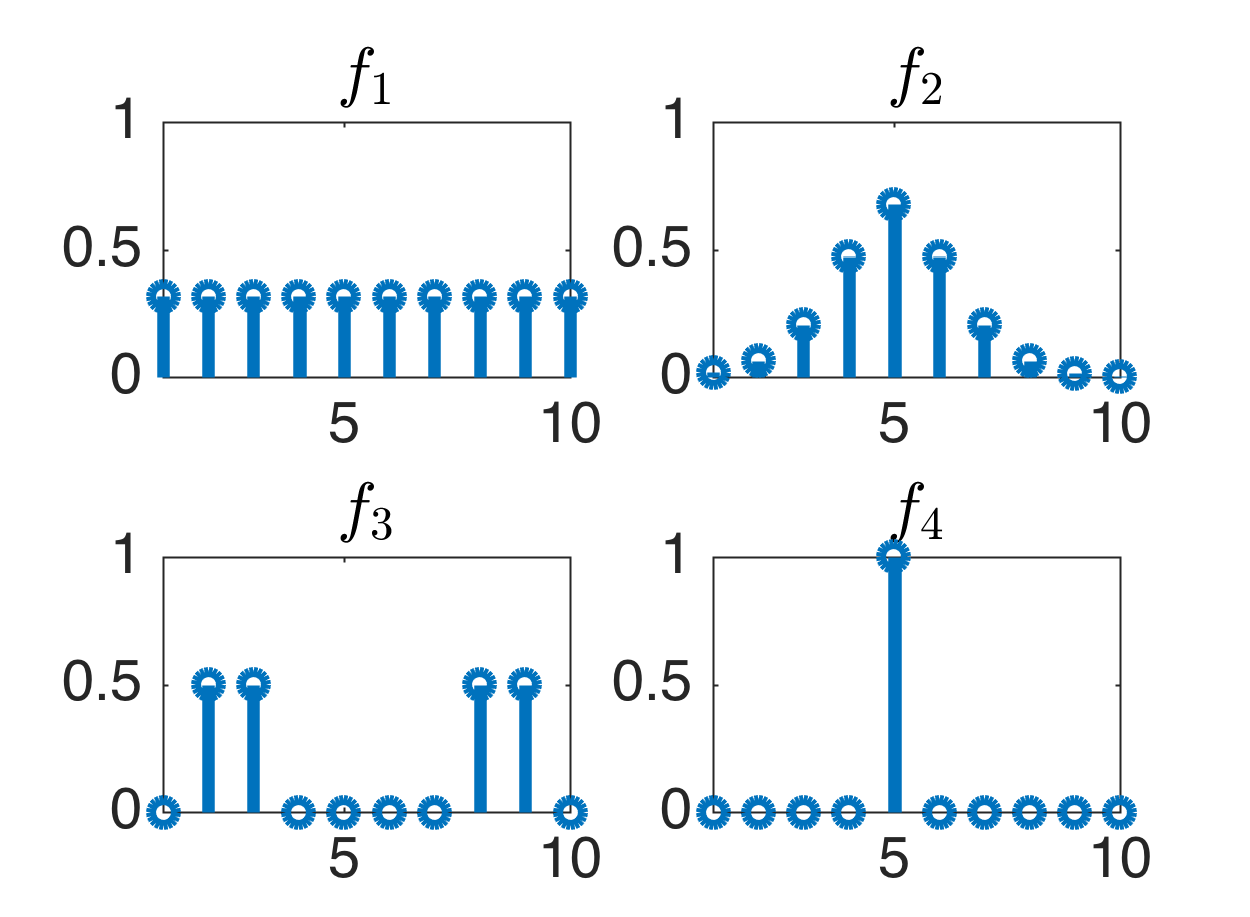}
\end{minipage}
\begin{minipage}[h]{.45\linewidth}
\centering
\begin{tabular}[h]{ |c|cccc|} 
 \hline 
  p  &   $s_p(f_1)$ &  $s_p(f_2)$    & $s_p(f_3)$   & $s_p(f_4)$ \\ 
 \hline 
  $  1.00 $ & $  0.32 $ & $  0.46 $ & $  0.50 $ & $  1.00 $ \\ 
  $  1.33 $ & $  0.56 $ & $  0.69 $ & $  0.71 $ & $  1.00 $ \\ 
  $  2.00 $ & $  1.00 $ & $  1.00 $ & $  1.00 $ & $  1.00 $ \\ 
  $  4.00 $ & $  0.56 $ & $  0.75 $ & $  0.71 $ & $  1.00 $ \\ 
  $   \infty $ & $  0.32 $ & $  0.68 $ & $  0.50 $ & $  1.00 $ \\ 
 \hline 
 \end{tabular}
\end{minipage}
\caption{The concentration $s_p(\cdot)$ of four different example signals (all with 2-norm equal to 1), for various values of $p$. %Example of concentration for various signals. 
Note that the position of the signal coefficients does not matter for this concentration measure. Different values of $p$ lead to different notions of concentration; for example, $f_2$ is more concentrated than $f_3$ if $p=\infty$ (it has a larger maximum absolute value), but less concentrated if $p=1$.}
\label{fig:concentration_example}
\end{figure}

\subsection{Concentration of the graph Laplacian eigenvectors}
The spectrum of the graph Laplacian replaces the frequencies as coordinates in the Fourier domain.
For the special case of shift-invariant graphs with circulant graph Laplacians \cite[Section 5.1]{grady}, the Fourier eigenvectors can still be viewed as pure oscillations. However, for more general graphs (i.e., all but the most highly structured), the oscillatory behavior of the Fourier eigenvectors must be interpreted more broadly. For
example, \cite[Fig. 3]{shuman2013emerging} displays the number of zero crossings of each eigenvector; that is, for each eigenvector, the number of pairs of connected vertices where the signs of the values of the eigenvector at the connected vertices are opposite. It is generally the case that the graph Laplacian eigenvectors associated with larger eigenvalues contain more zero crossings, yielding a notion of frequency to the graph Laplacian eigenvalues. However, despite this broader notion of frequency, the graph Laplacian eigenvectors are not always globally-supported, pure oscillations like the complex exponentials. In particular, they can feature sharp peaks, meaning that some of the Fourier basis elements can be much more similar to an element of the canonical basis of Kronecker deltas on the vertices of the graph. As we will see, uncertainty principles for signals on graphs are highly affected by this phenomenon. 

%{\color{red}
%On very "regular graphs" like rings, grids, or paths, the Fourier modes are intuitively the oscillating modes for given frequencies \cite{shuman2013emerging} \cite{shuman2015vertex}. Moreover, for a graph created by randomly sampling a smooth manifold, it has been shown that the Laplacian of the graph converges toward the Laplace-Beltrami operator of the manifold \cite{belkin2005towards}, whose eigenfunctions are also used as Fourier modes. However, when the graph structure is "less regular" (with isolated vertices or groups of highly connected vertices), the Laplacian spectrum can differ greatly from the notion of frequency. Some Fourier modes can be sharp peaks, hence close to elements of the canonical basis.
%%In some cases, the graph coherence $\mu_{\G}$ between the graph Fourier and the canonical basis can be high (close to one): some Fourier modes are sharp peaks, hence close to elements of the canonical basis. This is illustrated in the Example~\ref{ex:stargraph}. 
%As we will see, uncertainty is highly affected by this phenomenon. 
%%Indeed $\mu_{\G}$ is one crucial global graph parameter for uncertainty principles. For example, \eqref{eq: classic 00 and 11 uncertainty} tells us that a high coherence $\mu$ reduces the uncertainty. This is illustrated in the Example~\ref{ex:stargraph}.
%}

One way to compare a graph Fourier basis to the canonical basis is to compute the coherence between these two representations. 
%This has already been introduced in the introduction but we define it more formally here.
\begin{definition}[Graph Fourier Coherence $\mu_{\G}$]
Let $\G$ be a graph of $N$ vertices. Let $\{\delta_i\}_{i\in \{1,2,\ldots,N\}}$ denote the canonical basis of $\ell^2(\Cbb^N)$ of Kronecker deltas and let $\{u_{\l}\}_{\l\in \{0,1,\ldots,N-1\}}$ be the orthonormal basis of eigenvectors of the graph Laplacian of $\G$. The graph Fourier coherence is defined as:
$$\mu_{\G}=\max_{i,\l}|\langle \delta_i,u_{\l}\rangle|=\max_{i,\l}|u_{\l}(i)|=\max_{\l} s_{\infty}(u_{\l}).$$
\end{definition}
This quantity measures the similarity between the two sets of vectors. If the sets possess a common vector, then $\mu_{\G}=1$ (the maximum possible value for $\mu_{\G}$). If the two sets are maximally incoherent, such as the canonical and Fourier bases in the standard discrete setting, then $\mu_{\G}=1/\sqrt{N}$ (the minimum possible value).

%The entries of the Laplacian matrix are given by the graph connections hence the expression of graph Laplacian eigenvectors is greatly related to the graph structure. 
%As a consequence,  shall depend on the structure of the graph. 

Because the graph Laplacian matrix encodes the weights of the edges of the graph, the coherence $\mu_{\G}$ clearly depends on the structure of the underlying graph.  It remains an open question exactly how structural properties of weighted graphs such as the regularity, clustering, modularity, and other spectral properties can be linked to the concentration of the graph Laplacian eigenvectors. For certain classes of random graphs \cite{dekel, dumitriu, tran} or large regular graphs \cite{brooks}, the eigenvectors have been shown to be non-localized, globally oscillating functions (i.e., $\mu_{\G}$ is low). Yet, empirical studies such as \cite{mcgraw} show that graph Laplacian eigenvectors can be highly concentrated (i.e., $\mu_{\G}$ can be close to 1), particularly when the degree of a vertex is much higher or lower than the degrees of other vertices in the graph. The following example illustrates how $\mu_{\G}$ can be influenced by the graph structure.

\begin{example}\label{ex:pathgraph}
In this example, we discuss two classes of graphs that can have graph Fourier coherences. The first, called comet graphs, are studied in \cite{saito,Nakatsukasa2013mysteries}. They are  composed of a star with $k$ vertices connected to a center vertex, and a single branch of length greater than one extending from one neighbor of the center vertex (see Figure~\ref{fig:intro graph topology}, top). 
If we fix the length of the longer branch (it has length 10 in Figure \ref{fig:intro graph topology}), and increase $k$, the number of neighbors of the center vertex, the graph Laplacian eigenvector associated with the largest eigenvalue approaches a Kronecker delta centered at the center vertex of the star.
%In this configuration, some of the Laplacian eigenvectors become more concentrated as the number of branches of the star increases, with a shape tending to a Kronecker delta. 
As a consequence, the coherence between the graph Fourier and the canonical bases approaches 1 as $k$ increases. 

The second class are the modified path graphs, which we use several times in this contribution. %It is simple to visualize and allows to easily change the graph Fourier coherence, making it perfectly suitable for showing how our different bounds evolve. 
We start with a standard path graph of 10 nodes equally spaced (all edge weights are equal to one) and we move the first node out to the left; i.e., we reduce the weight between the first two nodes (see Figure~\ref{fig:intro graph topology}, bottom). The weight is related to the distance by $W_{12}=1/d(1,2)$ with $d(1,2)$ being the distance between nodes 1 and 2. When the weight between nodes 1 and 2 decreases, the eigenvector associated with the largest eigenvalue of the Laplacian becomes more concentrated, which increases the coherence $\mu_{\G}$.
These two examples of simple families of graphs illustrate that the topology of the graph can impact the graph Fourier coherence, and, in turn, uncertainty principles that depend on the coherence. % through the graph Fourier transform.

\begin{figure}[ht!]
\begin{center}
\includegraphics[width=0.25\textwidth]{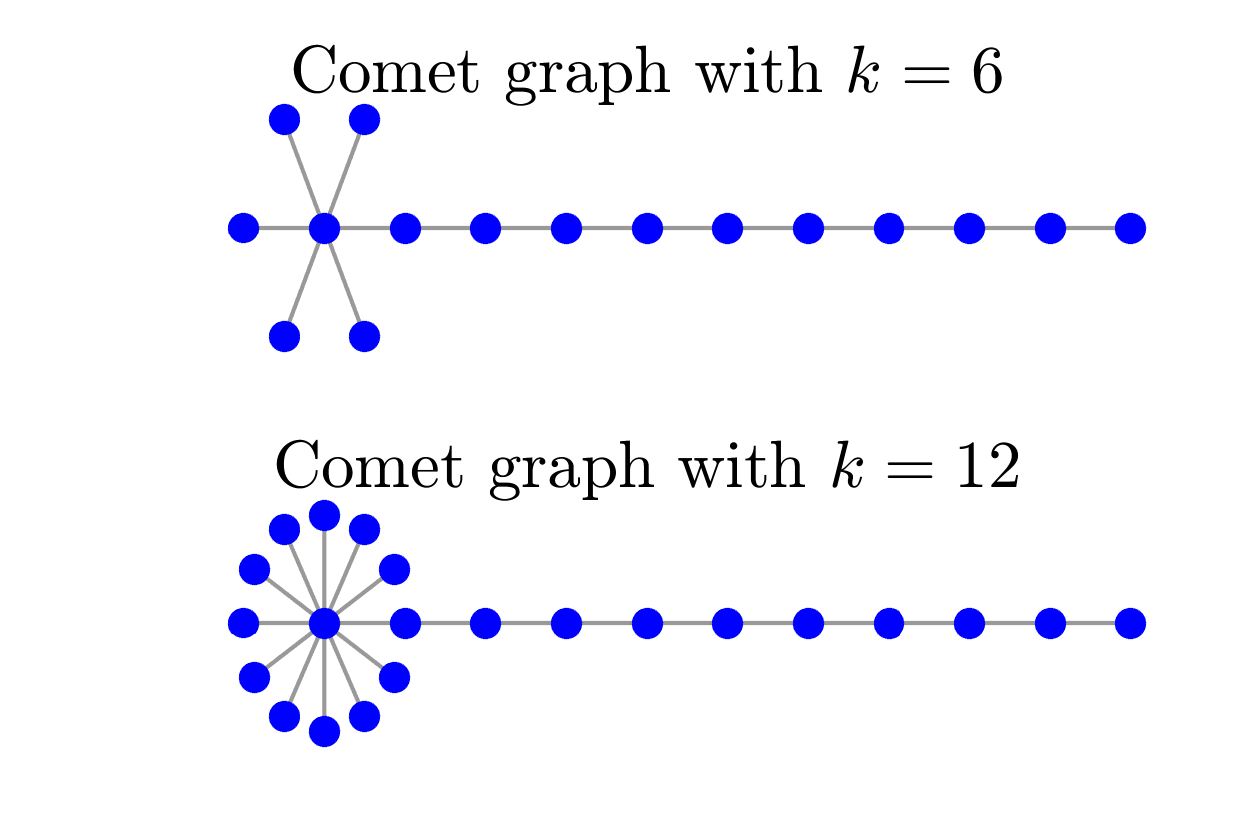}
\includegraphics[width=0.25\textwidth]{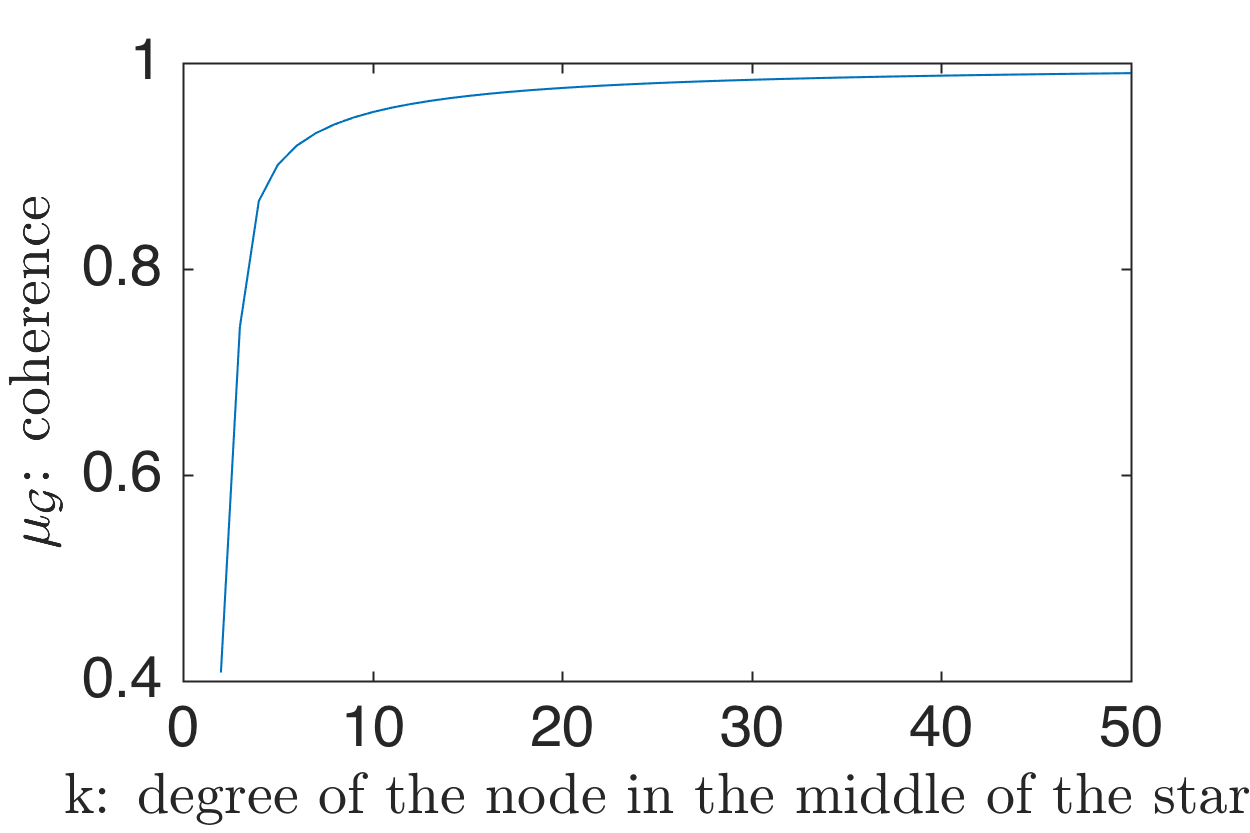}\\
\includegraphics[width=0.25\textwidth]{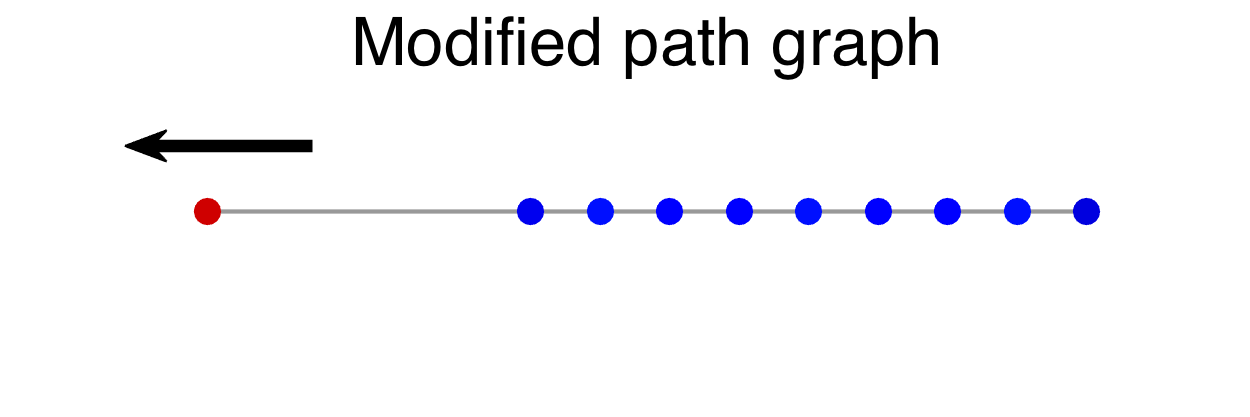}
\includegraphics[width=0.25\textwidth]{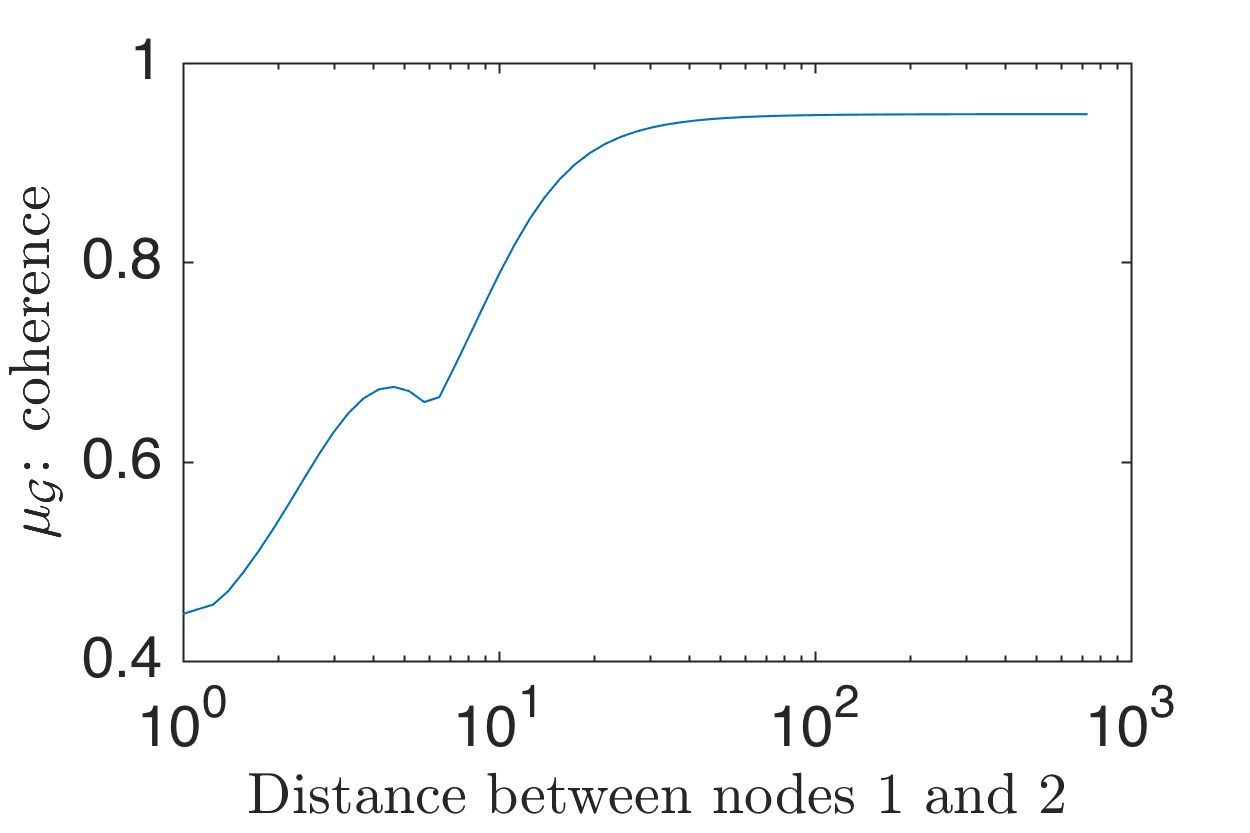}
\end{center}
\caption{Coherence between the graph Fourier basis and the canonical basis for the graphs described in Example~\ref{ex:pathgraph}. Top left: Comet graphs with  $k=6$ and $k=12$ branches, all of length one except for one of length ten. %(degree of the vertex being the head of the comet). 
Top right: Evolution of the graph Fourier coherence $\mu_{\G}$ with respect to $k$. Bottom left: Example of a modified path graph with $10$ nodes. Bottom right: Evolution of the coherence of the modified path graph with respect to the distance between nodes $1$ and $2$. As the degree of the comet's center vertex increases or the first node of the modified path is pulled away, the coherence $\mu_{\G}$ tends to the limit value $\sqrt{\frac{N-1}{N}}$.}
\label{fig:intro graph topology}
\end{figure}

In Figure~\ref{fig: modified path eig}, we display the eigenvector associated with the largest graph Laplacian eigenvalue for a modified path graph of 100 nodes, for several values of the weight $W_{12}$. %On the right are plotted the full set of eigenvectors, each one of them having a different color. 
Observe that the shape of the eigenvector has a sharp local change at node 1. 

\begin{figure}[ht!]
\begin{center}
\includegraphics[width=0.75\textwidth]{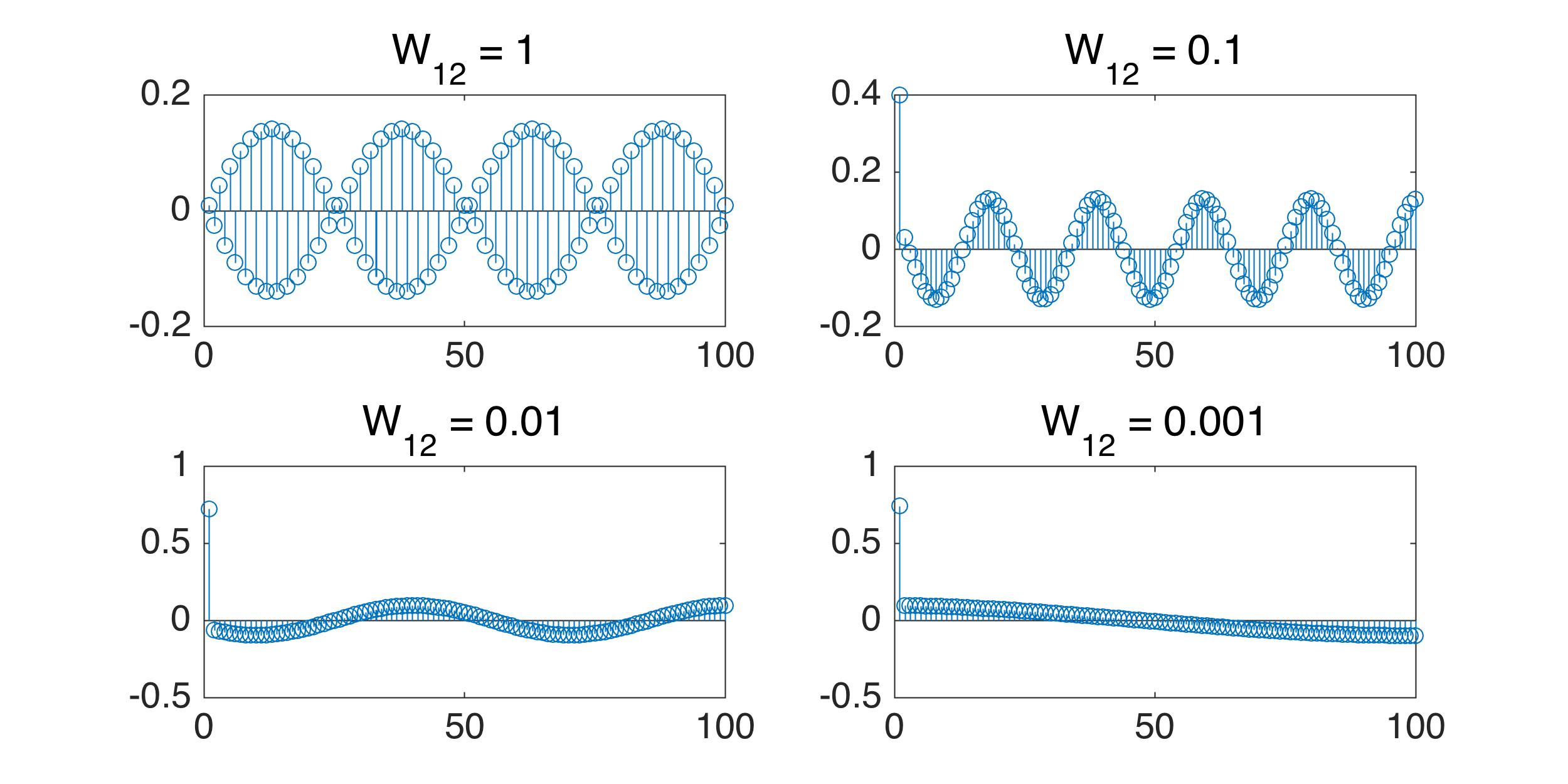}
\caption{Eigenvectors associated with the largest graph Laplacian eigenvalue of the modified path graph with $100$ nodes, for different values of $W_{12}$. As the distance between the first two nodes increases, the eigenvector becomes sharply peaked.\label{fig: modified path eig}}
\end{center}
\end{figure}
\end{example}

Example~\ref{ex:pathgraph} demonstrates an important point to keep in mind. A small local change in the graph structure can greatly affect the behavior of one eigenvector, and, in turn, a global quantity such as $\mu_\G$. However, intuitively, a small local change in the graph should not drastically change the processing of signal values far away, for example in a denoising or inpainting task. For this reason, in Section \ref{subsec:local}, we introduce a notion of local uncertainty that depicts how the graph is behaving locally.
%{\color{red}
%Example~\ref{ex:pathgraph} points out an important characteristics of graphs: local changes in the graph structure imply only local changes in the shape of the graph Laplacian eigenvectors. Away from these changes, we expect Laplacian based operations to stay the same. However, $\mu_{\G}$ is a global parameter and do not account for this locality. David: I don't really understand this statement. Doesn't the change in a single weight change the eigenvector everywhere?
%}

Note that not only special classes of graphs or pathological graphs yield highly localized graph Laplacian eigenvectors. Rather, graphs arising in applications such as sensor or transportation networks, or graphs constructed from sampled manifolds (such as the graph sampled from manifold 2 in Figure \ref{fig:intro manifold}) can also have graph Fourier coherences close to 1 (see, e.g., \cite[Section 3.2]{shuman2015vertex} for further examples).

%This will lead to limitations on coherence based uncertainty principles.
%Note that the shape of the eigenvectors is different at vertex 1 from the one at the other vertices. In fact, this graph has a very different behavior between node $1$ and the other nodes. This is a consequence from the difference of structure between node $1$ and the other nodes. 
%The uncertainty
%Since only the graph topology will infer the amount of uncertainty, this kind of graphs motivate a local uncertainty principle. This will be presented later on.

%From these two examples one can see that local changes in the graph structure imply local changes in the shape of the graph Laplacian eigenvectors, local changes for $\max_k|u_k|$ and a different value for $\mu_{\G}$. 

%This provides a strong argument in favor of a local analysis for functions defined on graphs with Laplacian based techniques and in particular an uncertainty principle relying on local graph quantities.
%Signals and functions defined on graphs are vectors of $\Cbb^N$ and bases for graph signal processing are bases of $\ell^2(\Cbb^N)$.
% However the shape of vectors may be influenced by the graph structure. For example the eigenvectors of the graph Laplacian depend on the graph as the entries of the Laplacian matrix are given by the graph connections. In that respect, $\mu$ shall depend on the shape of the graph. 
%\begin{remark}[on the spreading measures]\label{rem:spreadingmeasures}

 % \nati{Those limit the concentration of localized signals.}

%\clearpage

\section{Global uncertainty principles relating the concentration of graph signals in two domains} \label{Se:global_unc_princ}
%Before stating the first group of uncertainty principles, we must introduce some concentration and sparsity measures.

In this section, we derive basic uncertainty principles using concentration measures and highlight the limitations of those uncertainty principles.

\subsection{Direct applications of uncertainty principles for discrete signals}\label{Se:direct_extensions}
We start by applying three known uncertainty principles for discrete signals to the graph setting.
%Let us state the first (straight-forward) results of this contribution. 
\begin{theorem}\label{th:generaluncertainty}
Let $f\in\Cbb^N$ be a nonzero signal defined on a connected, weighted, undirected graph $\G$, let $\{ u_{\l}\}_{\l=0,1,\ldots,N-1}$ be a graph Fourier basis for $\G$, and let $\mu_{\G}=\max_{i,\l}|\langle\delta_i,u_{\l}\rangle|$. We have the following four uncertainty principles: 
\begin{itemize}
\item[(i)] the support uncertainty principle \cite{Elad02Generalized} 
\begin{equation}\label{eq:eladbruckstein}
\frac{\|f\|_0+\|\hat{f}\|_0}{2}\ge\sqrt{\|f\|_0\|\hat{f}\|_0}\ge\frac{1}{\mu_{\G}},
\end{equation}
\item[(ii)] the $\ell^p$-norm uncertainty principle \cite{rictorrefined}
\begin{equation} \label{eq:rictorr}
\|f\|_{p}\|\hat{f}\|_{p}\geq \mu_{\G}^{1-\frac{2}{p}}\|f\|_{2}^2 ,\qquad p\in[1,2], %\hbox{ and }
\end{equation}
\item[(iii)] the entropic uncertainty principle \cite{Maassen88generalized}
\begin{equation}\label{eq:entr}
H(f)+H(\hat{f})\ge -2\ln\mu_{\G}.
\end{equation}
%{\color{red}
\item[(iv)] the `local' uncertainty principle~\cite{folland1997} %, let $E$ be a subset of vertices of $G$ and $|E|$ denoting its size,
\begin{align}\label{eq:loc_folland}
\sum_{i\in \V_S}|f(i)|^2\leq |\V_S|  \|f\|_\infty^2 \le|\V_S|\mu_{\G}^2\|\hat{f}\|_1^2
\end{align}
for any subset $\V_S$ of the vertices $\V$ in the graph $\G$.
%}
\end{itemize}
\end{theorem}

The first uncertainty principle is given by a direct application of the Elad-Bruckstein inequality~\cite{Elad02Generalized}. It states that the sparsity of a function in one representation limits the sparsity in a second representation. As displayed in \eqref{Eq:elbr_gen}, the work of \cite{Elad02Generalized} holds for representations in any two bases. As we have seen, if we focus on the canonical basis $\{\delta_i\}_{i=1,\dots,N}$ and the graph Fourier basis $\{u_{\l}\}_{\l=0,\dots,N-1}$, the coherence $\mu_{\G}$ depends on the graph topology. For the ring graph, $\mu_{\G}=\frac{1}{\sqrt{N}}$, and we recover the result from the standard discrete case (regular sampling, periodic boundary conditions). However, for graphs where $\mu_{\G}$ is closer to 1, the uncertainty principle \eqref{eq:eladbruckstein} is much weaker and therefore less informative. For example, $\|\hat{f}\|_0 \|f\|_0 \geq \frac{1}{\mu_{\G}^2} \approx 1$ is trivially true of nonzero signals. The same caveat applies to \eqref{eq:rictorr}
 and \eqref{eq:entr}, which follow directly from \cite{rictorrefined} and \cite{Maassen88generalized}, respectively, by once again specifying the canonical and graph Fourier bases.  
 The last inequality \eqref{eq:loc_folland} is an adaptation %application 
 \cite[Eq. (4.1)]{folland1997} to the graph setting, 
%\begin{equation}
%\sum_{i\in E}|f(i)|^2\le|E| \|f\|_\infty^2\le|E|\mu_G^2\|\hat{f}\|_1^2,
%\end{equation}
using the Hausdorff-Young inequality of Theorem~\ref{theo: Hausdorff-Young graph} (see next section).
%{\color{red}
%%%%%%%%
%The third family of uncertainty principle are `local' uncertainty principles. The idea is presented in the review~\cite{folland1997} in the continuous case. A discrete case version is given in~\cite{donoho1989uncertainty}. 
It states that the energy of a function in a subset of the domain is bounded from above by the size of the selected subset and the sparsity of the function in the Fourier domain. %Since these principles do not rely on the underlying structure of the domain it applies directly to the graph setting. A graph version of Eq.(4.3) in~\cite{folland1997} can be stated as follow. Let $E$ be a subset of vertices of the graph, for any function $f$ defined on the graph,
%\begin{equation}
%\sum_{i\in E}|f(i)|^2\le|E| \|f\|_\infty^2\le|E|\mu_G^2\|\hat{f}\|_1^2,
%\end{equation}
%where $|E|$ is the size of the set $E$.
%The latter inequality is a consequence of the discrete Hausdorff-Young inequality stated later on in Th.~\ref{theo: Hausdorff-Young graph}. 
If the subset $\V_S$ is small and the function is sparse in the graph Fourier domain, this uncertainty principle limits the amount of energy of $f$ 
that fits insides of the subset of $\V_S$.
%is prevented to fit entirely inside $E$, showing an uncertainty in the vertex domain. 
Because $\V_S$ can be chosen to be a local region of the domain (the graph vertex domain in our case), Folland and Sitaram \cite{folland1997} refer to such principles as ``local uncertainty inequalities.'' 
%This is local in the sense that $E$ can be chosen to be a local region of the graph. 
However, the term $\mu_{\G}$ in the uncertainty bound is not local in the sense that it depends on the whole graph structure and not just on the topology of the subgraph containing vertices in $\V_S$.

The following example illustrates the relation between the graph, the concentration of a specific graph signal, and one of the uncertainty principles from Theorem \ref{th:generaluncertainty}. We return to this example in Section \ref{Se:limitations} to discuss further the limitations of these uncertainty principles featuring $\mu_{\G}$.

\begin{example} \label{Ex:lpunc}
Figure~\ref{fig:path_node_away_uncertainty1} shows the computation of the quantities involved in~\eqref{eq:rictorr}, with $p=1$ and different $\G$'s taken to be the modified path graphs of Example~\ref{ex:pathgraph}, with different distances between the first two vertices. We show the lefthand side of \eqref{eq:rictorr} for two different Kronecker deltas, one centered at vertex 1, and one centered at vertex 10. We have seen in Figure \ref{fig:intro graph topology} that as the distance between the first two vertices increases, the coherence increases, and therefore the lower bound on the right-hand side of \eqref{eq:rictorr} decreases. For $\delta_1$, the uncertainty quantity on the left-hand side of \eqref{eq:rictorr} follows a similar pattern. The intuition behind this is that as the weight between the first two vertices decreases, a few of the eigenvectors start to have local jumps around the first vertex (see Figure \ref{fig: modified path eig}). As a result, we can sparsely represent $\delta_1$ as a linear combination of those eigenvectors and $||\widehat{\delta_1}||_1$ is reduced. However,  since there are not any eigenvectors that are localized around the last vertex in the path graph, we cannot find a sparse linear combination of the graph Laplacian eigenvectors to represent $\delta_{10}$. Therefore, its uncertainty quantity on the left-hand side of \eqref{eq:rictorr} does not follow the behavior of the lower bound.
%One of the signals chosen is a Kronecker placed on the first node, $\delta_1$. Its uncertainty (product of spreadings in the vertex and Laplacian spectral domains) decreases with the connectivity strength of node 1. The uncertainty bound $1/\mu_{\G}$ has the same behavior. However, for the function $\delta_{10}$ located on node $10$, its uncertainty remains almost constant: its spreading behavior is not well-described by the uncertainty principle of~\eqref{eq:rictorr}.
%Intuition: as weight goes down, a few of the eigenvectors have local jumps around the first vertex. As a result, we can sparsely represent the delta there as a linear combination of those eigenvectors and $||\widehat{\delta_1}||_1$ is reduced. 
\begin{figure}[ht]
\begin{center}
\includegraphics[width=0.45\textwidth]{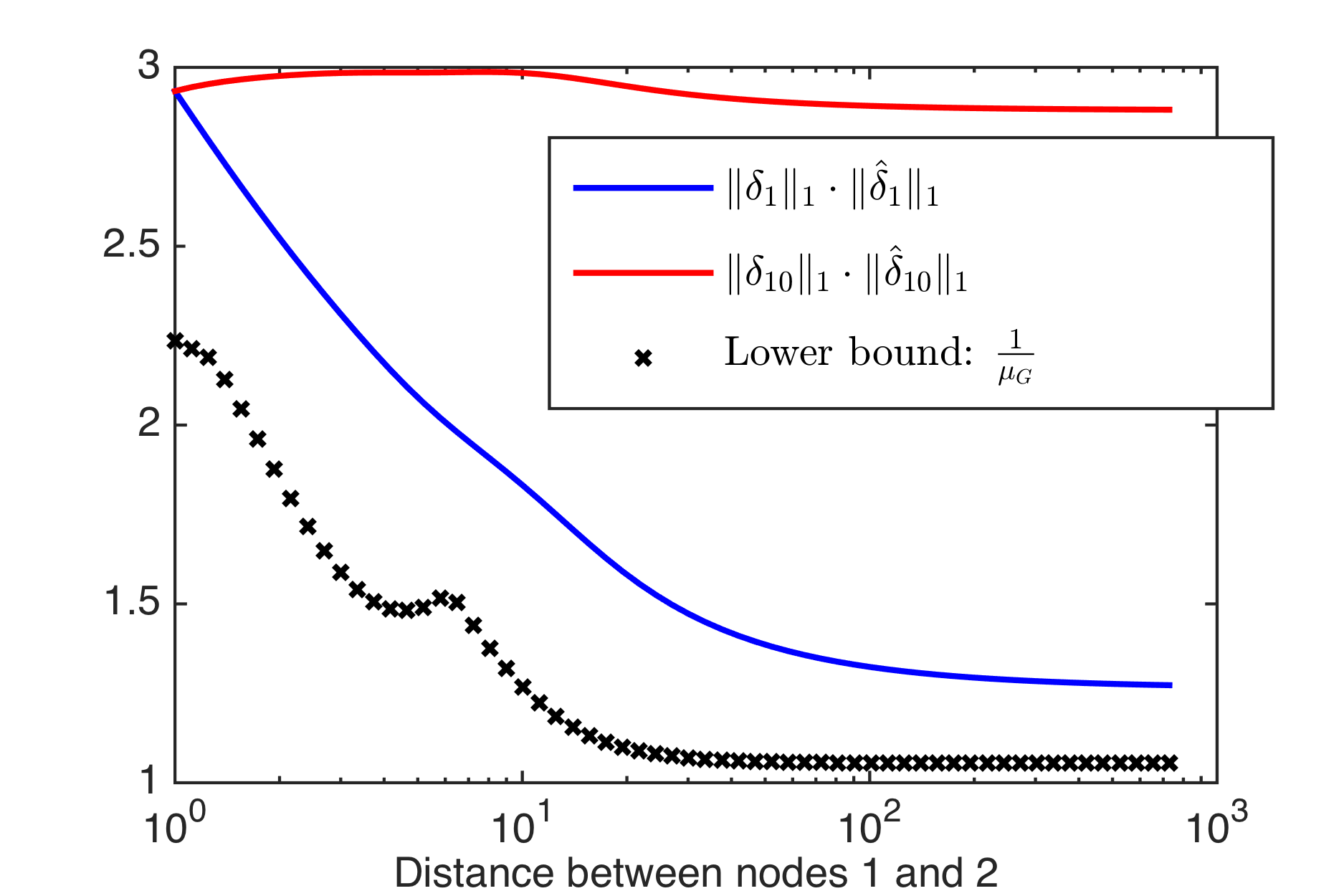}
\end{center}
\caption{Numerical illustration of the $\l^p$-norm uncertainty principle on a sequence of modified path graphs with different mutual coherences between the canonical basis of deltas and the graph Laplacian eigenvectors. For each modified path graph, the weight $W_{12}$ of the edge between the first two vertices is the reciprocal of the distance shown on the horizontal axis. The black crosses show the lower bound on the right-hand side of \eqref{eq:rictorr}, with $p=1$. The blue and red lines show the corresponding uncertainty quantity on the left-hand side of  \eqref{eq:rictorr}, for the graph signals $\delta_1$ and $\delta_{10}$, respectively.}
%between the Numerical computation of both side of 
 %for $p=1$ for two different functions, with respect to the distance between node 1 and 2. This distance is defined as the inverse of the weight $d=1/W_{12}$. }
\label{fig:path_node_away_uncertainty1}
\end{figure}
\end{example}

%%%%%%%%%%%%%%%%%%%%

\subsection{The Hausdorff-Young inequalities for signals on graphs}\label{Se:HYI}

The classical Hausdorff-Young inequality \cite[Section IX.4]{reed1975methods} is a fundamental harmonic analysis result behind the intuition that a high degree of concentration of a signal in one domain (time or frequency) implies a low degree of concentration in the other domain. This relation is used in the proofs of the entropy and $\ell^p$-norm uncertainty principles in the continuous setting. In this section, as we continue to explore the role of $\mu_{\G}$ and the differences between the Euclidean and graph settings, we extend the Hausdorff-Young inequality to graph signals.
 
%The Hausdorff-Young inequality for graphs, like its classic analogous, allows to bound norms in $\ell^p$-spaces of a signal and of its Fourier transform. This inequality is behind the intuition that concentration in one domain (time or frequency) implies spreading in the other. We recall that a small $\ell^p$-norm means sparsity for $p<2$ and spreading for $p>2$ (see previous section). The Hausdorff-Young relation is involved in the proof of the entropy and $\ell^p$-norm uncertainty principles in the continuous setting. Establishing it for the graph case provides a clearer view on the differences between the classical and graph cases and on the role of $\mu_{\G}$. From a broader point of view, it is also a basic result that is of value for any attempt to generalize harmonic analysis and signal processing to the graph setting. 
%The graph version of the Haudorff-Young inequality is the following.
\begin{theorem} \label{theo: Hausdorff-Young graph}
Let $\mu_{\G}$ be the coherence between the graph Fourier and canonical bases of a graph $\G$. Let $p,q>0$ be such that  $\frac{1}{p}+\frac{1}{q}=1$. For any signal $f \in \Cbb^N$ defined on $\G$ and $1 \leq p \leq 2$, we have
\begin{equation} \label{eq: HY 1}
\| \hat f \|_q \leq \mu_{\G}^{1-\frac{2}{q}} \| f \|_p .
\end{equation}
Conversely, for $2 \leq p \leq \infty$, we have
\begin{equation} \label{eq: HY 2}
\| \hat f \|_q \geq \mu_{\G}^{1-\frac{2}{q}} \| f \|_p.
\end{equation}
\end{theorem}
The proof of Theorem ~\ref{theo: Hausdorff-Young graph}, given in Sec.~\ref{sec:Rieszthorin}, is an extension of the classical proof using the Riesz-Thorin interpolation theorem. In the classical (infinite dimensional) setting, the inequality only depends on $p$ and $q$. On a graph, it depends on $\mu_{\G}$ and hence on the structure of the graph. %It is a first basic result showing the implication of $\mu_{\G}$ in the spreading of a function and its graph Fourier transform. 

Dividing both sides of each inequality in Theorem \ref{theo: Hausdorff-Young graph} by $\|f\|_2$ leads to bounds on the concentrations (or sparsity levels) of a graph signal %sparsity or spreading of a function
 and its graph Fourier transform. 
 \begin{corollary} \label{Co:splevel}
 Let $p,q>0$ be such that  $\frac{1}{p}+\frac{1}{q}=1$. For any signal $f \in \Cbb^N$ defined on the graph $\G$, we have
\begin{align*}
s_p(f) s_q(\hat{f}) \leq {\mu_{\G}^{\left| 1-\frac{2}{q}\right|}}.
\end{align*}
 \end{corollary}
 %As seen previously, $\|f\|_p/\|f\|_2$ and $\|\hat f\|_p/\|\hat f\|_2$ are 
 %concentration measures {\color{red} Maybe use the reciprocal for some values of $p$
%measures of spreading for the function and its graph Fourier transform respectively. 
Theorem~\ref{theo: Hausdorff-Young graph} and Corollary \ref{Co:splevel} assert that concentration or sparsity level of a graph signal in one domain (vertex or graph spectral) limits the concentration or sparsity level in the other domain. However, once again, if the coherence $\mu_{\G}$ is close to 1, the result is not particularly informative as $s_p(f) s_q(\hat{f})$ is trivially upper bounded by 1. The following numerical experiment illustrates the quantities involved in the Hausdorff-Young inequalities for graph signals. We again see that as the graph Fourier coherence increases, signals may be simultaneously concentrated in both the vertex domain and the graph spectral domain.
%We perform a numerical experiment showing the behavior of the involved quantities on a simple graph in the following example.
%{\color{red}
\begin{example}
Continuing with the modified path graphs of Examples \ref{ex:pathgraph} and \ref{Ex:lpunc}, we illustrate the bounds of the Hausdorff-Young inequalities for graph signals in Figure \ref{fig:hausdorffyoung}. For this example, we take the signal $f$ to be $\delta_1$, a Kronecker delta centered on the first node of the modified path graph. As a consequence, $\|\delta_1\|_p=1$ for all $p$, which makes it easier to compare the quantities involved in the inequalities. For this example, the bounds of Theorem \ref{theo: Hausdorff-Young graph} are fairly close to the actual values of $\|\hat{\delta_1}\|_q$.

%Using the example of the modified path graph presented in example~\ref{ex:pa	thgraph}, we illustrate on Fig.~\ref{fig:hausdorffyoung} the fact that graphs with low coherence may allow signals to be highly concentrated both in the vertex and spectral domain. 
%For this example, the input signal $f$ is a Kronecker delta placed on the first node, $f(1)=1$ and $f(i)=0$ for $i\neq1$. 
\begin{figure}[ht!]
\centering
\hfill
\begin{minipage}[b]{.4\linewidth}
\centerline{\includegraphics[width=\linewidth]{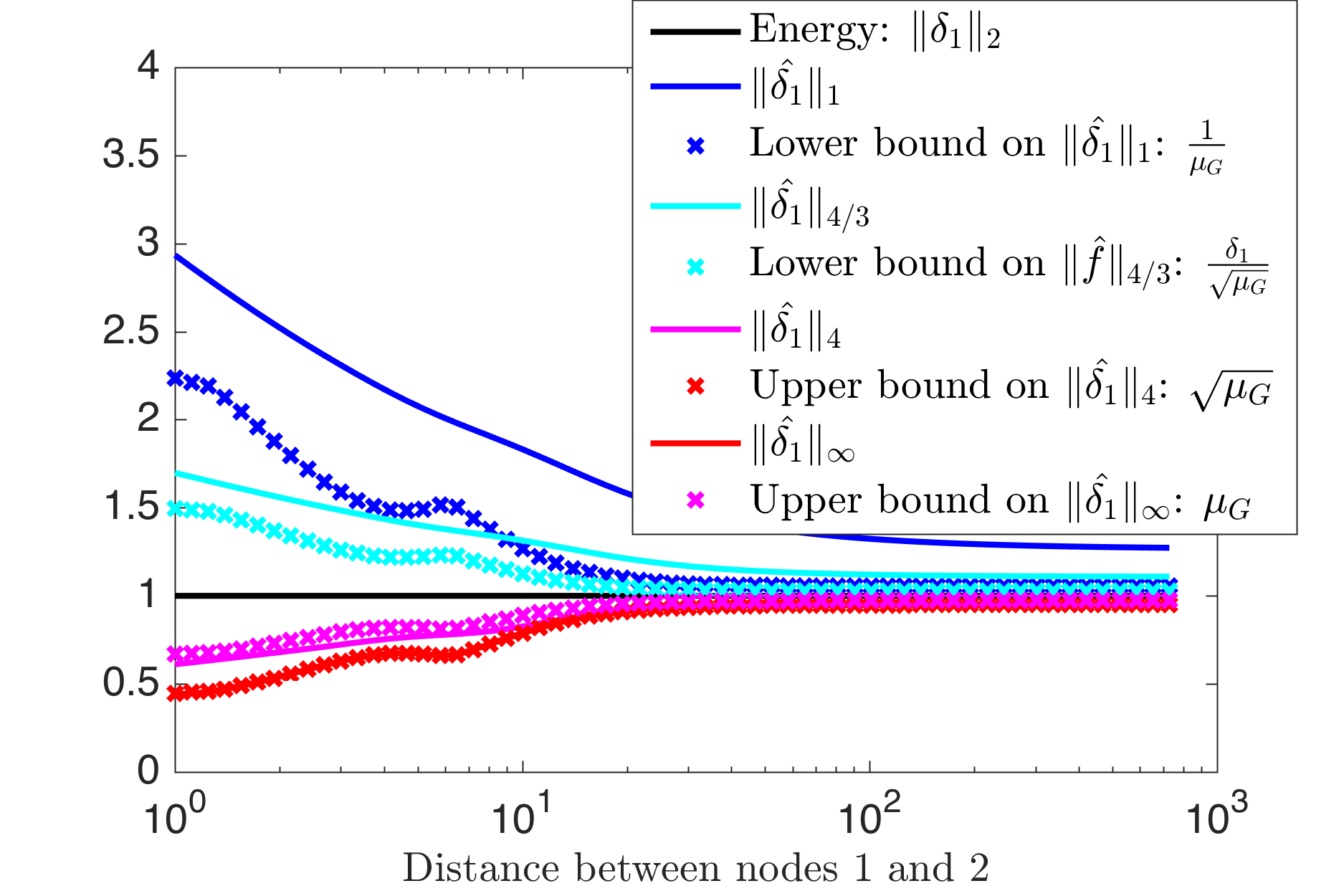}} 
\centerline{\small{~~(a)}}
\end{minipage}
\hfill
\begin{minipage}[b]{.4\linewidth}
\centerline{\includegraphics[width=\linewidth]{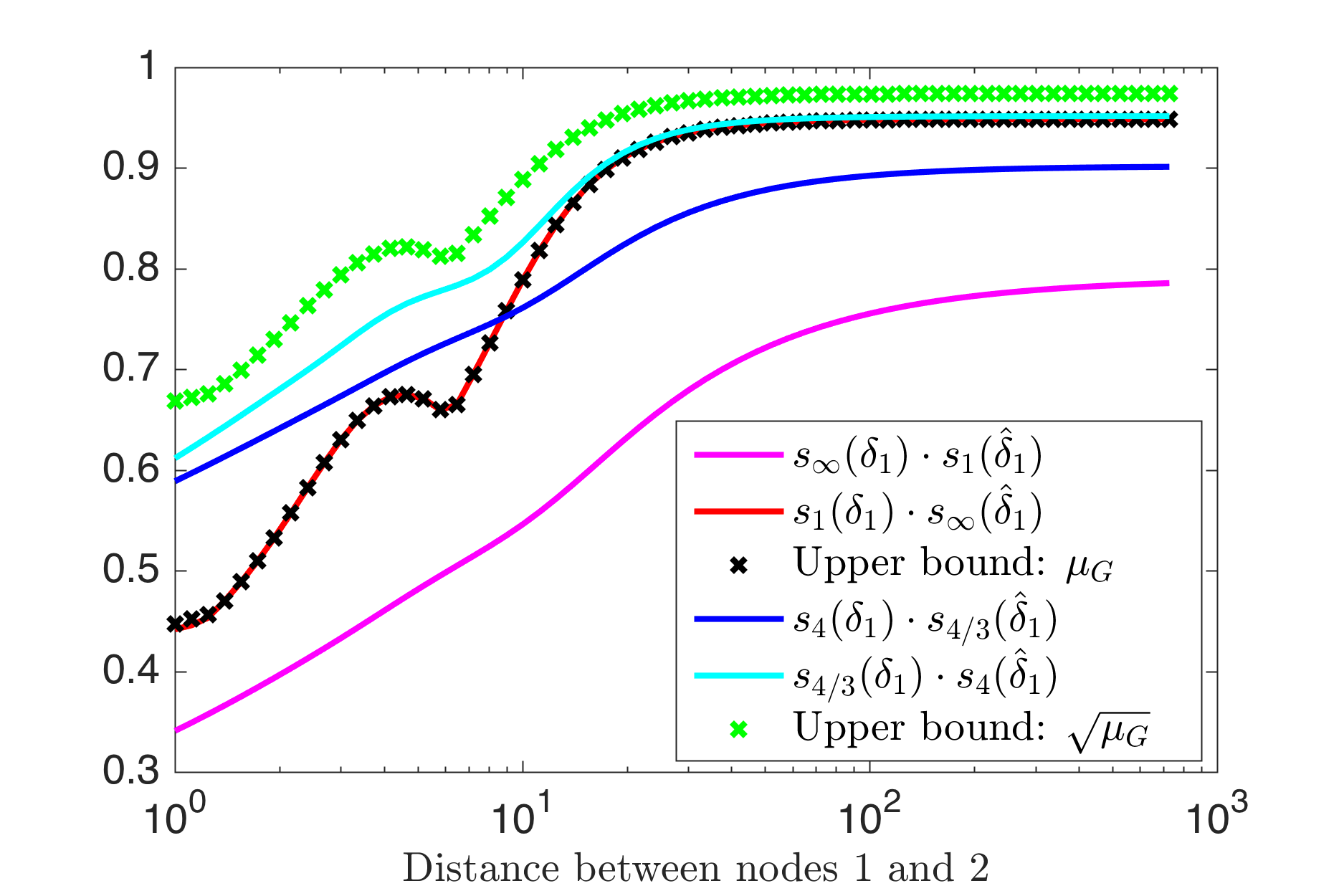}} 
\centerline{\small{~~(b)}}
\end{minipage}
\hfill
\hfill
%\begin{center}
%\includegraphics[width=0.45\textwidth]{figures/Bound11_HJ.png}
\caption{Illustration of the bounds of the Hausdorff-Young inequalities for graph signals on the modified path graphs with $f=\delta_1$. %(Ex. \ref{ex:pathgraph}). 
(a) The quantities in \eqref{eq: HY 1} and \eqref{eq: HY 2} for $q=1,\frac{4}{3},4,$ and $\infty$. (b) The quantities in Corollary \ref{Co:splevel} for the same values of $q$. %The input signal is a Kronecker delta placed on the node $i=1$. The abscissa represents the distance between nodes 1 and 2, $d=1/W_{12}$. 
}
\label{fig:hausdorffyoung}
%\end{center}
\end{figure}
\end{example}

%}

\paragraph{Sharpness of the graph Hausdorff-Young inequalities.}
For $p=q=2$, \eqref{eq: HY 1} and \eqref{eq: HY 2}  becomes equalities. Moreover, for $p=1$ or $p=\infty$, there is always at least one signal for which the inequalities \eqref{eq: HY 1} and \eqref{eq: HY 2}  become equalities, respectively.
%In the classical case, the discrete Hausdorff-Young inequality is sharp for all picket fence signal (i.e. equally space sequences of Kronecker) {\GBR We need a ref for that} . This cannot be generalized to the graph setting. However, we observe that 
%There is always at least one signal for which the inequality becomes an equality in the particular cases $p=1$, $p=2$ and $p=\infty$. 
Let $i_1$ and $\l_1$ satisfy $\mu_{\G}=\max_{i,\l}|u_{\l}(i)|=|u_{\l_1}(i_1)|$.
For $p=1$, let $f=\delta_{i_1}$. Then $||f||_1=1$, and $||\hat{f}||_\infty=\max_{\l} |\langle \delta_{i_1},u_{\l}\rangle|=\mu_{\G}$, and thus \eqref{eq: HY 1} is tight. 
For $p=\infty$, let $f=u_{\l_1}$. Then $||f||_\infty=\mu_{\G}$, $||\hat{f}||_1=||\widehat{u_{\l_1}}||_1=1$, and thus  \eqref{eq: HY 2} is tight. The red curve and its bound in Figure ~\ref{fig:hausdorffyoung}  show the tight case for $p=1$ and $q=\infty$.

%then for $f(n)=\delta_{i_1}(n)$, $\hat{f}(k)=u_k(i_1)$, $\|\delta_{i_1}\|_p=1$ for all $p$, $\|\hat{\delta}_{i_1}\|_{\infty}=\mu_{\G}$ and \eqref{eq: HY 1} is tight for $p=1$. For $f(i) = u_{k_1}(i)$, $\hat{f}(k)=\delta_{k_1}(k)$ and \eqref{eq: HY 2} is tight for $p=\infty$. The case $p = \infty$ has been plotted on Fig.~\ref{fig:hausdorffyoung} (red curves).

\subsection{Limitations of global concentration-based uncertainty principles in the graph setting} \label{Se:limitations}
%\noindent{\bf Limitations of this family of uncertainty principles in the graph setting.}
%\nati{I believe this subsection should go somewhere else. Maybe at the end of 3. Currently it is kind of strange.}
The motivation for this section was twofold. First, we wanted to derive the uncertainty principles for graph signals analogous to some of those that are so fundamental for signal processing on Euclidean domains. However, we also want to highlight the limitations of this approach (the second family of uncertainty principles described in Section \ref{Se:intro}) in the graph setting. The graph Fourier coherence is a global parameter that depends on the topology of the entire graph. Hence, it may be greatly influenced by a small localized changes in the graph structure. For example, in the modified path graph examples above, a change in a single edge weight leads to an increased coherence, and in turn significantly weakens the uncertainty principles  characterizing the  concentrations of the graph signal in the vertex and spectral domains. Such examples call into question the ability of such global uncertainty principles for graph signals to accurately describe phenomena in inhomogeneous graphs.   This is the primary motivation for our investigation into local uncertainty principles %proposed in proposed
 in Section \ref{subsec:local}. However, before getting there, we consider global uncertainty principles from the third family of uncertainty principles described in Section \ref{Se:intro} that bound the concentration of the analysis coefficients of a graph signal in a time-frequency transform domain.

\section{Graph signal processing operators and dictionaries}\label{sec:keyquantities}

As mentioned in Section \ref{Se:intro}, uncertainty principles can inform dictionary design. 
In the next section, we present uncertainty principles characterizing the concentration of the analysis coefficients of graph signals in different transform domains.  We focus on three different classes of dictionaries for graph signal analysis: (i) frames, (ii) localized spectral graph filter frames, and (iii) graph Gabor filter bank frames, where localized spectral graph filter frames are a subclass of frames, and graph Gabor filter bank frames are a subclass of localized spectral graph filter frames. In this section, we define these 
%In this section, we present uncertainty principles for three different classes of dictionaries for graph signals: (i) frames, (ii) localized spectral graph filter frames, and (iii) windowed graph Fourier frames.  The latter two classes of graph signal dictionaries belong to the general class of frames. Below we define these 
different classes of dictionaries, and highlight some of their mathematical properties. Note that our notation uses dictionary atoms that are double indexed by $i$ and $k$, but these could be combined into a single index $j$ for the most general case.
\begin{definition}[Frame]
A dictionary $\Dc = \{g_{i,k}\}$ is a frame if there exist constants $A$ and $B$ called the lower and upper frame bounds such that for all $f\in\Cbb^N$:
$$A\|f\|_2^2\le\sum_{i,k}|\langle f,g_{i,k}\rangle|^2\le B\|f\|_2^2.$$
If $A=B$, the frame is said to be a {tight frame}.
\end{definition}
For more properties of frames, see, e.g., \cite{christensen2003introduction,kovacevic_frames1,kovacevic_frames2}. Most of the recently proposed dictionaries for graph signals are either orthogonal bases (e.g., \cite{coifman2006diffusion,narang2012perfect,sakiyama}) , which are a subset of tight frames, or overcomplete frames (e.g., \cite{hammond2011wavelets,shuman2015vertex,shuman2013spectrum}). 

In order to define localized spectral graph filter frames,
%windowed graph Fourier frames \cite{shuman2012windowed,shuman2015vertex}, 
we need to first recall one way to generalize the translation %and modulation operators 
operator to the graph setting.

\begin{definition}[Generalized localization/translation operator on graphs \cite{hammond2011wavelets,shuman2015vertex}]\label{def:translation}
We localize (or translate) a kernel $\hat{g}$ to center vertex $i\in\{1,2,\ldots,N\}$ by applying the localization operator $\T_i$, whose action is defined as
$$
\T_ig(n) =  \sqrt{N} \sum_{\ell=0}^{N-1} \hat{g}(\lambda_\ell) \overline{\x_{\ell}(i)}\x_{\ell}(n).
$$
\end{definition}
Note that this generalized localization operator is a kernelized operator. It does not translate an arbitrary signal defined in the vertex domain to different regions of the graph, but rather localizes a pattern defined in the graph spectral domain to be centered at different regions of the graph. The smoothness of the kernel $\hat{g}(\cdot)$ to be localized can be used to bound the localization of the translated kernel around a center vertex $i$; i.e., if a smooth kernel $\hat{g}(\cdot)$ is localized to center vertex $i$, then the magnitude of $T_i g(n)$ decays as the distance between $i$ and $n$ increases \cite[Section 5.2]{hammond2011wavelets}, \cite[Section 4.4]{shuman2015vertex}. Except for special cases such as when $\G$ is a circulant graph with $\mu_{\G}=\frac{1}{\sqrt{N}}$ and the Laplacian eigenvectors are the DFT basis, the generalized localization operator of Definition \ref{def:translation} is not isometric. Rather, we have 
\begin{lemma}[\cite{shuman2015vertex}, Lemma 1] \label{Le:Tisom}
For any $g \in \Cbb^N$, 
\begin{align}\label{Eq:Tisometric}
|\hat{g}(0)| \leq ||T_i g||_2 \leq \sqrt{N} \nu_i ||\hat{g}||_2 \leq \sqrt{N} \mu_{\G} ||\hat{g}||_2,
\end{align}
which yields the following upper bound on the operator norm of $T_i$:
\begin{align*}
||T_i ||_{op} = \sup_{g \in \Cbb^N} \frac{||T_i g||_2}{||\hat{g}||_2}\leq \sqrt{N} \nu_i  \leq \sqrt{N} \mu_{\G},
\end{align*}
where $\nu_i=\max_{\l} \left|u_{\l}(i)\right|$.
\end{lemma}
%For the special case when , Lemma \ref{Le:Tisom} confirms that $T_i$ is an isometric operator. {\color{red}Double check the lower bound.} 
It is interesting to note that although the norm is not preserved when a kernel is localized on an arbitrary graph, it is preserved on average when translated to separately to every vertex on the graph:
\begin{eqnarray}\label{eq: sum translation conservation energy} 
\frac{1}{N}\sum_{i=1}^{N}\|T_{i}g\|_{2}^{2} = \sum_{i=1}^{N}\sum_{\ell=0}^{N-1}\left|\hat{g}(\lambda_{\ell})\bar{\x}_{{\ell}}(i)\right|^{2} =  \sum_{\ell=0}^{N-1}\left|\hat{g}(\lambda_{\ell})\right|^{2}\sum_{i=1}^{N}\left|\bar{\x}_{{\ell}}(i)\right|^{2} =  \|\hat{g}\|_{2}^{2}. 
\end{eqnarray}
%\end{remark}
%The properties of the generalized translation provided in the literature shows a connection between the graph structure and the spreading of a function when localized at a given vertex. We present 

The following example presents more precise insights on the interplay between the localization operator, the graph structure, and the %spreading 
concentration of localized functions.
\begin{example} \label{Ex:isom}
Figure~\ref{fig: uncertainty shift} illustrates the effect of the graph structure on the norms of localized functions.
%heat kern localization operator on the norm of the localized function and on its spreading for different graphs.  
%Let 
We take the kernel to be localized to be a heat kernel of the form $\hat{g}(\lambda_\ell) = e^{-\tau \lambda_\ell}$, for 
 %denote the heat kernel for the Laplacian spectral value $\lambda_\ell$, with 
 some constant $\tau>0$. We localize the kernel $\hat{g}$ to be centered at each vertex $i$ of the graph with the operator $\T_i$, and we compute and plot their $\ell^2$-norms $\|\T_ig\|_2$. The figure shows that when a center node $i$ and its surrounding vertices are relatively weakly connected, the $\ell^2$-norm of the localized heat kernel is large, and when the nodes are relatively well connected, the norm is smaller. Therefore, the norm of the localized heat kernel may be seen as a measure of vertex centrality.\footnote{In fact, the square norm of the localized heat kernel at vertex $i$ is, up to constants, the average diffusion distance from $i$ to all other vertices. It is therefore a genuine measure of centrality.} %\nati{This is true because we choose the heat kernel. I do not think we can say that in general!}
%{\color{red}
 %the variation of the $\ell^2$-norm may be understood as a change in the spreading of the window due to the irregular structure of the graph. 
%}
Moreover, in the case of the heat kernel, we can relate the $\ell^2$-norm of $T_i g$ to its 
concentration $s_1(T_1 g)$. 
%the %spreading 
Localized heat kernels are comprised entirely of nonnegative components; i.e., $\T_ig(n)\ge0$ for all $i$ and $n$. 
%Although the set $\{\T_ig\}_i$ keeps a similar localized shape around their respective central vertex, the concentration may vary depending on the local structure of the graph around the central vertex.
%The graph signal associated to the heat kernel stays positive when localized, %This is not a trivial property, it comes from the fact that 
This property comes from (i) the fact that $\T_ig(n)=\left(\hat{g}(\L)\right)_{in}$ (see~\cite{hammond2011wavelets}), and (ii) the non-trivial property that the entries of $\hat{g}(\L)$ are always nonnegative for the heat kernel~\cite{metzger}. Since $\T_ig(n) \ge 0$ for all $i$ and $n$,
we have
\begin{align} \label{Eq:one_norm_trans_window}
\|\T_ig\|_1= \sum_{n=1}^N T_i g(n)=\sqrt{N}\hat{g}(0)=\sqrt{N},  
\end{align}
where the second equality follows from \cite[Corollary 1]{shuman2015vertex}. Thus, recalling that a large value for $s_1(T_i g)$ means that $T_i g$ is concentrated, we can combine \eqref{Eq:Tisometric}
 and \eqref{Eq:one_norm_trans_window} to derive an upper bound on the concentration of $T_i g$:
% it implies $\|\T_ig\|_1= \sqrt{N} $ for all $i$, as $\sum_{n=1}^N T_if(n)=\sqrt{N}\hat{f}(0)$ for any $f$ on the graph. The following quantity:
%$$s(g)=\frac{\|g\|_1}{\|g\|_2}=\frac{\sqrt{N}}{\|g\|_2},
%$$
$$s_1(T_i g) = \frac{||T_i g||_2}{||T_i g||_1}=\frac{||T_i g||_2}{\sqrt{N}} \leq \nu_i ||\hat{g}||_2.$$
%is a concentration measure; % (see Sec.~\ref{sec:uncertainty1}). %A small value for $s(g)$ means that $g$ is concentrated. 
%a large value for $s_1(T_i g)$ means that $T_i g$ is concentrated.
Thus, $||T_i g||_2$ serves as  a measure of concentration, and  according to the numerical experiments of Figure \ref{fig: uncertainty shift}, localized heat kernels centered on the relatively well-connected regions of a graph tend to be less concentrated than the ones centered on relatively less well-connected areas. %The connectivity favors the spreading of functions.
Intuitively, the values of the localized heat kernels can be linked to the diffusion of a unit of energy from the center vertex to surrounding vertices over a fixed time. In the well-connected regions of the graph, energy diffuses faster, making the localized heat kernels less concentrated.

%\nati{Remember that the heat kernel refers to a diffusion process. As a result, the sparsity levels of the localized atoms are a good estimation of the vertex density. Densely sampled regions are more connected and the diffusion is faster. }
%
%{\color{red} May need to comment a bit more on how ``well-connected''  is relative to other vertices in the graph. Can we relate $\nu_i$ to van Mieghem paper to improve upper bound and also relate the two measures of centrality?}
%\nati{I add some stuff. I would not go for mathematical bounds here.}

\begin{figure}
\begin{center}
\includegraphics[width=0.3\textwidth]{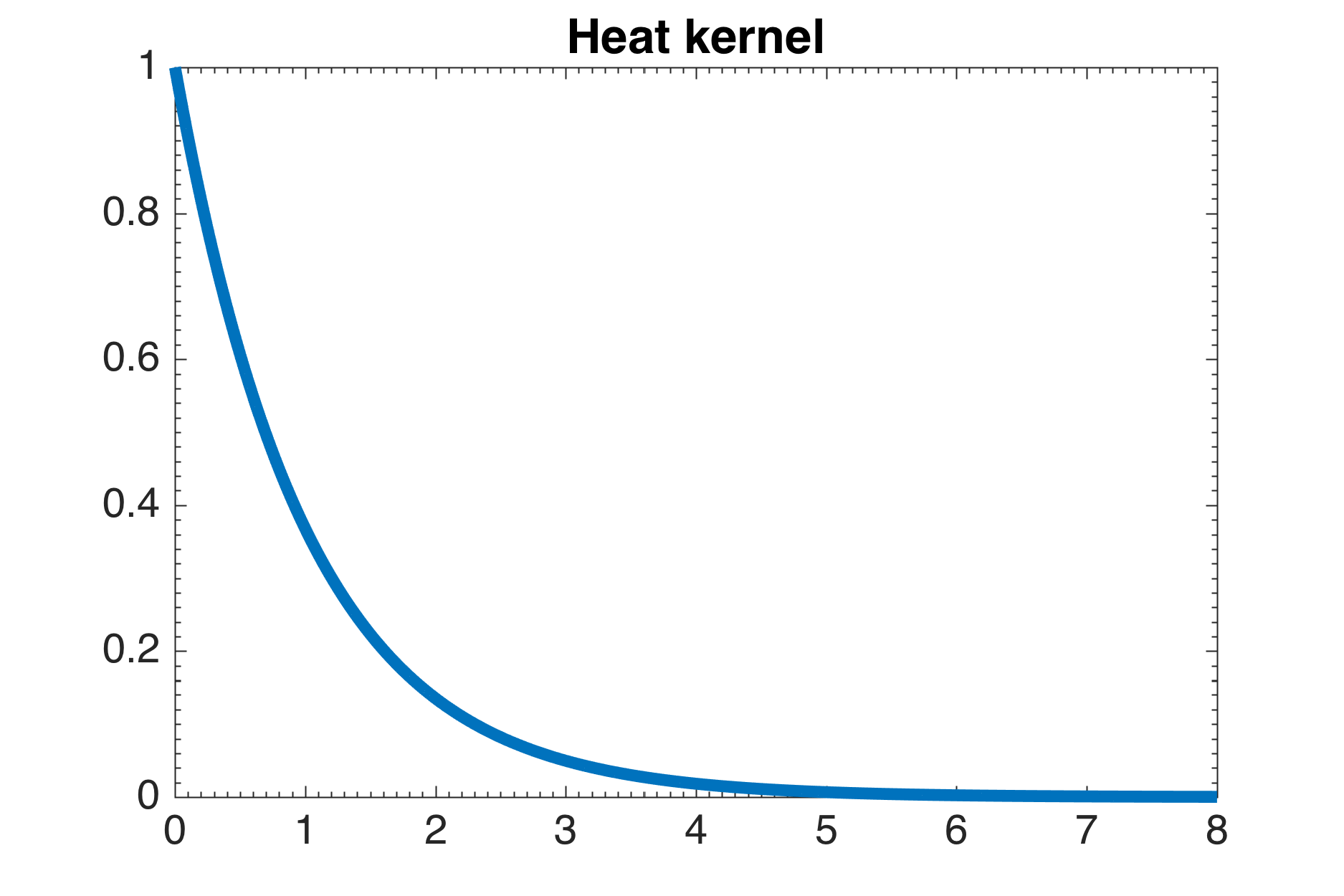} 
\includegraphics[width=0.3\textwidth]{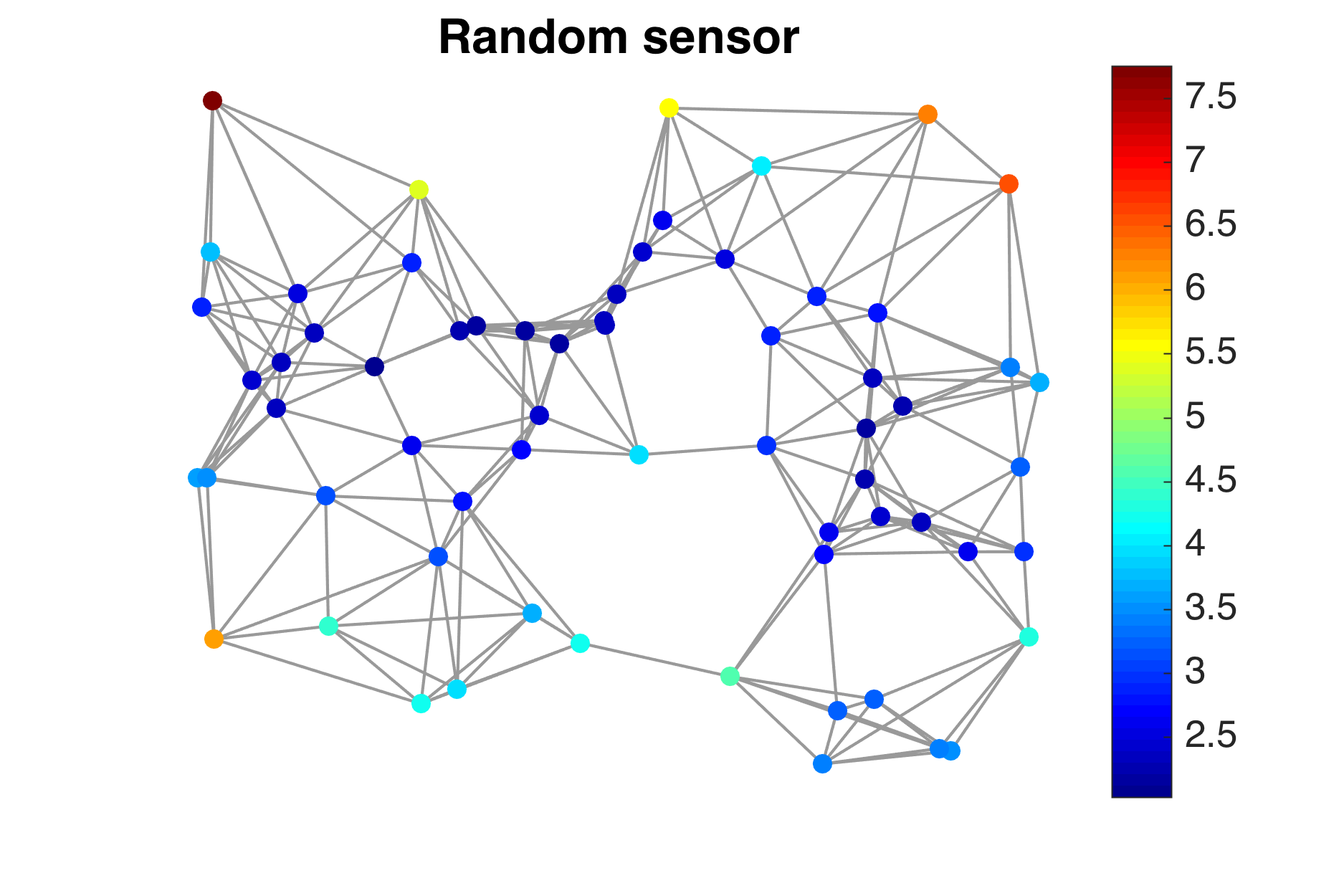} 
\includegraphics[width=0.3\textwidth]{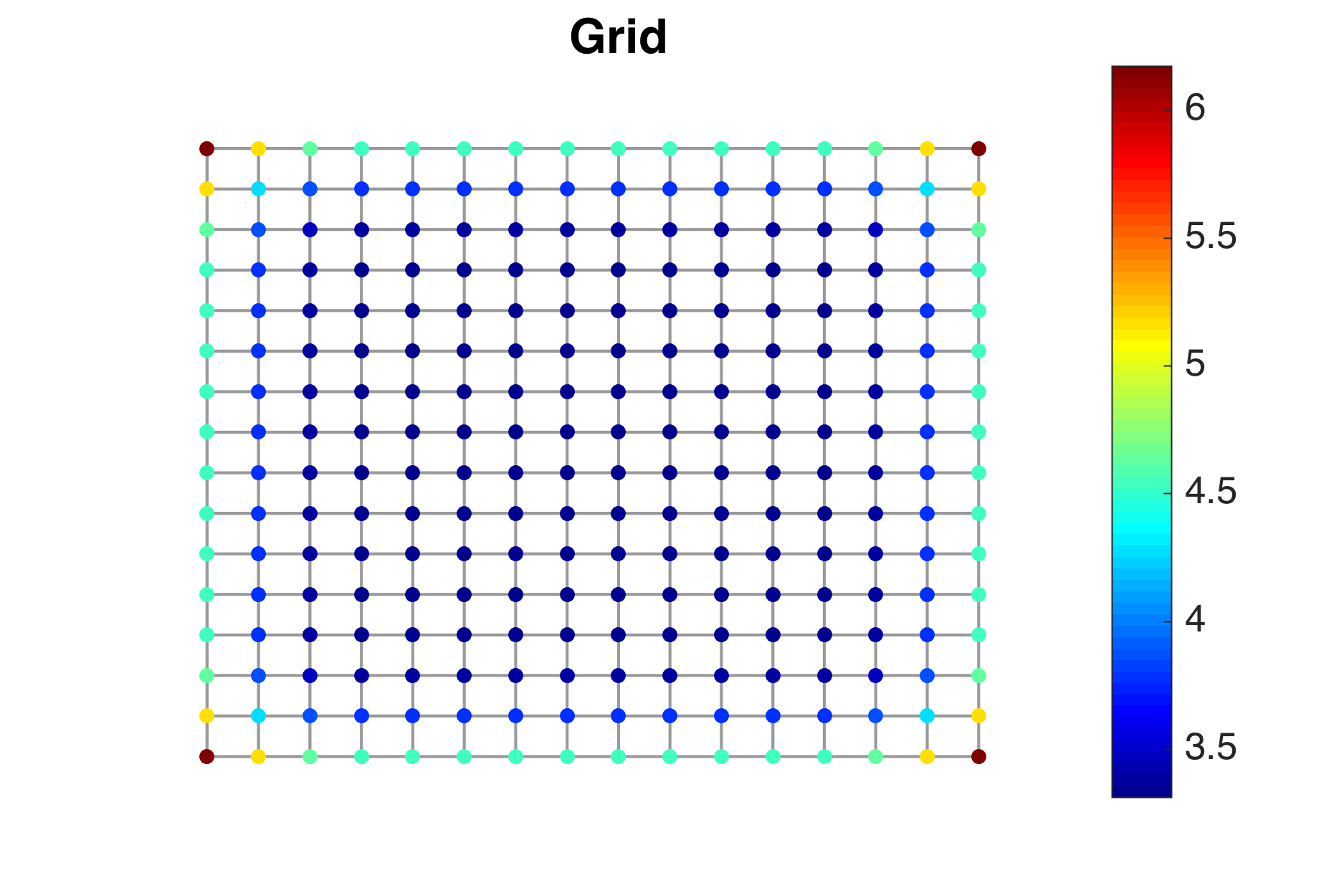} 
\par
\includegraphics[width=0.3\textwidth]{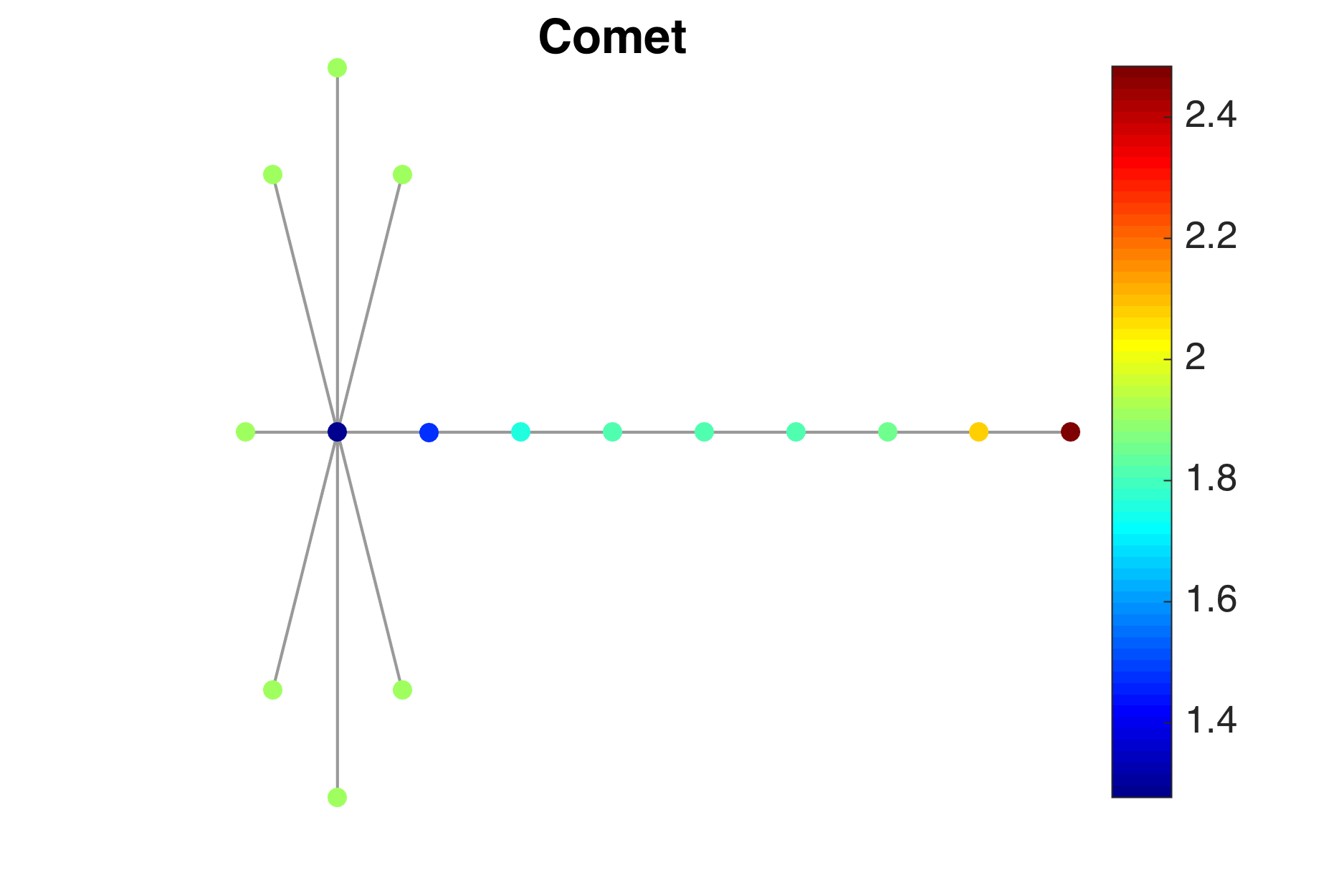} 
\includegraphics[width=0.3\textwidth]{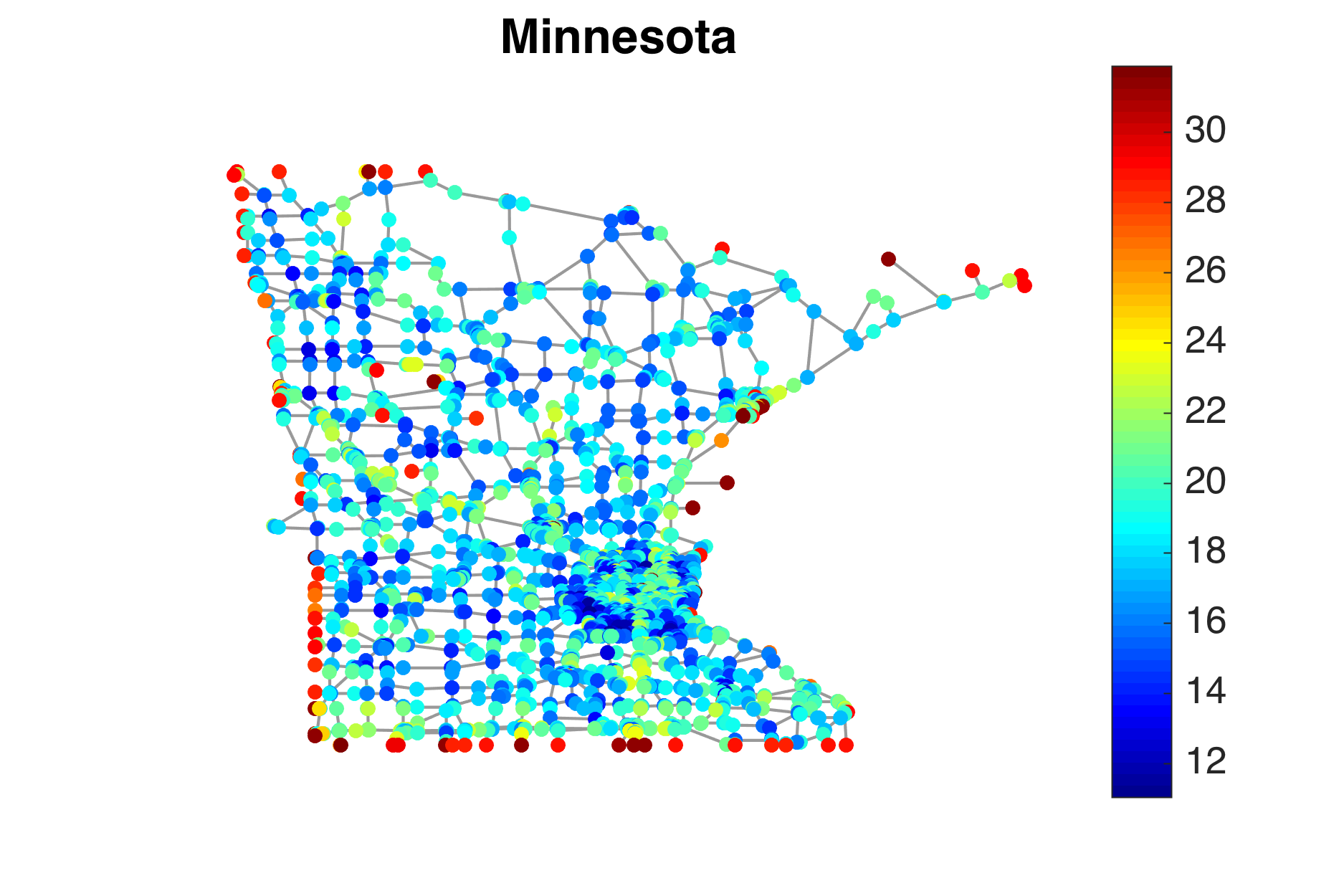} 
\includegraphics[width=0.3\textwidth]{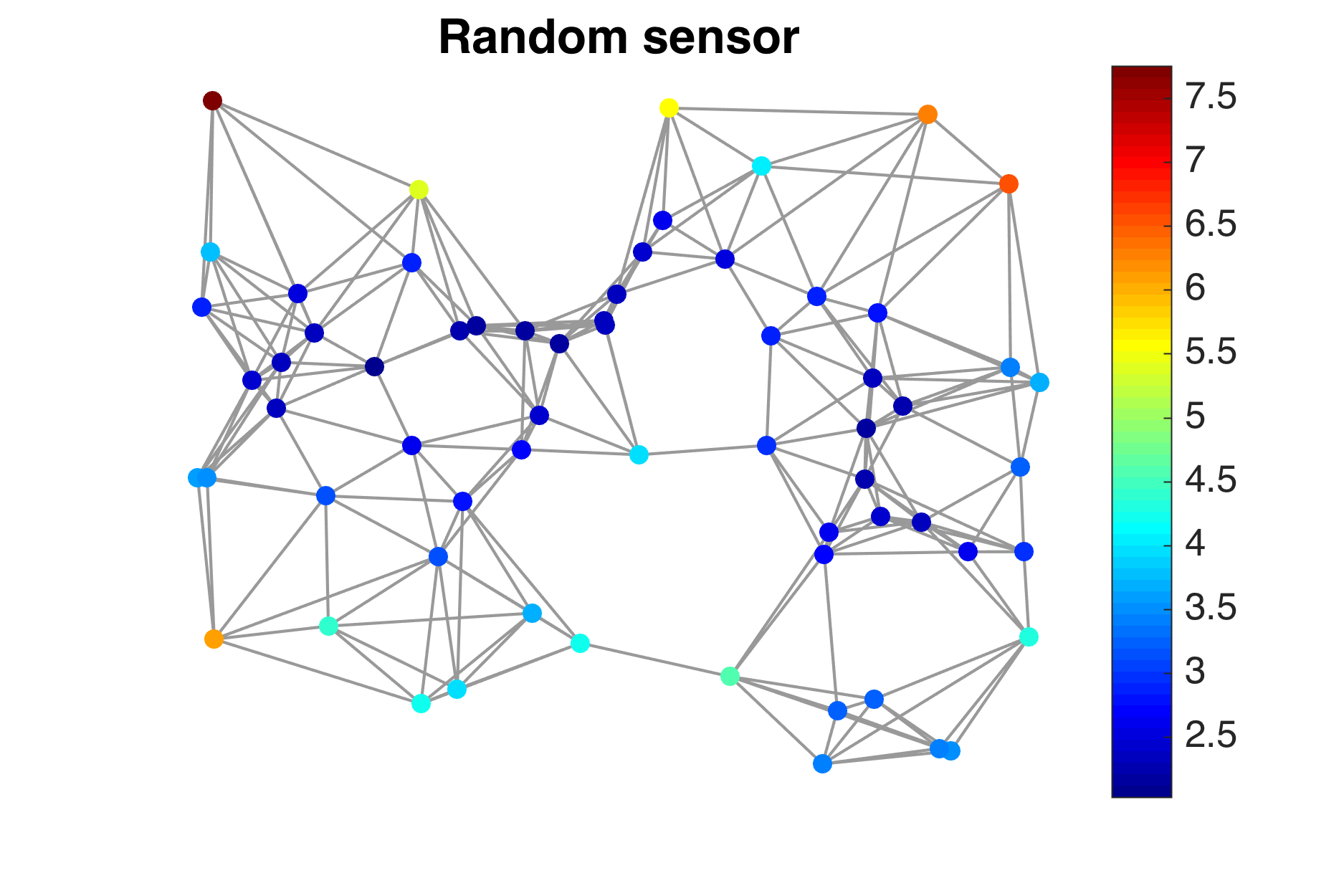} 
\end{center}
\caption{\label{fig: uncertainty shift} %{\color{red} 
The heat kernel $\hat{g}(\lambda_\ell) = e^{-10\frac{\lambda_\ell}{\lambda_\text{max}} }$ (upper left), and the norms of the localized heat kernels, $\left\{||T_i{g}||_2\right\}_{i=1,2,\ldots,N}$, on various graphs.
%Measure of the norm of localized heat kernels $T_i{g}$ with $\hat{g}(\lambda_\ell) = e^{-10\frac{\lambda_\ell}{\lambda_\text{max}} }$. 
For each graph and each center node $i$, the color of vertex $i$ is proportional to the value of $\|T_ig\|_2$. %The top left plot is the heat kernel defined continuously. 
Within each graph, nodes $i$ that are relatively less connected to their neighborhood seem to yield a larger norm $\|T_ig\|_2$.}
%}
\end{figure}

\end{example}

The main class of dictionaries for graph signals that we consider is localized spectral graph filter frames. 
\begin{definition}[Localized spectral graph filter frame]
Let  $\gg=\{\widehat{g_0}(\cdot),\widehat{g_1}(\cdot),\ldots,\widehat{g_{K-1}}(\cdot)\}$ be a sequence of kernels (or filters), where  each $\hg_k:\sigma(\L)\to \Cbb$
is a function defined on the graph Laplacian spectrum $\sigma(\L)$ of a graph $\G$. Define the quantity $G(\lambda):=\sum_{k=0}^{K-1}|\widehat{g_k}(\lambda_\ell)|^2$. Then
$\Dc_{\gg}=\{g_{i,k}\}=\{T_i g_k\}$ is a localized spectral graph filter dictionary, and it forms a frame if $G(\lambda)>0$ for all $\lambda \in \sigma(\L)$.
\end{definition}
In practice, each filter $\widehat{g_k}(\cdot)$ is often defined as a continuous function over the interval $[0,\lambda_{\max}]$ and then applied to the discrete set of eigenvalues in $\sigma(\L)$. The following lemma characterizes the frame bounds for a localized spectral graph filter frame.
\begin{lemma}[\cite{shuman2013spectrum}, Lemma 1]
Let  $\Dc_{\gg}=\{g_{i,k}\}=\{T_i g_k\}$ be a localized spectral graph filter frame of atoms on a graph $\G$ generated from the sequence of filters $\gg=\{\widehat{g_0}(\cdot),\widehat{g_1}(\cdot),\ldots,\widehat{g_{K-1}}(\cdot)\}$. The lower and upper frame bounds for $\Dc_{\gg}$ are given by $A=N\cdot \min_{\lambda \in \sigma(\L)} G(\lambda)$ and $B=N\cdot \max_{\lambda \in \sigma(\L)} G(\lambda)$, respectively. If $G(\lambda)$ is constant over $\sigma(\L)$, then  $\Dc_{\gg}$ is a tight frame.
\end{lemma}
%Note that even when $\Dc_{\gg}$ is a tight frame, its atoms do not necessarily have the same length, as demonstrated in Example \ref{Ex:isom} above. 
%\nati{I do not understand the previous sentence.}
Examples of localized spectral graph filter frames include the spectral graph wavelets of \cite{hammond2011wavelets}, the Meyer-like tight graph wavelet frames of \cite{leonardi_fmri,leonardi_multislice}, the spectrum-adapted wavelets and vertex-frequency frames of \cite{shuman2013spectrum}, and the learned parametric dictionaries of \cite{thanou_TSP_2014}.
The dictionaries constructions in \cite{hammond2011wavelets, shuman2013spectrum} choose the filters so that their energies are localized in different spectral bands. Different choices of filters lead to different tilings of the vertex-frequency space, and can for example lead to wavelet-like frames or vertex-frequency frames (analogous to classical windowed Fourier frames). The frame condition that $G(\lambda)>0$ for all $\lambda \in \sigma(\L)$ ensures that these filters cover the entire spectrum, so that no band of information is lost during analysis and reconstruction. 
% We should also remark that because the generalized translation and modulation operators $T_i$ and $M_k$ do not commute, a windowed graph Fourier frame with atoms of the form $M_k T_i g$ cannot be written as a special case of localized spectral graph filter frames. Rather, the two special classes of graph signal dictionaries we consider here represent distinct transforms to perform vertex-frequency analysis. 

In this paper, in order to generalize classical windowed Fourier frames, we often use a localized graph spectral filter bank where the kernels are uniform translates, which we refer to as a graph Gabor filter bank.
\begin{definition}[Graph Gabor filter bank]
When the $K$ kernels used to generate the localized graph spectral filter frame are uniform translates of each other, we refer to the resulting dictionary as a \emph{graph Gabor filter bank} or a \emph{graph Gabor filter frame}. If we use the warping technique of \cite{shuman2013spectrum} on these uniform translates, we refer to the resulting dictionary as a \emph{spectrum-adapted graph Gabor filter frame}.
%A graph Gabor filter bank is constructed by uniformly shifting a kernel $\hat{g}(\cdot)$ $K$ times between $0$ and $\lmax$. The resulting filters are given by:
%\nati{
%$$
%\widehat{g_k}(\lambda) = \hat{g}\left(\lambda - \frac{k}{K-1}\lambda_{\max}\right), k=0,1,\ldots,K-1.
%$$}
%The graph Gabor transform is the corresponding analysis operation and is written $\A_\gg$.
\end{definition}
Graph Gabor filter banks are generalizations of the short time Fourier transform. When $\hat{g}$ is smooth, the atoms are localized in the vertex domain.
%both in the vertex (thanks to the localization operator) and in the spectral domain (by construction). 
%When using our Theorem \ref{Co:Lieblocgraph} with a Gabor filter bank\footnote{and additional complex transformation (see \ref{Se:discrete_Lieb} for details)}, we recover Lieb uncertainty principle. 
In this contribution, for all graph Gabor filter frames, we use the following mother window:
$\hat{g}(t) = \sin \left( 0.5 \pi \cos (\pi (t-0.5))^2 \right), $ for $t\in [-0.5, 0.5]$ and $0$ elsewhere. A few desirable properties of this choice of window are (a) it is perfectly localized in the spectral domain in $[-0.5 ,0.5]$, (b) it is smooth enough to be approximated by a low order polynomial, and (c) the frame formed by uniform translates (with an even overlap) is tight.

\begin{definition}[Analysis operator]\label{def:analysis}
The analysis operator of a dictionary $\Dc=\{g_{i,k}\}$ to a signal $f\in\Cbb^N$ is given by
$$ \A_{\Dc} f(i,k) = \langle f,g_{i,k}\rangle.$$
When $\Dc=\{g_{i,k}\}=\{T_i g_k\}$ is a localized spectral graph filter frame, we denote it with $\A_{\gg}$. In all cases, we view $A_{\Dc}$ as a function from $\Cbb^N$ to $\Cbb^{|\Dc|}$, and thus we use $||\A_{\Dc} f||_p$ (or $||\A_{\gg} f||_{p}$) to denote a vector norm of the analysis  coefficients. 
\end{definition}

\section{Global uncertainty principles bounding the concentration of the analysis coefficients of a graph signal in a transform domain}\label{Se:Lieb_global}

\subsection{Discrete version of Lieb's uncertainty principle}

%{\color{red}
%
%Make the following discrete argument more precise. Include continuous statements here?
%Following Lieb~\cite{lieb1990integral}, we prove a discrete version of the ambiguity function uncertainty principle. This discrete version seems to be known~\cite{ricaud2014optimally} but we could not find any proof in the literature.
%To prove this theorem we need the two following lemmas. 
%
%To compute the ambiguity function in the standard setting, $g$ is localized and modulated, creating a set of atoms $g_{i,k} = \M_k \T_i g$ ($\T_i$ being the localization operator to vertex $i$ and $\M_k$ the modulation at "frequency" $k$). Let $\A_\ggf(i,k) = \scp{f}{g_{i,k}}$ be the cross-ambiguity function of $f$ with window $g$, $\langle\cdot,\cdot\rangle$ denotes the scalar product. Lieb in \cite{lieb1990integral} has established bounds on the $\ell^p$-norm of the cross-ambiguity function in the continuous case. If we transpose them to the discrete setting, we find 

Lieb's uncertainty principle in the continuous one-dimensional setting \cite{lieb1990integral} states that the cross-ambiguity function of a signal cannot be too concentrated in the time-frequency plane. In the following, we transpose these statements to the discrete periodic setting, and then generalize them to frames and signals on graphs. The following discrete version of Lieb's uncertainty principle is partially presented in \cite[Proposition 2]{eusipcooptamb}.
%for $p \in [1,2]$:
%\begin{equation}\label{eq: lieb p 1 2}
%\| \A_\ggf \|_p  = \left( \sum_{i=0}^{N-1} \sum_{k=0}^{N-1} | \A_\ggf(i,k)|^p \right)^\frac{1}{p} \geq N^{\frac{1}{p}} \|f\|_2 \|g\|_2,
%\end{equation}
%and for $p \in [2,\infty)$:
%\begin{equation} \label{eq: lieb p 1 infty}
%\| \A_\ggf \|_p = \left( \sum_{i=0}^{N-1} \sum_{k=0}^{N-1} | \A_\ggf(i,k)|^p \right)^\frac{1}{p} \leq N^{\frac{1}{p}} \|f\|_2 \|g\|_2.
%\end{equation}
%To make an uncertainty principle, both expressions are needed. If we take $p=1$ in  \eqref{eq: lieb p 1 2}, we are sure that the sum of all element is greater than $N \|g\|_2 \|f\|_2$. Now setting $p=\infty$ in the dual expression \eqref{eq: lieb p 1 infty} tells us that all elements are smaller than $\|g\|_2\|f\|_2$. This implies that $\A_{\gg}f $ has at least $N$ non zero elements. 

\begin{theorem} \label{theo:Classical_Lieb_discrete}
Define the discrete Fourier transform (DFT) as 
$$\hat{f}[k]=\frac{1}{\sqrt{N}}\sum_{n=0}^{N-1}f[n] \exp\left(\frac{-i2\pi k n}{N}\right) ,$$
and the discrete windowed Fourier transform (or discrete cross-ambiguity function) as (see, e.g., \cite[Section 4.2.3]{mallat2008wavelet})
$$\A_{\Dc_{DWFT}} f[u,k]=\sum_{n=0}^{N-1}f[n]\overline{g[n-u]}\exp\left(\frac{-i2\pi k n}{N} \right).$$
For two discrete signals of period $N$, we have for $2 \leq p < \infty$
\begin{align} \label{eq: lieb p 1 infty}
\| \A_{\Dc_{DWFT}} f \|_p = \left(\sum_{u=1}^N \sum_{k=0}^{N-1} | \A_{\Dc_{DWFT}} f [u,k]|^p\right)^{\frac{1}{p}} \leq N^\frac{1}{p}\| f \|_2 \| g\|_2,
\end{align}
and for $1 \leq p \leq 2$
\begin{align} \label{eq: lieb p 1 2}
\|\A_{\Dc_{DWFT}} f \|_p = \left(\sum_{u=1}^N \sum_{k=0}^{N-1} | \A_{\Dc_{DWFT}} f [u,k]|^p\right)^{\frac{1}{p}} \geq N^\frac{1}{p}\| f \|_2 \| g\|_2.
\end{align}
\end{theorem}
These inequalities are proven in Section \ref{Se:discrete_Lieb} of the Appendix.  Note that the minimizers of this uncertainty principle are the so-called "picket fence" signals, trains of regularly spaced diracs.

\subsection{Generalization of Lieb's uncertainty principle to frames}

%{\color{red}
%The first featured theorem concerns the tight frames made of localized filter-banks $\{\T_ig_k\}_{i,k}$. This latter property is used to obtain the right hand side inequalities. However the main result (left hand side) holds for any set $\{g_{i,k}\}_{i,k}$ that is a tight frame. The generalization to any frame will be presented in the next section as Theorem \ref{theo ambiguity bound}.
%}

\begin{theorem} \label{theo: amb tight}
Let  $\Dc=\{g_{i,k}\}$ be a frame of atoms in $\Cbb^N$, with lower and upper frame bounds $A$ and $B$, respectively. For any signal $f\in \Cbb^N$ and any $p \geq 2$, we have
\begin{equation} \label{eq: lieb generalization p 2 inf_main}
\|\A_{\Dc}f\|_{p}\leq B^{\frac{1}{p}}\left(\max_{i,k}\|g_{i,k}\|_{2}\right)^{1-\frac{2}{p}}\|f\|_{2}.
\end{equation}
For any signal $f\in \Cbb^N$ and any $1 \leq p \leq 2$, we have
\begin{equation}  \label{eq: lieb generalization p 1 2_main}
\|\A_{\Dc}f\|_{p}\geq A^{\frac{1}{p}}\left(\max_{i,k}\|g_{i,k}\|_{2}\right)^{1-\frac{2}{p}}\|f\|_{2}.
\end{equation}
Combining \eqref{eq: lieb generalization p 2 inf_main}
and \eqref{eq: lieb generalization p 1 2_main}, 
%for $p\ge 2$
%$$
%s_p(\mathcal{A}_{g}f)=\frac{\|\mathcal{A}_{g}f\|_{p}}{\|\mathcal{A}_{g}f\|_{2}}
%\leq\frac{B^{\frac{1}{p}}}{A^{\frac{1}{2}}}\left(\max_{i,k}\|g_{i,k}\|_{2}\right)^{\left|1-\frac{2}{p}\right|}
%$$
%and for $1\leq p \leq 2$
%$$
%s_p(\mathcal{A}_{g}f)=\frac{\|\mathcal{A}_{g}f\|_{2}}{\|\mathcal{A}_{g}f\|_{p}}
%\leq\frac{B^{\frac{1}{2}}}{A^\frac{1}{p}}\left(\max_{i,k}\|g_{i,k}\|_{2}\right)^{\left|1-\frac{2}{p}\right|}
%$$
for any $p\in [1,\infty]$, we have
\begin{align}\label{Eq:ft1a}
s_p(\mathcal{A}_{\Dc}f) \leq\frac{B^{\min\{\frac{1}{2},\frac{1}{p}\}}}{A^{\max\{\frac{1}{2},\frac{1}{p}\}}}\left(\max_{i,k}\|g_{i,k}\|_{2}\right)^{\left|1-\frac{2}{p}\right|}.
\end{align}
When $\Dc$ is a tight frame with frame bound $A$, \eqref{Eq:ft1a} reduces to 
$$
s_p(\mathcal{A}_{\Dc}f) \leq A^{-\left|\frac{1}{2}-\frac{1}{p} \right|}\left(\max_{i,k}\|g_{i,k}\|_{2}\right)^{\left|1-\frac{2}{p}\right|}.
$$
\end{theorem}
A proof is included in Section \ref{Se:lieb_gen_proof} of the Appendix. The proof of Theorem \ref{theo:Classical_Lieb_discrete} in Section \ref{Se:discrete_Lieb} of the Appendix also demonstrates that this uncertainty principle is indeed a generalization of the discrete periodic variant of Lieb's uncertainty principle.
%%%%%%%%%%%%%
%\begin{featuredtheorem} \label{theo: amb tight_old}
%Let $\G$ be a graph, $\{g_{i,k}=\T_ig_k\}_{i,k}$ be a tight frame on $\G$ with frame bound $A$ and $f\in \Cbb^N$ any signal defined on $\G$. For $p\geq 2$ we have the following inequality
%$$
%\frac{\|\mathcal{B}_{g}f\|_{p}}{\|\mathcal{B}_{g}f\|_{2}}
%\leq A^{\frac{1}{p}-\frac{1}{2}}\left(\max_{i,k}\|g_{i,k}\|_{2}\right)^{1-\frac{2}{p}}
%\leq A^{\frac{1}{p}-\frac{1}{2}}\left(\sqrt{N}\mu_{\G}  \max_k \|g_k\|_{2}\right)^{1-\frac{2}{p}}.
%$$
%Moreover for $1\leq p \leq 2$, we have
%$$
%\frac{\|\mathcal{B}_{g}f\|_{p}}{\|\mathcal{B}_{g}f\|_{2}}
%\geq A^{\frac{1}{p}-\frac{1}{2}}\left(\max_{i,k}\|g_{i,k}\|_{2}\right)^{1-\frac{2}{p}}
%\geq A^{\frac{1}{p}-\frac{1}{2}} \left(\sqrt{N}\mu_{\G} \max_k \|g_k\|_{2} \right)^{1-\frac{2}{p}}.
%$$
%\end{featuredtheorem}

%{\color{red}
%In the standard discrete setting,...
%Let us choose the particular case where the graph is the ring (standard discrete setting) and the frame is the short-time Fourier transform $\{\T_i\M_kg\}_{i,k\in[1,\cdots,N]^2}$. This gives $A=N$, $\|\A_{\gg}f\|_2=\sqrt{N} \|g\|_2\|f\|_2$ and $\|g_{i,k}\|_2=1$ for all $i$ and $k$. Replacing these values in the inequalities, we get the discrete version of Lieb's uncertainty principle (see Appendix). It demonstrates that this uncertainty principle is indeed a generalization of Lieb's one.
%}

\subsection{Lieb's uncertainty principle for localized spectral graph filter frames}

Lemma \ref{Le:Tisom} implies that $\max_{i,k} ||T_i g_k||_2 \leq  \sqrt{N} \mu_{\G}  \max_k\| \widehat{g_k}\|_2$. Therefore the following is a corollary to Theorem \ref{theo: amb tight} for the case of localized spectral graph filter frames.
%{\color{red}
%Corollary of previous result
%We highlight the graph dependence by providing a second (weaker) bound to the uncertainty inequalities. We bound the localization operator $\T_i$ to let $\mu_{\G}%$ appear:
%$$
%\|g_{i,k}\|_2 = \| \T_ig_k\|_2 \leq  \sqrt{N} \mu_{\G} \sqrt{\sum_l\sum_n|\hat{g}_k(l)|^2 |u_l(n)|^2}\leq \sqrt{N} \mu_{\G} \| g_k\|_2.
%$$
%}

%Get rid of frame bounds in a second inequality?

%$\Dc_G$ specify sequence of filters/kernels $\{\widehat{g_k}(\cdot)\}$
\begin{theorem} \label{Co:Lieblocgraph}
Let  $\Dc_{\gg}=\{g_{i,k}\}=\{T_i g_k\}$ be a localized spectral graph filter frame of atoms on a graph $\G$ generated from the sequence of filters $\gg=\{\widehat{g_0}(\cdot),\widehat{g_1}(\cdot),\ldots,\widehat{g_{K-1}}(\cdot)\}$. For any signal $f\in \Cbb^N$ on $\G$ and 
for any $p\in [1,\infty]$, we have
\begin{align}\label{Eq:Llg1a}
s_p(\mathcal{A}_{\gg}f) \leq
\frac{B^{\min\{\frac{1}{2},\frac{1}{p}\}}}{A^{\max\{\frac{1}{2},\frac{1}{p}\}}}\left(\max_{i,k}\|g_{i,k}\|_{2}\right)^{\left|1-\frac{2}{p}\right|} \leq
\frac{B^{\min\{\frac{1}{2},\frac{1}{p}\}}}{A^{\max\{\frac{1}{2},\frac{1}{p}\}}}\left(\sqrt{N}\mu_{\G} \max_{k}\|\widehat{g_k}\|_{2}\right)^{\left|1-\frac{2}{p}\right|},
\end{align}
where $A=\min_{\lambda \in \sigma(\L)} G(\lambda)$ is the lower frame bound and $B=\max_{\lambda \in \sigma(\L)} G(\lambda)$ is the upper frame bound. 
When $\Dc$ is a tight frame with frame bound $A$, \eqref{Eq:Llg1a} reduces to 
\begin{align}\label{Eq:Llg1atight}
s_p(\mathcal{A}_{\gg}f)  
\leq A^{-\left|\frac{1}{2}-\frac{1}{p} \right|}\left(\max_{i,k}\|g_{i,k}\|_{2}\right)^{\left|1-\frac{2}{p}\right|}
\leq A^{-\left|\frac{1}{2}-\frac{1}{p} \right|}\left(\sqrt{N}\mu_{\G} \max_{k}\|\widehat{g_k}\|_{2}\right)^{\left|1-\frac{2}{p}\right|}.
\end{align}
\end{theorem}

The bounds depend on the frame bounds $A$ and $B$, which are fixed with the design of the filter bank. However, in the tight frame case,  we can 
choose the filters in a manner such that the bound
 $A$ does not depend on the graph structure. For example, if the $\hg_k$ are defined continuously on the interval $[0,\lambda_{\rm max}]$ and $\sum_{k=0}^{M-1}\left|\hg_{k}(\lambda)\right|^{2}$ is equal to a constant for all $\lambda$, $A$ is not affected by a change in the values of the Laplacian eigenvalues, e.g., from a change in the graph structure. The second quantity, $\max_{i,k}\|g_{i,k}\|_{2}$, reveals the influence of the graph. The maximum $\ell_2$-norm of the atoms depends on the filter design, but also, as discussed previously in Section \ref{sec:keyquantities}, on the graph topology. 
However, the bound is not local as it depends on the maximum $\|g_{i,k}\|_2$ over all localizations $i$ and filters $k$, which takes into account the entire graph structure. %(all the graph structure is taken into account). 
% {\color{red}
% However, the role of the $\ell^2$-norm is made clearer in this theorem: the quantity $\|g_{i,k}\|_2$ can serve as a spreading measure and a criterion to measure the quality of the local graph structure, from the perspective of the uncertainty principle. David: I don't understand this. We just said it was not local. \nati{I agree. We should just remove this. In fact, we find that this theorem is taking the maximum over all local bounds of Theorem \ref{theo:local_uncertainty}. Shall we say that here?}
%}

The second bounds in \eqref{Eq:Llg1a} and \eqref{Eq:Llg1atight} also suggest how the filters can be designed so as to improve the uncertainty bound.
The quantity $\| \widehat{g_k}\|_2 = \left( \sum_\ell |\hat{g}_k(\lambda_\ell)|^2 \right)$ depends on the distribution of the eigenvalues $\lambda_\ell$, and, as consequence, on the graph structure. However, the distribution of the eigenvalues can be taken into account when designing the filters in order to reduce or cancel this dependency \cite{shuman2013spectrum}. 

In the following example, we compute the first uncertainty bound in \eqref{Eq:Llg1atight} for different types of graphs and filters. It provides some insight on the influence of the graph topology and filter bank design on the uncertainty bound. % and of the filter-bank on the bound.
\begin{example}\label{ex:filterbank} %   [Filter-bank cross-ambiguity function and the uncertainty principle]\label{ex:filterbank}
We use the techniques of \cite{shuman2013spectrum} to construct four tight localized spectral graph filter frames for each of eight different graphs. Figure \ref{fig: gabor filterbank} shows an examples of the four sets of filters for a 64 node sensor network. For each graph, two of the sets of filters (b and d in Figure \ref{fig: gabor filterbank}) are adapted via warping to the distribution of the graph Laplacian eigenvalues so that each filter contains an appropriate number of eigenvalues (roughly equal in the case of translates and roughly logarithmic in the case of wavelets). The warping avoids filters containing zero or very few eigenvalues at which the filter has a nonzero value.
%, and plot the 
%We build 4 filter-banks being tight frames (plotted on Fig.\ref{fig: gabor filterbank}). Two of them (b and d on the figure) are adapted to the distribution of the spectrum so each filter contains an appropriate number of eigenvalues, following the warping technique of \cite{shuman2013spectrum}. It avoids filters containing 0 or very few eigenvalues. 
These tight frames are designed such that { $A=N$}, and thus Theorem~\ref{Co:Lieblocgraph} yields
 $$
s_{\infty}(\A_{\gg}f)=\frac{\|\A_{\gg}f\|_\infty}{\|\A_{\gg}f\|_2} \leq { N^{-\frac{1}{2}}} \max_{i,k}\|\T_ig_k\|_2 \leq \mu_{\G}\max_k ||\widehat{g_k}||_2.
$$
Table  \ref{tab:uncertainty tight} displays the values of the first concentration bound $\max_{i,k}\|\T_ig_k\|_2$ for each graph and frame pair.
%The uncertainty bound for various graphs has been computed in Table~\ref{tab:uncertainty tight}. 
%{\color{red}
%Revisit the following write-up after Table is redone.
The uncertainty bound is %small 
largest when the graph is far from a regular lattice (ring or path). As expected, the worst cases are for highly inhomogeneous graphs like the comet graph or %some random graphs.
a modified path graph with one isolated vertex. 
%The bound is related to $\mu_{\G}$, large bounds are for graphs with a small graph Fourier coherence. 
The choice of the filter bank may also decrease or increase the bound, depending on the graph.
%}
%$$
%\frac{\|\A_{\gg}f\|_2}{\|\A_{\gg}f\|_\infty} \geq \frac{1}{\max_{i,k}\|\T_ig_k\|_2}.
%$$
\begin{figure}[htb!]
\centering
\begin{minipage}[b]{.24\linewidth}
\centerline{\small{~~Graph Gabor}}
\centerline{\small{~~Filter Frame}}
\centerline{\includegraphics[width=\linewidth]{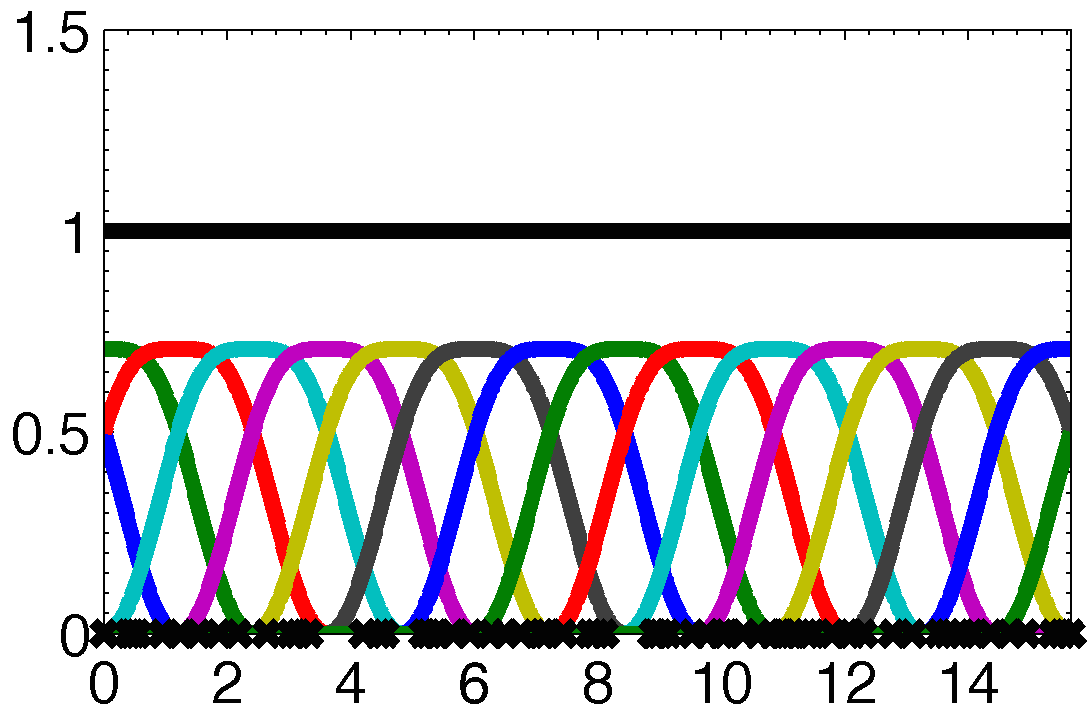}} 
\centerline{\small{~~$\lambda$}}
\centerline{\small{~~(a)}}
\end{minipage}
\hfill
\begin{minipage}[b]{.24\linewidth}
\centerline{\small{~~Spectrum-Adapted}}
\centerline{\small{~~Graph Gabor}}
\centerline{\small{~~Filter Frame}}
\centerline{\includegraphics[width=\linewidth]{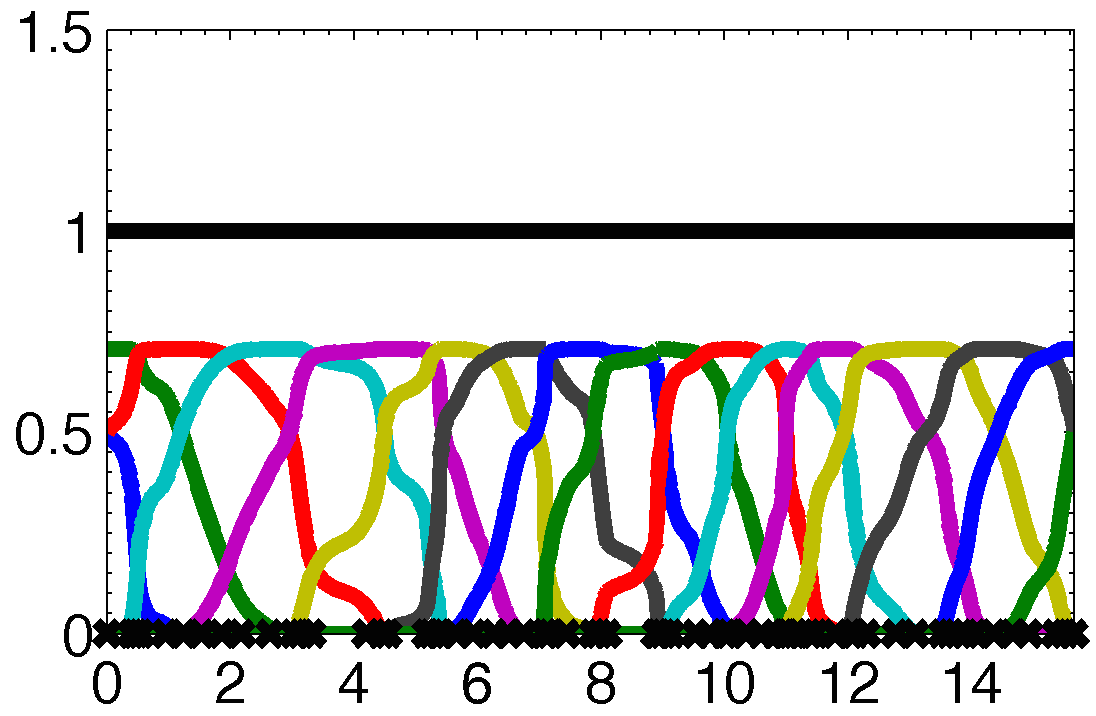}} 
\centerline{\small{~~$\lambda$}}
\centerline{\small{~~(b)}}
\end{minipage}
\hfill
\begin{minipage}[b]{.24\linewidth}
\centerline{\small{~~Log-Warped}}
\centerline{\small{~~Tight Graph}}
\centerline{\small{~~Wavelet Frame}}
\centerline{\includegraphics[width=\linewidth]{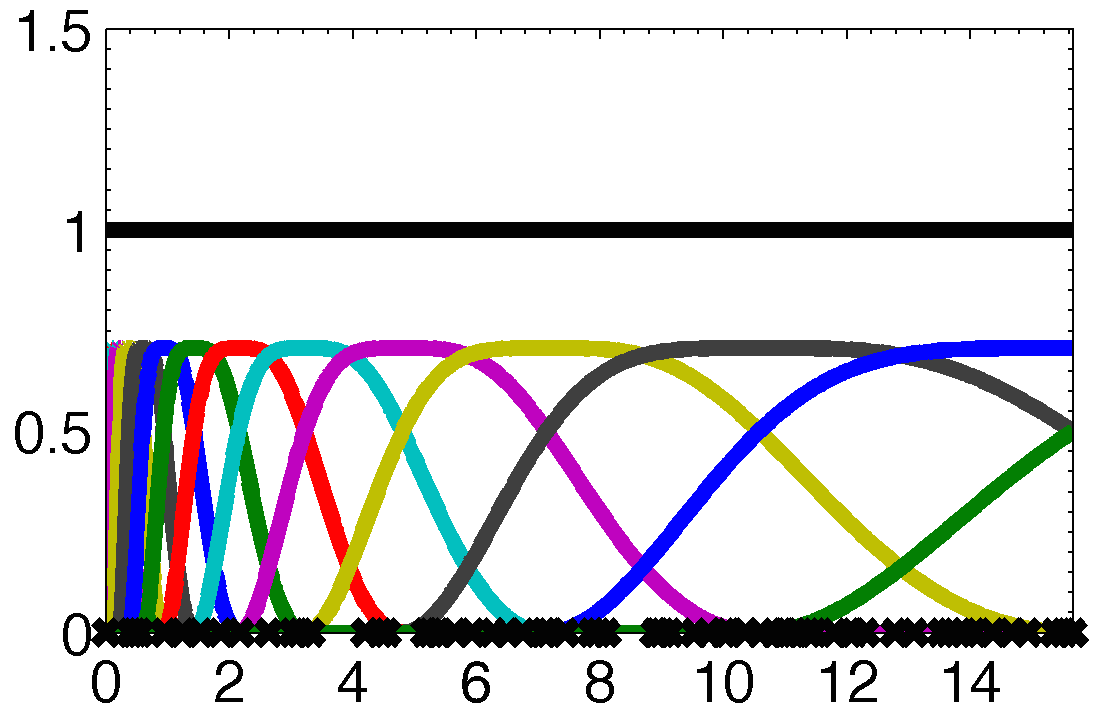}} 
\centerline{\small{~~$\lambda$}}
\centerline{\small{~~(c)}}
\end{minipage}
\hfill
\begin{minipage}[b]{.24\linewidth}
\centerline{\small{~~Spectrum-Adapted}}
\centerline{\small{~~Tight Graph}}
\centerline{\small{~~Wavelet Frame}}
\centerline{\includegraphics[width=\linewidth]{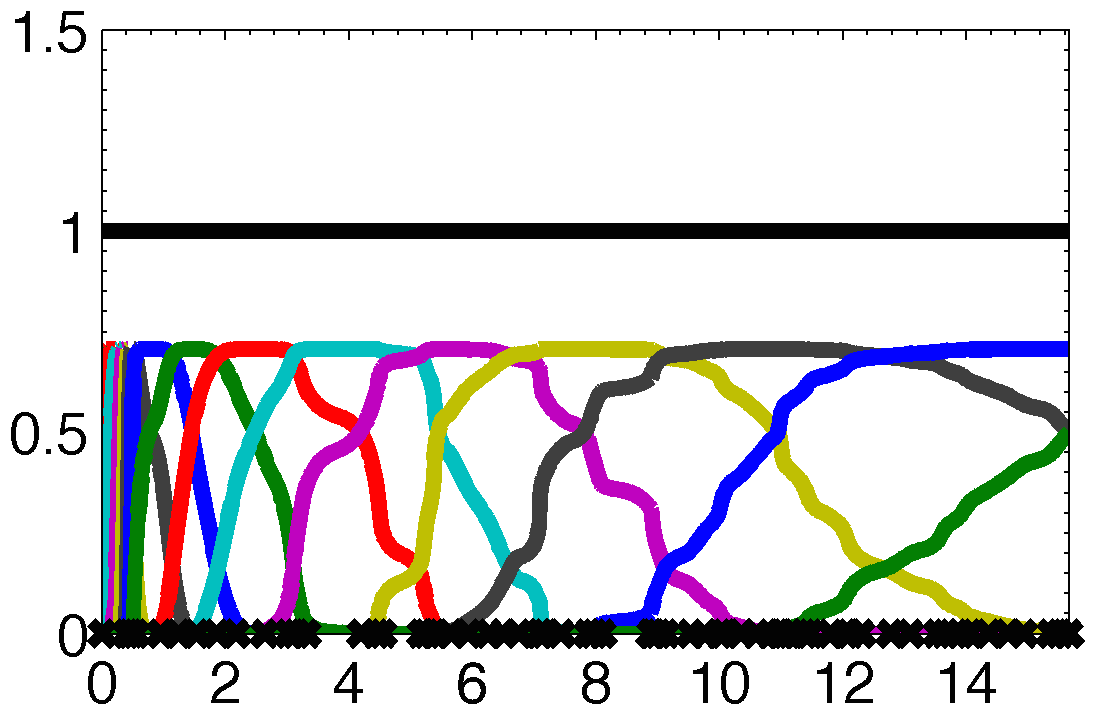}} 
\centerline{\small{~~$\lambda$}}
\centerline{\small{~~(d)}}
\end{minipage} 
%\begin{center}
%\includegraphics[width=0.23\textwidth]{figures/FB_translates_uniform.png}
%\includegraphics[width=0.23\textwidth]{figures/FB_translates_adapted.png}
%\includegraphics[width=0.23\textwidth]{figures/FB_wavelets_uniform.png}
%\includegraphics[width=0.23\textwidth]{figures/FB_wavelets_adapted.png}
%\end{center}
\caption{
Four different filter bank designs of \cite{shuman2013spectrum}, shown for a random sensor network with 64 nodes.  
%An example of the shape of four different filter-banks for a fixed graph. 
Each colored curve is a filter defined continuously on $[0,\lambda_{\rm max}]$, and each filter bank has 16 such filters. They are designed such that $G(\lambda)=1$ for all $\lambda$ (black line), and thus all four designs yield tight localized spectral graph filter frames. %$A=B=1$ for all $\lambda_k$ (black line). 
%The magnitude of the adapted filters b and d varies with the density of eigenvalues on the spectral domain in order to keep approximatively a constant energy when discretized on the spectral values.
%{Here, the black line is $G(\lambda)=1$ and 
%t
The frame bounds here are $A=B=N$. 
\label{fig: gabor filterbank}}
\end{figure}

%We observe that the smaller the $mu$ the bigger the uncertainty. Moreover, it seems that adapting the filter to the spectrum increase the uncertainty. resulting in a motivation for \cite{shuman2013spectrum}.
\begin{table}[thb] 
\begin{center}\begin{tabular}{ |c|ccccc|} 
 \hline 
  &  & Graph Gabor  & Spectrum-Adapted & Log-Warped  & Spectrum-Adapted  \\ 
Graph & $\mu_{\G}$ &(Uniform Translates) &  Graph Gabor& Wavelets & Wavelets \\ 
 \hline
 % Inverted value
%Ring                   & $  0.12 $ & $  3.05 $ &\textbf{  3.63 } & $  2.28 $ & $  2.23 $ \\ 
%Random sensor network  & $  0.90 $ & $  1.43 $ & $  1.45 $ & \textbf{1.47  }& $  1.46 $ \\ 
%Random regular         & $  0.47 $ & $  2.36 $ &\textbf{   2.42 } & $  1.81 $ & $  1.89 $ \\ 
%Erdos Renyi            & $  0.79 $ &\textbf{ 1.58 } & $  1.55 $ & $  1.58 $ & $  1.51 $ \\ 
%Comet                  & $  0.98 $ & $  1.43 $ &\textbf{1.44  }& $  1.44 $ & $  1.44 $ \\ 
%Path                   & $  0.18 $ & $  2.23 $ &\textbf{ 2.65 } & $  1.95 $ & $  1.94 $ \\ 
%Modified path $d=0.1$  & $  0.48 $ & $  1.44 $ & $  1.51 $ &\textbf{1.74  }& $  1.72 $ \\ 
%Modified path $d=0.01$ & $  0.70 $ & $  1.41 $ & $ 1.47 $ & $  1.43 $ &\textbf{ 1.55 } \\ 
Ring                    & $  0.12 $ & $  0.33 $ & $\textbf{  0.28 }$ & $  0.44 $ & $  0.45 $ \\ 
Random sensor network   & $  0.90 $ & $  0.70 $ & $  0.69 $ & $\textbf{  0.68 }$ & $  0.69 $ \\ 
Random regular          & $  0.43 $ & $  0.41 $ &$\textbf{   0.40 }$ & $  0.57 $ & $  0.53 $ \\ 
Erdos Renyi             & $  0.93 $ & $  0.68 $ & $  0.68 $ & $  0.68 $ & $\textbf{  0.67  }$\\ 
Comet                   & $  0.98 $ &$  0.70 $ & $  0.70 $ & $  0.70 $& $  0.70 $ \\ 
Path                    & $  0.18 $ & $  0.45 $ & $\textbf{  0.38 }$ & $  0.51 $ & $  0.51 $ \\ 
Modified path: $W_{12}=0.1$   & $  0.48 $ & $  0.69 $ & $  0.66 $ & $\textbf{ 0.57 }$ & $  0.58 $ \\ 
Modified path: $W_{12}=0.01$  & $  0.70 $ & $  0.71 $ & $  0.68 $ & $  0.70 $ & $\textbf{ 0.65 }$ \\ 
 \hline 
 \end{tabular}\end{center} 
 \caption{Numerical values of the uncertainty bound $\max_{i,k}\|\T_ig_k\|_2$ of Example~\ref{ex:filterbank} for various graphs of $64$ nodes. 
 %{\GBR give $1/\mu_{\G}$ instead of $\mu_{\G}$? (to compare with the other theorems)}{\nati Well in that case, it would be more meaningful to write $\frac{1}{\mu_{\G} \max_k\|g_k\|_2}$. However, in this quantity is bellow $1$ for many of those graphs. The bound is not tight enough}. {\GBRB bon alors on laisse comme ca.}.
 \label{tab:uncertainty tight} } 
\end{table} 

\end{example}

%\subsection{Spectral-vertex spread example}
The uncertainty principle in Theorem \ref{Co:Lieblocgraph} bounds the concentration of the graph Gabor transform coefficients. In the next example, we examine these coefficients for a series of signals with different vertex and spectral domain localization properties.
%\nati{
\begin{example}[Concentration of the graph Gabor coefficients for signals with varying vertex and spectral domain concentrations.] %Spreading in the vertex and spectral domain.] 
In Figure \ref{Example_ambiguity_sensor_tau_effect}, we 
analyze a series of signals on a random sensor network of 100 vertices. Each signal is created by localizing a kernel $\widehat{h_{\tau}}(\lambda) = e^{-\frac{\lambda^2}{\lambda_{\text{max}}^2} \tau^2}$ to be centered at vertex 1 (circled in black). To generate the four different signals, we vary the value of the parameter $\tau$ in the heat kernel. We plot the four localized kernels in the graph spectral and vertex domains in the first two columns, respectively. The more we ``compress'' $\hat{h}$ in the graph spectral domain (i.e. we reduce its spectral spreading by increasing $\tau$), the less concentrated the localized atom becomes in the vertex domain.
%
%plot the vertex and spectral spreading of several signals on a random sensor network of $100$ nodes. The signal is created by localizing the kernel $\hat{h}(\lambda) = e^{-\frac{\lambda^2}{\lambda_{\text{max}}^2} \tau^2}$ to the first node of the graph (with the black circle). In this example, we illustrate how the concentration evolves  when $\tau$ varies. The first and second column show the signal in the specral and in the vertex domain. 
The joint vertex-frequency representation $|\A_\gg T_1h_\tau(i,k)|$ of each signal is shown in the third column, which illustrates the trade-off between concentration in the vertex and the spectral domains. The concentration of these graph Gabor transform coefficients is the quantity bounded by the uncertainty principle presented in Theorem \ref{Co:Lieblocgraph}. 
%The  concentration of these graph Gabor transform coefficients in the third column illustrate the the trade-off between concentration in the vertex and the spectral domain. The final two columns show the concentration of the graph Gabor transform coefficients of these signals, the quantity bounded by the uncertainty principles presented in Theorem \ref{Co:Lieblocgraph}. 
% provide a limit for this concentration. 
In the last row of the Figure \ref{Example_ambiguity_sensor_tau_effect}, $\tau=\infty$ which leads to a Kronecker delta for the kernel and a constant on the vertex domain. On the contrary, when the kernel is constant, with $\tau=0$ (top row), the energy of the graph Gabor coefficients %energy spreads 
stays concentrated around one vertex but spreads along all frequencies. %This effect is due to the irregular structure of the graph. In between, we observe a trade-off between spectral and vertex spread.

\begin{figure}[ht!]
\centering
\begin{minipage}[b]{0.07\linewidth}
$\tau=0$ \\
\vspace{.8in}

$\tau=4$ \\
\vspace{.8in}

$\tau=10$ \\
\vspace{.8in}

$\tau=\infty$
\vspace{1in}

\end{minipage}
\hfill
\begin{minipage}[b]{0.9\linewidth}
\centerline{\includegraphics[width=1.1\textwidth]{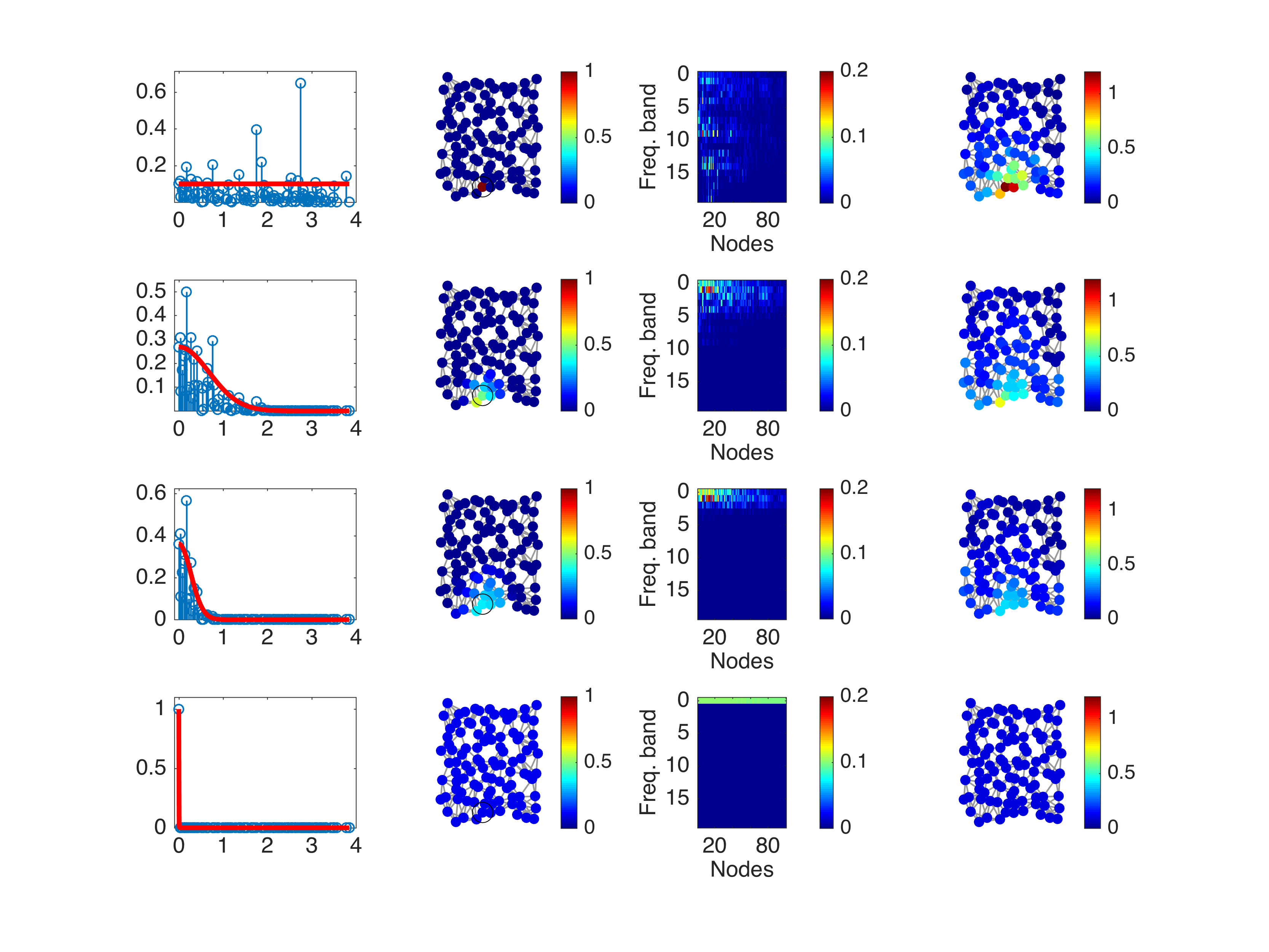}}
\end{minipage}
\caption{Graph Gabor transform of four different signals $f_{\tau}=T_1 h_{\tau}$, with each row corresponding to a signal with a different value of the parameter $\tau$. Each of the signals is a kernel localized to vertex 1, with the kernel to be localized equal to $\widehat{h_{\tau}}(\lambda)=e^{-\frac{\lambda^2}{\lambda_{\text{max}}^2} \tau^2}$. The underlying graph is a random sensor network of $100$ vertices. First column: the kernel $h_{\tau}(\lambda)$ is shown in red and the localized kernel $\widehat{f_{\tau}}$ is shown in blue, both in the graph spectral domain. Second column: the signal $f_{\tau}$ in the vertex domain (the center vertex 1 is circled). Third column: $|\A_\gg T_1h_\tau (i,k)|$, the absolute value of the Gabor transform coefficients for each vertex $i$ and each of the 20 frequency bands $k$. Fourth column: since it is hard to see where on the graph the transform coefficients are concentrated when the nodes are placed on a line in the third column, we display the value $\sum_{k=0}^{19}  |\A_\gg T_1h_\tau (i,k)|$ on each vertex $i$ in the network. This figure illustrates the tradeoff between the vertex and the frequency concentration.}
\label{Example_ambiguity_sensor_tau_effect}
\end{figure}

\end{example}

%}

%%%%%%%%%%%%%%%%%%%%%%%%%%%%
%\clearpage

\section{Local uncertainty principles for signals on graphs} \label{subsec:local}

In the previous section, we defined a global bound for the concentration of the localized spectral graph filter frame analysis coefficients. In the classical setting, such a global bound is also local in the sense that each part of the domain has the same structure, due to the regularity of the underlying domain. 
% \nati{Such a bound is sufficient in classical analysis since it characterizes the full domain.} 
However, this is not the case for the graph setting where the domain is irregular.  Example~\ref{ex:pathgraph} shows that a ``bad'' structure (a weakly connected node) in a small region of the graph reduces the uncertainty bound even if the rest of the graph is well behaved. Functions localized near the weakly connected node can be highly concentrated in both the vertex and  frequency domains, whereas functions localized away from it are barely impacted. Importantly, \emph{the worst case determines the global uncertainty bound}.  As another example, suppose one has two graphs $G_1$ and $G_2$ with two different structures, each of them having a different uncertainty bound. The uncertainty bound for the graph $\G$ that is the union of these two disconnected graphs is the minimum of the uncertainty bounds of the two disconnected graphs, which is suboptimal for one of the two graphs.
%has also an uncertainty bound, which is the smallest bound between the two graphs, which is suboptimal for one of the two graphs.

 %As a result, we ask us 
 In this section, we ask the following questions. Where does this worse case happen? Can we find a local principle that more accurately characterizes the uncertainty in other parts of the graph? %  How is the remaining part of the graph behaving? 
In order to answer this question, we investigate the concentration of the analysis coefficients of the frame atoms, which are localized signals in the vertex domain.
%a selected localized signal, an atom of the frame. 
This technique % has already been 
is used in the classical continuous case by Lieb~\cite{lieb1990integral}, who
%In its contribution, Lieb 
defines the (cross-) ambiguity function, %which is 
the STFT of a short-time Fourier atom. The result is a joint time-frequency uncertainty principle that does not depend on the localization in time or in frequency of the analyzed atom.  %{\color{red} Isn't it strange that we talk about Lieb here even though the results from the previous sections are generalizations of Lieb's theorems?}

%We conclude this section by introducing notation for the 
Thus, we start by generalizing to the graph setting the definition of ambiguity (or cross-ambiguity) functions from time-frequency analysis of one-dimensional signals. % to the graph setting.
\begin{definition}[Ambiguity function]
The ambiguity function of a localized spectral frame $\Dc=\{g_{i,k}\}=\{T_i g_k\}$ is defined as:
$$
\mathbb{A}_\gg(i_0,k_0,i,k) = \A_\gg T_{i_0} g_{k_0} (i,k) = \langle T_{i_0} g_{k_0} , T_i g_k \rangle
$$
\end{definition}
When the kernels $\{\widehat{g_k}\}_{k=0,1,\ldots,M-1}$ are appropriately warped uniform translates,
%\footnote{In fact, the kernels also need to be appropriately warped}, 
the operator $\mathbb{A}_\gg$ becomes a generalization of the short time Fourier transform. 
Additionally, the ambiguity function assesses the degree of coherence (linearly dependence) between the atoms  $T_{i_0} g_{k_0}$ and $ T_{i} g_{k}$.
%describe the vertex/frequency overlap between nodes $i$ and $i_0$, and frequency band $k$ and $k_0$.  
In the following, we use this ambiguity function to  probe \emph{locally} the structure of the graph, and derive local uncertainty principles.
%\nati{
%Inspired by Lieb, we generalize the ambiguity functions to the graph setting. We propose a theorem that given a spectral graph frame, bounds the concentration of its atom. 
%Using the ambiguity function of a graph vertex/frequency frame, we will be then able to probe \emph{locally} the structure of the graph.
%}

%Since the graph topology can be very different from one node to another, it is of major interest to know which structure or which part of a graph may weaken the uncertainty principle. The previous theorems have paved the way for such a characterization. It is now crucial to define \emph{a local uncertainty principle} and be able to measure the evolution of the uncertainty bounds from one place to another on a graph.

\subsection{Local uncertainty principle}

In order to probe the local uncertainty of a graph, we take a set of localized kernels in the graph spectral domain and center them at different local regions of the graph in the vertex domain. The atoms resulting from this construction are jointly localized in both the vertex and graph spectral domains, where "localized" means that the values of the function are zero or close to zero away from some reference point. 
%We create these functions by using kernels $\{g_k\}_k$ and localization operators $\{\T_i\}_i$. 
By ensuring that the atoms are localized or have support within a small region of the graph, we focus on the properties of the graph in that region. In order to get a local uncertainty principle, we apply the frame operator to these localized atoms, and analyze the concentration of the resulting coefficients. In doing so, we develop an uncertainty principle relating these concentrations to the local graph structure.  

%This is the main idea used here to investigate the locality, uncertainty and its relation with the graph structure.

To prepare for the theorem, we first state a lemma that gives a hint to how the scalar product of two localized functions depends on the graph structure and properties. In the following, we multiply two kernels $\hat{g}$ and $\hat{h}$ in the graph spectral domain. For notation, we represent the product of these two kernels in vertex domain as $g \cdot h$. 
%Indeed, the uncertainty principle compare the localization or concentration of functions and we need to understand how the graph structure influences it. 
%The value of the scalar product depends on the shape of the functions; it is high if their respective values coincide or are close in the different locations over the graph (for example if they are localized in the same region). 
%We denote by $(g\cdot h)$ the multiplication of two functions $g$ and $h$.
\begin{lemma}\label{lemma:Perraudin}
For two kernels $\hat{g}$, $\hat{h}$ and two nodes $i,j$, the localization operator satisfies
\begin{eqnarray}
<\T_{i}g,\T_{j}h> & = & \sqrt{N}\T_{i}(g\cdot h)(j) \label{eq:norm_product_tig_el}
\end{eqnarray}
and 
\begin{equation}
\left(\sum_{i}\left|<\T_{i}g,\T_{j}h>\right|^{p}\right)^{\frac{1}{p}}=\sqrt{N}\|\T_{j}(g\cdot h)\|_{p}.\label{eq:norm_product_tig}
\end{equation}
\end{lemma}
Equation~\eqref{eq:norm_product_tig_el} shows more clearly the conditions on the kernels and nodes under which the scalar product is small. Let us take two examples. First, suppose $\hat{g}$ and $\hat{h}$ have a compact support on the spectrum and do not overlap (kernels localized in different places), then $\hat{g}\cdot \hat{h}$ is zero everywhere on the spectrum, and therefore the scalar product on the left-hand side of \eqref{eq:norm_product_tig_el} is also equal to zero. Second, assume $i$ and $j$ are distant from each other. Then $|\T_i(g\cdot h)(j)|$ is small if $\hat{g}$ and  $\hat{h})$ are reasonably smooth. In other words, the two atoms $T_i g$ and $T_j h$  must be localized both in the same area of graph in the vertex domain and the same spectral region in order for the  scalar product to be large. This localization depends on the atoms, but also on the graph structure.
\begin{proof}[Proof of Lemma \ref{lemma:Perraudin}]
\begin{eqnarray*}
<\T_{i}g,\T_{j}h> & = & <\widehat{\T_{i}g},\widehat{\T_{j}h}> =  N\sum_{\ell=0}^{N-1}\hat{g}(\lambda_{\ell})u_{\ell}(i)\bar{\hat{h}}(\lambda_{\ell})\bar{u}_{\ell}(j)\\
 & = & N\sum_{\ell=0}^{N-1}\left(\hat{g}\cdot \hat{h}\right)(\lambda_{\ell})u_{\ell}(i)\bar{u}_{\ell}(j) = \sqrt{N}\T_{i}(g\cdot h)(j).
\end{eqnarray*}
Moreover, a direct computation shows 
\[
\left(\sum_{i}\left|<\T_{i}g,\T_{j}h>\right|^{p}\right)^{\frac{1}{p}}=\left(\sum_{i}\left|\sqrt{N}\T_{j}(g\cdot h)(i)\right|^{p}\right)^{\frac{1}{p}}=N^{\frac{1}{2}}\|\T_{j}(g\cdot h)\|_{p}.
\]
\end{proof}
The following theorem provides inequalities giving a local uncertainty principle. The local bound depends of the localization of the atom $\T_{i_0}g_{k_0}$ both in the graph and spectral domains. The center vertex $i_0$ and kernel $\hat{g}_{k_0}$ can be chosen to be any vertex and kernel; however, the locality property of the uncertainty principle appears when $\T_{i_0}g_{k_0}$ 
is concentrated around node $i_0$ in the vertex domain and %{\color{red} and/or}
 around a small portion of the spectrum in the graph spectral domain.
%has a small spread around some reference point $i_0$ in the vertex domain. Alternatively, the locality can be in the spectral domain or in both domains. 
Once again, we measure the concentration with $\l^p$-norms.
\begin{theorem}[Local uncertainty]\label{theo:local_uncertainty}
Let $\{ \T_ig \}_{\left\{i\in[1,N],k\in[0,M-1]\right\}}$ be a localized spectral graph filter frame with lower frame bound $A$ and upper frame bound $B$. For any $i_{0}\in[1,N],k_{0}\in[0,M-1]$ such that $\|\T_{i_{0}}g_{k_{0}}\|_{2}>0$, the quantity
\begin{equation} \label{Eq:local_ineq}
\|\A_{\gg}\T_{i_0}g_{k_0}\|_{p}
= \left(\sum_{k=1}^{M}\sum_{i=1}^{N}\left|<\T_{i}g_{k},\T_{i_{0}}g_{k_{0}}>\right|^{p}\right)^{\frac{1}{p}}
= \sqrt{N}\left(\sum_{k=1}^{M}\|\T_{i_0}(g_{k_0}\cdot g_{k})\|_{p}^{p}\right)^{\frac{1}{p}}
\end{equation}
satisfies for $p \in [1,\infty]$
%$1 \leq p \leq \infty$ 
\begin{equation}\label{eq:localuncertainty}
s_p\left(\A_{\gg}\T_{i_0}g_{k_0}\right) %=
%\frac{\|\A_{\gg}\T_{i_0}g_{k_0}\|_{p}}{\|\A_{\gg}\T_{i_0}g_{k_0}\|_{2}}
\leq\frac{B^{\min\{\frac{1}{p},1-\frac{1}{p}\}}\|\T_{\tilde{i}_{i_0,k_0}}g_{\tilde{k}_{i_0,k_0}}\|_{2}^{\left|1-\frac{2}{p}\right|}}{A^{\frac{1}{2}}}
\leq \frac{ B^{\min\{\frac{1}{p},1-\frac{1}{p}\}} \left( \sqrt{N}\nu_{\tilde{i}_{i_0,k_0}} \| g_{\tilde{k}_{i_0,k_0}} \|_2 \right)^{\left|1-\frac{2}{p}\right|}}{A^{\frac{1}{2}}},
\end{equation}
where $\nu_i$ is defined in Lemma \ref{Le:Tisom}, $\tilde{k}_{i_0,k_0}=\argmax_{k}\|\T_{i_0}(g_{k_0}\cdot g_{k})\|_{\infty}$, and $\tilde{i}_{i_0,k_0}=\argmax_{i}\left|\T_{i_0}(g_{k_0}\cdot g_{\tilde{k}_{i_0,k_0}})(i)\right|$. 
%Additionally thanks to Hölder inequality we have for $q\leq2$ and $\frac{1}{p}+\frac{1}{q}=1$
%\begin{equation}
%s_q\left(\A_{\gg}\T_{i_0}g_{k_0}\right) =\frac{\|\A_{\gg}\T_{i_0}g_{k_0}\|_{2}}{\|\A_{\gg}\T_{i_0}g_{k_0}\|_{q}}
%\leq\frac{\|\A_{\gg}\T_{i_0}g_{k_0}\|_{p}}{\|\A_{\gg}\T_{i_0}g_{k_0}\|_{2}}
%\leq\frac{B^{\frac{1}{p}}\|\T_{\tilde{i}}g_{\tilde{k}}\|_{2}^{1-\frac{2}{p}}}{A^{\frac{1}{2}}}
%\leq \frac{ B^{\frac{1}{p}} \left( \sqrt{N}\tilde{\mu}_{\tilde{i}} \| g_{\tilde{k}} \|_2 \right)^{1-\frac{2}{p}}}{A^{\frac{1}{2}}}.
%\end{equation}
\end{theorem}

The bound in \eqref{eq:localuncertainty} is local, because we get a different bound for each $i_0$, $k_0$ pair. For each such pair, the bound depends on the quantities $\tilde{i}_{i_0,k_0},\tilde{k}_{i_0,k_0}$, which are maximizers over a set of all vertices and kernels, respectively; however, as we discuss in Example \ref{Ex:local} below, $\tilde{i}_{i_0,k_0}$ is typically close to $i_0$, and $\tilde{k}_{i_0,k_0}$ is typically close to $k_0$. For this reason, this bound typically depends only on local quantities.

\begin{proof}[Proof of Theorem \ref{theo:local_uncertainty}]
For notational brevity in this proof, we omit the indices $i_0,k_0$ for the quantities $\tilde{i}$ and $\tilde{k}$. First, note that
\[
\|\A_{\gg}\T_{i_0}g_{k_0}\|_{\infty}
= \max_{k}\sqrt{N}\|\T_{i_0}(g_{k_0}\cdot g_{k})\|_{\infty}\leq\|\T_{\tilde{i}}g_{\tilde{k}}\|_{2}\|\T_{i_0}g_{k_0}\|_{2},
\]
where $\tilde{k}_{i_0,k_0}=\argmax_{k}\|\T_{i_0}(g_{k_0}\cdot g_{k})\|_{\infty}$
and $\tilde{i}_{i_0,k_0}=\argmin_{i}\left|\T_{i_0}(g_{k_0}\cdot g_{\tilde{k}})(i)\right|$. 
Let us then interpolate the two following expressions:
\begin{align}
\|\A_{\gg}\T_{i_0}g_{k_0}\|_{2}&\leq B^{\frac{1}{2}}\|\T_{i_0}g_{k_0}\|_{2} \\
\hbox{ and~~~} \|\A_{\gg}\T_{i_0}g_{k_0}\|_{\infty}&\le\|\T_{\tilde{i}}g_{\tilde{k}}\|_{2}\|\T_{i_0}g_{k_0}\|_{2}.
\end{align}
We use the Riesz-Thorin Theorem (Theorem \ref{theo:Riesz-Thorin}) with $p_{1}=q_{1}=p_{2}=2$,
$q_{2}=\infty$, $M_{p}=B^{\frac{1}{2}}$ and $M_{q}=\|\T_{\tilde{i}}g_{\tilde{k}}\|_{2}$. Note that $\A_{\gg}$ is a bounded operator from the Hilbert space spanned by $\T_{i_0}g_{k_0}$ (isomorphic to a one-dimensional Hilbert space) to the one spanned by $\{T_{i_0}g_{k_0}\}_{i,k}$. 
We take $t=\frac{2}{r_{2}}$ and find $r_{1}=2$, leading to
\[
\|\A_{\gg}\T_{i_0}g_{k_0}\|_{r_{2}}\le B^{\frac{1}{r_{2}}}\|\T_{\tilde{i}}g_{\tilde{k}}\|_{2}^{1-\frac{2}{r_{2}}}\|\T_{i_0}g_{k_0}\|_{2}.
\]
Since $\A_{\gg}$ is a frame, we also have $\|\A_{\gg}\T_{i_0}g_{k_0}\|_{2}\ge A^{\frac{1}{2}}\|\T_{i_0}g_{k_0}\|_{2}$, which yields:
\[
\frac{\|\A_{\gg}\T_{i_0}g_{k_0}\|_{2}}{||\A_{\gg}\T_{i_0}g_{k_0}\|_{p}}\ge\frac{A^{\frac{1}{2}}}{B^{\frac{1}{p}}\|\T_{\tilde{i}}g_{\tilde{k}}\|_{2}^{1-\frac{2}{p}}}.
\]
Finally, thanks to Hölder's inequality, we have for $p\leq2$ and $\frac{1}{p}+\frac{1}{q}=1$
\begin{eqnarray*}
\frac{\|\A_{\gg}\T_{i_0}g_{k_0}\|_{2}}{\|\A_{\gg}\T_{i_0}g_{k_0}\|_{p}} &  \leq & 
\frac{\|\A_{\gg}\T_{i_0}g_{k_0}\|_{q}}{\|\A_{\gg}\T_{i_0}g_{k_0}\|_{2}} \\
& \leq & \frac{B^{\frac{1}{q}}\|\T_{\tilde{i}}g_{\tilde{k}}\|_{2}^{1-\frac{2}{p}}}{A^{\frac{1}{2}}} \\
&\leq&\frac{B^{1-\frac{1}{p}}\|\T_{\tilde{i}}g_{\tilde{k}}\|_{2}^{\frac{2}{p}-1}}{A^{\frac{1}{2}}} \\
& \leq & \frac{ B^{1-\frac{1}{p}} \left( \sqrt{N}\nu_{\tilde{i}} \| g_{\tilde{k}} \|_2 \right)^{\frac{2}{p}-1}}{A^{\frac{1}{2}}}.
\end{eqnarray*}
\end{proof}

The next corollary shows that in many cases, the local uncertainty inequality~\eqref{eq:localuncertainty} is  sharp (becomes an equality). To obtain this,
we require that the frame 
 %a first requirement is that the frame 
 $\A_{\gg}$ is tight and 
% . A second requirement is that the graph and the kernel $\hat{g}$ are such that $k_0 = \argmax_{k} \| \T_{i_0}(g_k\cdot g_{k_0})\|_\infty$ and $i_0 = \argmax_j |\T_{i_0}g_{k_0}^2(j) |$; i.e., 
 %the maximum of 
 $|\langle\T_ig_k,T_{i_0}g_{k_0}\rangle|$ is maximized when $k=k_0$ and $i=i_0$. 
%This last requirement is more difficult to check, but it is quite intuitive when considering graphs : it holds for example when the frame is constructed on a ring graph.
\begin{corollary}\label{corol:localuncertainty}
Under the assumptions of Theorem \ref{theo:local_uncertainty} and, assuming additionally
\begin{enumerate}
\item $\A_{\gg}$ is a tight frame with frame-bound $A$,
\item $k_0 = \argmax_{k} \| \T_{i_0}(g_k\cdot g_{k_0})\|_\infty$, and 
\item $i_0 = \argmax_j |\T_{i_0}g_{k_0}^2(j) |$,
\end{enumerate} 
%and $g_{k_0}$ and the graph being such that  
% and
we have 
\begin{equation} \label{eq:local_uncertainty_with_hypothesis}
s_\infty\left(\A_{\gg}\T_{i_0}g_{k_0}\right) 
= \frac{\|\A_{\gg}\T_{i_0}g_{k_0}\|_{\infty}}{\|\A_{\gg}\T_{i_0}g_{k_0} \|_{2}}
= \frac{\|\T_{i_0} g_{k_0} \|_{2} }{A^{\frac{1}{2}}}.
\end{equation}
\end{corollary}
\begin{proof}
The proof follows directly from the two following equalities. For the denominators, since the frame is tight, we have:
\[
\|\A_{\gg}\T_{i_0}g_{k_0}\|_{2} = A^{\frac{1}{2}}\|\T_{i_0}g_{k_0}\|_{2}.
\]
For the numerators, we have
%in the case where $\max_{k} \| \T_{i_0}(g_k\cdot g_{k_0})\|_\infty = \| \T_{i_0}(g_{k_0}^2)\|_\infty$, and $\max_j |\T_{i_0}g_{k_0}^2(j) | = |\T_{i_0}g_{k_0}^2(i_0) | $, we find using \eqref{eq:norm_product_tig_el}:

\begin{eqnarray}
\|\A_{\gg}\T_{i_0}g_{k_0}\|_{\infty} 
&=& \max_{i,k} |\scp{\T_{i}g_{k}}{\T_{i_0}g_{k_0}} | \nonumber \\
&=& \sqrt{N}\max_{i,k} |\T_{i_0}(g_{k}\cdot g_{k_0})(i) | \label{Eq:loccor1}\\
&=& \sqrt{N}\max_{k} \|\T_{i_0}(g_{k}\cdot g_{k_0})\|_\infty \nonumber \\
&=& \sqrt{N}\|\T_{i_0} g_{k_0}^2\|_\infty \label{Eq:loccor2} \\
&=& \sqrt{N}|\T_{i_0}g_{k_0}^2(i_0) | \label{Eq:loccor3} \\
&=& \scp{\T_{i_0}g_{k_0} }{\T_{i_0} g_{k_0} } \label{Eq:loccor4} \\
&=& \|\T_{i_0} g_{k_0} \|_2^2, \nonumber
\end{eqnarray}
where \eqref{Eq:loccor1} and \eqref{Eq:loccor4} follow from \eqref{eq:norm_product_tig_el}, \eqref{Eq:loccor2} follows from the second hypothesis, and \eqref{Eq:loccor3} follows from the third hypothesis.
\end{proof}
\begin{corollary}\label{corol:lower_bound_on_concentration}
Under the assumptions of Theorem \ref{theo:local_uncertainty}, we have
\begin{align}\label{Eq:local_lower_b}
s_\infty(\A_{\gg}\T_{i_{0}}g_{k_{0}}) = \frac{\|\A_{\gg}\T_{i_{0}}g_{k_{0}}\|_{\infty}}{\|\A_{\gg}\T_{i_{0}}g_{k_{0}}\|_{2}}\geq \frac{\|\T_{i_{0}}g_{k_{0}}\|_{2}}{B^{\frac{1}{2}}}
,
\end{align}
which is a lower bound on the concentration measure.
\begin{proof}
We have
\begin{eqnarray} \label{Eq:loccor5}
\|\A_{\gg}\T_{i_{0}}g_{k_{0}}\|_{\infty}= \max_{i,k} |\scp{\T_{i}g_{k}}{\T_{i_0}g_{k_0}} | \geq  \left|\scp{\T_{i_{0}}g_{k_{0}}}{\T_{i_{0}}g_{k_{0}}}\right| =  \|\T_{i_{0}}g_{k_{0}}\|_{2}^{2}. 
%& = & \sqrt{N}\max_{k}\|\T_{i_{0}}(g_{k}\cdot g_{k_{0}})\|_{\infty}\label{Eq:loccor5} \\
% & \geq & \sqrt{N}\|\T_{i_{0}}g_{k_{0}}^{2}\|_{\infty} \nonumber\\
% & \geq & \sqrt{N}\left|\T_{i_{0}}g_{k_{0}}^{2}(i_{0})\right| \nonumber\\
% & = & \scp{\T_{i_{0}}g_{k_{0}}}{\T_{i_{0}}g_{k_{0}}} \\
% & = & \|\T_{i_{0}}g_{k_{0}}\|_{2}^{2} \label{Eq:loccor6}\nonumber
\end{eqnarray}
%where \eqref{Eq:loccor5} and  \eqref{Eq:loccor5}
%follow from \eqref{eq:norm_product_tig_el}.
Additionally, because %$\gg$
$\{T_i g_k\}_{i=1,2,\ldots,N;k=0,1,\ldots,M-1}$ 
is a frame, we have
\begin{align} \label{Eq:loccor6}
\|\A_{\gg}\T_{i_{0}}g_{k_{0}}\|_{2}\leq B^{\frac{1}{2}}\|\T_{i_{0}}g_{k_{0}}\|_{2}.
\end{align}
Combining \eqref{Eq:loccor5} and \eqref{Eq:loccor6} yields the desired inequality in \eqref{Eq:local_lower_b}.
%Putting the two inequalities together, we find the result
%\[
%\frac{\|\A_{\gg}\T_{i_{0}}g_{k_{0}}\|_{\infty}}{\|\A_{\gg}\T_{i_{0}}g_{k_{0}}\|_{2}}\geq \frac{\|\T_{i_{0}}g_{k_{0}}\|_{2}}{B^{\frac{1}{2}}}
%\]
\end{proof}
\end{corollary}
%Using 
Together, Theorem \ref{theo:local_uncertainty} and Corollary \ref{corol:lower_bound_on_concentration} yield lower and upper bounds on the local sparsity levels $s_\infty(\A_{\gg}\T_{i_{0}}g_{k_{0}})$: 
\[
\frac{\|\T_{\tilde{i}}g_{\tilde{k}}\|_{2}}{A^{\frac{1}{2}}} \geq s_\infty(\A_{\gg}\T_{i_{0}}g_{k_{0}}) = \frac{\|\A_{\gg}\T_{i_{0}}g_{k_{0}}\|_{\infty}}{\|\A_{\gg}\T_{i_{0}}g_{k_{0}}\|_{2}}\geq \frac{\|\T_{i_{0}}g_{k_{0}}\|_{2}}{B^{\frac{1}{2}}}.
\]

\subsection{Illustrative examples}

In order to better understand this local uncertainty principle, we illustrate it with some examples.
\begin{example}[Local uncertainty on a sensor network] \label{Ex:local}
Let us concentrate on the case where $p=\infty$. Theorem \ref{theo:local_uncertainty} tells us that
\begin{equation}\label{eq:localuncertaintyexample}
\frac{\|\A_{\gg}\T_{i_0}g_{k_0}\|_{\infty}}{\|\A_{\gg}\T_{i_0}g_{k_0}\|_{2}}
\leq \frac{\|\T_{\tilde{i}_{i_0,k_0}}g_{\tilde{k}_{i_0,k_0}}\|_{2}}{A^{\frac{1}{2}}}
\leq \frac{  \left( \sqrt{N}\nu_{\tilde{i}_{i_0,k_0}} \| g_{\tilde{k}_{i_0,k_0}} \|_2 \right)}{A^{\frac{1}{2}}},
\end{equation}
meaning that the concentration of $\A_{\gg}\T_{i_0}g_{k_0}$ is limited by $\frac{1}{\|\T_{\tilde{i}}g_{\tilde{k}_{i_0,k_0}}\|_{2}}$. One question is to what extent this quantity is local or reflects the local behavior of the graph.
%The main point of the discussion is to clarify with this uncertainty principle is local in the spectral domain, i.e: $g_{\tilde{k}}$ overlaps with $g_{k_0}$ and in the vertex domain, i.e: the node $\tilde{i}$ is close from the node $i_0$.
As a general illustration for this discussion, we present in Fig.~\ref{fig: local uncertainty filter} quantities related to the local uncertainty of a random sensor network of $100$ nodes evaluated for two different values of $k$ (one in each column) and all nodes $i$. 

\begin{figure}[htb]
\caption{Illustration of Theorem \ref{theo:local_uncertainty} and related variables $\tilde{i}$ and $\tilde{k}$ for a random sensor graph of $100$ nodes. Top figure: the $8$ uniformly translated kernels $\{\widehat{g_k}\}_k$ (in 8 different colors) defined on the spectrum and giving a tight frame. Each row corresponds to quantities related to the local uncertainty principle. The first column concerns the kernel (filter) in blue on the top figure, the second is associated with the orange one. On a sensor graph, the local uncertainty level (inversely proportional to the local sparsity level plotted here) is far from constant from one node to another or from one frequency band to another. 
}
\centering
 \vspace{-2cm}
\begin{tabular}{m{0.3\textwidth} m{0.3\textwidth} m{0.3\textwidth}}
 Graph Gabor filter bank 
 & \includegraphics[width=0.2\textwidth]{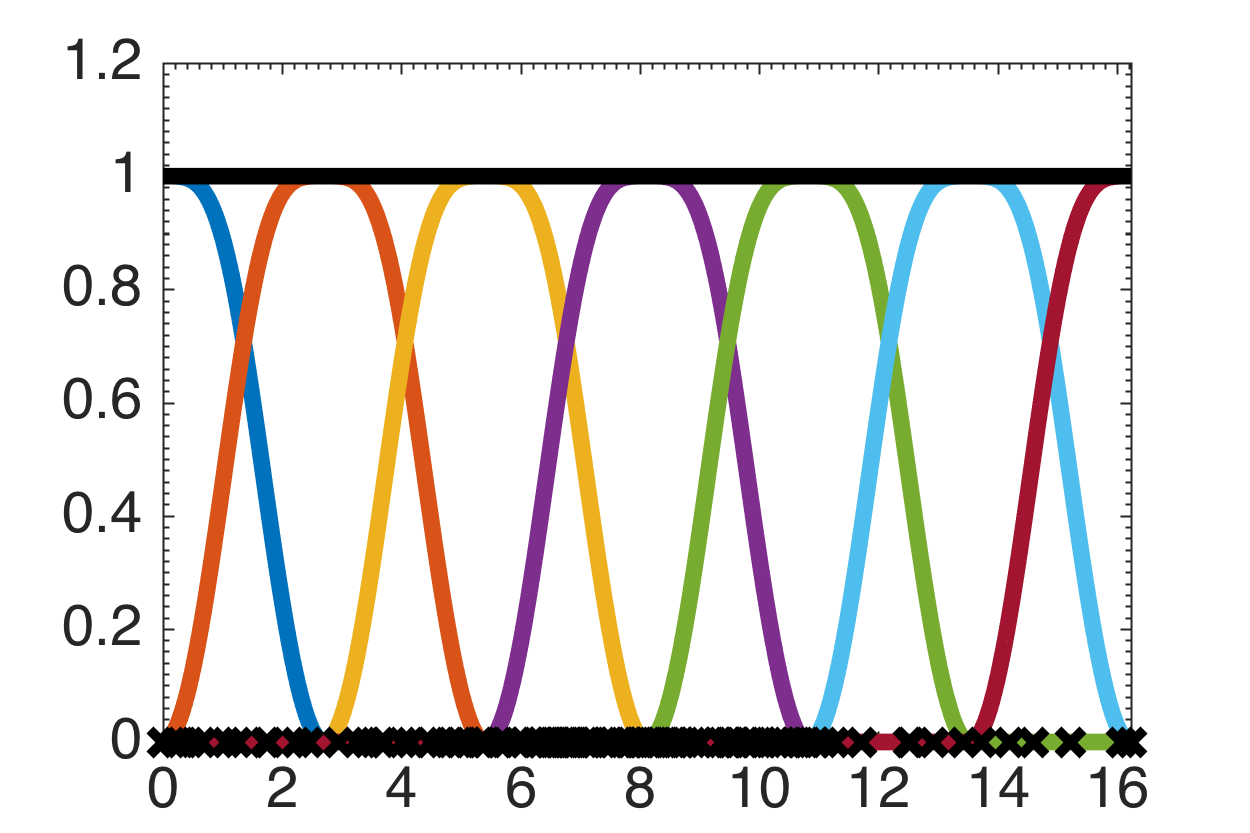}   &\\
& First filter $k_0=0$ (blue) & Second Filter $k_0=1$ (orange) \\
Local sparsity level: 
$$s_\infty(\A_{\gg}\T_{i_0}g_{k_0}) = \frac{\|\A_{\gg}\T_{i_0}g_{k_0}\|_{\infty}}{\|\A_{\gg}\T_{i_0}g_{k_0}\|_{2}}$$
&  \includegraphics[width=0.25\textwidth]{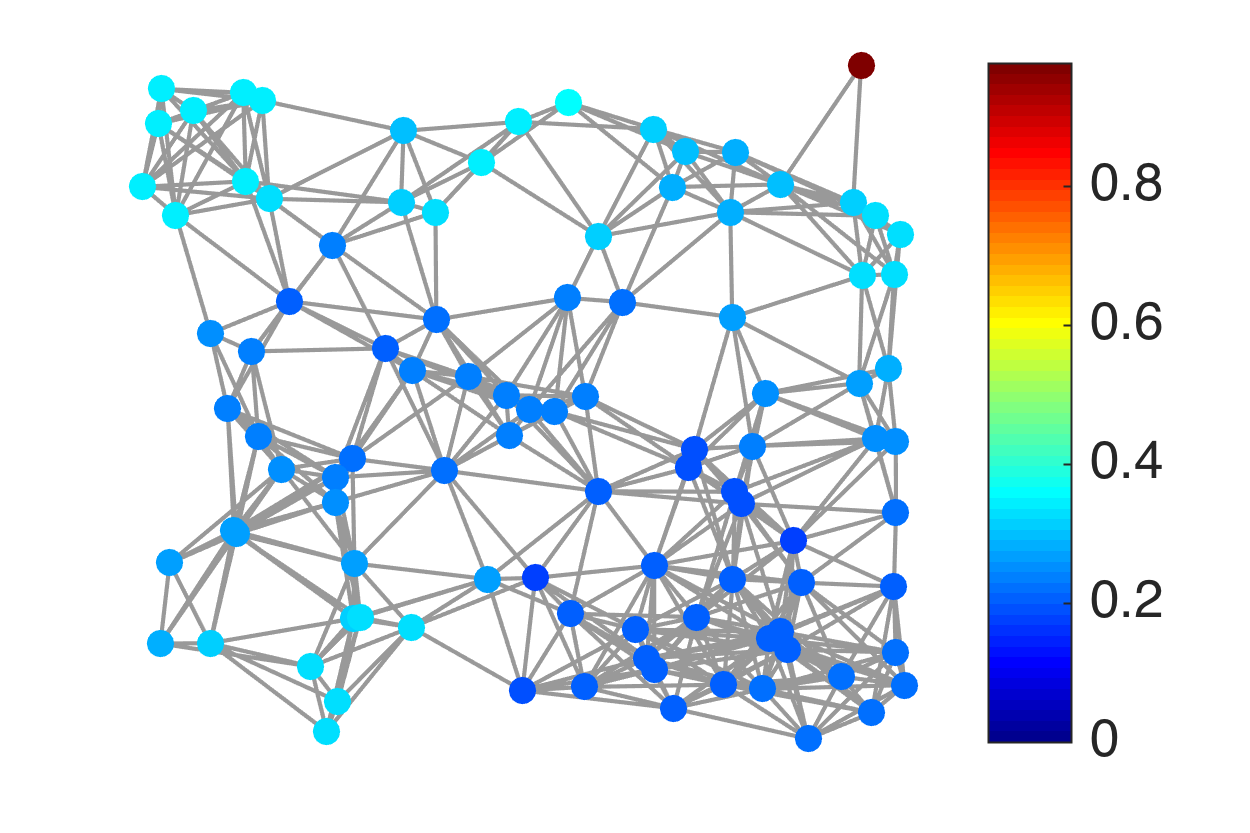}  
&  \includegraphics[width=0.25\textwidth]{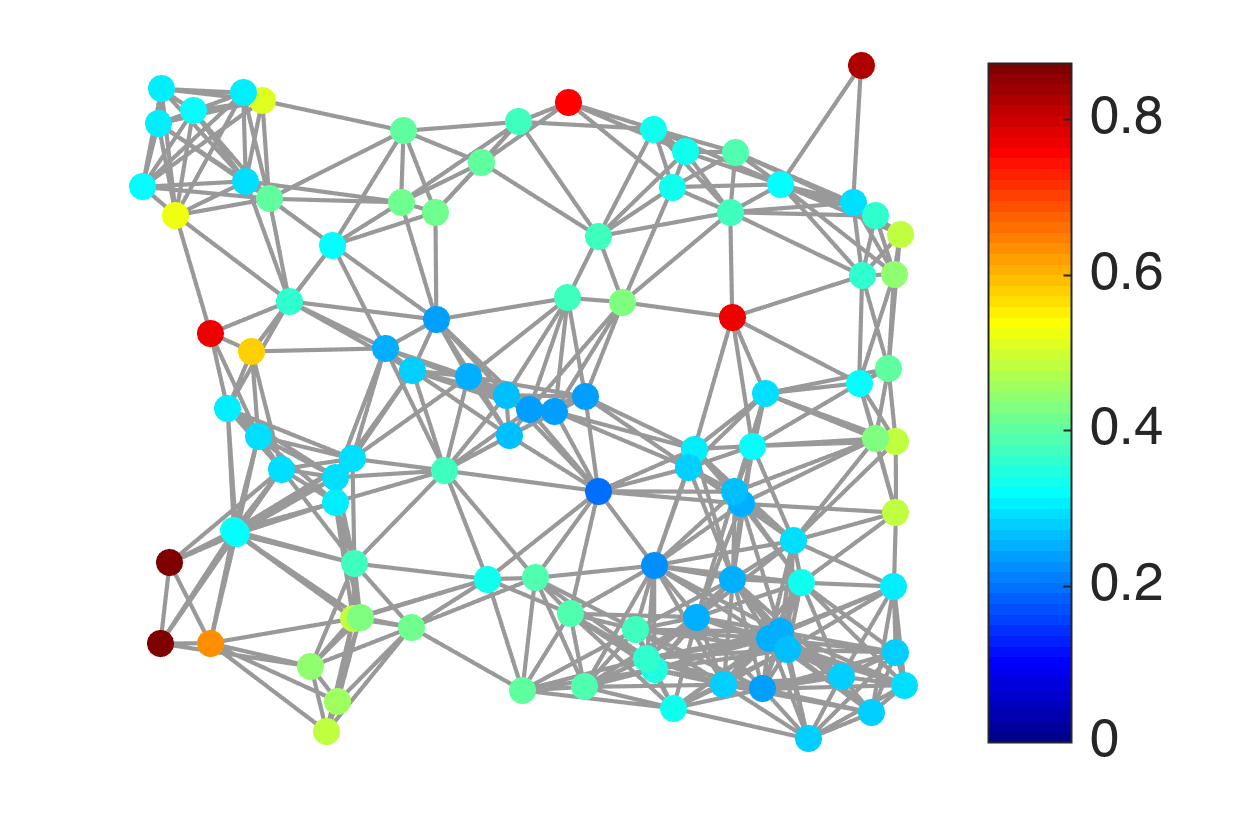}  \\ 
Upper bound on local sparsity from Theorem \ref{theo:local_uncertainty}: 
$$\frac{\|\T_{\tilde{i}_{i_0,k_0}}g_{\tilde{k}_{i_0,k_0}}\|_{2}}{A^{\frac{1}{2}}}$$
&  \includegraphics[width=0.25\textwidth]{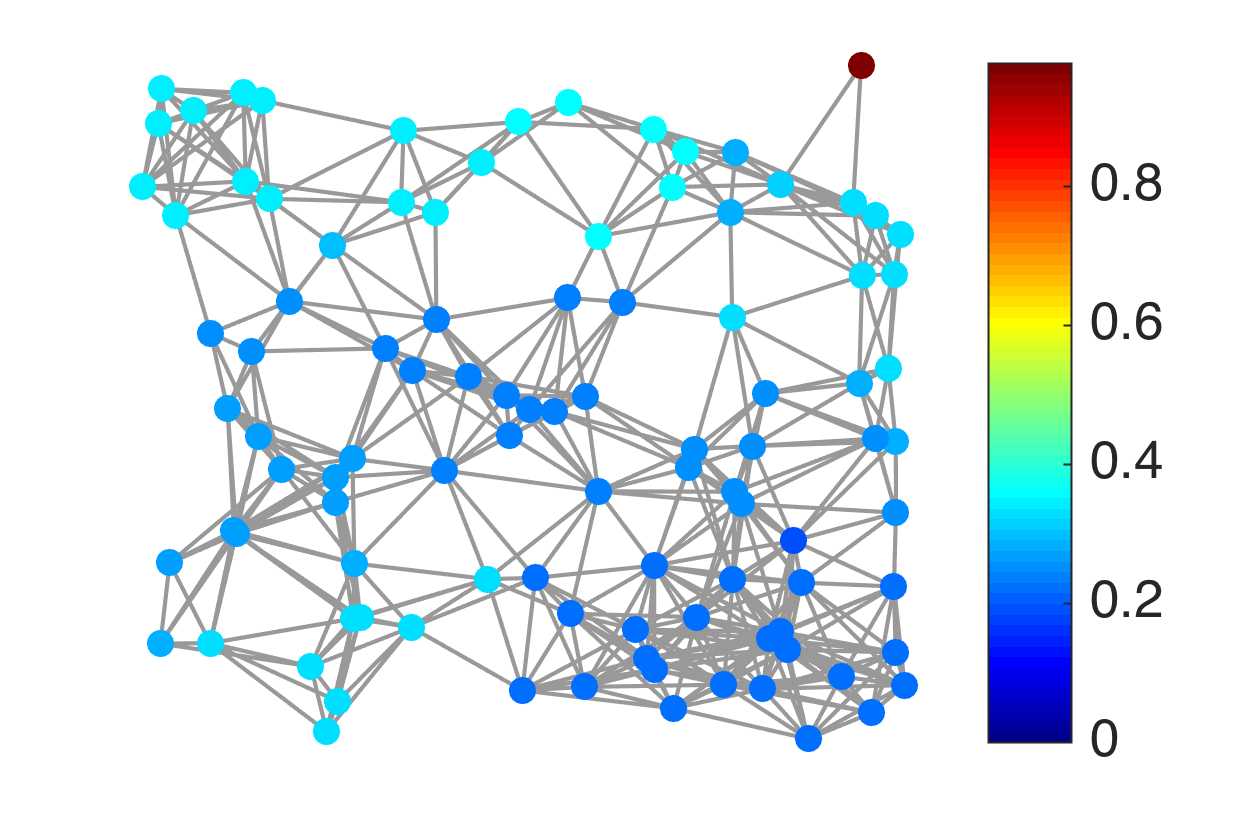} \vspace{-0.25cm}
& \includegraphics[width=0.25\textwidth]{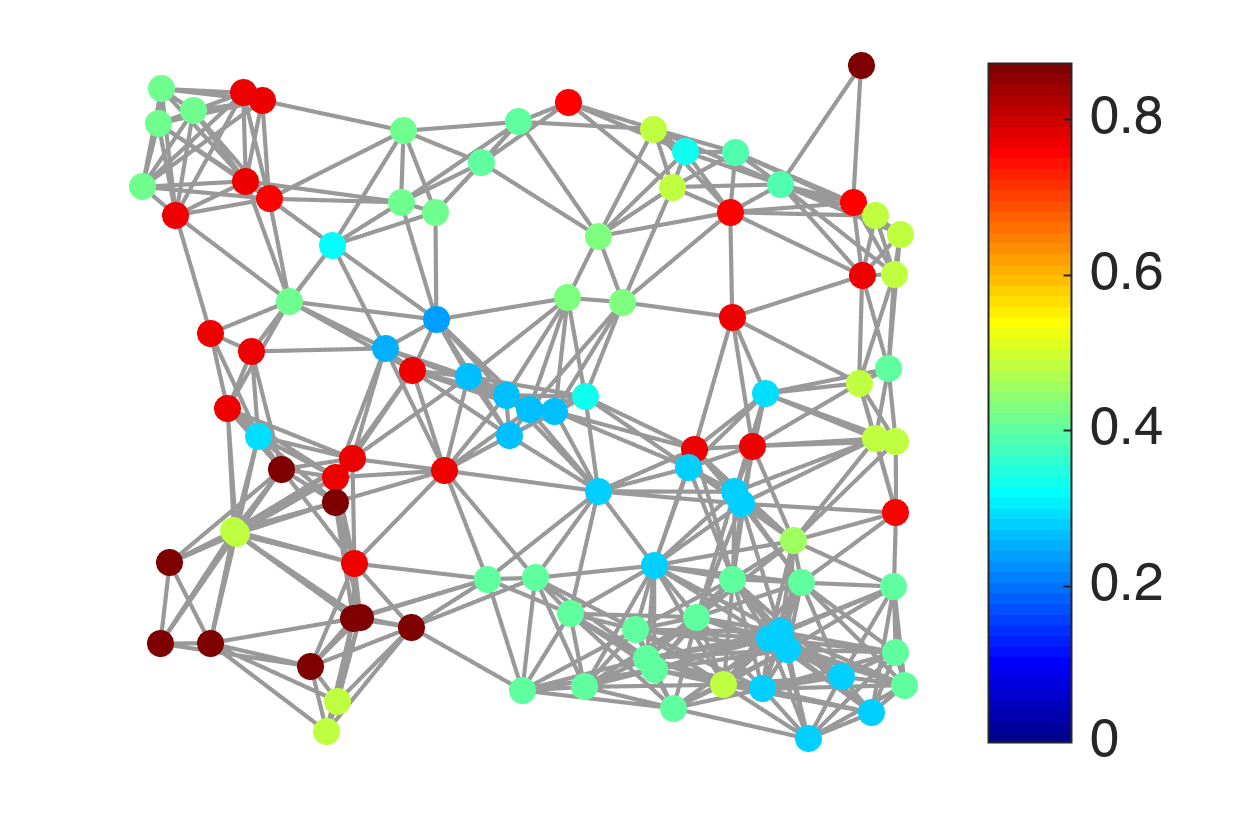} \vspace{-0.25cm} \\
Maximimizing filter index:
$$\tilde{k}_{k_0,i_0}$$
&  \includegraphics[width=0.25\textwidth]{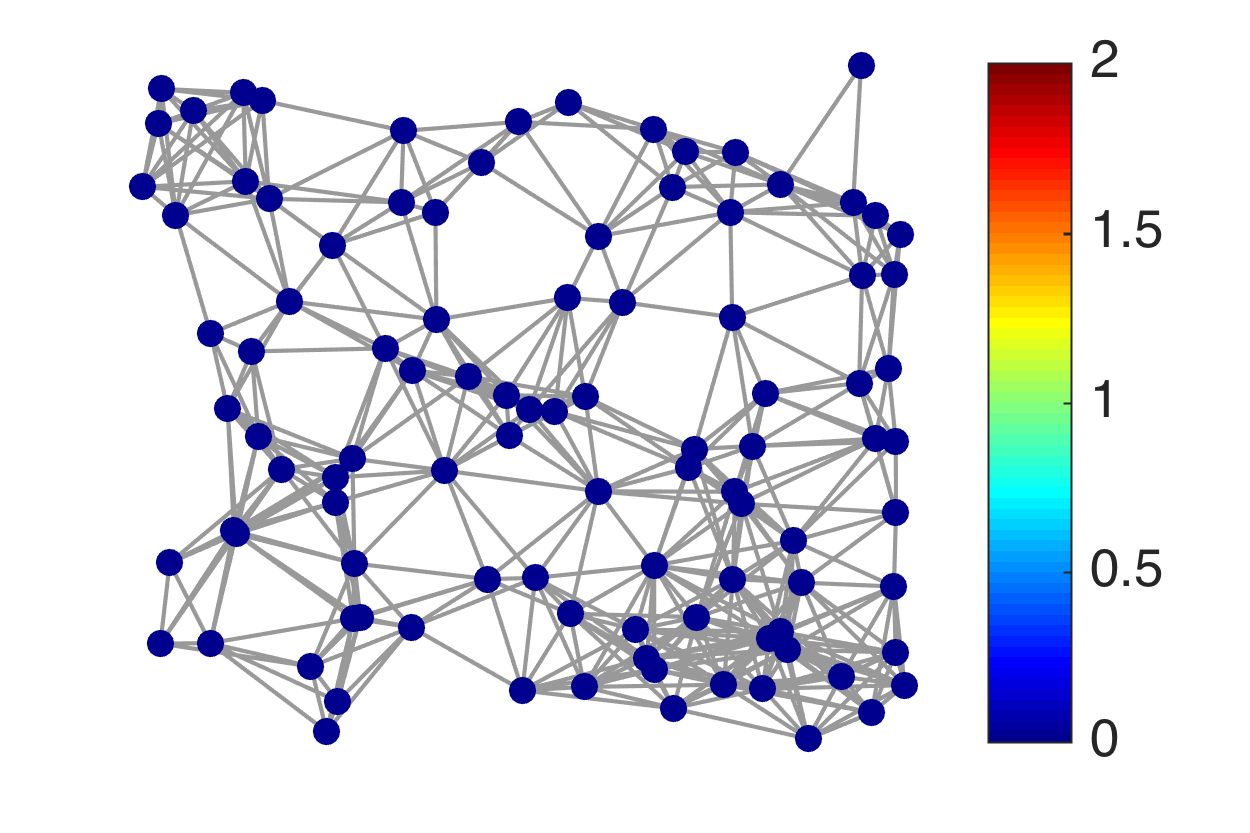} \vspace{-0.25cm} 
&  \includegraphics[width=0.25\textwidth]{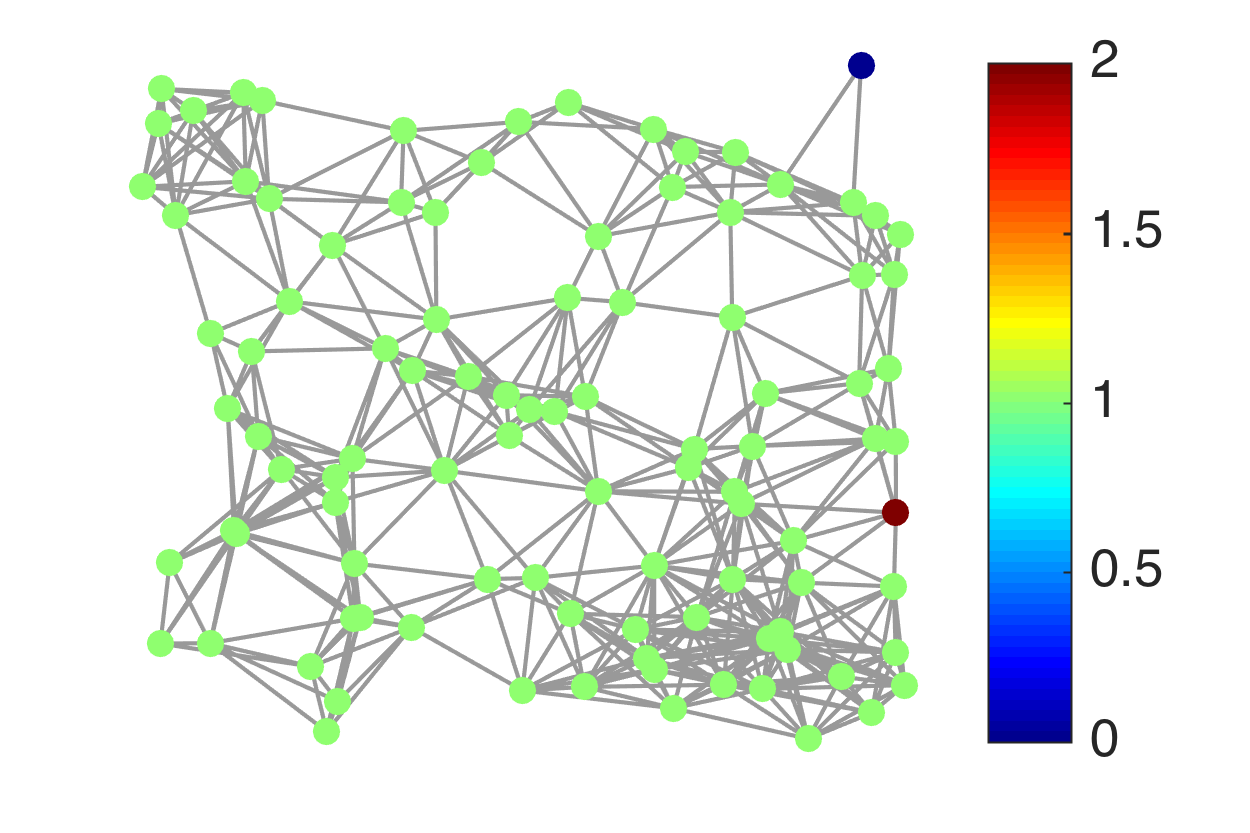}  \vspace{-0.25cm}\\ 
Hop distance between $\tilde{i}_{k_0,i_0}$ and $i_0$ 
&  \includegraphics[width=0.25\textwidth]{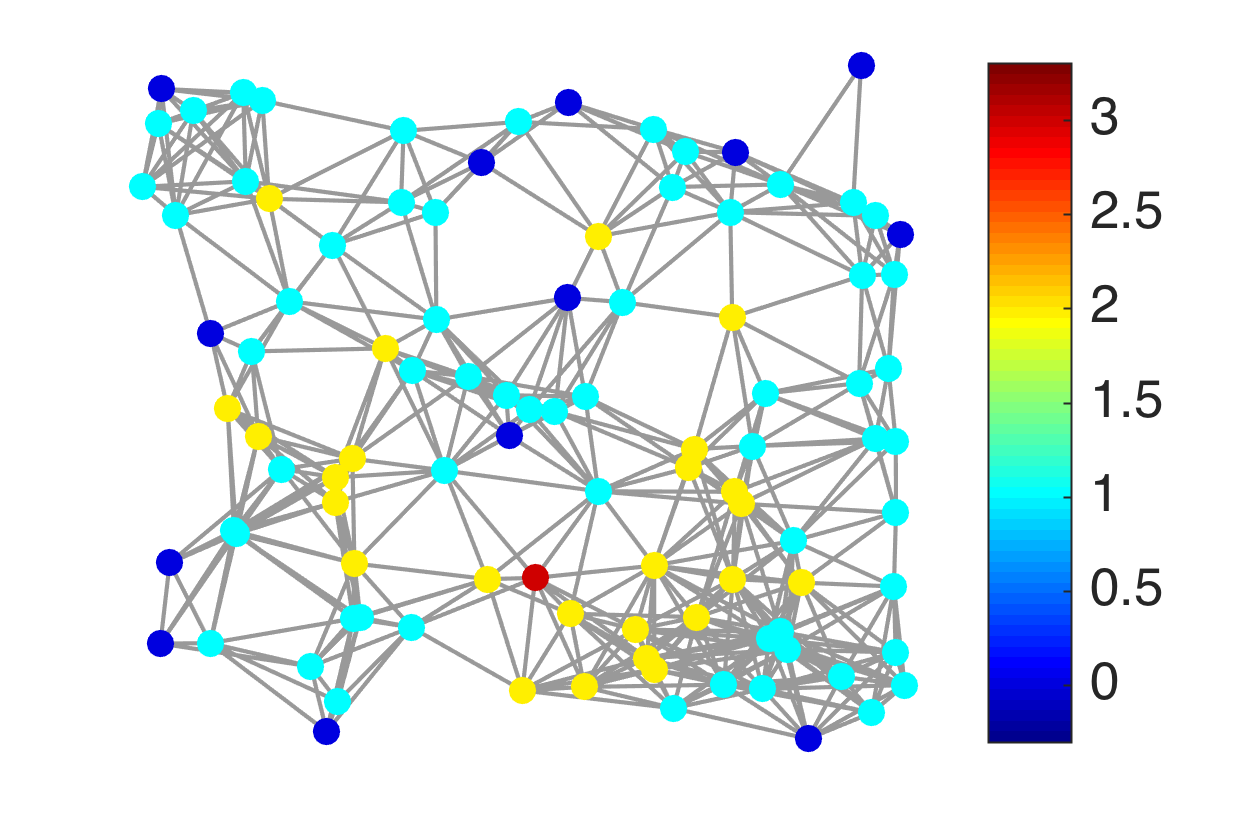} \vspace{-0.25cm}
&  \includegraphics[width=0.25\textwidth]{figures/LU_hope_dist_11.png}   \vspace{-0.25cm}\\
Lower bound on the local sparsity from Corollary \ref{corol:lower_bound_on_concentration}:
$$
\frac{\|\T_{i_{0}}g_{k_{0}}\|_{2}}{B^{\frac{1}{2}}}
$$
&  \includegraphics[width=0.25\textwidth]{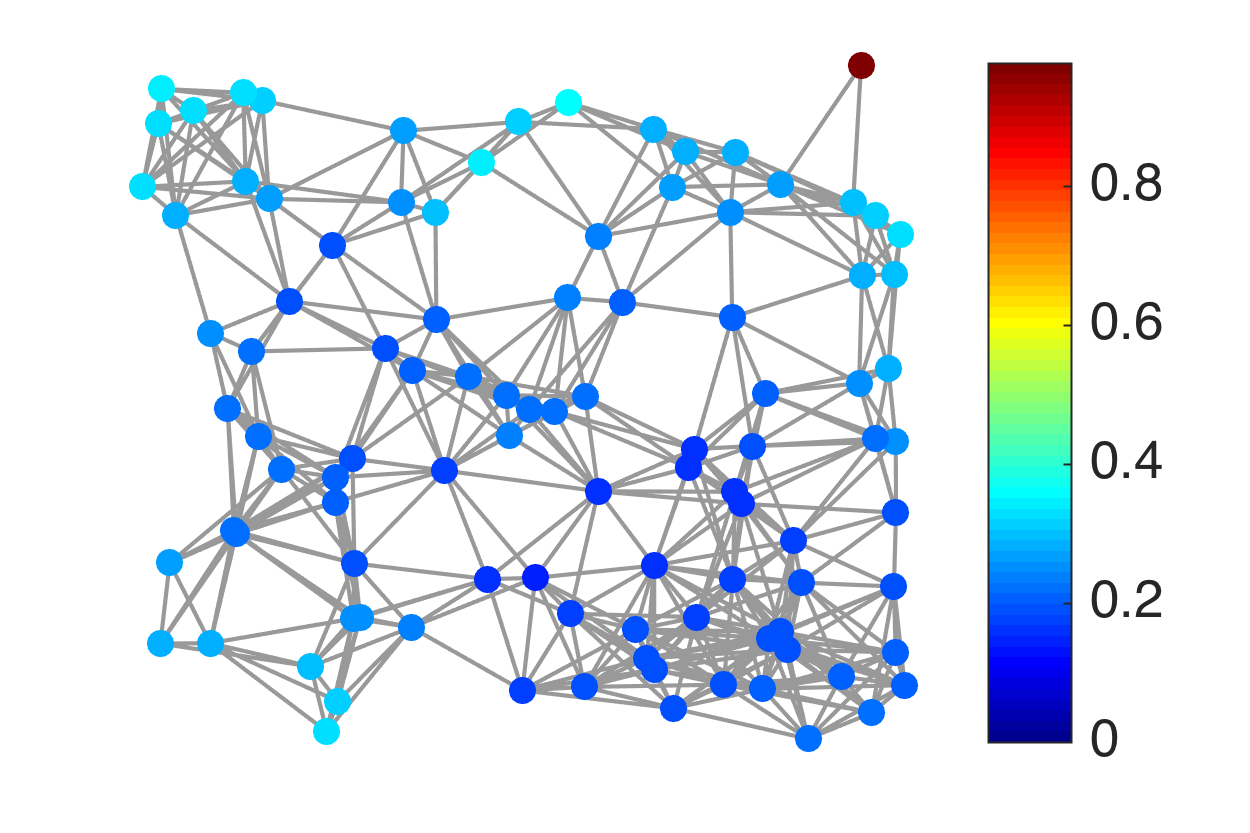} \vspace{-0.25cm} 
&  \includegraphics[width=0.25\textwidth]{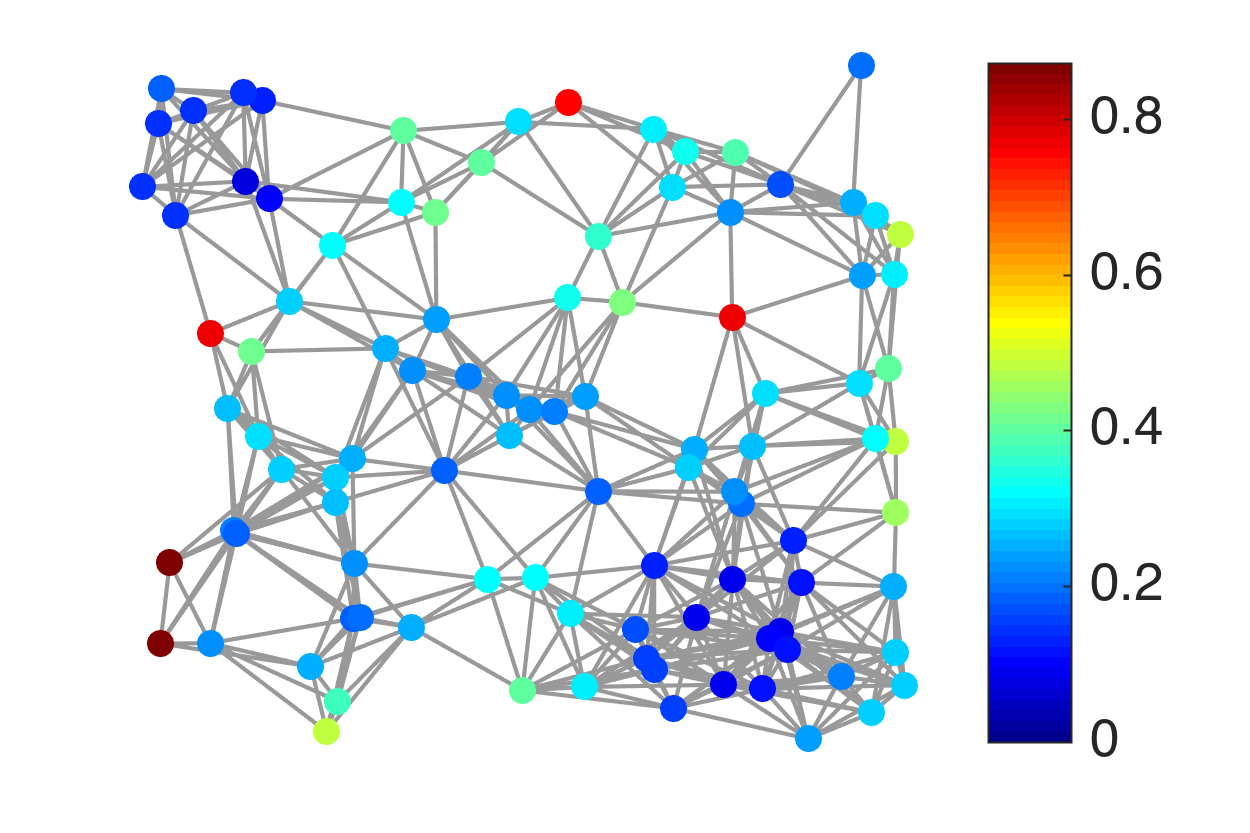}  \vspace{-0.25cm}\\
Relative error between $s_\infty(\A_{\gg}\T_{i_0}g_{k_0})$ and $B^{-\frac{1}{2}}\|T_{i_0}g_{k_0}\|_2$:
$$\frac{s_\infty(\A_{\gg}\T_{i_0}g_{k_0})-B^{-\frac{1}{2}}\|T_{i_0}g_{k_0}\|_2}{s_\infty(\A_{\gg}\T_{i_0}g_{k_0})}$$
&  \includegraphics[width=0.25\textwidth]{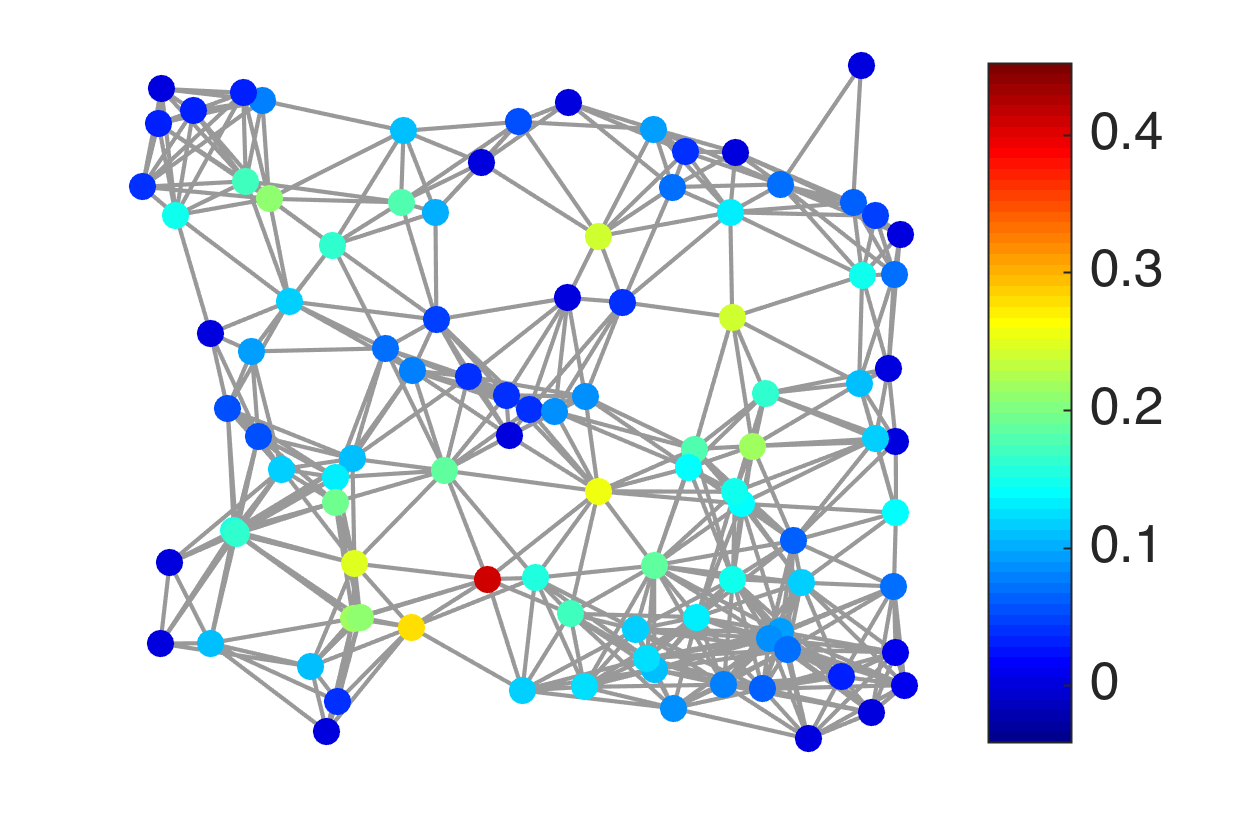}  \vspace{-0.25cm}
& \includegraphics[width=0.25\textwidth]{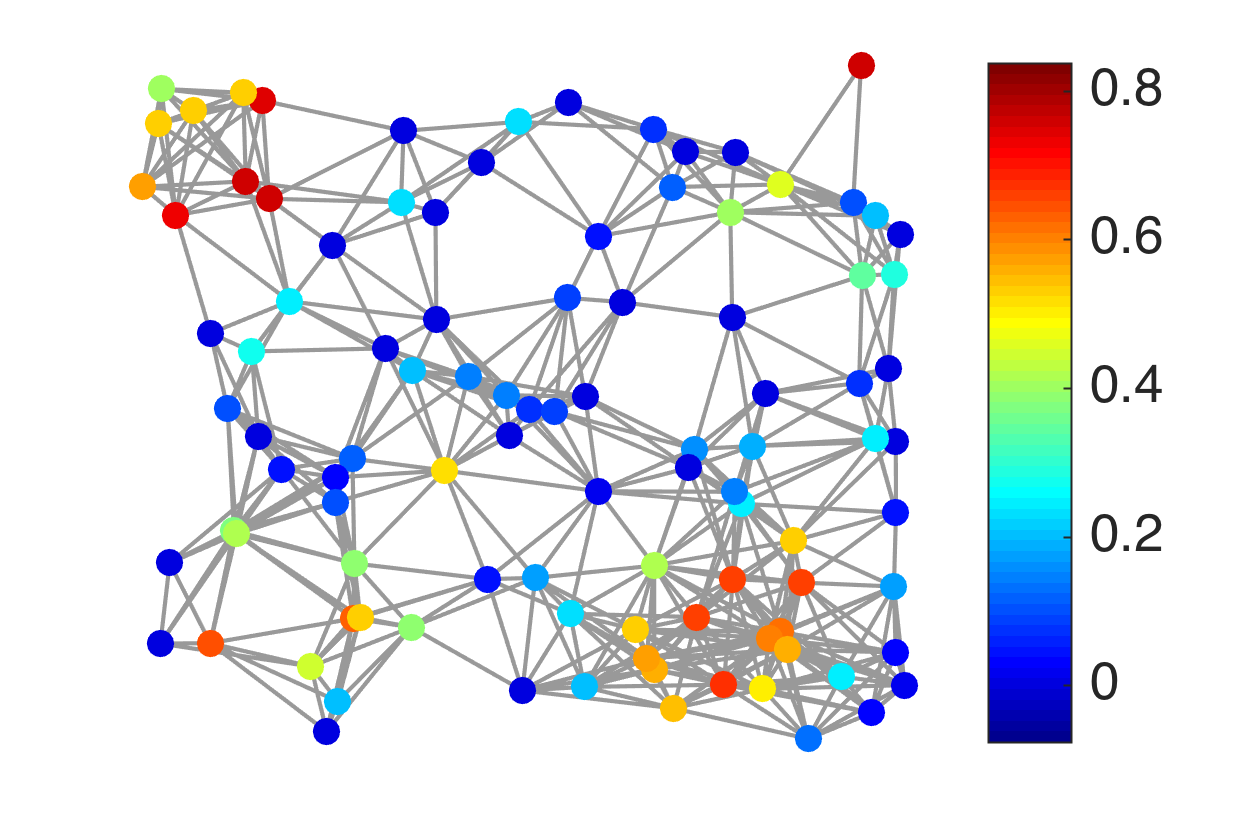} 
\end{tabular}
\label{fig: local uncertainty filter}
\end{figure}
\clearpage

The first row (not counting the top figure) shows the local sparsity levels  of $\A_{\gg}\T_{i_0}g_{k_0}$ in terms of the $\ell^\infty$-norm (left hand side of~\eqref{eq:localuncertaintyexample}) at each node of the graph. The second row shows the values of the upper bound on local sparsity for each node of the graph (middle term of~\eqref{eq:localuncertaintyexample}). The values of both rows are strikingly close. Note that for this type of graph, %the uncertainty is high 
local sparsity/concentration is lowest 
where the nodes are well connected. %This behaviour was already observed earlier on through the modified path example.

We focus now on the values of $\tilde{k}$ and $\tilde{i}$ as they are crucial in Theorem \ref{theo:local_uncertainty}. We also give insights that explain when a tight bound is obtained, as stated in Corollary~\ref{corol:localuncertainty}.
%\begin{itemize}
% % Benjamin version
% \item Usually on regular domains, filters $g_{k}$ and $g_{k_{0}}$ with $k\neq k_{0}$ overlap only partially and 
% \[
% \|g_{k_{0}}^{2}\|_{2}^{2}=\sum_{\ell}\left(g_{k_{0}}^{2}(\ell)\right)^{2}>\sum_{\ell}\left(g_{k_{0}}(\ell)g_{k}(\ell)\right)^{2}=\|g_{k}\cdot g_{k_{0}}\|_{2}^{2}.
% \]
% It the case for the first filter (show on row 3) of Fig.~\ref{fig: local uncertainty filter} where $\tilde{k}=k_0$ (uniform green color). However, the second filter does not satisfy this property: all nodes are blue (value zero) except one on the top right corner. This is due to the fact that this node is less connected to the rest and there is a Laplacian eigenvector well localized on this it. As a consequence, the localization on the graph is affected in a counter-intuitive manner.
% Nathanael version
%\item %Their exist no 
There is not a simple way to determine the value of $\tilde{k}$, because it depends not only on the node $i_0$ and the filters $\widehat{g_k}$, but also on the graph Fourier basis. However, the definition $\tilde{k} = \text{arg}\max_{k} \| \T_{i_0}(g_k\cdot g_{k_0})\|_\infty$ implies that the two kernels $\widehat{g_{\tilde{k}}}$ and $\widehat{g_{k_0}}$ have to overlap ``as much as possible'' in the graph Fourier domain in order to maximize the infinity-norm. In the case of a Gabor filter bank like the one presented in the first line of Fig.~\ref{fig: local uncertainty filter}, $k_0 = \tilde{k}$ for most of the nodes.  This happens because the filters $\widehat{g_{k}}$ and $\widehat{g_{k_{0}}}$ do not overlap much if $k\neq k_{0}$, i.e when
\[
\|\widehat{g_{k_{0}}}^{2}\|_{2}^{2}=\sum_{\ell}\left(\widehat{g_{k_{0}}}^{2}(\lambda_\ell)\right)^{2}\text{>>}\sum_{\ell}\left(\widehat{g_{k_{0}}}(\lambda_\ell)\widehat{g_{k}}(\lambda_\ell)\right)^{2}=\|\widehat{g_{k}}\cdot \widehat{g_{k_{0}}}\|_{2}^{2}.
\]
In fact, in the case of Fig.~\ref{fig: local uncertainty filter}, $\tilde{k}$ is bounded between $k_0-1$ and $k_0+1$ because there is no overlap with the other filters.
In Fig.~\ref{fig: local uncertainty filter}, we plot $\tilde{k}(i)$ for $k_0=0$ and $k_0=1$. For the first filter, we have $\tilde{k}_{i_0,k_0} = k_0$ for all vertices $i_0$. The second filter follows the same rule except for two nodes. The isolated node on the north east is less connected to the rest and there is a Laplacian eigenvector well localized on it. As a consequence, the localization on the graph is affected in a counter-intuitive manner.

%\item 
Let us now concentrate on the second important variable: $\tilde{i}$. Under the assumption that the kernels $\widehat{g_k}$ are smooth,
%\footnote{As a informal definition, a smooth kernel is a a kernel well approximated by a low-order polynomial.}, 
the energy of localized atoms $\T_{i_0}g_k$ reside inside a ball centered at $i_0$~\cite{shuman2015vertex}. 
Thus, the node $j$ maximizing $ |\T_{i_0}(g_{k_0} g_{\tilde{k}})(j)|$ cannot be far from the node $i_0$. Let us define the hop distance $h_{\G}(i,j)$ as the length of the shortest path\footnote{A \emph{path} in a graph is a tuple of vertices $(v_1,v_2,...,v_p)$ with the property that $[v_i,v_{i+p}]\in \E$ for $1\leq i \leq p-1$. Two nodes $v_i,v_j$ are connected by a path if there is exist such  tuple with $v_1=v_i$ and $v_p=v_j$. The length of a path is defined as the cardinality of the path tuple minus one.} between nodes $i$ and $j$. If the kernels $\widehat{g_k}$ are polynomial functions of order $K$, the localization operator $T_{i_0}$ concentrates all of the energy of $T_{i_0}g_k$ inside a $K$-radius ball centered in $i_0$. Since the resulting kernel $\widehat{g_{k_0}} \widehat{g_{\tilde{k}}}$ is a polynomial of order $2K$, $\tilde{i}$ will be at a distance of at most of $2K$ hops from the node $i_0$. In general, $\tilde{i}$ is close to $i_0$. In fact, the distance $h_{\G}(i, \tilde{i})$ is related to the smoothness of the kernel $\widehat{g_{k_0}}\widehat{g_{\tilde{k}}}$ \cite{shuman2015vertex}. 
To illustrate this effect, we present in Fig.~\ref{fig:rel_error_ntig_inf} the average and maximum hop distance $h_{\G}(i,\tilde{i})$. In this example, we control the concentration of a kernel $\hat{g}$ with a dilation parameter $a$: $\widehat{g_a}(x) = \hat{g}(ax)$. Increasing $a$ compresses the kernel in the Fourier domain and increases the spread of the localized atoms in the vertex domain. Note that even for high spectral compression, the hop distance $h_{\G}(i,\tilde{i})$ remains low. Additionally, we also compute the mean relative error between $\|\T_i g^2 \|_\infty$ and $|\T_i g^2(i) |$. This quantity asserts how well $\|\T_i g \|_2^2$ estimates $\|\T_i g^2 \|_\infty$.\footnote{From Lemma \ref{lemma:Perraudin}, when $\|T_i g^2\|_\infty = |T_ig^2(i)|$, then $\|T_i g^2\|_\infty =  \|\T_i g \|_2^2$.}
%For this last quantity, the heat kernel has much lower values than the wavelets. It comes from the fact that usually for low frequency decreasing kernels, the maximum of the atoms $\T_ig$ is located in $i$. 
Returning to Fig.~\ref{fig: local uncertainty filter}, %Let us now analyze us to explain 
the fourth  row %line of  which 
shows the hop distance between $i_0$ and $\tilde{i}$. It never exceed $3$ for both the first and the second filter, which is a good sign of locality. 

%Moreover, the fact that the second filter is not decreasing explains why the hope distance is a slightly bigger for this one.
\begin{figure}[ht!]
\begin{minipage}[t]{.38\linewidth}
\begin{center}
\includegraphics[width=0.95\textwidth]{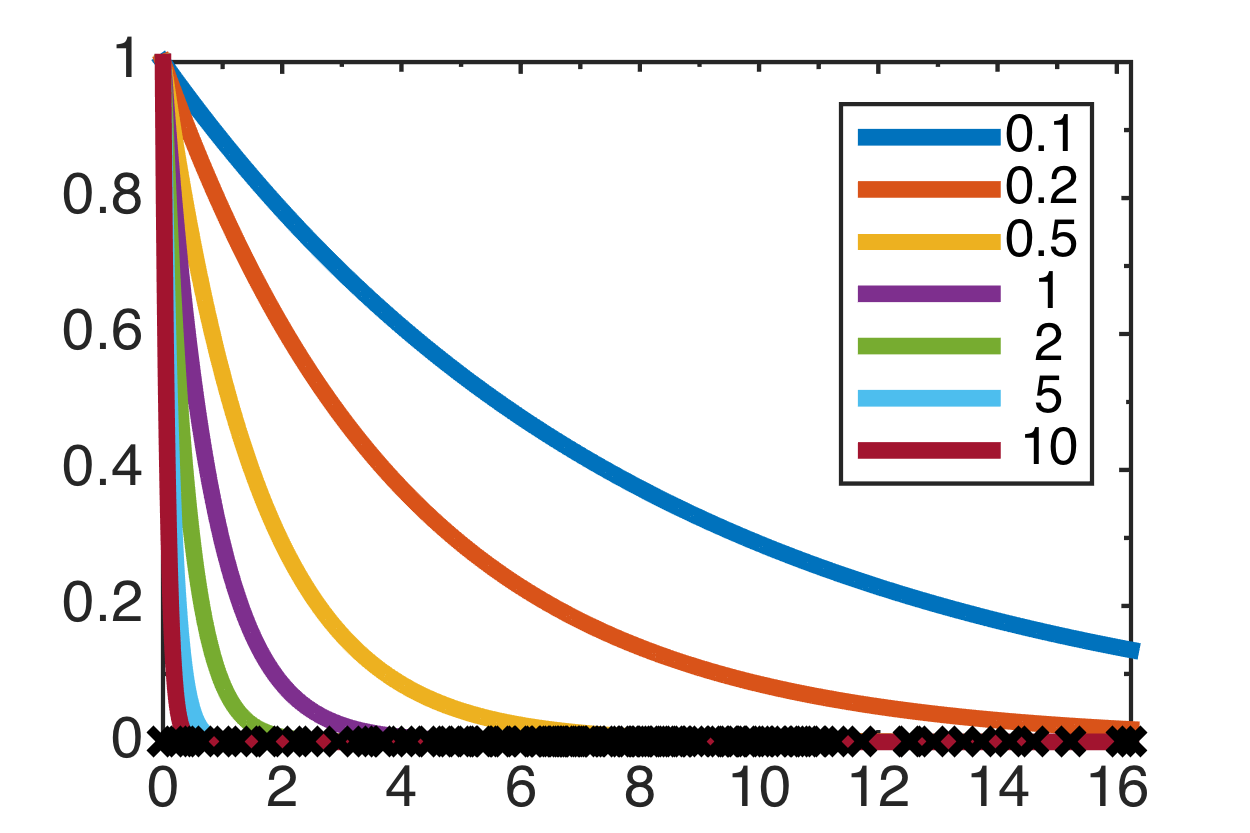}
\includegraphics[width=0.95\textwidth]{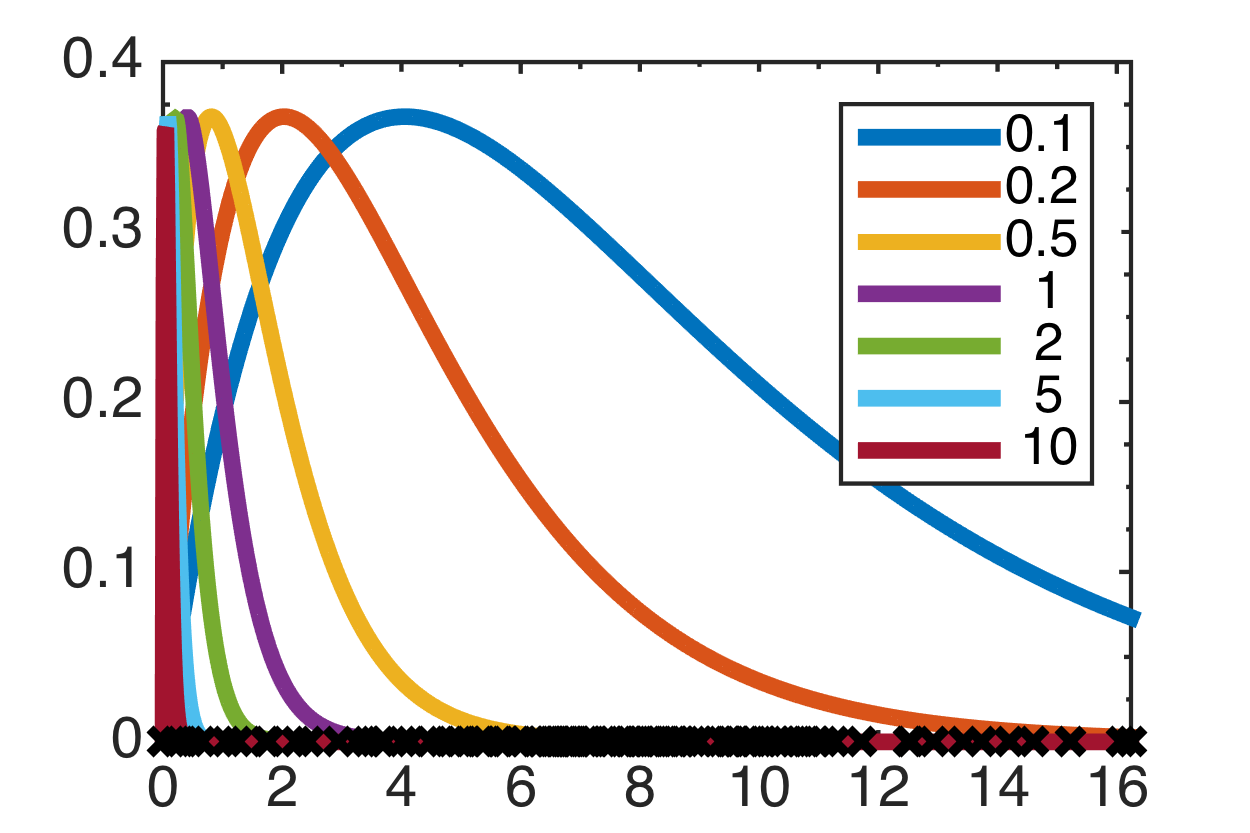} \\
\end{center}
\end{minipage}
\begin{minipage}[t]{.45\linewidth}
\begin{tabular}{|c|c|c|c|} 
\hline
$a$ & $\text{mean}_i~\frac{\|\T_i g \|_\infty - |\T_i g(i) |}{\|\T_i g \|_\infty} $ in \% & $\text{mean}_i~h_{\G}(\tilde{i},i) $  & $ \max_i h_{\G}(\tilde{i},i)  $ \\ 
\hline
\hline
\multicolumn{4}{|c|}{Heat kernel} \\
\hline
0.1 & 0 & 0 & 0 \\ 
0.2 & 0 & 0 & 0 \\ 
0.5 & 2.01 & 0.28 & 1 \\ \textbf{}
1 & 5.87 & 0.89 & 2 \\ 
2 & 7.45 & 1.39 & 3 \\ 
5 & 8.59 & 2.04 & 4 \\ 
10 & 2.63 & 2.08 & 4 \\ 
\hline
\hline
\multicolumn{4}{|c|}{Wavelet kernel} \\
\hline
0.1 & 0 & 0 & 0 \\ 
0.2 & 0 & 0 & 0 \\ 
0.5 & 9.03 & 0.62 & 1 \\ 
1 & 10.99 & 1.07 & 2 \\ 
2 & 17.69 & 1.67 & 3 \\ 
5 & 29.67 & 2.07 & 4 \\ 
10 & 33.45 & 2.48 & 6 \\
\hline
\end{tabular}
\end{minipage}
\caption{Localization experiment using the sensor graph of Fig.~\ref{fig: local uncertainty filter}. The heat kernel is defined as $\hat{g}(ax) = e^{-\frac{10 \cdot ax}{\lambda_{\max}}}$ and the wavelet kernel $\hat{g}(ax)= \sqrt{40} \cdot ax \cdot e^{-\frac{40\cdot ax}{\lmax}}$. For a smooth kernel $\hat{g}$, the hop distance $h_{\G}$ between $i$ and $\tilde{i} =\argmax_j |\T_ig(j)|$ is small.} %{\color{red} Is the first column the mean over all i? \nati{yes} Does $g$ need to be squared in the definition of $\tilde{i}$? \nati{no} }}
\label{fig:rel_error_ntig_inf}
\end{figure}
%\end{itemize}
In practice we can not %know what are the 
always determine the values of $\tilde{k}$ and $\tilde{i}$, but as we have seen, the quantity $B^{-\frac{1}{2}}\|\T_{i}g_{k_0}\|_2$ may still be a good estimate of the local sparsity level. %{\color{red} Is this shown in Row 5? Double check if there should be a $\sqrt{N}$ in the numerator/denominator.\nati{I think now it is clear.}}
%If we assume that $\tilde{k}=k_0$, we have
%\[
%\|\A_{\gg}T_{i0}g_{k0}\|_{\infty} 
%\approx \sqrt{N}\T_{i0}g_{k0}^2(i_0)
%= \scp{\T_{i0}g_{k0} }{\T_{i0} g_{k0} }  
%= \|\T_{i0} g_{k0} \|_2^2.
%\]
%Note that this estimate is always bigger than the bound, which might be an undesirable consequence. 
Row 5 of Fig.~\ref{fig: local uncertainty filter} shows these %uncertainty 
estimates, and the last row shows the relative error between these estimates and the actual local sparsity levels. We observe that for the first kernel, the estimate  gives a sufficiently rough approximation of the local sparsity levels. For the second kernel, the approximation error is low for most of the nodes, but not all.

%In order to analyze, we might want to compute the local uncertainty bound for each node. There exist techniques to estimate efficiently $\|\T_ig\|_2$ \cite{perraudin2015stationarity}. However, the computation of $\tilde{k}$ and $\tilde{i}$ remains very expensive in term of computation. As result, we propose an estimation of the local uncertainty bound. If we set $\tilde{k}=k_0$,

\end{example}

%\subsection{Comparison between local and global uncertainty principle on a modified path graph}

In the next example, we compare the local and global uncertainty principles on a modified path graph.

\begin{example}
%\nati{
On a $64$ node modified path graph (see Example~\ref{ex:pathgraph} for details), we compute the graph Gabor transform of the signals 
$f_1=T_1 g_0$ and $f_2=T_{64}g_0$.
%{\color{red} $f_1=\delta_1$ (localized on the first only) and $f_2=\delta_{64}$.}
%We decrease the weight of the edge connecting the first and second nodes of the path to see the effect on the concentration of the ambiguity function. 
In Figure \ref{fig:path_node_away_ambiguity}, we show the evolution of the graph Gabor transforms of the two signals with respect to the distance $d=1/W_{12}$ from the first to the second vertex in the graph. As the first node is pulled away, a localized eigenvector appears centered on the isolated vertex. Because of this, %the localization operator is able to concentrate signal in both domains. This is illustrated with the 
as this distance increases, the
signal $f_1$ becomes concentrated in both the vertex and graph spectral domains, leading to graph Gabor transform coefficients that are highly concentrated (see the top right plot in Fig.  \ref{fig:path_node_away_ambiguity}). %the Gabor coefficients become more and more concentrated. 
However, since the graph modification is local, it does not drastically affect %the representation
the graph Gabor transform coefficients of the signal $f_2$ (middle row of Fig.  \ref{fig:path_node_away_ambiguity}), whose energy is concentrated on the far end of the path graph.

In Figure~\ref{fig:bound Agf}, we plot the evolution of the uncertainty bounds as well as the concentration of the Gabor transform coefficients of $f_1$ and $f_2$. 
%While the change of the weight connecting the first two vertices barely affects the analysis coefficients for $\delta_{64}$ (whose energy is concentrated at the far end of the path), $\delta_1$ becomes concentrated in both domains, and therefore its coefficients become quite sparse. 
The global uncertainty bound from Theorem  \ref{Co:Lieblocgraph} tells us that 
$$s_1(\A_\gg f) \leq \max_{i,k} ||T_i g_k||_2,\hbox{ for any signal }f.$$
The local uncertainty bound from Theorem \ref{theo:local_uncertainty} tells us that 
$$s_1(\A_\gg T_{i_0} g_{k_0}) \leq ||T_{\tilde{i}_{i_0,k_0}}g_{\tilde{k}_{i_0,k_0}}||_2, \hbox{ for all } i_0 \hbox{ and } k_0.$$
Thus, we can view the global uncertainty bound as an upper bound % maximum 
%of
on all of the local uncertainty bounds. 
 %The concentration of the signal $f_2$ is almost not affected by the change in node $1$. The local bound does also almost not change. As for signal $f_1$, it becomes very concentrated in both domains as the node is pulled away. The global bound is taking into account of the worst case and can be seen of the maximum of all local bounds. 
In fact the bumps in the global uncertainty bound in Figure~\ref{fig:bound Agf} correspond to the local bound with $i_0=1$ and   different frequency bands $k_0$. We plot the local bounds for $i_0=1$ and $k_0=0$ and $k_0=2$. %}

\begin{figure}[ht!]
\begin{center}
\hspace{.75in} $d=1$ \hspace{.75in} $d\approx 11$ \hspace{.75in} $d\approx 17$ \hspace{.75in} $d=27$ \hspace{.65in} $d=81$  \hspace{.6in} \\
\begin{minipage}{.1\linewidth}
$f_1=T_1g_0$ \vspace{1in}
\end{minipage}
\includegraphics[width=0.16\textwidth]{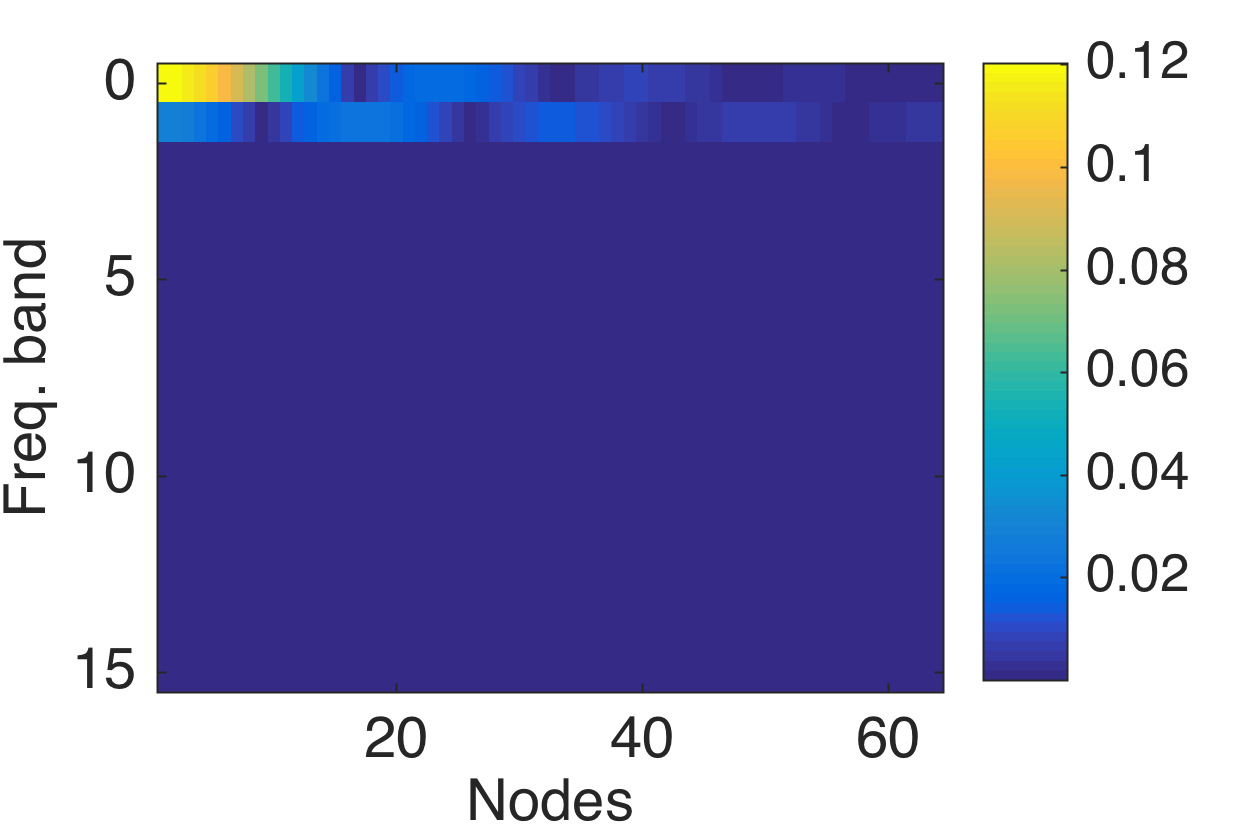}
\includegraphics[width=0.16\textwidth]{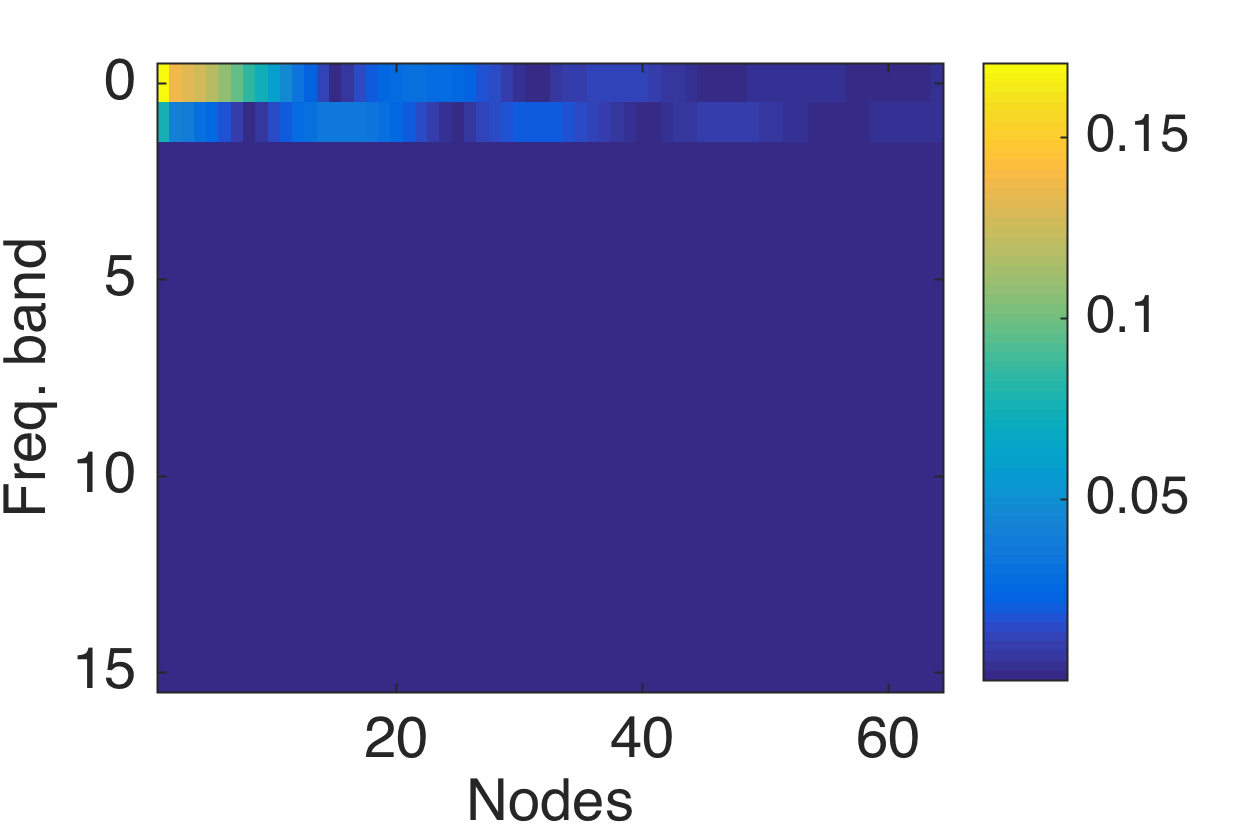}
\includegraphics[width=0.16\textwidth]{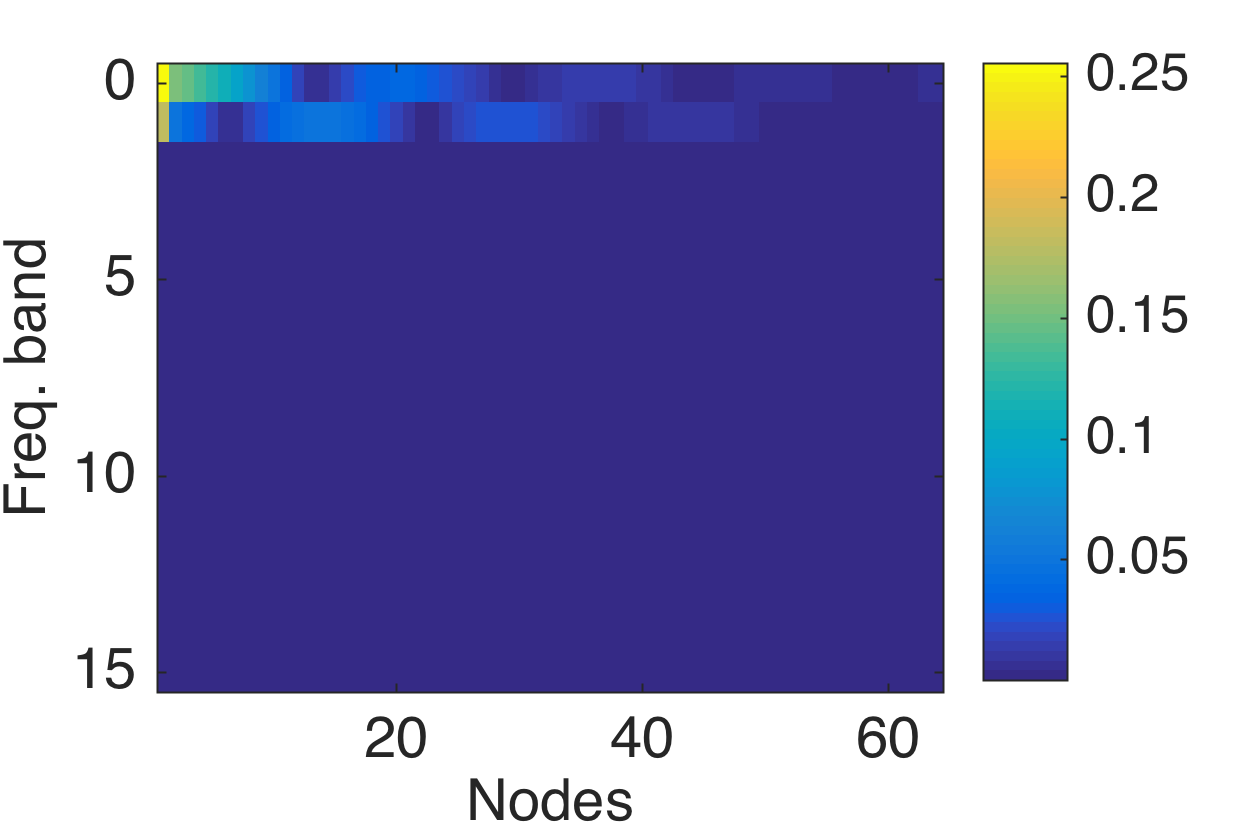}
\includegraphics[width=0.16\textwidth]{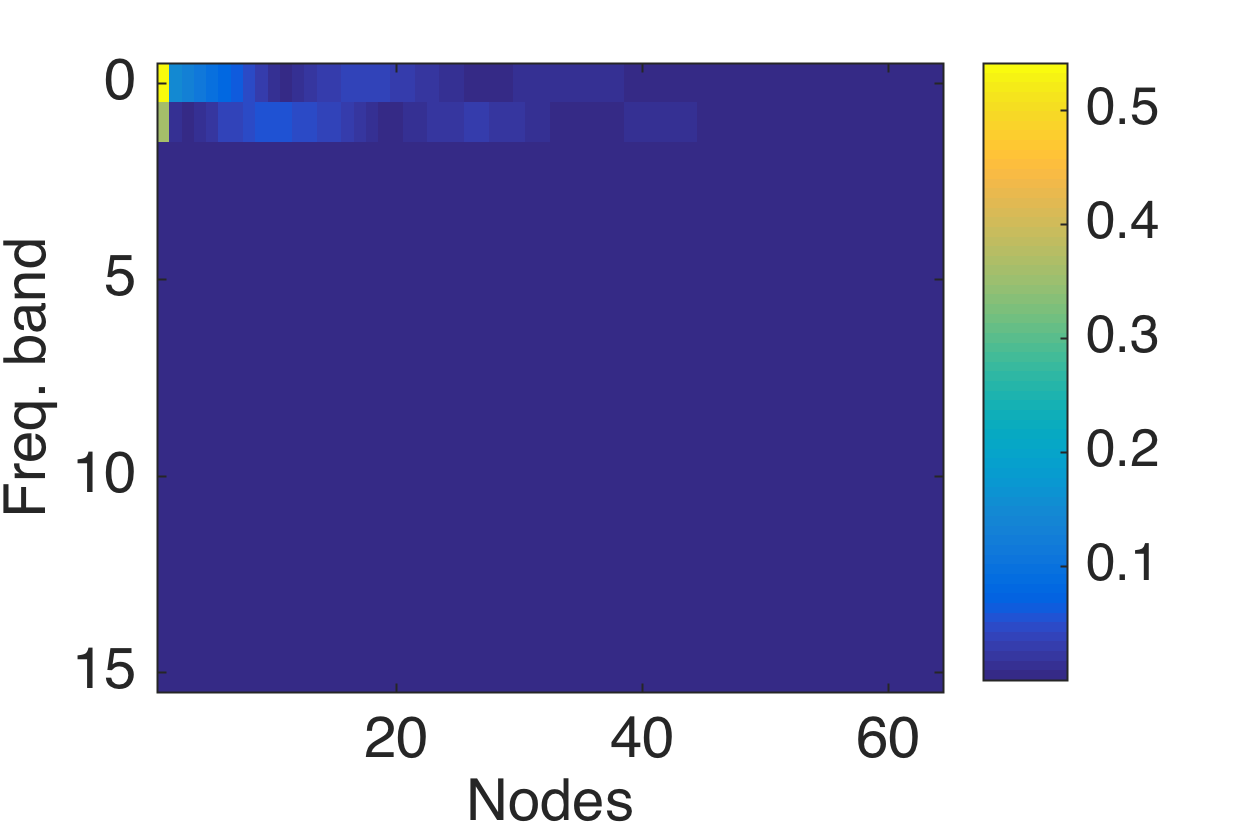}
\includegraphics[width=0.16\textwidth]{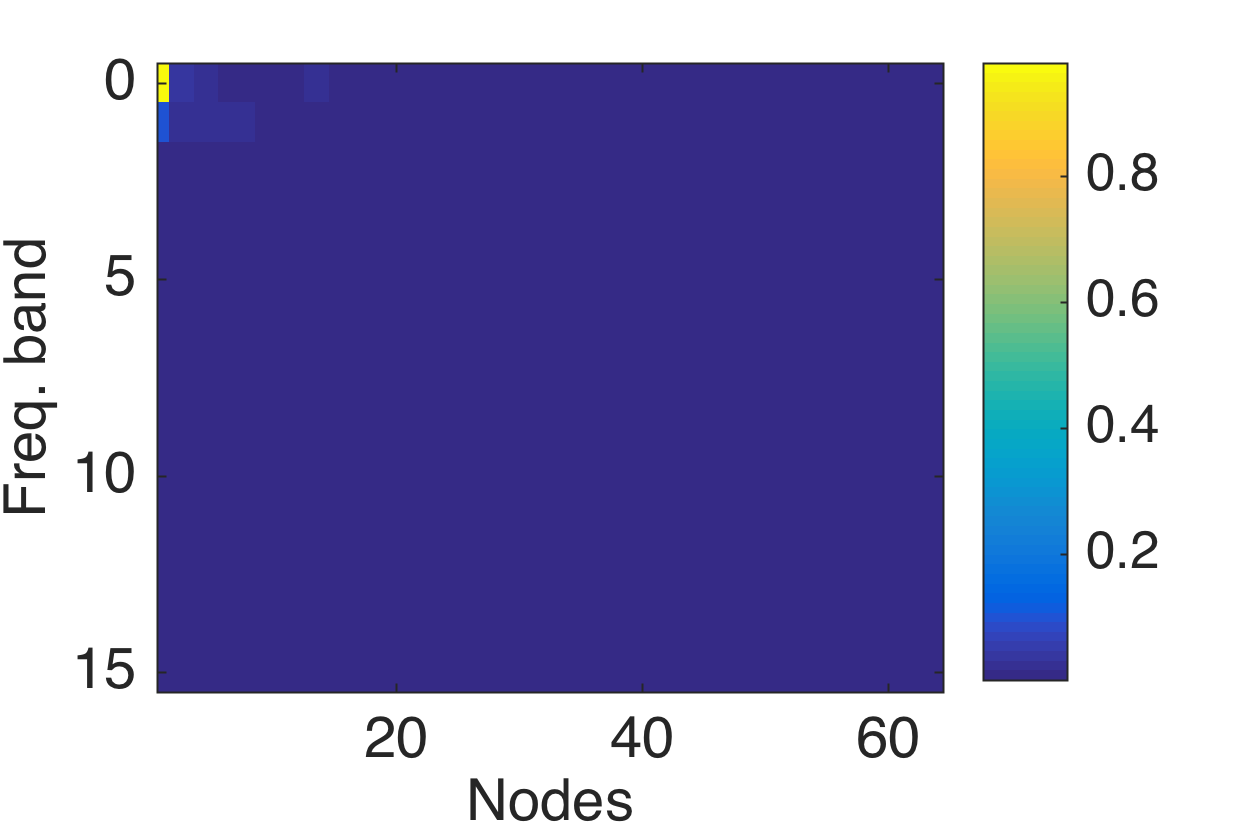} \\
\vspace{-.3in}
\begin{minipage}{.1\linewidth}
$f_2=T_{64}g_0$ \vspace{1in}
\end{minipage}
\includegraphics[width=0.16\textwidth]{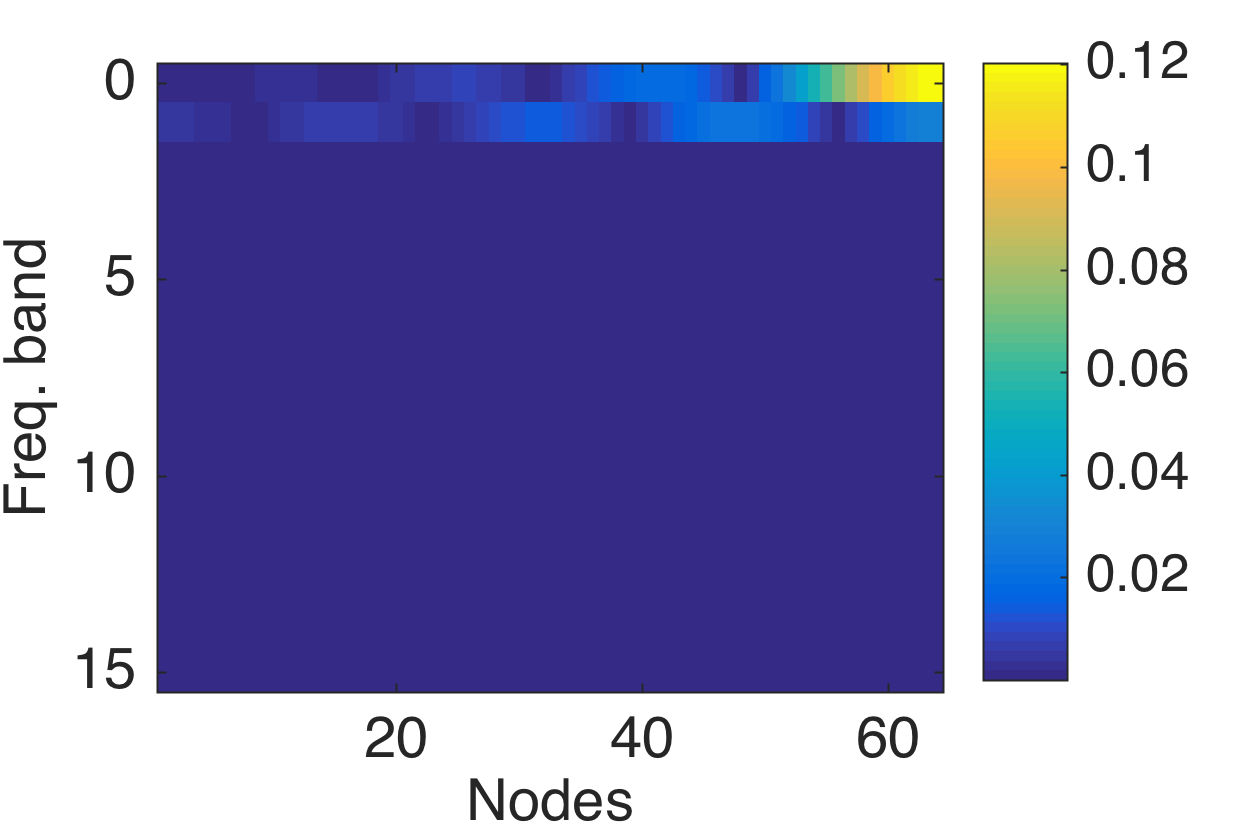}
\includegraphics[width=0.16\textwidth]{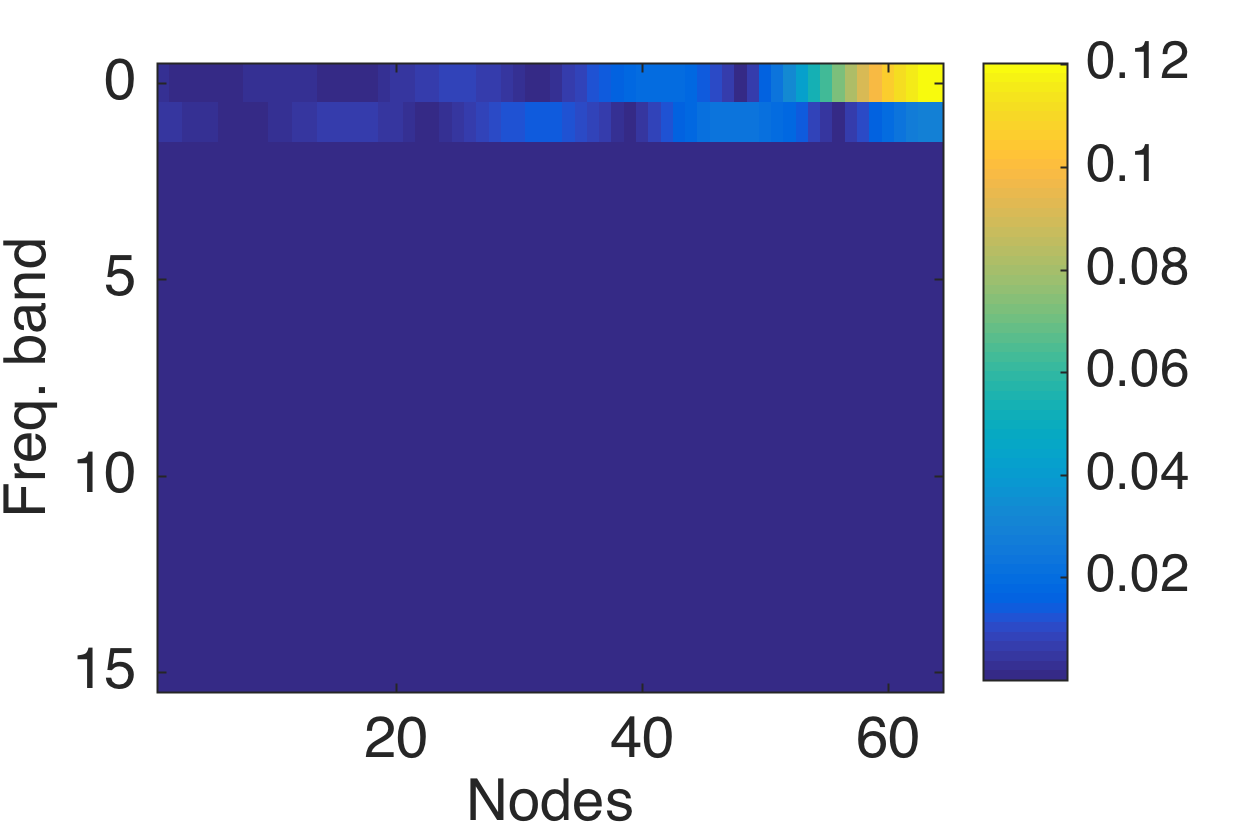}
\includegraphics[width=0.16\textwidth]{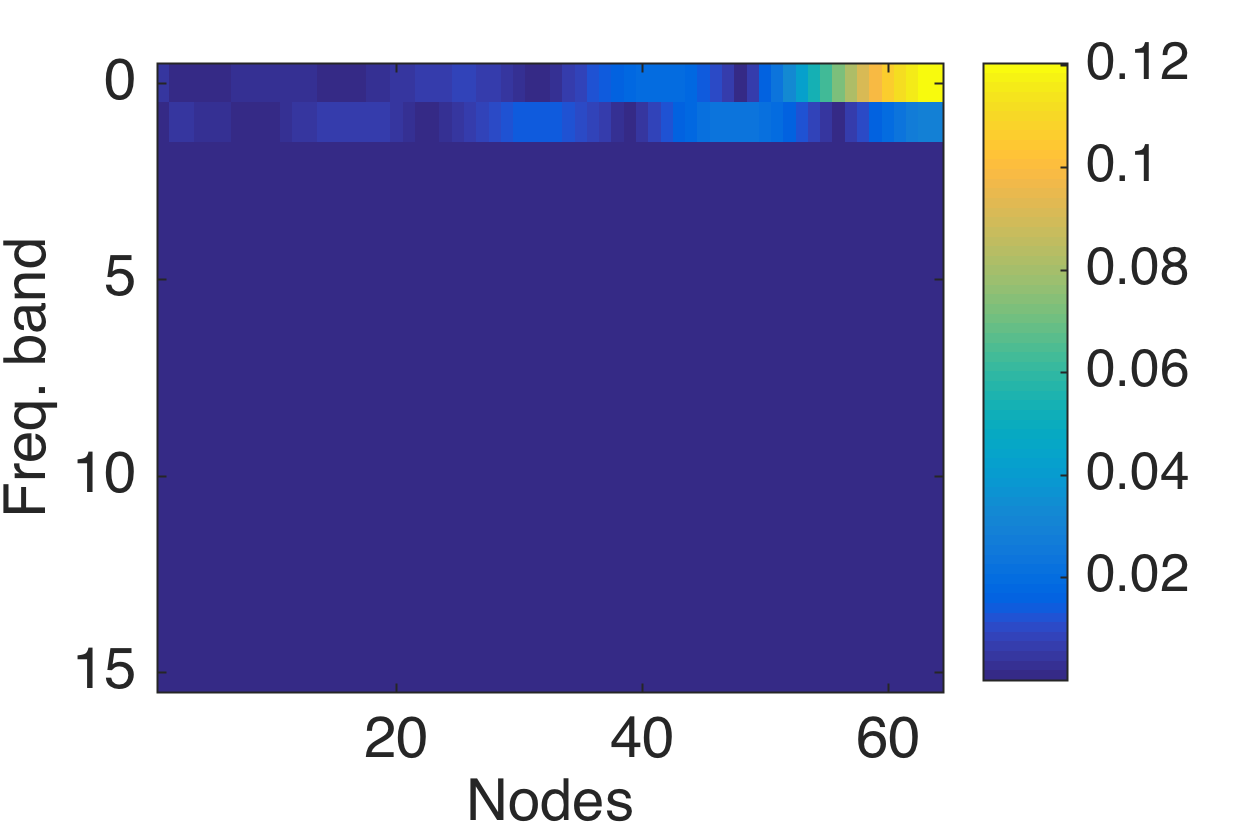}
\includegraphics[width=0.16\textwidth]{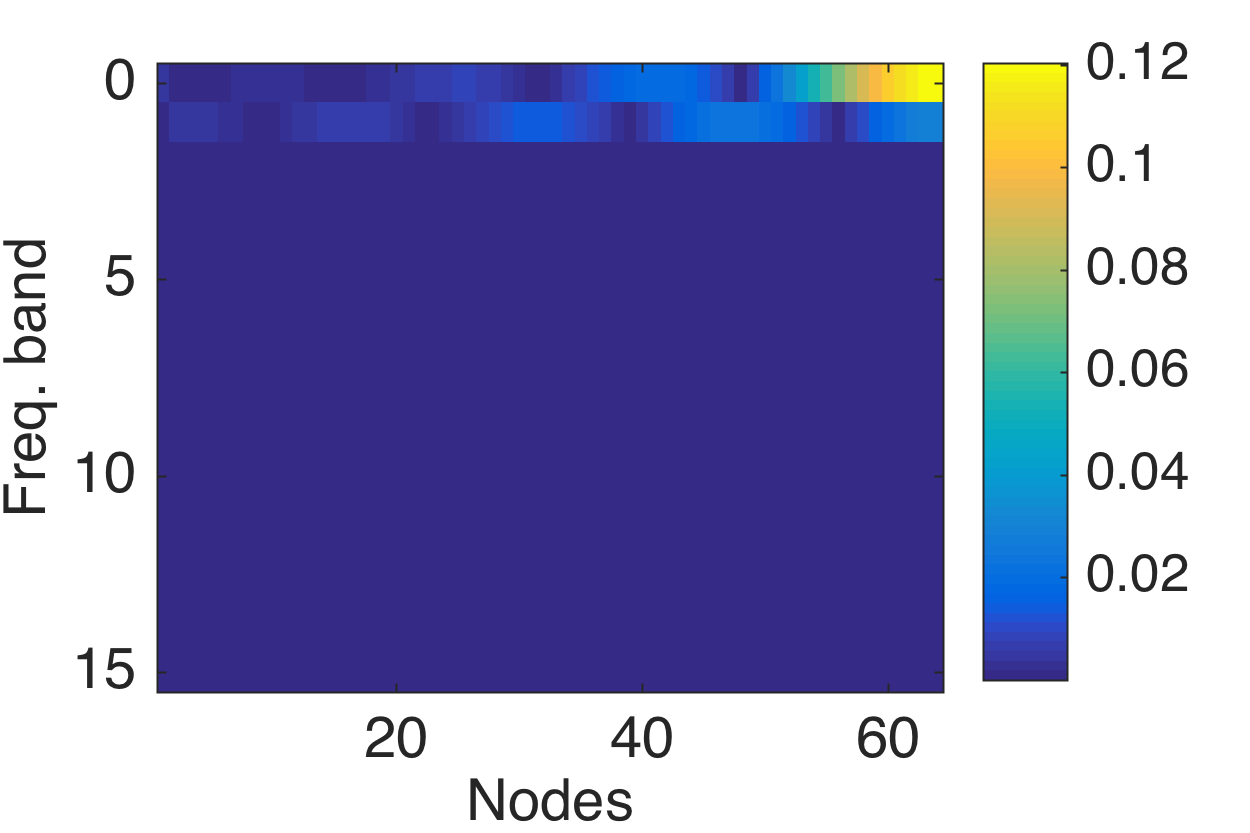}
\includegraphics[width=0.16\textwidth]{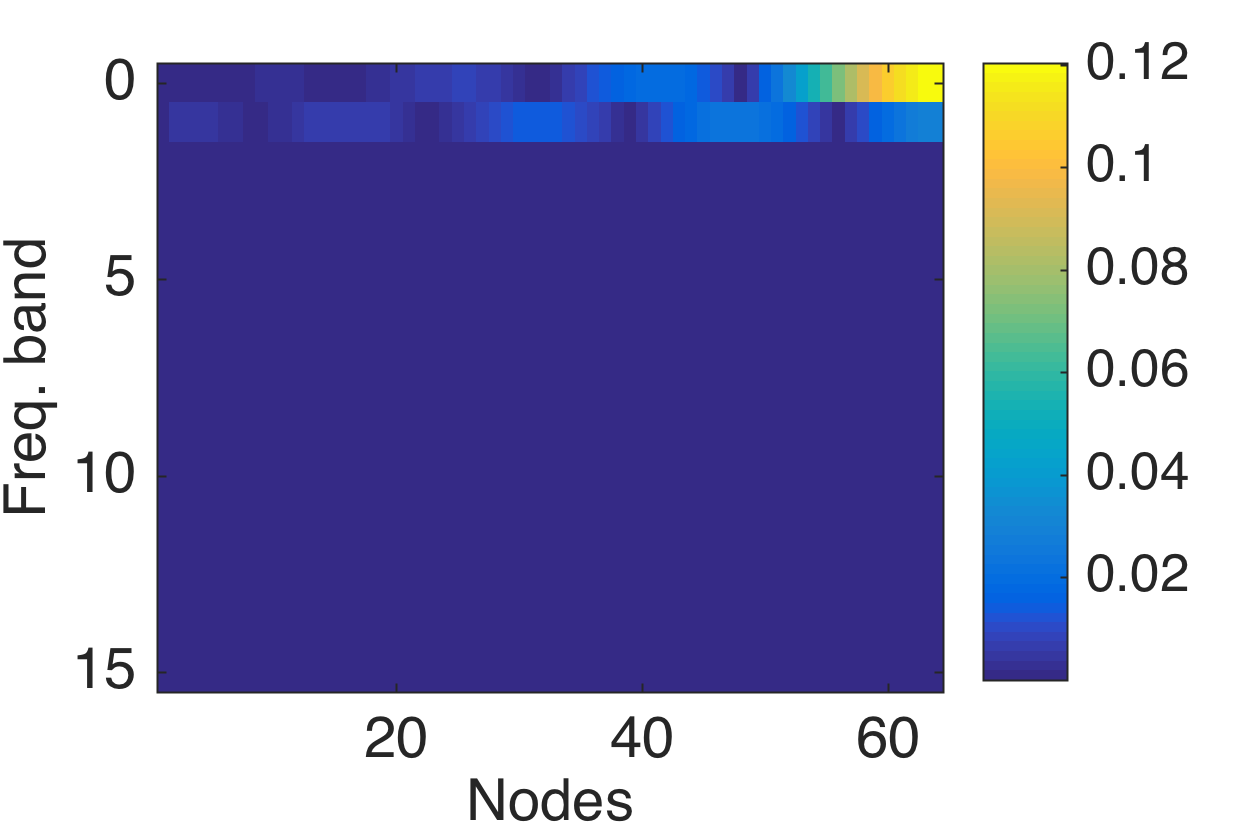} \\
\vspace{-.3in}
\begin{minipage}{.1\linewidth}
$f_1=T_{1}g_0$ \\ 
Vertex domain \vspace{0.75in} 
\end{minipage}
\includegraphics[width=0.16\textwidth]{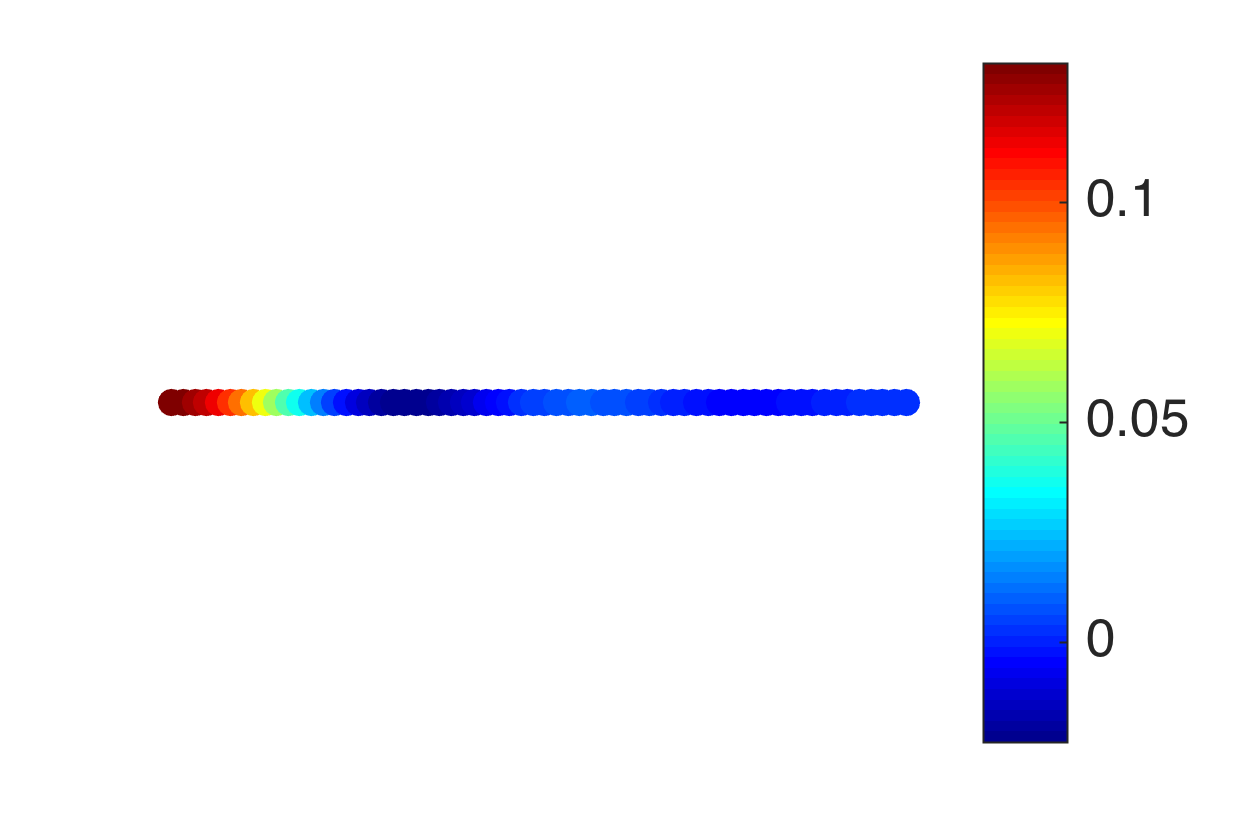}
\includegraphics[width=0.16\textwidth]{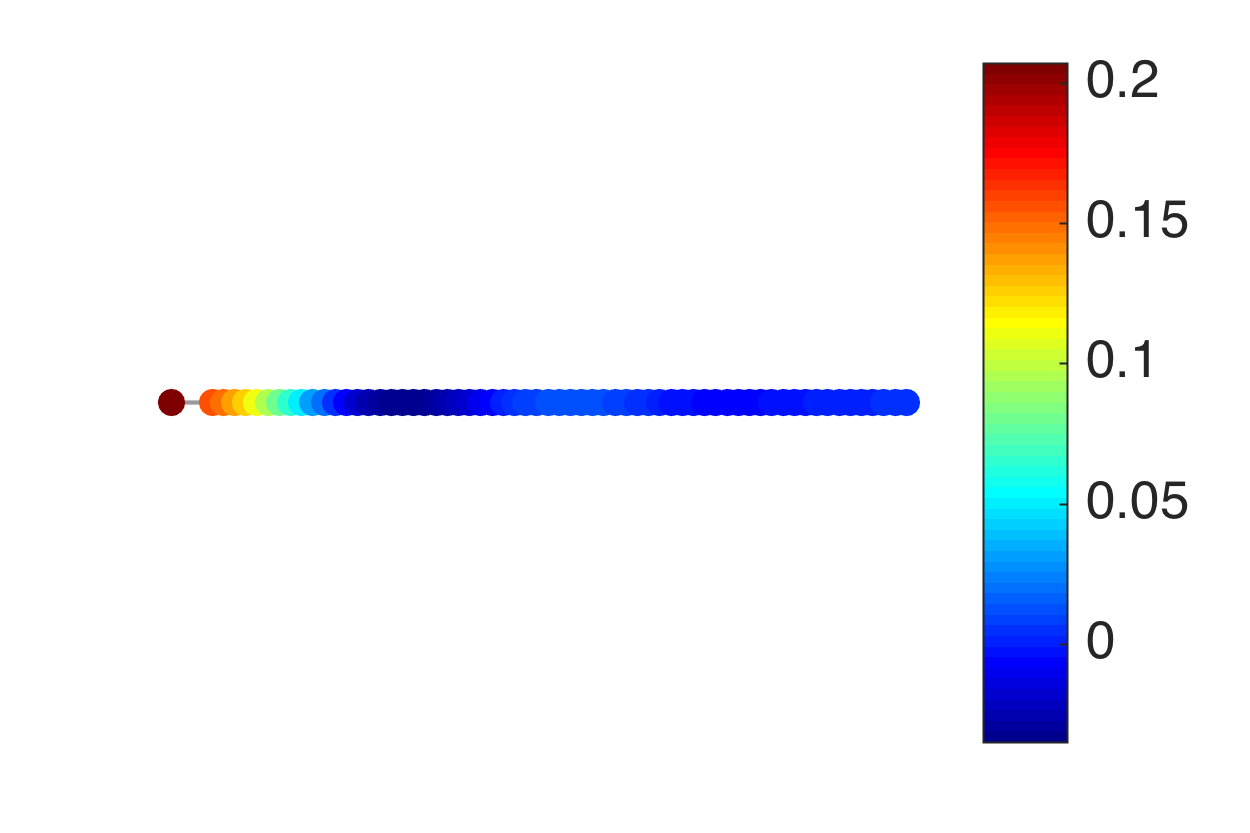}
\includegraphics[width=0.16\textwidth]{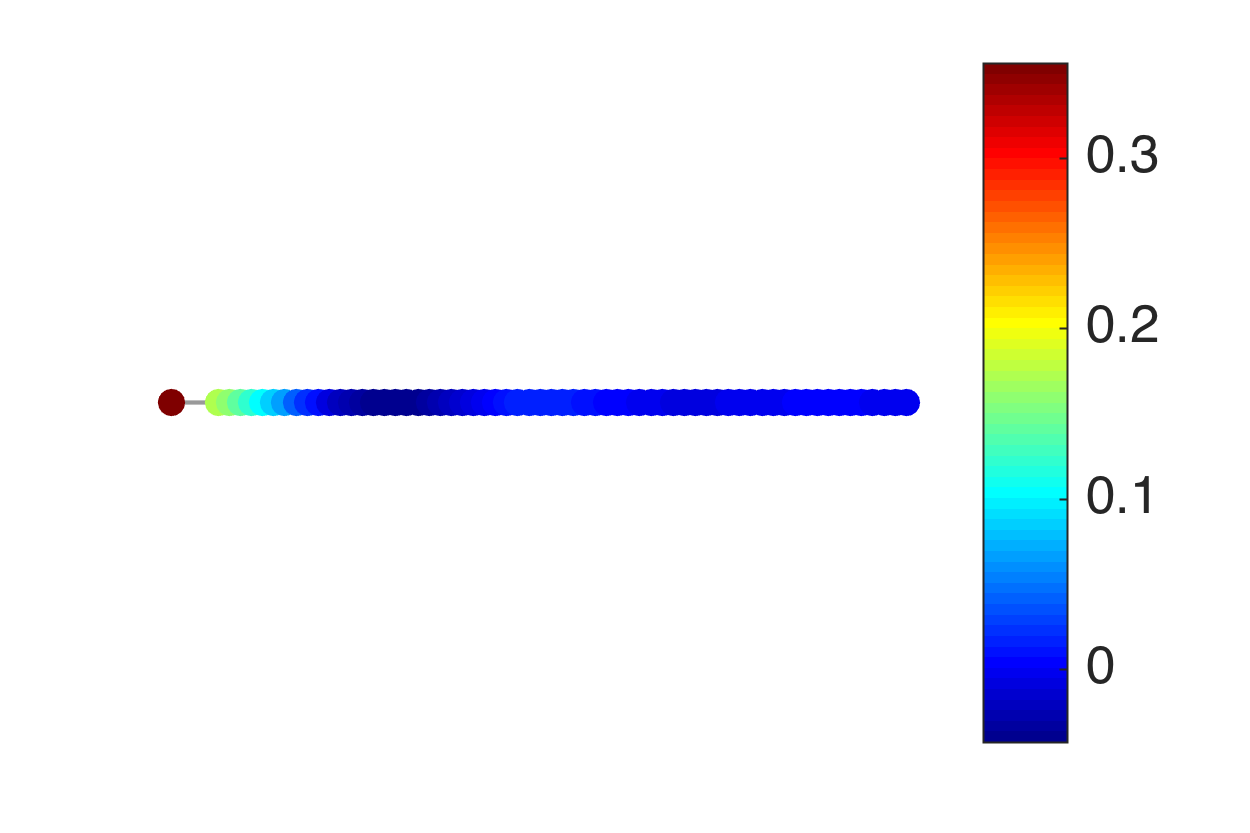}
\includegraphics[width=0.16\textwidth]{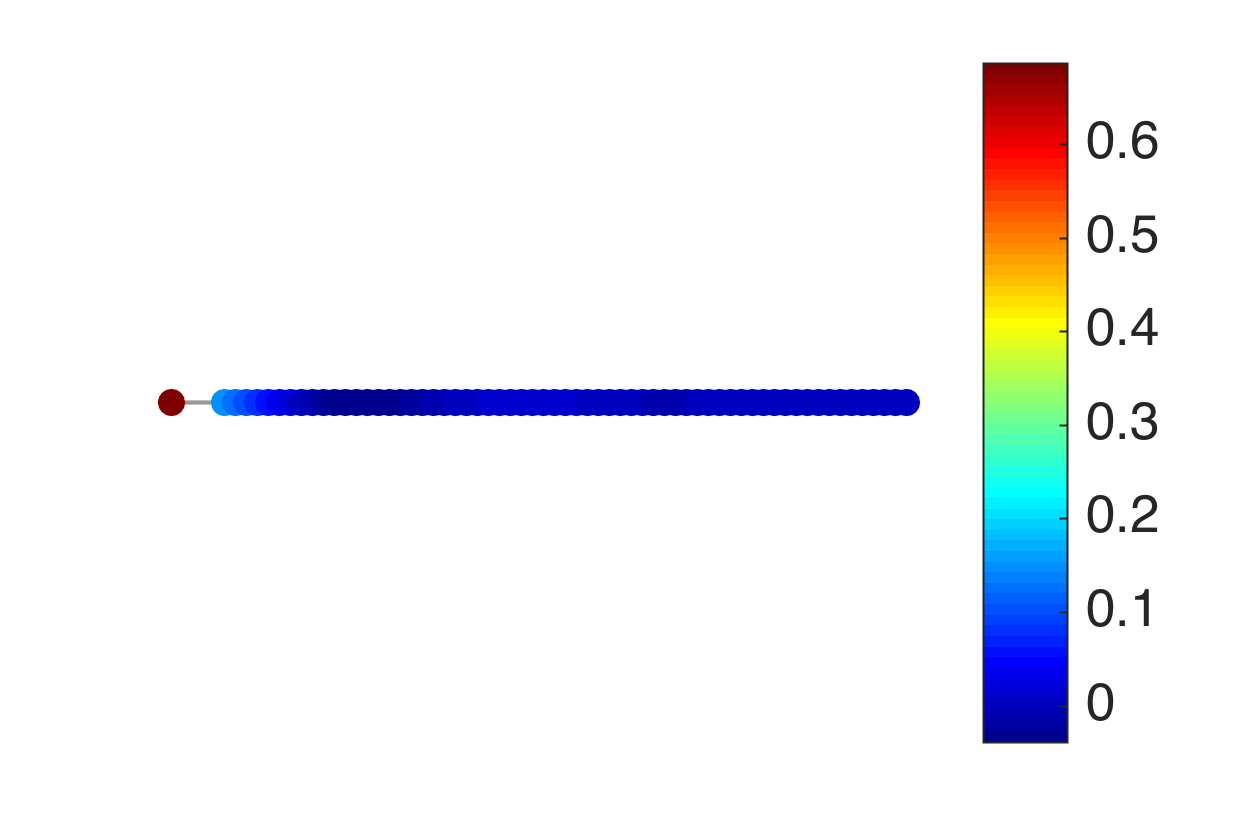}
\includegraphics[width=0.16\textwidth]{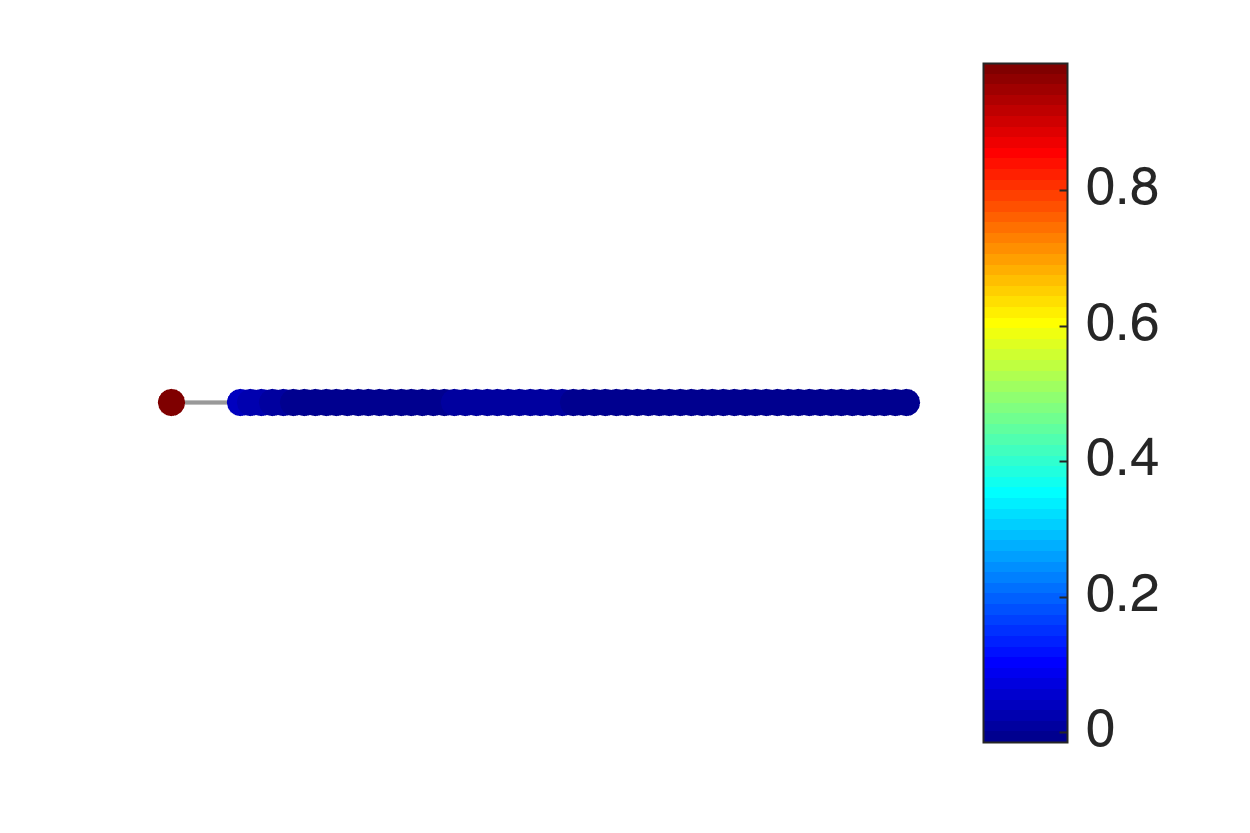} \\
\end{center}
\vspace{-.7in}
\caption{%{\color{red} Do we want to also change this one to translated atoms to match the next figure?} \nati{Here you have the figure with the atoms. Do you prefer it?}
%Graph Gabor transforms of $f_2=\delta_1$ and $f_2=\delta_{64}$ for 5 different distances between vertices 1 and 2 of the modified path graph. The distance $d=1/W_{12}$ is the inverse of the weight of the edge connecting the first two vertices in the path. The node $64$ is not affected by the change in the graph structure, because its energy is concentrated on the opposite side of the path graph. The graph Gabor coefficients of $f_1$, however, become highly concentrated as a graph Laplacian eigenvector becomes localized on vertex 1 as the distance increases.
Graph Gabor transforms of $f_1=T_1g_0$ and $f_2=T_{64}g_0$ for 5 different distances between vertices 1 and 2 of the modified path graph. The distance $d=1/W_{12}$ is the inverse of the weight of the edge connecting the first two vertices in the path. The node $64$ is not affected by the change in the graph structure, because its energy is concentrated on the opposite side of the path graph. The graph Gabor coefficients of $f_1$, however, become highly concentrated as a graph Laplacian eigenvector becomes localized on vertex 1 as the distance increases. The bottom row shows that as the distance between the first two vertices increases, the atom $T_1 g_0$ also converges to a Kronecker delta centered on vertex 1.
} 
%The first node is pulled away leading to a localized eigenvector and a very concentrated Gabor transform.}
\label{fig:path_node_away_ambiguity}
\end{figure}

\begin{figure}[ht!]
\begin{center}
\includegraphics[width=0.45\textwidth]{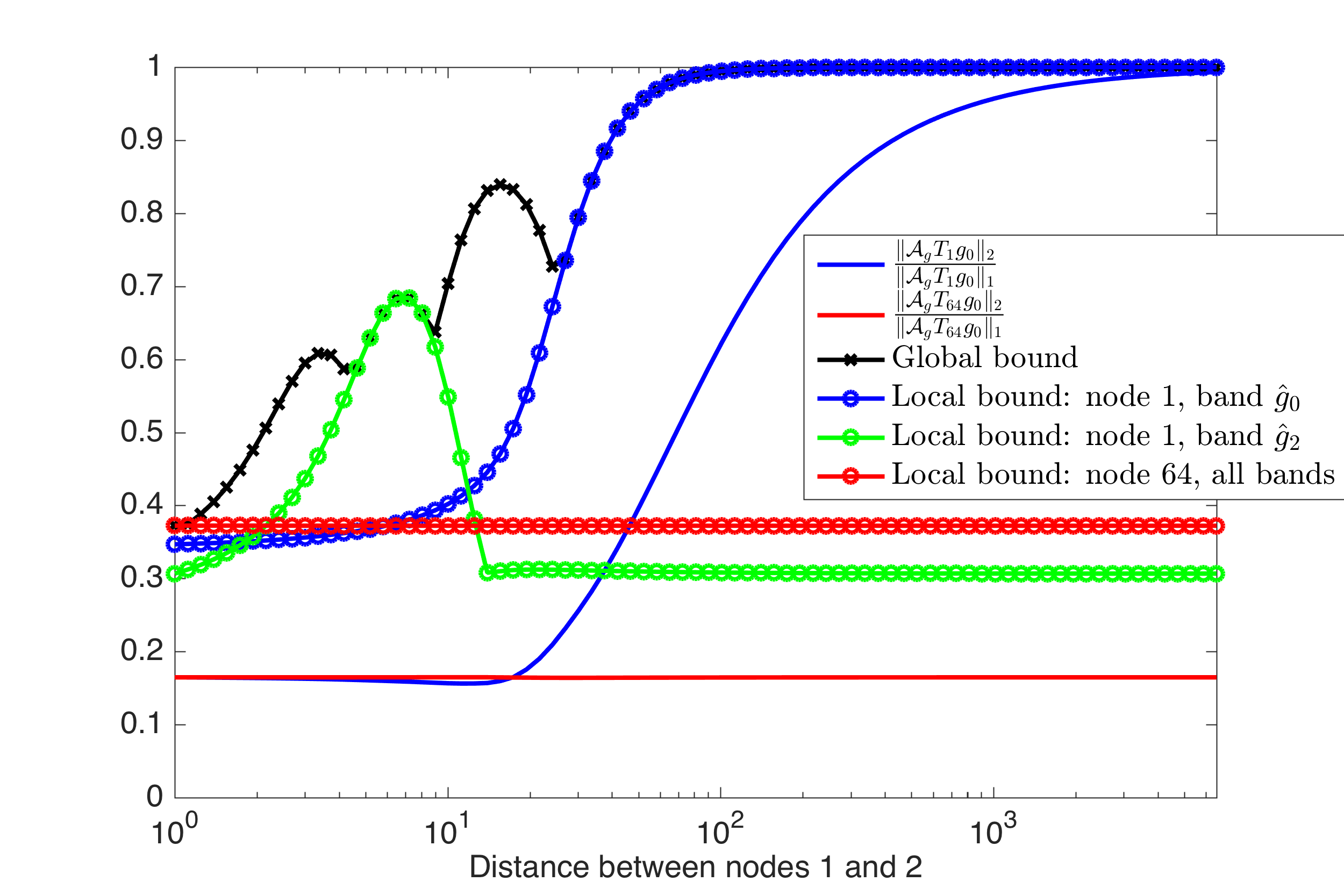}
\end{center}
\caption{ %{\color{red} Change this so that the first two lines are $\A_{\gg}T_1 g_0$ and $\A_{\gg}T_{64}g_0$ (or some other filter)? If so, try to choose the filter such that the localized kernel is close to a delta on each of those vertices, and perhaps plot that. \nati{Done, The additional plots are in figure 12}} 
Concentration of the graph Gabor coefficients of % \nati{ 
$f_1=T_1g_0$ and $f_2=T_{64}g_0$ %} 
with respect to the distance between the first two vertices in the modified path graph, along with the upper bounds on this concentration from  Theorem  \ref{Co:Lieblocgraph} (global uncertainty) and Theorem \ref{theo:local_uncertainty} (local uncertainty). Each bump of the global bound corresponds to a local bound of a given spectral band of node $1$. For clarity, we plot only bands $\widehat{g_0}$ and $\widehat{g_2}$ for node $1$. For node $64$, the local bound is barely affected by the change in graph structure, and the sparsity levels of the graph Gabor transform coefficients of $T_{64}g_0$ also do not change much.}
\label{fig:bound Agf}
\end{figure}
\end{example}

%{\color{red} 
\subsection{Single kernel analysis}

Let us focus on the case where we analyze a single kernel $\hat{g}$. Such an analysis is relevant when we model the signal as a linear combination of different localizations of a single kernel:
%the signal follows the following model
$$
f(n) = \sum_{i=1}^N w_i T_ig(n)
$$
This model has been proposed in different contributions { \cite{perraudin2016stationary, gadde2015probabilistic, zhang2015graph}}, and has also been used as an interpolation model, e.g., in \cite{pesenson_splines} and \cite[Section V.C]{shuman2016multiscale}. 
In this case, we could ask the following question. If we measure the signal value at node $j$, how much information do we get about $w_j$? We can answer this by looking at the overlap between the atom $T_jg$ and the other atoms. When $T_jg$ has a large overlap with the other atoms, the value of $f(j)$ does not tell us much about $w_j$. However, in the case where $T_jg$ has a very small overlap with the other atoms (an isolated node for example), knowing $f(j)$ gives an excellent approximation for the value of $w_j$. The following theorem uses the sparsity level of $g(\L)T_jg$ to analyze the overlap between the atom $T_jg$ and the other atoms.
%s can be analyzed through the sparsity level of $g(\L)T_ig$. We have the following theorem.
%Let us now analyze the overlap between the atoms generated by a single kernel $g$. To do so, we use the following theorem. }
\begin{theorem} \label{theo:local_atom_overlap}
For a kernel $\hat{g}$, the overlap between the atom localized to center vertex $j$ and the other atoms satisfies
\[
O_{p}(j)=\frac{\left(\sum_{i}\left|<\T_{i}g,\T_{j}g>\right|^{p}\right)^{\frac{1}{p}}}{\left(\sum_{i}\left|<\T_{i}g,\T_{j}g>\right|^{2}\right)^{\frac{1}{2}}}=\frac{||g(\L)\T_{j}g\|_{p}}{||g(\L)\T_{j}g\|_{2}} = \frac{\|\T_{j}g^{2}\|_{p}}{\|\T_{j}g^{2}\|_{2}}
\]
\begin{proof}
This result follows directly from the application of \eqref{eq:norm_product_tig} in Lemma \ref{lemma:Perraudin}.
\end{proof}
\end{theorem}
% 
% This motivate the irregular sampling inpainting example

\subsection{Application: non-uniform sampling}

\begin{example}[Non-uniform sampling for graph inpainting] \label{Ex:adaptated_sampling}
In order to motivate Theorem \ref{theo:local_atom_overlap} from a practical signal processing point of view, we use it to optimize the sampling of a signal over a graph. To asses the quality of the sampling, we solve a small inpainting problem where only a part of a signal is measured and the goal is to reconstruct the entire signal. Assuming that the signal varies smoothly in the vertex domain, we can formulate the inverse problem as:
\begin{align}\label{Eq:rec_smooth}
\mathop{\rm argmin}_x  x^T \L x \hspace{0.25cm} \text{ s. t. } \hspace{0.25cm} y = Mx,
\end{align}
where $y$ is the observed signal, $M$ the inpainting masking operator and $x^T \L x$ the graph Tikhonov regularizer ($\L$ being the Laplacian). In order to generate the original signal, we filter Gaussian noise on the graph with a low pass kernel $\hat{h}$. The frequency content of the resulting signal will be close to the shape of the filter $\hat{h}$. For this example, we use the low pass kernel $\hat{h}(x) = \frac{1}{1+\frac{100}{\lmax} x}$ to generate the smooth signal.
%we use a heat kernel as the filtering kernel.

For a given number of measurements, %the goal is to obtain the best recovery as possible. T
the traditional idea is to randomly sample the graph. Under that strategy, the measurements are distributed across the network.
%will distributed everywhere thus covering every part of the network. 
Alternatively, we can use our local uncertainty principles to create an adapted mask. The intuitive idea %is the following, 
that nodes with less uncertainty (higher local sparsity values) should be sampled with higher probability because their value can be inferred less easily from other nodes. 
%In other words, if a signal is smooth on the graph, it cannot be very concentrated on the part with a lot of uncertainty. 
Another way to picture this fact is the following. Imagine that we want to infer a quantity over a random sensor network. In the more densely populated parts of the network, the measurements are more correlated and redundant. As result, a lower sampling rate is necessary. On the contrary, in the parts where there are fewer sensors, the information has less redundancy and a higher sampling rate is necessary.
%In fact, uncertainty arises from the fact that two representations are incoherent. This incoherence allows to recover a signal with a low number of measurements in one domain with an assumption into the other. For instance, in compressed sensing, the signal is assumed to be sparse in one domain. 
%\nati{
%To probe our local uncertainty, we use the low frequency kernel %$\hat{g}(x) = e^{-\tau x}$
%$\hat{g}(x) = \frac{1}{1+\frac{100}{\lmax} x}$ to generate low smooth signal. {\color{red} Is someone going to complain that you are cheating by using the same exact filter? \nati{How about now?}} Theorem \ref{theo:local_atom_overlap} gives us a measure of local uncertainty. 
The heat kernel $\hat{g}(x)=e^{-\tau x}$ is a %good atom 
convenient choice to probe the local uncertainty of a graph, %. Indeed, 
because $\widehat{g^{2}}(x)=e^{-2\tau x}$ is also a heat kernel, resulting in a sparsity level depending only on $\|\T_{j}g^{2}\|_{2}$. Indeed we have $\|\T_{j}g^{2}\|_{1}=\sqrt{N}$. The local uncertainty bound of Theorem \ref{theo:local_atom_overlap}  becomes:
\[
O_{1}(j)=\frac{\|\T_{j}g^{2}\|_{1}}{\|\T_{j}g^{2}\|_{2}}=\frac{\sqrt{N}}{\|\T_{j}g^{2}\|_{2}}.
\]
% We observe that it evolves across the graph nodes with $\frac{\sqrt{N}}{\|\T_ig^2\|_2}$, where $\T_ig^2$ is the localized kernel $g^2$ at node $i$. 
Based on this measure, we design a second random sampled mask with a probability proportional to %the inverse of %the uncertainty: 
$\|\T_ig^2\|_2$; % {\color{red} why squared?} \nati{Fixed}
that is, the higher the overlap level at vertex $j$, the smaller the probability that vertex $j$ is %to be 
chosen as a sampling point, and vice-versa. For each sampling ratio, we performed $100$ experiments and averaged the results. For each experiment, we also randomly generated new graphs. 
The experiment was carried out using open-source code: the UNLocBoX~\cite{perraudin2014unlocbox} and the GSPBox~\cite{perraudin2014gspbox}.
Figure \ref{fig:introinpainting} presents the result of this experiment for a sensor graph and a community graph.
%}
In the sensor graph, we observe that our local measure of uncertainty varies smoothly on the graph and is higher in the more dense part. Thus, the likelihood of sampling poorly connected vertices is higher than the likelihood of sampling well connected vertices. %more vertices are sampled 
In the community graph, we observe that the uncertainty is highly related to the size of the community. The larger the community, the larger the uncertainty (or, equivalently, the smaller the local sparsity value). In both cases, the adapted, non-uniform random sampling performs better than random uniform sampling. 

\begin{figure}[ht!]
\centering
%\begin{minipage}[b]{.24\linewidth}
%\centerline{\small{Local uncertainty}}
%\centerline{\small{(Sensor network)}}
%\centerline{\includegraphics[width=\linewidth]{figures/adapted_sampling_uncertainty_sensor.png}} 
%\centerline{\small{(a)}}
%\end{minipage}
%\hfill
%\begin{minipage}[b]{.24\linewidth}
%\centerline{\small{Local uncertainty}}
%\centerline{\small{(Community graph)}}
%\centerline{\includegraphics[width=\linewidth]{figures/adapted_sampling_uncertainty_cummunity.png}} 
%\centerline{\small{(b)}}
%\end{minipage}
%\hfill
\begin{minipage}[b]{.24\linewidth}
\centerline{\small{$\|T_ig\|_2 \propto $ sampling distribution  }}
\centerline{\small{(Sensor network)}}
\centerline{\includegraphics[width=\linewidth]{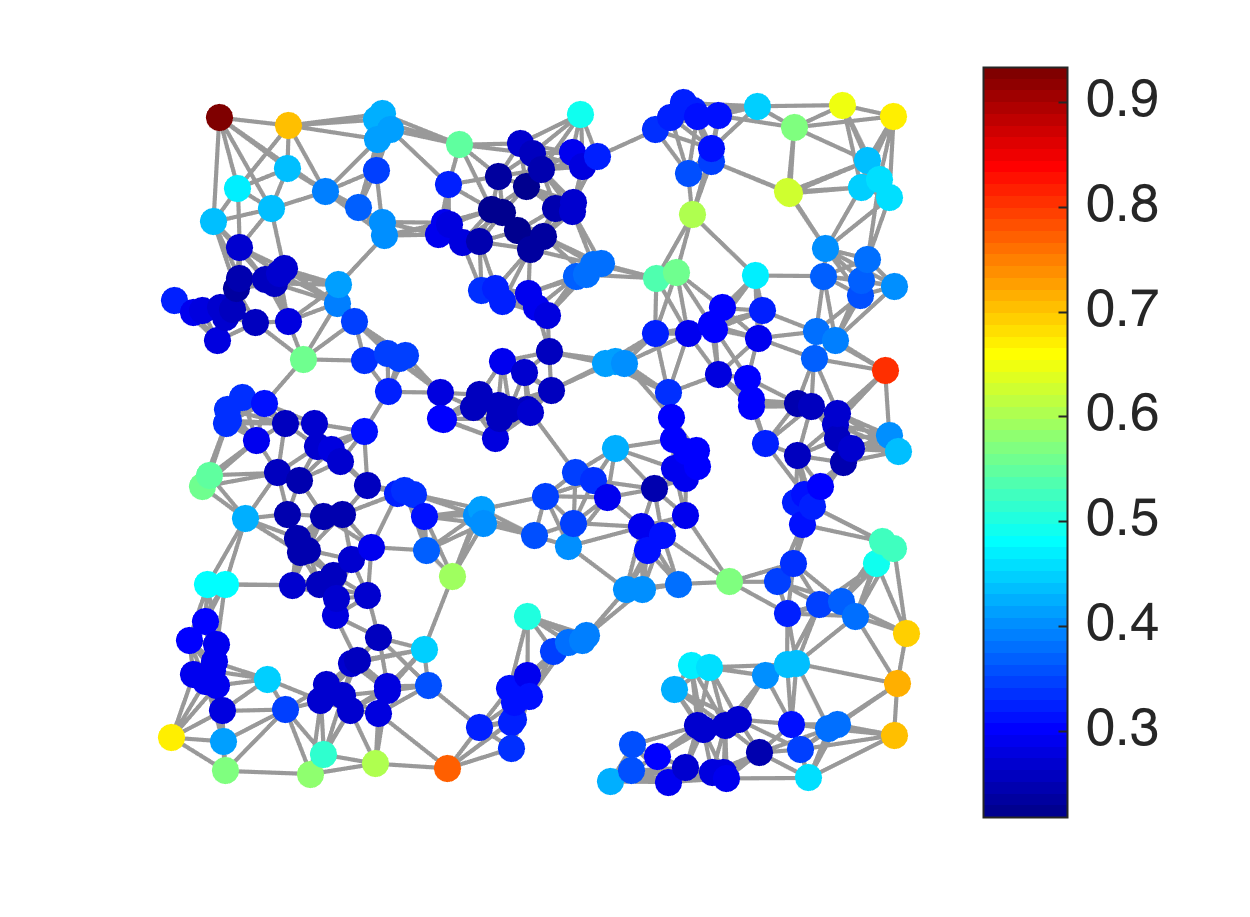}} 
\centerline{\small{(a)}}
\end{minipage}
\hfill
\begin{minipage}[b]{.24\linewidth}
\centerline{\small{$\|T_ig\|_2 \propto $ sampling distribution  }}
\centerline{\small{(Community graph)}}
\centerline{\includegraphics[width=\linewidth]{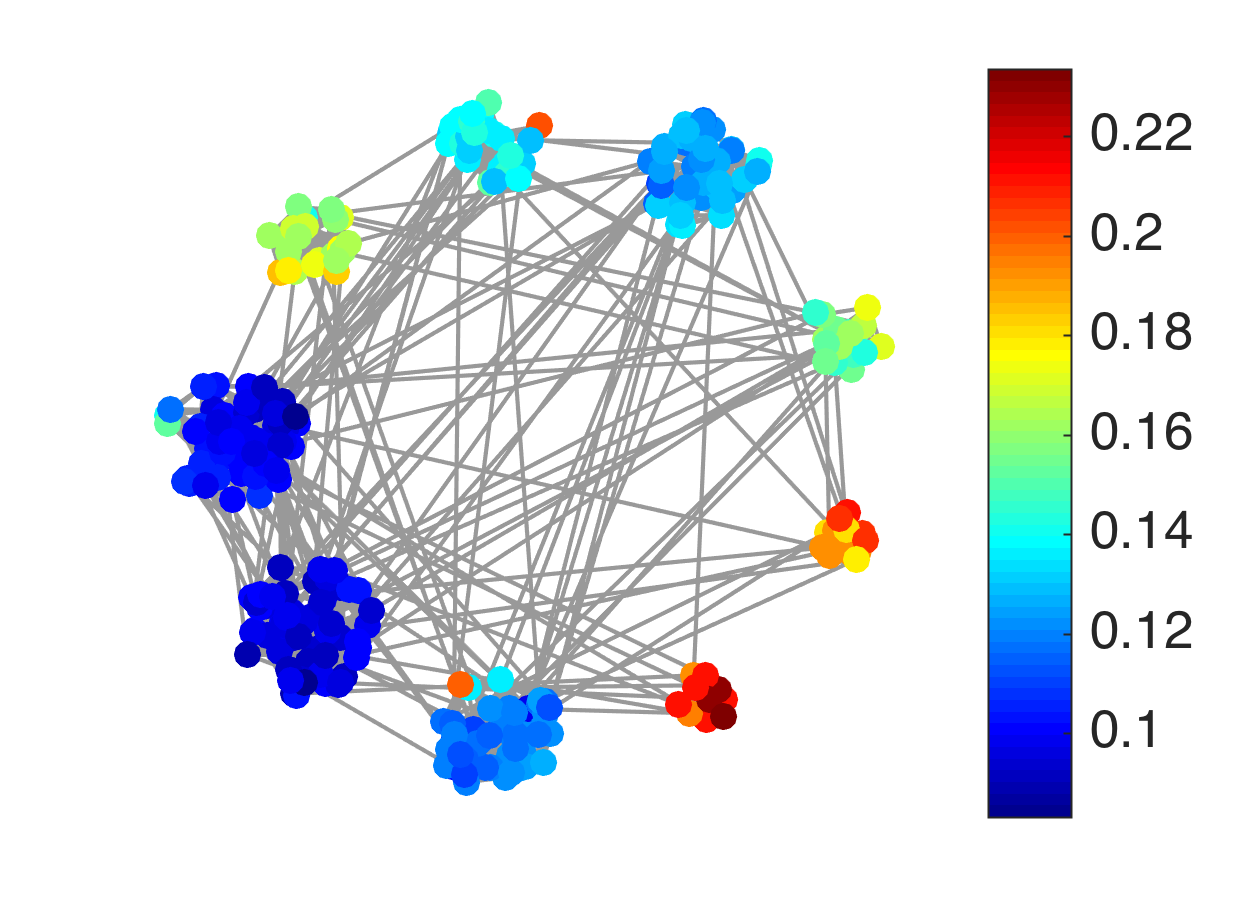}} 
\centerline{\small{(b)}}
\end{minipage}
%\hfill
%\\
\vspace{.1in}
\begin{minipage}[b]{.24\linewidth}
\centerline{\small{Reconstruction error}}
\centerline{\small{(Sensor network)}}
\centerline{\includegraphics[width=\linewidth]{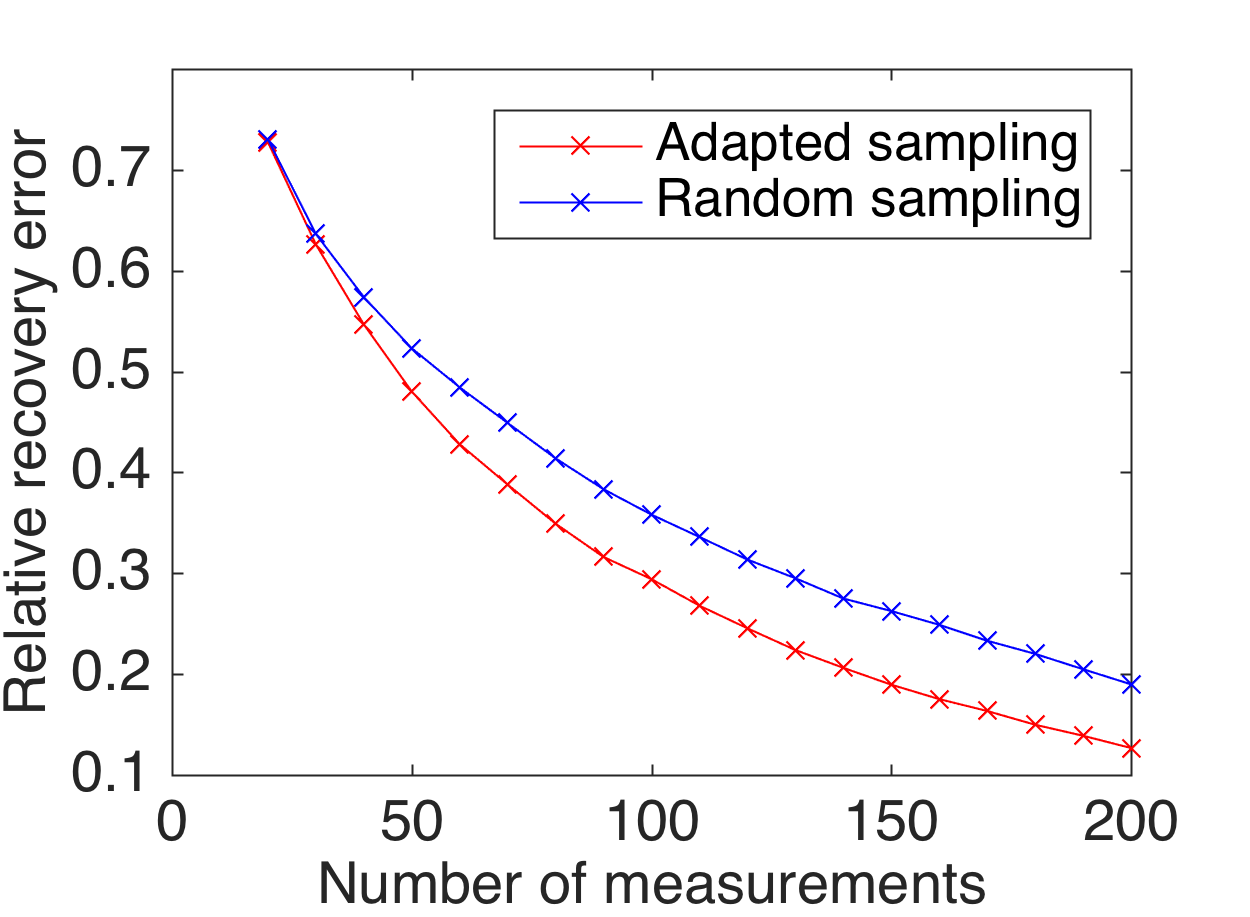}}
\centerline{\small{(c)}}
\end{minipage}
\hfill
\begin{minipage}[b]{.24\linewidth}
\centerline{\small{Reconstruction error}}
\centerline{\small{(Community graph)}}
\centerline{\includegraphics[width=\linewidth]{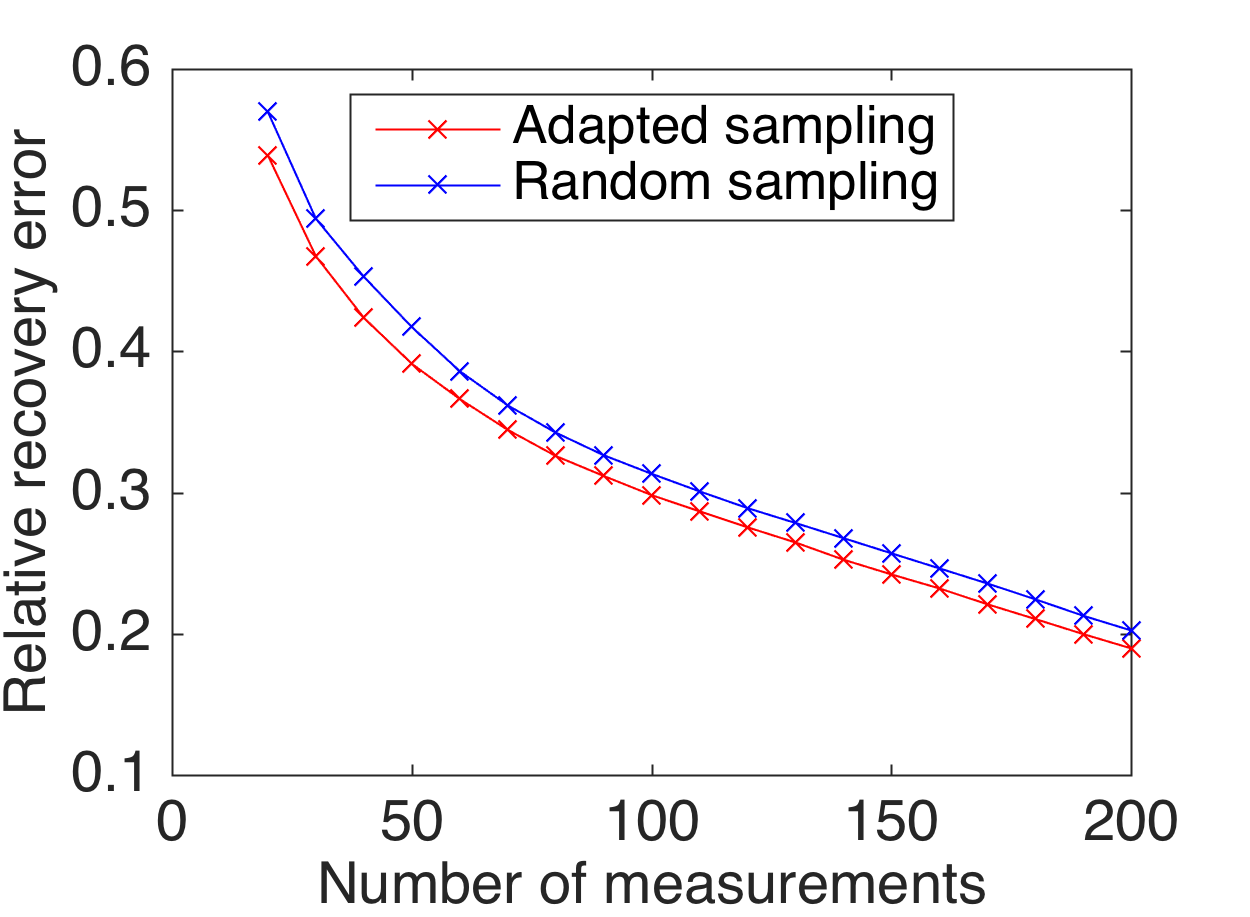}} 
\centerline{\small{(d)}}
\end{minipage}
\hfill
\\
\vspace{.1in}
\begin{minipage}[b]{.19\linewidth}
\centerline{\small{Smooth signal}}
\centerline{\includegraphics[width=\linewidth]{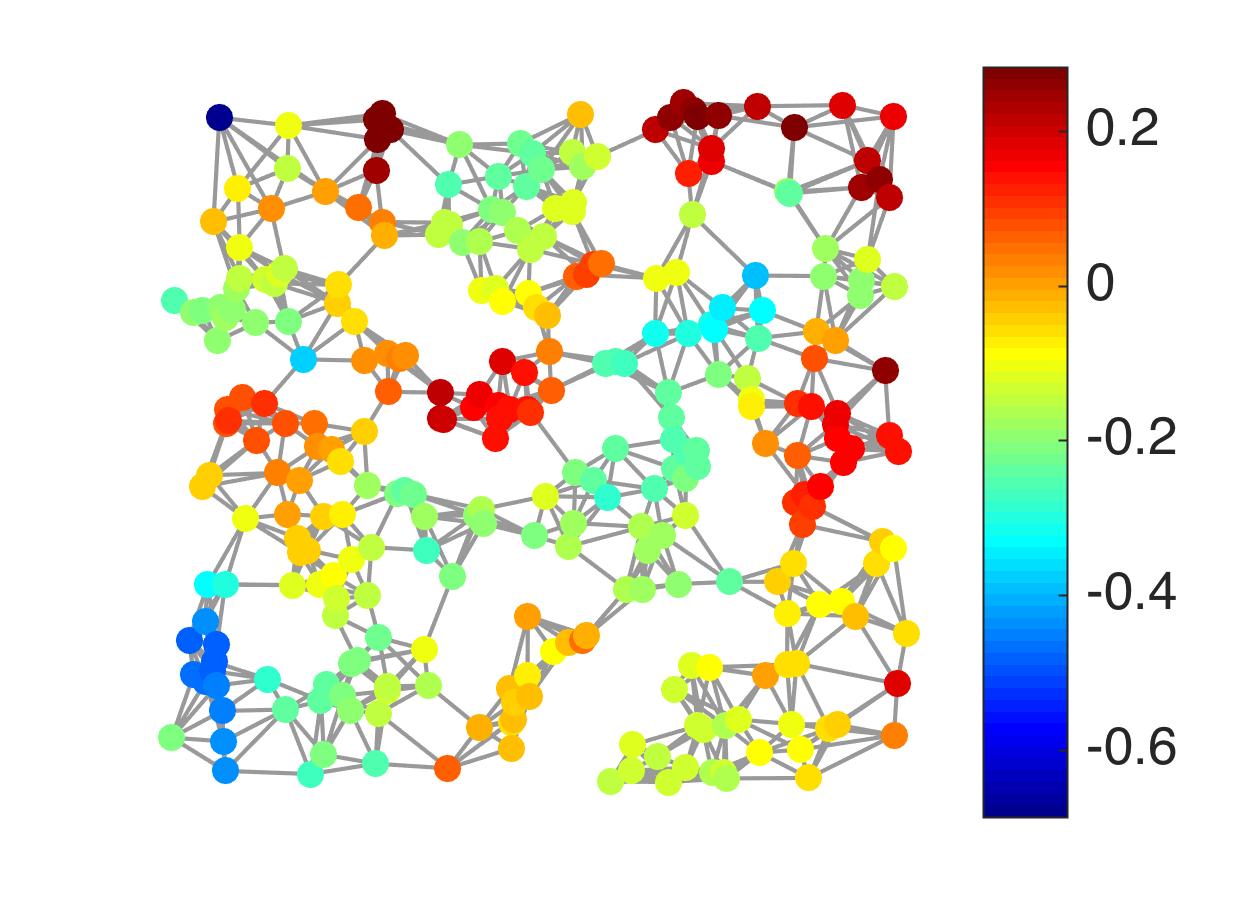}} 
\centerline{\small{(e)}}
\end{minipage}
\hfill
\begin{minipage}[b]{.19\linewidth}
\centerline{\small{Sample locations}}
\centerline{\small{(Uniform sampling)}}
\centerline{\includegraphics[width=\linewidth]{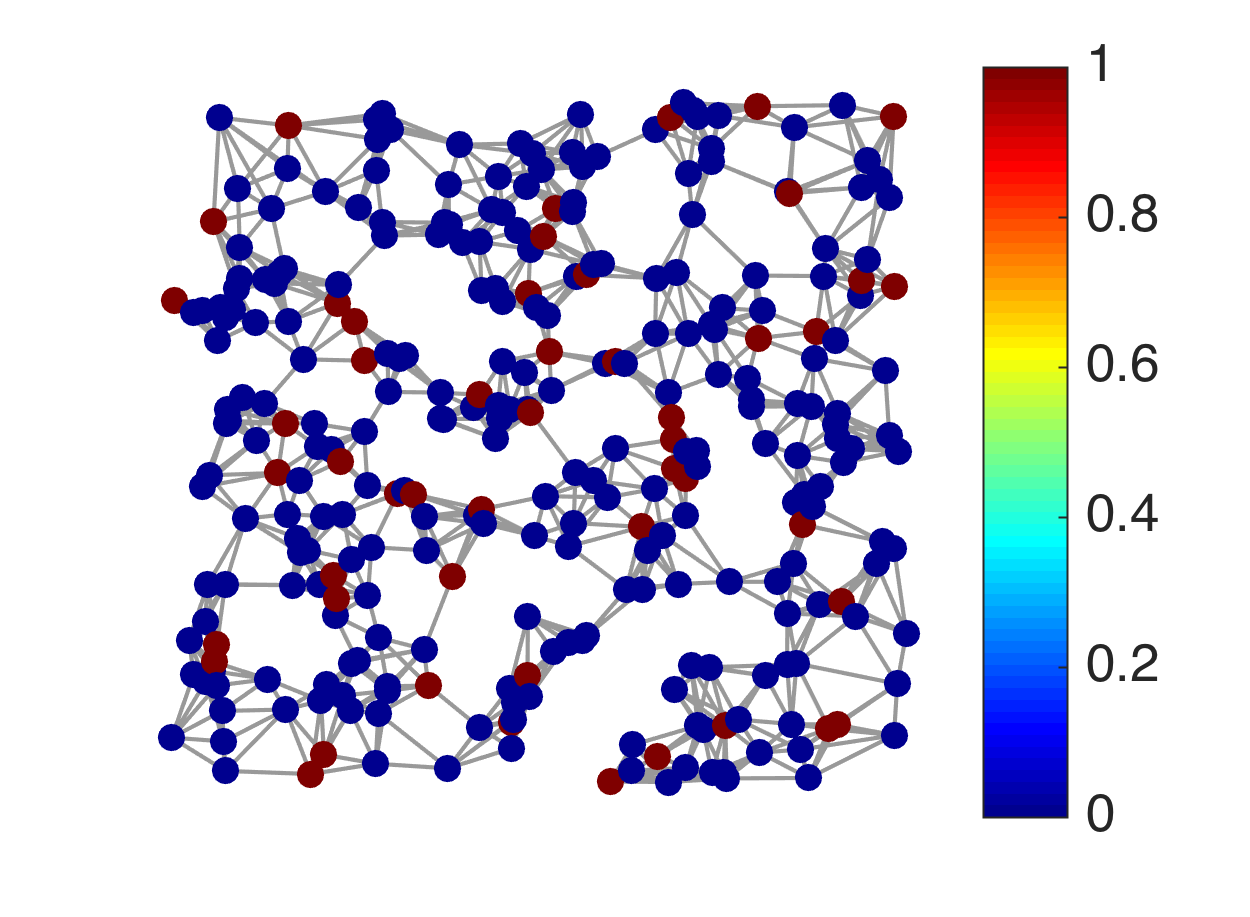}} 
\centerline{\small{(f)}}
\end{minipage}
\hfill
\begin{minipage}[b]{.19\linewidth}
\centerline{\small{Sample locations}}
\centerline{\small{(Non-uniform sampling)}}
\centerline{\includegraphics[width=\linewidth]{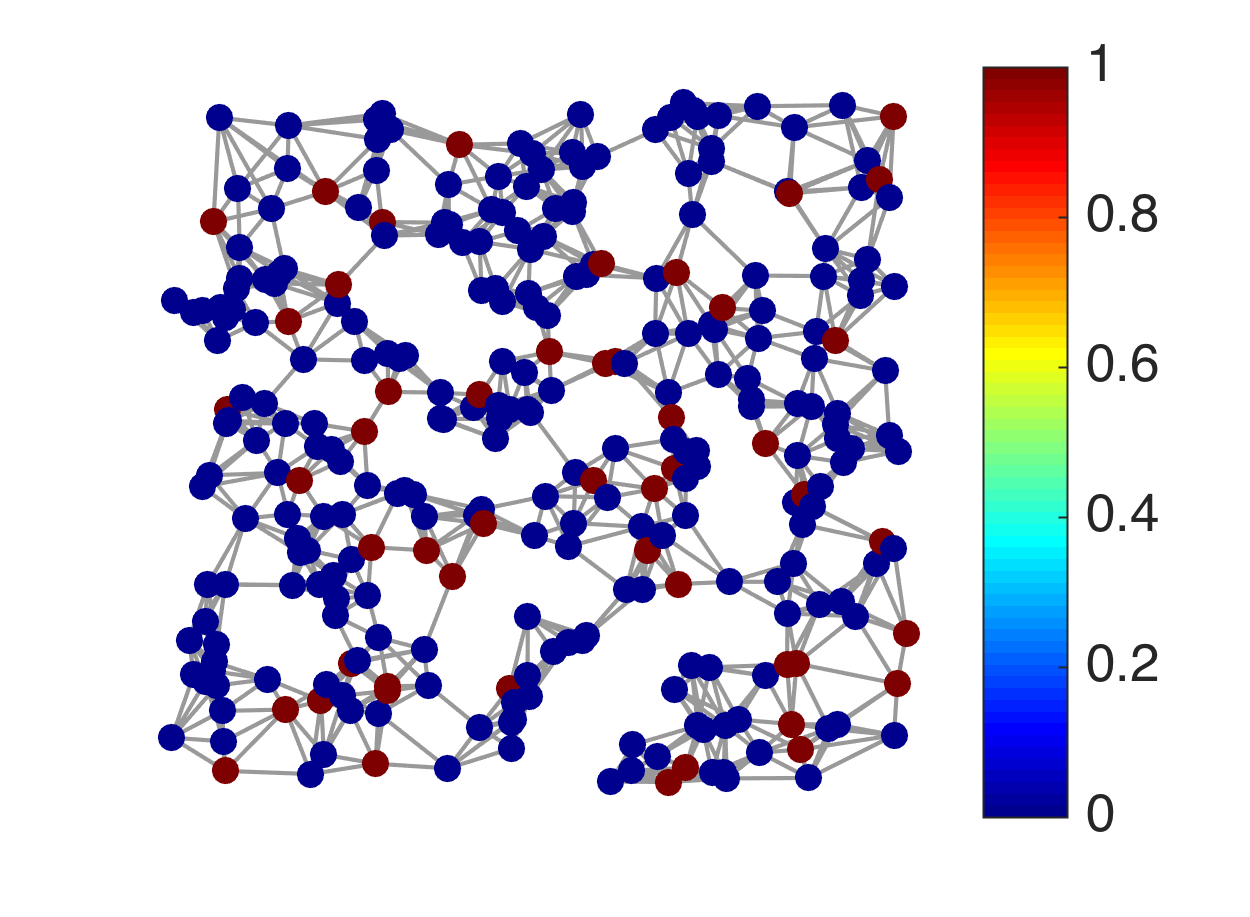}} 
\centerline{\small{(g)}}
\end{minipage}
\hfill
\begin{minipage}[b]{.19\linewidth}
\centerline{\small{Reconstruction}}
\centerline{\small{(Uniform sampling)}}
\centerline{\includegraphics[width=\linewidth]{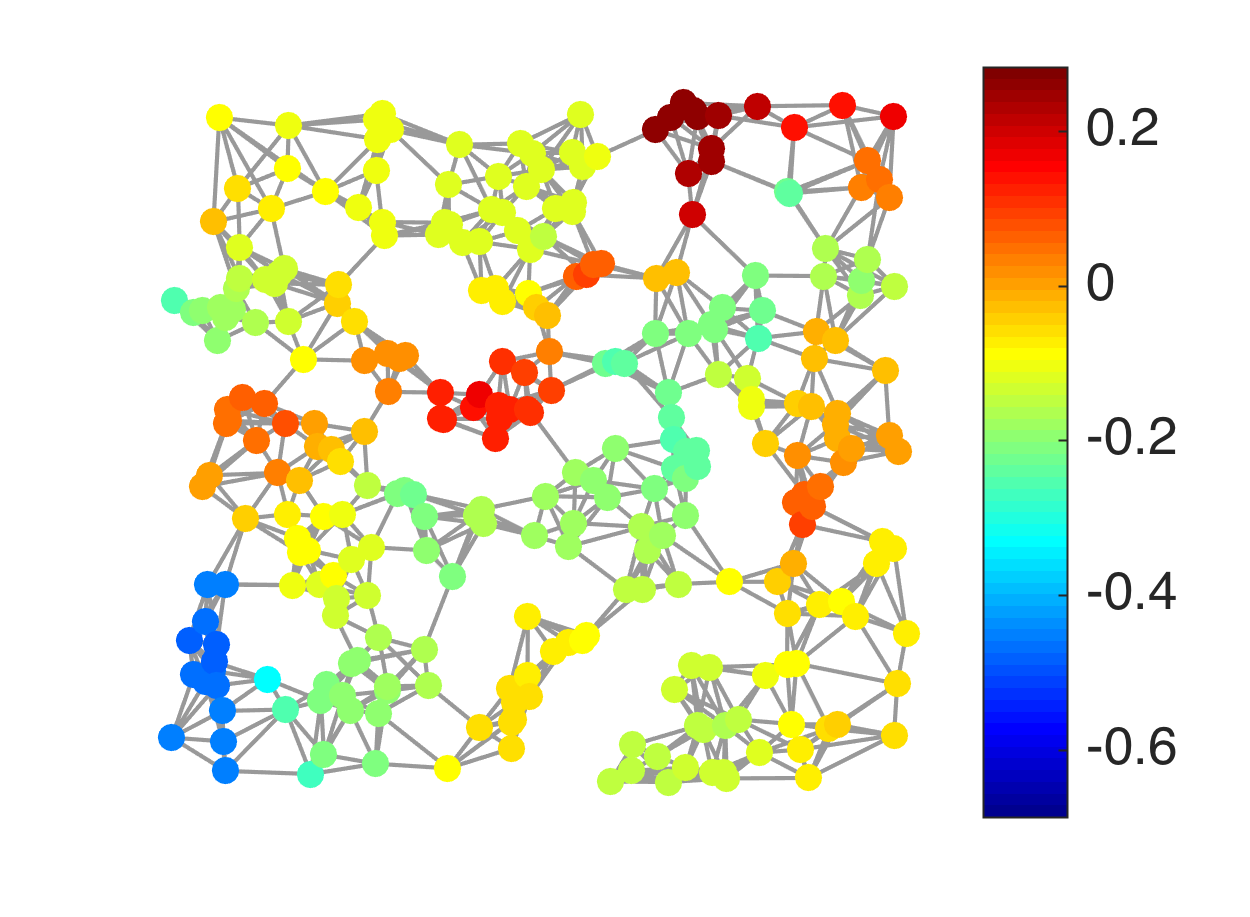}} 
\centerline{\small{(h)}}
\end{minipage}
\hfill
\begin{minipage}[b]{.19\linewidth}
\centerline{\small{Reconstruction}}
\centerline{\small{(Non-uniform sampling)}}
\centerline{\includegraphics[width=\linewidth]{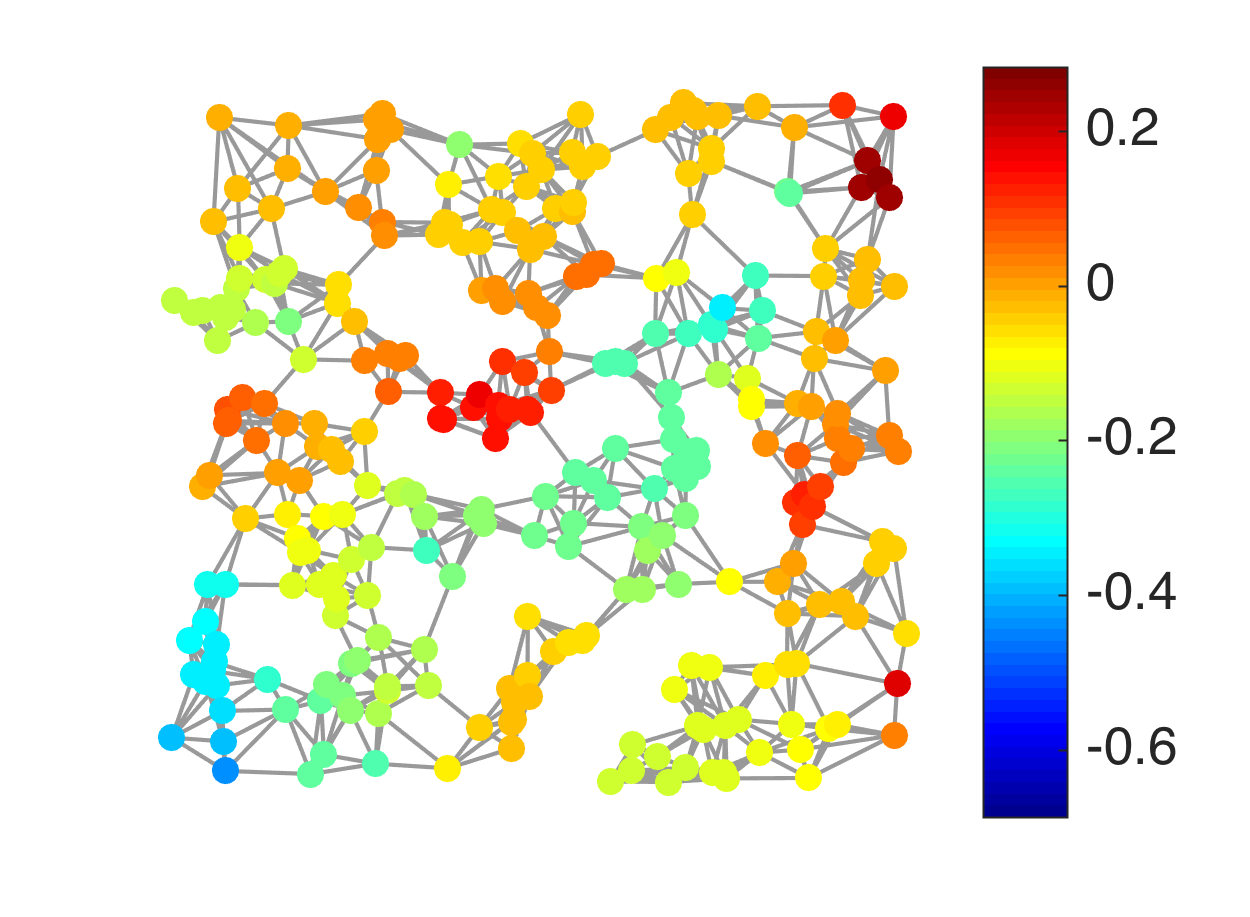}} 
\centerline{\small{(i)}}
\end{minipage}
\caption {Comparison of random uniform sampling and random non-uniform sampling according to a distribution based on the local sparsity values. Top row: (a)-(b) The random non-uniform sampling distribution is proportional to $\|T_ig\|_2$ (for different values of $i$), which is shown here for 
%\nati{$\|T_ig\|_2$ proportional to the sampling distribution} on 
sensor and community graphs with 300 vertices. (c)-(d) the errors resulting from using the different sampling methods on each graph, with the reconstruction in \eqref{Eq:rec_smooth}. 
Bottom row: an example of a single inpainting experiment. (e) the smooth signal, (f)-(g) the locations selected randomly according to the uniform and non-uniform sampling distributions, (h)-(i) the reconstructions resulting from the two different sets of samples.
%(a) \nati{ Top left: signal recovery error with respect of the number of measurements. Top right: visualization of the overlap level uncertainty (the inverse of the local uncertainty values) for one randomly generated sensor graph and one community graph. Both graphs have $300$ nodes. %Each part of the graph has its own uncertainty level. 
%Note how the community graph has different overlap levels in each community. Using adapted random sampling based on the uncertainty level allows us to greatly improve the reconstruction error. 
%Bottom: A signal inpainting experiment. Visually, the signal is better reconstructed.} 
%{\color{red} To do: change sensor network in (a) to match the ones on bottom row}
}
%(1) remove titles from png files, (2) change first two to the actual distributions rather than uncertainty levels? (3) add locations selected by uniform random for constrast.}\nati{Done, I did also adapt the legend.}} 
\label{fig:introinpainting}
\end{figure}

\end{example}

Other works are also starting to use uncertainty principles to develop sampling theory for signals on graphs. In \cite{puy2015random}, the cumulative coherence is used to optimize the sampling distribution. This can be seen as sampling proportionally to $\|T_ig\|_2^2$, where $\hat{g}$ is a specific rectangular kernel, in order to minimize the cumulative coherence of band-limited signals. 
In \cite{tsitsvero2015signals}, Tsitsvero et al. make a link between uncertainty and sampling to obtain a non-probabilistic sampling method.
%They assume a different setting that allows them to use previous samples to optimize the next ones.
While %in the current paper 
non-uniform random sampling is only an illustrative example in this paper, we are currently working on a separate contribution that uses our uncertainty theory to optimize sampling.

\section{Conclusion}
The global uncertainty principles  discussed in Section \ref{Se:global_unc_princ} may be less informative when applied to signals residing on inhomogeneous graphs, because the structure of a specific area of the graph can affect global quantities such as the coherence $\mu_{\G}$, which play a key role in the uncertainty bounds. Our main contribution was to suggest a new way of considering uncertainty by incorporating a notion of locality; specifically, we focused on the concentration of the analysis coefficients under a linear transform whose dictionary atoms are generated by localizing kernels defined in the graph spectral domain to different areas of the graph. The equivalent physical approach would be to say that the uncertainty on the measurements depends on the medium where the particle is located. 
Comparing the first inequality in \eqref{eq:localuncertainty} from the local uncertainty Theorem \ref{theo:local_uncertainty} with the first inequality in \eqref{Eq:Llg1a} from the global uncertainty Theorem \ref{Co:Lieblocgraph}, we see that the latter global bound can be viewed as the maximum of the local bounds over all regions of the graph and all regions of the spectrum.\footnote{The leading constants in the middle terms of \eqref{Eq:Llg1a} and 
\eqref{eq:localuncertainty} are equal for $p>2$. When $p<2$,  
% Theorm \ref{theo:local_uncertainty} is a generalization of Theorem \ref{Co:Lieblocgraph} for $p>2$. However, when $p<2$,
%we have for the global bound $\frac{B^{\frac{1}{2}}}{A\frac{1}{p}}$
%and for the local bound 
%\[
%\frac{B^{1-\frac{1}{p}}}{A^{\frac{1}{2}}}=\frac{B^{\frac{1}{2}}}{A^{\frac{1}{p}}}\cdot\frac{B^{\frac{1}{2}-\frac{1}{p}}}{A^{\frac{1}{2}-\frac{1}{p}}}=\frac{B^{\frac{1}{2}}}{A^{\frac{1}{p}}}\cdot\left(\frac{B}{A}\right)^{\frac{1}{2}-\frac{1}{p}}
%\]
there is a constant factor $\left(\frac{B}{A}\right)^{\frac{1}{2}-\frac{1}{p}}\geq1$ between the two bounds. This factor is equal to $1$ in the case of a tight frame ($A=B$).}
% We haven't find a way to improve the local bound for the case $p<2$.}
% {\color{red} [Any intuition as to why it is $\frac{B^{\frac{1}{p}}}{A^{\frac{1}{2}}}$ in \eqref{eq:localuncertainty} and  $\frac{B^{\min\{\frac{1}{2},\frac{1}{p}\}}}{A^{\max\{\frac{1}{2},\frac{1}{p}\}}}$ in  \eqref{Eq:Llg1a}?]}\nati{ I've left a note after the proof.} 
This supports our view that the benefit of the global uncertainty principle is restricted to the behavior in the region of the graph with the least favorable structure. The local uncertainty principle, on the other hand, provides information about each region of the graph separately.

The key quantities  $\{||T_i g_k||_2\}_{i,k}$ appear in both the global and local uncertainty principles. While we know that smoother kernels $\widehat{g_k}$ lead to atoms of the form $T_i g_k$ being more concentrated in the vertex domain, further study of the norms of these atoms is merited, as they seem to carry some notions of both uncertainty and centrality. 

Finally, we showed in Example \ref{Ex:adaptated_sampling} how this local notion of uncertainty can be used constructively in the context of a sampling and interpolation experiment. The uncertainty quantities suggest to sample non-uniformly, often with higher weight given to less connected vertices. We envision future work applying these local uncertainty principles to other signal processing tasks, as well as extending the notion of local uncertainty to other types of dictionaries for graph signals.

\section*{Acknowledgment}
This work has been supported by the Swiss National Science Foundation research project \textit{Towards Signal Processing on Graphs}, grant number: 2000\_21/154350/1.

%%%%%%%%%%%%%%%%%%%%%%%%%%%%%%%%%%%
\section{Appendix}

%%%%%%%%%%%%%%%%%%%%%%%%%%%%%%%%%%%%%%%%%%%
\subsection{Hausdorff-Young inequalities for graph signals}
\label{sec:Rieszthorin}
%In the following we recall and establish in the graph case some key results used to compute spreading and uncertainty relations. The first one is the Riesz-Thorin theorem which is only recalled here as it is independent of the underlying domain space. It is a crucial result which is involved in all of the $\ell^p$-norms and entropic uncertainty principles.
%The second is the Hausdorff-Young inequality for graphs which relates $\ell^p$-norms of a function and its graph Fourier transform. Since the definition of the Fourier transform has been adapted to the graph setting, the bound is different from the standard one.
%%The third inequality is Young's one, concerning the convolution product. We define a generalization of the convolution product for the graph setting in~\eqref{defconvolution} and it is necessary to demonstrate its validity.  

%%     Riesz-Thorin interpolation theorem
%%%%%%%%%%%%%%%%%%%%%%%%%%%%%%%%%%%%%%%%%%%
%\subsection{Riesz-Thorin interpolation}
To prove the Hausdorff-Young inequalities for graph signals, we start by restating the Riesz-Thorin interpolation theorem, which can be found in~\cite[Section IX.4]{reed1975methods}. This theorem is valid for any measure spaces with $\sigma$-finite measures, and hence in the finite dimensional case. 
\begin{theorem}[Riesz-Thorin]\label{theo:Riesz-Thorin}
Assume $\TT$ is a bounded linear operator from $\ell^{p_1}$ to $\ell^{p_2}$ and  from $\ell^{q_1}$ to $\ell^{q_2}$; i.e., there exist constants $M_p$ and $M_q$ such that
$$\| \TT f \|_{p_2} \leq M_p \| f \|_{p_1} \hbox{~~and~~~} \| \TT f \|_{q_2} \leq M_q \| f \|_{q_1}. $$ 
Then for any $t$ between 0 and 1, $\TT$ is also a bounded operator from $\ell^{r_1}$ to $\ell^{r_2}$ :
$$\| \TT f \|_{r_2} \leq M_r\| f \|_{r_1},$$  
with
\begin{equation*}
\frac{1}{r_1}= \frac{t}{p_1}+\frac{1-t}{q_1}, \hspace{1cm} \frac{1}{r_2}= \frac{t}{p_2}+\frac{1-t}{q_2},
\end{equation*}
and
\begin{equation*}
M_r = M_p^t M_q^{1-t}.
\end{equation*}
%and $t$ is any number between $0$ and $1$.
\end{theorem}

We shall also need the following reverse form of the result: 
\begin{corollary} \label{theo:Riesz-Thorin-converse}
Assume $\TT$ is a bounded invertible linear operator  from $\ell^{p_1}$ to $\ell^{p_2}$ and from $\ell^{q_1}$ to $\ell^{q_2}$, with bounded left-inverse from $\ell^{p_2}$ to $\ell^{p_1}$ and from $\ell^{q_2}$ to $\ell^{q_1}$; i.e., there exist constants $N_p$ and $N_q$ such that 
\begin{align}\label{Eq:bop1}
\| \TT^{-1} g \|_{p_1} \leq N_p \| g \|_{p_2} \hbox{~~and~~~} \| \TT^{-1}g \|_{q_1} \leq N_q \| g \|_{q_2}, 
\end{align}
or, equivalently, there exist constants $M_p$ and $M_q$ such that 
\begin{align}\label{Eq:bop2}
\| \TT f \|_{p_2} \geq M_p \| f \|_{p_1} \hbox{~~and~~~}\| \TT f \|_{q_2} \geq M_q \| f \|_{q_1}. 
\end{align} 
Then for any $t$ between 0 and 1, 
\begin{align}\label{Eq:rtcor}
\| \TT f \|_{r_2} \geq M_r\| f \|_{r_1},
\end{align}
with
\begin{equation*}
\frac{1}{r_1}= \frac{t}{p_1}+\frac{1-t}{q_1}, \hspace{1cm} \frac{1}{r_2}= \frac{t}{p_2}+\frac{1-t}{q_2},
\end{equation*}
and
\begin{equation*}
M_r = M_p^t M_q^{1-t}.
\end{equation*}
\end{corollary}
\begin{proof}
If $\TT$ is invertible and has a left-inverse $\TT^{-1}$ that satisfies $\TT^{-1}\TT f=f$ for all $f$, then the equivalence of \eqref{Eq:bop1} and \eqref{Eq:bop2} follows from taking $g=\TT f$, $f=\TT^{-1}g$, $M_p=N_p^{-1}$, and $M_q=N_q^{-1}$. The
proof of \eqref{Eq:rtcor} follows from the application of Theorem \ref{theo:Riesz-Thorin},
 with $\TT$ replaced by $\TT^{-1}$ and $f$ by $\TT f$.
\end{proof}

%%%%%%%%%%%%%%%%%%%%
%\subsection{Hausdorff-Young theorem}

\begin{proof}[Proof of Theorem \ref{theo: Hausdorff-Young graph} (Hausdorff-Young inequalities for graph signals)]
First, we have the Parseval equality $\| f \|_2^2 = \| \hat{f} \|_2^2$.
Second, we have
\begin{eqnarray*}
\| \hat{f} \|_\infty =  \max_\ell \left| \sum_{n=1}^N \x_\ell^*(n) f(n) \right| \leq  \max_\ell  \sum_{n=1}^N \left| \x_\ell^*(n) f(n) \right| \leq  \mu_{\G} \sum_{n=1}^N | f(n) | = \mu_{\G} \| f \|_1 . \label{premice2}
\end{eqnarray*}
Applying the Riesz-Thorin theorem with $p_1=2$, $p_2=2$, $M_p= 1$, $q_1=1$, $q_2=\infty$,  $M_q=\mu_{\G}$, $t=\frac{2}{q}$, $r_1=p$, and $r_2=q$ leads to the first inequality \eqref{eq: HY 1}. The proof of the converse is similar, as we have
\begin{eqnarray*}
\| f\|_\infty =  \max_i \left| \sum_{\ell=0}^{N-1} \x_\ell(i) \hat{f}(\ell) \right| \leq  \max_i  \sum_{\l=0}^{N-1} \left| \x_\ell(i) \hat{f}(\l) \right| \leq  \mu_{\G} \sum_{\l=0}^{N-1} | \hat{f}(\l) | = \mu_{\G} \|\hat{f} \|_1. 
\end{eqnarray*}
The graph Fourier transform is invertible, so  \eqref{eq: HY 2}  then follows from Corollary~\ref{theo:Riesz-Thorin-converse}, with $p_1=\infty$, $p_2=1$, $M_p= \mu_{\G}^{-1}$, $q_1=2$, $q_2=2$,  $M_q=1$, $t=\frac{2}{q}-1$, $r_1=p$, and $r_2=q$.
\end{proof}

\subsection{Variations of Lieb's uncertainty principle}

\subsubsection{Generalization of Lieb's uncertainty principle to frames} \label{Se:lieb_gen_proof}

%\subsection{Filter-banks based uncertainty principle}
%Computing the $p$-norm of the ambiguity function is a way of evaluating the spread of a signal in the vertex and frequency domain, as well as its localization in these domains. 
%Uncertainty principles for the ambiguity function have been established in the continuous case~\cite{lieb1990integral} but not in the graph case.
%In order to establish an uncertainty principle for the cross-ambiguity function of Def.~\ref{def:analysisFB} we first need a result giving a condition for the set of vectors $\{\T_ig_k=g_{i,k}\}$ to be a frame. The proof is given in the Appendix.
%\begin{lemma} \label{lemma: vertex frequency frame}
%Let $\{\hat{g}_k\}_{k}$ be a set of $M\ge 1$ bounded continuous functions on $[0,\lmax]$. The set $\{g_{i,k}\}_{i,k}$ is a frame if the $\{\hat{g}_k\}$ are such that $\sum_{k=1}^{M}\left|\hg_{k}(\lambda_{\ell})\right|^{2}>0$ for all $\lambda_{\ell}$ in the spectrum of the graph Laplacian. The frame bounds are
%$$
%A\|f\|_{2}^{2}\leq\|\mathcal{B}_{g}f\|_{2}^{2}\leq B\|f\|_{2}^{2}
%$$
%with $A=\min_{\ell}\sum_{k=1}^{M}\left|\hg_{k}(\lambda_{\ell})\right|^{2}$ and $B=\max_{\ell}\sum_{k=1}^{M}\left|\hg_{k}(\lambda_{\ell})\right|^{2}$.
%\end{lemma} 

\begin{proof}[Proof of Theorem \ref{theo: amb tight}]
Let  $\Dc=\{g_{i,k}\}$ be a frame of atoms in $\Cbb^N$, with lower and upper frame bounds $A$ and $B$, respectively. We show the following two inequalities, which together yield \eqref{Eq:ft1a}. First, for any signal $f\in \Cbb^N$ and any $p\ge 2$,
\begin{align}\label{Eq:wtsft1a}
s_p(\A_{\Dc}f)=\frac{\|\A_{\Dc}f\|_{p}}{\|\A_{\Dc}f\|_{2}}
\leq\frac{B^{\frac{1}{p}}}{A^{\frac{1}{2}}}\left(\max_{i,k}\|g_{i,k}\|_{2}\right)^{1-\frac{2}{p}}
%\leq\frac{B^{\frac{1}{p}}}{A^{\frac{1}{2}}}\left(\sqrt{N}\mu_{\G}  \max_k \|g_k\|_{2}\right)^{1-\frac{2}{p}},
\end{align}
Second, for any signal $f\in \Cbb^N$ and any $1\leq p \leq 2$,
\begin{align}\label{Eq:wtsft1b}
\frac{1}{s_p(\A_{\Dc}f)}=\frac{\|\A_{\Dc}f\|_{p}}{\|\A_{\Dc}f\|_{2}}
\geq\frac{A^\frac{1}{p}}{B^{\frac{1}{2}}}\left(\max_{i,k}\|g_{i,k}\|_{2}\right)^{1-\frac{2}{p}}
%\geq\frac{A^\frac{1}{p}}{B^{\frac{1}{2}}}\left(\sqrt{N}\mu_{\G} \max_k \|g_k\|_{2} \right)^{1-\frac{2}{p}}.
\end{align}
%
%The following theorem is a more general version of Featured theorem \ref{theo: amb tight} applying to any frame.
%\begin{theorem} \label{theo ambiguity bound}
%Let $\{g_{i,k}\}_{i,k}$ be a filter-bank defined on a graph $\G$ with lower and upper frame bounds $A$ and $B$ respectively. Let $f \in \Cbb^N$ be a signal on $\G$. One has for $p\ge 2$
%$$
%\frac{\|\mathcal{B}_{g}f\|_{p}}{\|\mathcal{B}_{g}f\|_{2}}
%\leq\frac{B^{\frac{1}{p}}}{A^{\frac{1}{2}}}\left(\max_{i,k}\|g_{i,k}\|_{2}\right)^{1-\frac{2}{p}}
%\leq\frac{B^{\frac{1}{p}}}{A^{\frac{1}{2}}}\left(\sqrt{N}\mu_{\G}  \max_k \|g_k\|_{2}\right)^{1-\frac{2}{p}},
%$$
%and for $1\leq p \leq 2$
%$$
%\frac{\|\mathcal{B}_{g}f\|_{p}}{\|\mathcal{B}_{g}f\|_{2}}
%\geq\frac{A^\frac{1}{p}}{B^{\frac{1}{2}}}\left(\max_{i,k}\|g_{i,k}\|_{2}\right)^{1-\frac{2}{p}}
%\geq\frac{A^\frac{1}{p}}{B^{\frac{1}{2}}}\left(\sqrt{N}\mu_{\G} \max_k \|g_k\|_{2} \right)^{1-\frac{2}{p}}.
%$$
%\end{theorem}
%\begin{proof}
For any $f$, the frame $\Dc$ satisfies
\begin{equation}
\sqrt{A}\|f\|_{2}\le\|\A_{\Dc}f\|_{2}\le\sqrt{B}\|f\|_{2}.\label{eq: bound norm 2}
\end{equation}
The computation of the sup-norm gives
\begin{equation}\label{eq: bound norm inf}
\|\A_{\Dc}f\|_{\infty}  =  \max_{i,k}\left|\langle f, g_{i,k} \rangle\right| \le \|f\|_{2}\max_{i,k}\|g_{i,k}\|_{2}.
\end{equation}
From \eqref{eq: bound norm 2}, $\A_{\Dc}$ is a linear bounded operator form $\ell_2$ to $\ell_2$ by $\sqrt{B}$. Similarly, from \eqref{eq: bound norm inf}, this operator is also bounded from $\ell_2$ to $\ell_\infty$ by $\max_{i,k}\|g_{i,k}\|_{2}$. Interpolating between $\ell_2$ and $\ell_\infty$ with the Riesz-Thorin theorem leads to
\begin{equation} \label{eq: lieb generalization p 2 inf}
\|\A_{\Dc}f\|_{p}\leq B^{\frac{1}{p}}\left(\max_{i,k}\|g_{i,k}\|_{2}\right)^{1-\frac{2}{p}}\|f\|_{2}.
\end{equation}
We combine~\eqref{eq: bound norm 2} and~\eqref{eq: lieb generalization p 2 inf} to obtain \eqref{Eq:wtsft1a}.
The second inequality \eqref{Eq:wtsft1b} is obtained using the following instance of H\"older's inequality:
$$\|\A_{\Dc}f\|_{2}^{2}\le\|\A_{\Dc}f\|_{\infty}\|\A_{\Dc}f\|_{1},$$
which implies that
\begin{equation}
\|\A_{\Dc}f\|_{1}\geq\frac{\|\A_{\Dc}f\|_{2}^{2}}{\|\A_{\Dc}f\|_{\infty}}\geq\frac{A\|f\|_{2}}{\max_{i,k}\|g_{i,k}\|_{2}}.\label{eq: bound norm 1}
\end{equation}
We then use Corollary~\ref{theo:Riesz-Thorin-converse}, the converse of Riesz-Thorin, to interpolate between \eqref{eq: bound norm 1} and \eqref{eq: bound norm 2}, and we find for $p\in[1,2]$:
\begin{equation}  \label{eq: lieb generalization p 1 2}
\|\A_{\Dc}f\|_{p}\geq A^{\frac{1}{p}}\left(\max_{i,k}\|g_{i,k}\|_{2}\right)^{1-\frac{2}{p}}\|f\|_{2}.
\end{equation}
Combining \eqref{eq: lieb generalization p 1 2} with the second inequality in \eqref{eq: bound norm 2} yields \eqref{Eq:wtsft1b}.
\end{proof}

\subsubsection{Discrete version of Lieb's uncertainty principle} \label{Se:discrete_Lieb}
%{\color{red}
%\begin{itemize}
%\item As a corollary of the frame version (relation to ring graph)
%\item Direct proof (one way only)
%\end{itemize}

\begin{proof}[Proof of Theorem \ref{theo:Classical_Lieb_discrete}]
Theorem \ref{theo:Classical_Lieb_discrete} is actually a particular case of Theorem \ref{theo: amb tight}. To see why,  we need to understand %in detail 
the transformation between the graph framework used in this contribution and the classical discrete periodic case. %generalized by the ring graph. 
The DFT basis vectors $\left\{u_k(n)=\frac{1}{\sqrt{N}}\exp\left( \frac{i2\pi k n}{N}\right)\right\}_{k=0,1,\ldots,N-1}$ can also be chosen as the eigenvectors of the graph Laplacian for a ring graph with $N$ vertices \cite{strang1999discrete}. The frequencies of the DFT, which correspond up to a sign to the inverse of the period of the eigenvectors, are not the same as the graph Laplacian eigenvalues on the ring graph, which are all positive. We can, however, form a bijection between the set $\sigma(\L)$ of graph Laplacian eigenvalues and the set of $N$ frequencies of the DFT, by associating one member from each set sharing the same eigenvector. At this point, instead of considering graph filters as continuous functions evaluated on the Laplacian eigenvalues, we can define a graph filter as a mapping from each individual eigenvalue to a complex number. Note that an eigenvalue with multiplicity $2$ can have two different outputs (e.g., $\lambda_3=\lambda_4=1$, but the filter has different values at $\lambda_3$ and $\lambda_4$). With this bijection and view of the graph spectral domain, we can recover the classical discrete periodic setting by forming a ring graph with $N$ vertices. Because the classical translation and modulation preserve 2-norms, the discrete windowed Fourier atoms of the form
$$g_{u,k}[n]=g[n-u]\exp \left(\frac{i2\pi k n}{N}\right)$$ 
%{\color{red} Check the constant in the atom form above}
all have the same norm $||g||_2$. %\nati{Seems ok for me.}
Together these  $N^2$  atoms
comprise a tight frame on the ring graph with frame bounds $A =B= N \| g \|_2^2$. 
Inserting these values into  \eqref{eq: lieb generalization p 2 inf_main}
and \eqref{eq: lieb generalization p 1 2_main} yields \eqref{eq: lieb p 1 infty} and \eqref{eq: lieb p 1 2}.
\end{proof}

For the case of $p\geq 2$, we also provide an alternative direct proof following similar ideas to those used in Lieb's proof for the continuous case \cite{lieb1990integral}. The arguments below follow the sketch of the proof of Proposition 2 in \cite{eusipcooptamb} and supporting personal communication from Bruno Torr{\'e}sani. 
%{\color{red}Can we double check with Bruno that it is OK to put this argument here and how he would like to be acknowledged?}
We need two lemmas. The first one is a direct application of Theorem~\ref{theo: Hausdorff-Young graph}, where here $\mu_{\G}=1/\sqrt{N}$.
%
%New version:
\begin{lemma} \label{HY_discrete}
Let $f\in \Cbb^N$ and $p$ be the H\"older conjugate of $p'$ ($\frac{1}{p} +\frac{1}{p'} = 1$). Then for 
$1\leq p \leq 2 $, we have
%be the H\"older conjugate of $p'$ ($\frac{1}{p} +\frac{1}{p'} = 1$) and $f\in \Cbb^N$, then
 $$\| \hat f \|_{p'} \leq N^{\frac{1}{p'}-\frac{1}{2}} \| f \|_{p}.$$
 Conversely, for $2 \leq p \leq \infty$, we have
 $$\| \hat f \|_{p'} \geq N^{\frac{1}{p'}-\frac{1}{2}}  \| f \|_{p}.$$
\end{lemma}
%
%Old version:
%\begin{lemma} \label{HY_discreteold}
%Let $ p \geq 2$ be the H\"older conjugate of $p'$ ($\frac{1}{p} +\frac{1}{q} = 1$) and $f\in \Cbb^N$, then
% $$\| \hat f \|_{p'} \leq N^{\frac{1}{p'}-\frac{1}{2}} \| f \|_{p}.$$
% 
% Conversely, let $1\leq p \leq 2$, then 
% $$\| \hat f \|_{p'} \geq N^{\frac{1}{2}-\frac{1}{p'}}  \| f \|_{p}.$$
%\end{lemma}
The second lemma is an equivalent of Young's inequality in the discrete case.
We denote the circular convolution between two discrete signals $f,g$ by $f\ast g$. The circular convolution satisfies $\widehat{f\ast g} = \hat{f}\cdot \hat{g}$. %(multiplication in the Fourier domain). 
\begin{lemma} \label{classical_young}
Let $f\in L^p$, $g\in L^q$, where $1 \leq p,q,r \leq \infty$ satisfy $1+\frac{1}{r} =\frac{1}{p} +\frac{1}{q} $. Then
\begin{equation*}
\| f \ast g \|_r \leq \| f\|_p \| g\|_q.
\end{equation*}
\end{lemma}
\begin{proof}
The proof is based on the following inequalities \cite[p. 174]{pinsky}
\begin{eqnarray}
\| f \ast g \|_1 &\leq& \| f\|_1 \| g\|_1 \label{eq:demo_young_classic111}\\
\| f \ast g \|_\infty &\leq& \| f\|_\infty \| g\|_1 \label{eq:demo_young_classicii1}\\
\| f \ast g \|_\infty &\leq& \| f\|_p \| g\|_{p^{\prime}}, \label{eq:demo_young_classicipq}
\end{eqnarray}
where $\frac{1}{p} +\frac{1}{p^{\prime}} = 1$. For a fixed function $g\in L^q$, we define an operator $\TT_g$ by $(\TT_{g}f)(n) = (f * g)(n)$. Using \eqref{eq:demo_young_classic111} and \eqref{eq:demo_young_classicii1}, we observe that this operator is bounded from $L^1$ to $L^1$ by $\|g\|_1$ and from $L^\infty$ to $L^\infty$ by $\|g\|_1$. Thus, we can apply the Riesz-Thorin theorem to this operator to get
\begin{align}\label{Eq:RT1fg}
\|f * g \|_p \leq \| f \|_p \| g \|_1.
\end{align}
Similarly, for a fixed function $f \in L^p$, we define another operator $T_f$ by $(T_{f}g)(n) = (f * g)(n)$. From \eqref{Eq:RT1fg} and  \eqref{eq:demo_young_classicipq}, we observe that this new operator is bounded from $L^1$ to $L^p$ by  $ \| f \|_p $ and from $L^{p^{\prime}}$ to $L^\infty$ by $\| f \|_p $.
%, where $\frac{1}{p}+\frac{1}{p^{\prime}}=1$. 
%Again the 
One more application of the Riesz-Thorin theorem leads to the desired result:
$$
\| f \ast g \|_r \leq \| f\|_p \| g\|_q,
$$
where $1+\frac{1}{r} =\frac{1}{p} +\frac{1}{q} $.
\end{proof}

%The ambiguity function uncertainty principle in the discrete case is the following.
\begin{proof}[Alternative proof of Theorem \ref{theo:Classical_Lieb_discrete} for the case $p\geq 2$]
Suppose $p>2$ and let $\frac{1}{p}+\frac{1}{p^{\prime}}=1$. We denote the DFT by $\F$. Noting that 
%For the proof we use the following convention and $\F$ denotes the Fourier transformation.then 
$\frac{p}{p^{\prime}}>1$, we have
\begin{eqnarray}
\| \A_{\Dc_{DWFT}} f \|_p^p 
& = & \sum_{u=1}^N \sum_{k=0}^{N-1} | \A_{\Dc_{DWFT}} f [u,k]|^p \nonumber \\
%&= & \sum_{u=1}^N \sum_{k=0}^{N-1} \left| \sum_{n=1}^N f[n] \overline{g}[u-n] e^{-2 \pi i k \frac{n}{N}} \right|^p \nonumber \\
&= & N^\frac{p}{2} \sum_{u=1}^N \sum_{k=0}^{N-1} \left| \F\{ f[\cdot] g[u-\cdot]\}[k] \right|^p  \nonumber \\
&= & N^\frac{p}{2}  \sum_{u=1}^N \left\| \F\{ f[\cdot] g[u-\cdot]\} \right\|_p^p \nonumber \\
&\leq & N^\frac{p}{2} \sum_{u=1}^N N^{\frac{p}{p^{\prime}}-\frac{p}{2}} \left\|  f[\cdot] g[u-\cdot] \right\|_{p^{\prime}}^p \label{discrete_lieb_HJ} \\
&= & N^{p-\frac{p}{p^{\prime}}}  \sum_{u=1}^N \left( \sum_{n=1}^N | f[n] g[u-n]|^{p^{\prime}} \right)^\frac{p}{p^{\prime}} \nonumber  \\
&= & N^{p-\frac{p}{p^{\prime}}} \sum_{u=1}^N \left( \sum_{n=1}^N | f^{p^{\prime}}[n]||g^{p^{\prime}}[u-n]| \right)^\frac{p}{p^{\prime}} \nonumber  \\
&= & N^{p-\frac{p}{p^{\prime}}} \sum_{u=1}^N \left( (| f^{p^{\prime}}| *|g^{p^{\prime}}|)[u] \right)^\frac{p}{p^{\prime}} \nonumber  \\
&= & N^{p-\frac{p}{p^{\prime}}} \left\| | f^{p^{\prime}}| *|g^{p^{\prime}}| \right\|_\frac{p}{p^{\prime}}^\frac{p}{p^{\prime}} \nonumber  \\
&\leq & N^{p-\frac{p}{p^{\prime}}}  \| f^{p^{\prime}}\|_\alpha^\frac{p}{p^{\prime}} \|g^{p^{\prime}}\|_\beta^\frac{p}{p^{\prime}}    \label{discrete_lieb_Y}\\
&=&N \| f^{p^{\prime}}\|_\alpha^\frac{p}{p^{\prime}} \|g^{p^{\prime}}\|_\beta^\frac{p}{p^{\prime}} \nonumber, 
\end{eqnarray}
for any $1\leq \alpha, \beta \leq \infty$ satisfying $\frac{1}{\alpha}+\frac{1}{\beta}=p^{\prime}$. Equation
\eqref{discrete_lieb_HJ} follows from the Hausdorff-Young inequality given in Lemma \ref{HY_discrete} and \eqref{discrete_lieb_Y} follows from the Young inequality given in Lemma \ref{classical_young} with $r= \frac{p}{p^{\prime}}$.
Now we can perform a change variable $a=\alpha p^{\prime}$ and $b=\beta p^{\prime}$ so that $\frac{1}{a}+\frac{1}{b}=1$, and \eqref{discrete_lieb_Y} becomes
\begin{align}\label{Eq:precov}
\| \A_{\Dc_{DWFT}} f \|_p^p 
\leq  N \| f^{p^{\prime}}\|_\alpha^\frac{p}{p^{\prime}} \|g^{p^{\prime}}\|_\beta^\frac{p}{p^{\prime}} =  N \| f\|_a^p \|g\|_b^p. 
\end{align}
Finally, we take $a=b=2$ and take the $p^{th}$ root of \eqref{Eq:precov} to show the first half of Theorem \ref{theo:Classical_Lieb_discrete}. Note that we cannot follow the same line of logic for the case $1 \leq p \leq 2$ without a converse of the Young's inequality in Lemma  \ref{classical_young}. 

% Note that $\frac{1}{a}+\frac{1}{b}=1$ with $\frac{p}{p-1}\leq a,b \leq p$. We select = $a=b=2$ and take the $p$ root to find the first equation of the theorem. 

%For $p\in[1,2]$, we apply Featured Theorem \ref{theo: amb tight} as this theorem is a particular case it.

\end{proof}
\bibliography{bib2}

\begin{thebibliography}{10}
\expandafter\ifx\csname url\endcsname\relax
  \def\url#1{\texttt{#1}}\fi
\expandafter\ifx\csname urlprefix\endcsname\relax\def\urlprefix{URL }\fi
\expandafter\ifx\csname href\endcsname\relax
  \def\href#1#2{#2} \def\path#1{#1}\fi

\bibitem{shuman2013emerging}
{D. I Shuman}, S.~K. Narang, P.~Frossard, A.~Ortega, P.~Vandergheynst, The
  emerging field of signal processing on graphs: {E}xtending high-dimensional
  data analysis to networks and other irregular domains, IEEE Signal Process.
  Mag. 30~(3) (2013) 83--98.

\bibitem{sandryhaila2014discrete}
A.~Sandryhaila, J.~M.~F. Moura, Discrete signal processing on graphs:
  {F}requency analysis, IEEE. Trans. Signal Process. 62~(12) (2014) 3042--3054.

\bibitem{Crovella2003}
M.~Crovella, E.~Kolaczyk, Graph wavelets for spatial traffic analysis, in:
  Proc. {IEEE INFOCOM}, Vol.~3, 2003, pp. 1848--1857.

\bibitem{Maggioni_biorthogonal}
M.~Maggioni, J.~C. Bremer, R.~R. Coifman, A.~D. Szlam, Biorthogonal diffusion
  wavelets for multiscale representations on manifolds and graphs, in: Proc.
  {SPIE} Wavelet XI, Vol. 5914, 2005.

\bibitem{szlam}
A.~D. {Szlam}, M.~{Maggioni}, R.~R. {Coifman}, J.~C. {Bremer}, Jr.,
  {Diffusion-driven multiscale analysis on manifolds and graphs: top-down and
  bottom-up constructions}, in: Proc. {SPIE} Wavelets, Vol. 5914, 2005, pp.
  445--455.

\bibitem{coifman2006diffusion}
R.~R. Coifman, M.~Maggioni, Diffusion wavelets, Appl. Comput. Harmon. Anal.
  21~(1) (2006) 53--94.

\bibitem{bremer_packets}
J.~C. Bremer, R.~R. Coifman, M.~Maggioni, A.~D. Szlam, Diffusion wavelet
  packets, Appl. Comput. Harmon. Anal. 21~(1) (2006) 95--112.

\bibitem{lafon_coarse}
S.~Lafon, A.~B. Lee, Diffusion maps and coarse-graining: {A} unified framework
  for dimensionality reduction, graph partitioning, and data set
  parameterization, IEEE Trans. Pattern Anal. Mach. Intell. 28~(9) (2006)
  1393--1403.

\bibitem{wang}
W.~Wang, K.~Ramchandran, Random multiresolution representations for arbitrary
  sensor network graphs, in: Proc. IEEE Int. Conf. Acc., Speech, and Signal
  Process., Vol.~4, 2006, pp. 161--164.

\bibitem{narang_lifting_graphs}
S.~K. Narang, A.~Ortega, Lifting based wavelet transforms on graphs, in: Proc.
  {APSIPA ASC}, Sapporo, Japan, 2009, pp. 441--444.

\bibitem{jansen}
M.~Jansen, G.~P. Nason, B.~W. Silverman, Multiscale methods for data on graphs
  and irregular multidimensional situations, J. R. Stat. Soc. Ser. B Stat.
  Methodol. 71~(1) (2009) 97--125.

\bibitem{gavish}
M.~Gavish, B.~Nadler, R.~R. Coifman, Multiscale wavelets on trees, graphs and
  high dimensional data: {T}heory and applications to semi supervised learning,
  in: Proc. Int. Conf. Mach. Learn., Haifa, Israel, 2010, pp. 367--374.

\bibitem{hammond2011wavelets}
D.~K. Hammond, P.~Vandergheynst, R.~Gribonval, Wavelets on graphs via spectral
  graph theory, Appl. Comput. Harmon. Anal. 30~(2) (2011) 129--150.

\bibitem{ram2011generalized}
I.~Ram, M.~Elad, I.~Cohen, Generalized tree-based wavelet transform, IEEE
  Trans. Signal Process. 59~(9) (2011) 4199--4209.

\bibitem{narang2012perfect}
S.~K. Narang, A.~Ortega, Perfect reconstruction two-channel wavelet
  filter-banks for graph structured data, IEEE. Trans. Signal Process. 60~(6)
  (2012) 2786--2799.

\bibitem{leonardi_multislice}
N.~Leonardi, D.~{Van De Ville}, Tight wavelet frames on multislice graphs, IEEE
  Trans. Signal Process. 61~(13) (2013) 3357--3367.

\bibitem{ekambaram_globalsip}
V.~N. Ekambaram, G.~C. Fanti, B.~Ayazifar, K.~Ramchandran, Critically-sampled
  perfect-reconstruction spline-wavelet filter banks for graph signals, in:
  Proc. Glob. Conf. Signal Inf. Process., Austin, TX, 2013, pp. 475--478.

\bibitem{narang_bior_filters}
S.~K. Narang, A.~Ortega, Compact support biorthogonal wavelet filterbanks for
  arbitrary undirected graphs, IEEE Trans. Signal Process. 61~(19) (2013)
  4673--4685.

\bibitem{liu_coarsening}
P.~Liu, X.~Wang, Y.~Gu, Coarsening graph signal with spectral invariance, in:
  Proc. IEEE Int. Conf. Acc., Speech, and Signal Process., Florence, Italy,
  2014, pp. 1070--1074.

\bibitem{sakiyama}
A.~Sakiyama, Y.~Tanaka, Oversampled graph {L}aplacian matrix for graph filter
  banks, IEEE Trans. Signal Process. 62~(24) (2014) 6425--6437.

\bibitem{nguyen}
H.~Q. Nguyen, M.~N. Do, Downsampling of signals on graphs via maximum spanning
  trees, IEEE Trans. Signal Process. 63~(1) (2015) 182--191.

\bibitem{shuman2013spectrum}
{D. I Shuman}, C.~Wiesmeyr, N.~Holighaus, P.~Vandergheynst, Spectrum-adapted
  tight graph wavelet and vertex-frequency frames, IEEE Trans. Signal Process.
  63~(16) (2015) 4223--4235.

\bibitem{shuman2015vertex}
{D. I Shuman}, B.~Ricaud, P.~Vandergheynst, Vertex-frequency analysis on
  graphs, Appl. Comput. Harmon. Anal. 40~(2) (2016) 260--291.

\bibitem{shuman2016multiscale}
{D. I Shuman}, M.~Faraji, P.~Vandergheynst, A multiscale pyramid transform for
  graph signals, IEEE. Trans. Signal Process.

\bibitem{donoho1989uncertainty}
D.~L. Donoho, P.~B. Stark, Uncertainty principles and signal recovery, SIAM J.
  Appl. Math 49~(3) (1989) 906--931.

\bibitem{Donoho01uncertainty}
D.~L. Donoho, X.~Huo, Uncertainty principles and ideal atomic decomposition,
  IEEE Trans. Inf. Theory 47~(7) (2001) 2845--2862.

\bibitem{Elad02Generalized}
M.~Elad, A.~M. Bruckstein, A generalized uncertainty principle and sparse
  representation in pairs of bases, IEEE Trans. Inf. Theory 48~(9) (2002)
  2558--2567.

\bibitem{gribonval2003sparse}
R.~Gribonval, M.~Nielsen, Sparse representations in unions of bases, IEEE
  Trans. Inf. Theory 49~(12) (2003) 3320--3325.

\bibitem{candes2006quantitative}
E.~J. Candes, J.~Romberg, Quantitative robust uncertainty principles and
  optimally sparse decompositions, Found. Comput. Math. 6~(2) (2006) 227--254.

\bibitem{rictorrefined}
B.~Ricaud, B.~Torr{\'e}sani, Refined support and entropic uncertainty
  inequalities, IEEE Trans. Inf. Theory 59~(7) (2013) 4272--4279.

\bibitem{ricaud_SPIE_2013}
B.~Ricaud, {D. I Shuman}, P.~Vandergheynst, On the sparsity of wavelet
  coefficients for signals on graphs, in: {SPIE} Wavelets and Sparsity, San
  Diego, California, 2013.

\bibitem{mcgraw}
P.~N. {McGraw}, M.~Menzinger, Laplacian spectra as a diagnostic tool for
  network structure and dynamics, Phys. Rev. E 77~(3) (2008) {031102--1} --
  {031102--14}.

\bibitem{saito}
N.~Saito, E.~Woei, On the phase transition phenomenon of graph {L}aplacian
  eigenfunctions on trees, {RIMS} Kokyuroku 1743 (2011) 77--90.

\bibitem{folland1997}
G.~Folland, A.~Sitaram, The uncertainty principle: {A} mathematical survey, J.
  {Fourier} Anal. Appl. 3~(3) (1997) 207--238.

\bibitem{mallat2008wavelet}
S.~G. Mallat, A Wavelet Tour of Signal Processing, 3rd ed., Academic Press,
  2008.

\bibitem{agaskar_spie}
A.~Agaskar, Y.~M. Lu, An uncertainty principle for functions defined on graphs,
  in: Proc. {SPIE}, Vol. 8138, San Diego, CA, 2011, pp. 81380T--1 --
  81380T--11.

\bibitem{agaskar_icassp}
A.~Agaskar, Y.~M. Lu, Uncertainty principles for signals defined on graphs:
  {Bounds} and characterizations, in: Proc. IEEE Int. Conf. Acc., Speech, and
  Signal Process., Kyoto, Japan, 2012, pp. 3493--3496.

\bibitem{agaskarlu13}
A.~Agaskar, Y.~Lu, A spectral graph uncertainty principle, IEEE Trans. Inf.
  Theory 59~(7) (2013) 4338--4356.

\bibitem{pasdeloup}
B.~Pasdeloup, R.~Alami, V.~Gripon, M.~Rabbat, Toward an uncertainty principle
  for weighted graphs, {ArXiv e-prints}.

\bibitem{tsitsvero2015signals}
M.~Tsitsvero, S.~Barbarossa, P.~Di~Lorenzo, Signals on graphs: Uncertainty
  principle and sampling, arXiv preprint arXiv:1507.08822.

\bibitem{slepian1961prolate}
D.~Slepian, H.~O. Pollak, Prolate spheroidal wave functions, fourier analysis
  and uncertainty—i, Bell System Technical Journal 40~(1) (1961) 43--63.

\bibitem{Maassen88generalized}
H.~Maassen, J.~Uffink, Generalized entropic uncertainty relations, Phys. Rev.
  Lett. 60~(12) (1988) 1103--1106.

\bibitem{reed1975methods}
M.~Reed, B.~Simon, Methods of Modern Mathematical Physics, Vol. 2.: Fourier
  Analysis, Self-Adjointness, Academic Press, 1975.

\bibitem{grady}
L.~J. Grady, J.~R. Polimeni, Discrete Calculus, Springer, 2010.

\bibitem{lieb1990integral}
E.~H. Lieb, {Integral bounds for radar ambiguity functions and Wigner
  distributions}, J. Math. Phys. 31~(3) (1990) 594.

\bibitem{sandryhaila_2013}
A.~Sandryhaila, J.~M.~F. Moura, Discrete signal processing on graphs, IEEE.
  Trans. Signal Process. 61~(7) (2013) 1644--1656.

\bibitem{chung1997spectral}
F.~R.~K. Chung, Spectral Graph Theory, Vol. 92 of the {CBMS} Regional
  Conference Series in Mathematics, {American Mathematical Society}, 1997.

\bibitem{renyi1961measures}
A.~R\'enyi, On measures of entropy and information, in: Proc. 4-th Berkeley
  Symp. Math. Statist. and Probability, 1961, pp. 547--561.

\bibitem{RicTorrACOM}
B.~Ricaud, B.~Torr{\'e}sani, A survey of uncertainty principles and some signal
  processing applications, Adv. Comput. Math. 40~(3) (2014) 629--650.

\bibitem{dekel}
Y.~Dekel, J.~R. Lee, N.~Linial, Eigenvectors of random graphs: {N}odal domains,
  Random Structures \& Algorithms 39~(1) (2011) 39--58.

\bibitem{dumitriu}
I.~Dumitriu, S.~Pal, Sparse regular random graphs: {S}pectral density and
  eigenvectors, Ann. Probab. 40~(5) (2012) 2197--2235.

\bibitem{tran}
L.~V. Tran, V.~H. Vu, K.~Wang, Sparse random graphs: {E}igenvalues and
  eigenvectors, Random Struct. Algo. 42~(1) (2013) 110--134.

\bibitem{brooks}
S.~Brooks, E.~Lindenstrauss, Non-localization of eigenfunctions on large
  regular graphs, Israel Journal of Mathematics 193~(1) (2013) 1--14.

\bibitem{Nakatsukasa2013mysteries}
Y.~Nakatsukasa, N.~Saito, E.~Woei, Mysteries around the graph {L}aplacian
  eigenvalue 4, Linear Algebra Appl. 438~(8) (2013) 3231--3246.

\bibitem{christensen2003introduction}
O.~Christensen, Frames and Bases, Birkh\"{a}user, 2008.

\bibitem{kovacevic_frames1}
J.~Kova\v{c}evi{\'{c}}, A.~Chebira, Life beyond bases: {T}he advent of frames
  (part {I}), IEEE Signal Process. Mag. 24~(4) (2007) 86--104.

\bibitem{kovacevic_frames2}
J.~Kova\v{c}evi{\'{c}}, A.~Chebira, Life beyond bases: {T}he advent of frames
  (part {II}), IEEE Signal Process. Mag. 24~(5) (2007) 115--125.

\bibitem{metzger}
B.~Metzger, P.~Stollmann, Heat kernel estimates on weighted graphs, Bull.
  London Math. Soc. 32~(4) (2000) 477--483.

\bibitem{leonardi_fmri}
N.~Leonardi, D.~{Van De Ville}, Wavelet frames on graphs defined by {FMRI}
  functional connectivity, in: Proc. IEEE Int. Symp. Biomed. Imag., Chicago,
  IL, 2011, pp. 2136--2139.

\bibitem{thanou_TSP_2014}
D.~Thanou, {D. I Shuman}, P.~Frossard, Learning parametric dictionaries for
  signals on graphs, IEEE. Trans. Signal Process. 62~(15) (2014) 3849--3862.

\bibitem{eusipcooptamb}
H.~Feichtinger, D.~Onchis-Moaca, B.~Ricaud, B.~Torr{\'e}sani, C.~Wiesmeyr, A
  method for optimizing the ambiguity function concentration, in: Proc. Eur.
  Signal Process. Conf. (EUSIPCO), 2012, pp. 804--808.

\bibitem{perraudin2016stationary}
N.~Perraudin, P.~Vandergheynst, Stationary signal processing on graphs, arXiv
  preprint arXiv:1601.02522.

\bibitem{gadde2015probabilistic}
A.~Gadde, A.~Ortega, A probabilistic interpretation of sampling theory of graph
  signals, arXiv preprint arXiv:1503.06629.

\bibitem{zhang2015graph}
C.~Zhang, D.~Flor{\^e}ncio, P.~A. Chou, Graph signal processing--a
  probabilistic framework.

\bibitem{pesenson_splines}
I.~Pesenson, Variational splines and {Paley-Wiener} spaces on combinatorial
  graphs, Constr. Approx. 29~(1) (2009) 1--21.

\bibitem{perraudin2014unlocbox}
N.~Perraudin, D.~Shuman, G.~Puy, P.~Vandergheynst, Unlocbox a matlab convex
  optimization toolbox using proximal splitting methods, arXiv preprint
  arXiv:1402.0779.

\bibitem{perraudin2014gspbox}
N.~Perraudin, J.~Paratte, D.~Shuman, V.~Kalofolias, P.~Vandergheynst, D.~K.
  Hammond, Gspbox: A toolbox for signal processing on graphs, arXiv preprint
  arXiv:1408.5781.

\bibitem{puy2015random}
G.~Puy, N.~Tremblay, R.~Gribonval, P.~Vandergheynst, Random sampling of
  bandlimited signals on graphs, arXiv preprint arXiv:1511.05118.

\bibitem{strang1999discrete}
G.~Strang, The discrete cosine transform, {SIAM} Review 41~(1) (1999) 135--147.

\bibitem{pinsky}
M.~A. Pinsky, Introduction to {Fourier} Analysis and Wavelets, Vol. 102 of the
  Graduate Studies in Mathematics, {American Mathematical Society}, 2002.

\end{thebibliography}


%\end{thebibliography}

\end{document}